%% file: main.tex
\documentclass{article}

\usepackage{microtype}
\usepackage{graphicx}
\usepackage{subcaption}
\usepackage[linesnumbered,ruled,vlined]{algorithm2e}
\usepackage{booktabs}
\usepackage{enumitem}
\usepackage{multirow}

\usepackage{longtable}
\usepackage{makecell}
\usepackage[table]{xcolor}
\usepackage{pdflscape}

\usepackage{iclr2027_conference,times}
\usepackage{hyperref}

\usepackage{amsmath}
\usepackage{amssymb}
\usepackage{mathtools}
\usepackage{amsthm}
\usepackage{wrapfig}
\usepackage{etoc}
\usepackage{booktabs}
\usepackage{longtable}
\usepackage{multirow}
\usepackage{makecell}
\usepackage{pdflscape}

\usepackage{needspace}
\usepackage{setspace}
\usepackage{bbm}
\usepackage{titletoc}
\usepackage{float}
\usepackage[capitalize,noabbrev]{cleveref}

\theoremstyle{plain}
\newtheorem{theorem}{Theorem}[section]

\newtheorem{lemma}[theorem]{Lemma}
\newtheorem{corollary}[theorem]{Corollary}
\theoremstyle{definition}
\newtheorem{definition}[theorem]{Definition}
\newtheorem{assumption}[theorem]{Assumption}
\theoremstyle{remark}

\usepackage[textsize=tiny]{todonotes}
\usepackage{graphics}

\allowdisplaybreaks

\title{Distributed Learning with Selective State Space Models: Architecture-Aware Convergence Analysis}

\author{%
  Adam Piaseczny \\
  Purdue University \\
  \texttt{apiasecz@purdue.edu} \\
  \And
  Md Kamran Chowdhury Shisher \\
  University of Colorado Colorado Springs \\
  \texttt{mshisher@uccs.edu} \\
  \AND
  Shiqiang Wang \\
  University of Exeter \\
  \texttt{shiqiang.wang@ieee.org} \\
  \And
  Christopher G. Brinton \\
  Purdue University \\
  \texttt{cgb@purdue.edu} \\
}
\iclrfinalcopy
\begin{document}
\maketitle
\lhead{Preprint. Under review.}

\begin{abstract}
Modern state space models (SSMs), such as Mamba2, provide a compelling alternative to transformers by combining linear-time sequence modeling with recurrent state-space dynamics. However, the behavior of SSMs in distributed learning settings remains poorly understood. In particular, the existing standard federated learning methods are largely architecture-agnostic, and do not account for the stability, selectivity, and state-space parameterization that characterize modern selective SSMs. To address this, we derive architecture-aware gradient and smoothness bounds for single- and multi-layer selective SSMs, and convergence bounds for FedAvg and FedProx, characterizing how recurrent stability, input-dependent discretization, and state projection norms affect federated optimization. We then numerically validate the single-layer bounds on sequences generated by a teacher SSM, using a learner that follows the analyzed recurrence. We use this analysis to formulate expectations about the effects of local training and client heterogeneity, and examine these expectations by comparing nine federated learning algorithms on Mamba2 language modeling across six text domains. These experiments illustrate how SSM-specific bounds can provide a basis for interpreting the behavior of practical federated learning algorithms.
\end{abstract}

\input{section_files/introduction}
\input{section_files/related_work}
\input{section_files/system_model}

\input{section_files/theoretical_analysis}

\input{section_files/experiments}

\section{Conclusion}
In this work, we studied distributed learning with selective SSMs through architecture-aware theory, deriving bounds that connect federated optimization to recurrent stability, input-dependent discretization, projection norms, client drift, and heterogeneity. We then validated the derived bounds and provided experiments that demonstrated their relevance and practical applications.  

\subsection*{AI use statement}
In this work, we used generative AI tools for: assisting in the writing of proofs, formulating mathematical claims (specifically, refining author-derived mathematical statements in response to feedback on proof correctness), proposing or refining hypotheses, designing or providing feedback on research methodology or experiments, implementing methods, supporting qualitative and thematic data analysis and interpreting results. 

We have not used generative AI tools for: helping develop theoretical models or conceptual frameworks, generating synthetic data sets, providing critical ingredients for proving mathematical claims, assisting with translation, or cleaning and reformatting a dataset. Additionally, we used generative AI tools for: creating or modifying scientific figures or images (limited to suggestions for figure titles and labels), editing a research paper to improve readability, and suggesting a structure for a research paper.

Our use of generative AI was three-fold. First, generative AI tools were used to provide feedback on theoretical arguments, hypotheses, methodological choices, and interpretations of experimental results. For all proofs, theorems, corollaries, and other mathematical statements in this work, the underlying claims and arguments were first derived and formalized. The resulting statements and proofs were subsequently provided to an AI agent for feedback on their correctness, logical consistency, and the validity of the stated claims. The agent's feedback was independently evaluated, and any resulting revisions were made manually. Similarly, generative AI tools were used to check the consistency of our experimental analyses and interpretations with the reported results. All final analyses, interpretations, and claims were determined by the authors.

Second, AI agents assisted with the implementation of the code used in our experiments. While some experiments use synthetic data generated by deterministic code, AI agents were only utilized to assist with writing the code, not generate the synthetic data. All AI-assisted code was manually inspected for correctness and subsequently evaluated using validation tests before being used in the reported experiments.

Third, we used large language models to improve the presentation of the manuscript, including wording, readability, organization, section flow, and textual elements such as figure titles. The scientific figures and images themselves were not generated or modified using generative AI (though their titles and labels were adjusted with suggestions from AI models). All suggested edits were reviewed by the authors and incorporated manually where appropriate.

We have reviewed all AI-assisted work and take responsibility for the final content of this work, including text, claims, code, and other artifacts produced with the aid of generative AI.

\subsection*{Reproducibility Statement}
We provide complete proofs for the theoretical results presented in this work. For the empirical results, Appendix~\ref{appendix:detailed_exp_setup} provides detailed descriptions of the experimental setups, including the datasets and algorithms used, hyperparameter tuning procedures, and hardware used to conduct the experiments. We additionally provide our codebase containing implementations of the methods used to obtain the reported results, together with instructions for running the experiments as supplementary material.

\bibliographystyle{unsrtnat}
\bibliography{references}

\newpage
\appendix
\onecolumn
\clearpage

\input{appendix}

\clearpage

\end{document}

%% file: section_files/introduction.tex
\section{Introduction}\label{sec:introduction}
The modern generative AI landscape is largely dominated by transformer architecture \citep{NIPS2017_3f5ee243}, with the majority of the state-of-the-art large language models (LLMs) employing them
\citep{agarwal2025gpt, team2026gemma, qiu2026design, team2026kimi}. Despite this widespread success, transformers face critical limitations in certain scenarios, often driven by the core efficiency bottleneck of their self-attention mechanism, scaling quadratically as $\mathcal O(L^2)$ with sequence length $L$ in both time and memory complexity \citep{tay2022efficient}. To address this inefficiency, a variety of alternative mechanisms and architectures have been proposed, including sparse attention \citep{child2019generating, zaheer2020big}, linear attention approximations \citep{pmlr-v119-katharopoulos20a, choromanski2021rethinking}, and modernized recurrent frameworks \citep{peng-etal-2023-rwkv}. Among these innovations, state space models (SSMs) \citep{gu2022efficiently} have emerged as one of the most promising new paradigms. These modern SSMs, such as Mamba \citep{gu2024mamba}, overcome the quadratic bottleneck by operating with a strictly linear-time complexity, $\mathcal O(L)$, during training, and a constant-sized latent state, $\mathcal O(1)$.

Driven by these computational advantages, the SSM architecture has undergone rapid development, achieving competitive performance at a reduced computational cost compared to transformers. Architectural innovations, such as the Structured State Space (S4) sequence model, HiPPO initialization \citep{gu2022efficiently}, and selective mechanisms \citep{gu2024mamba} have enabled SSMs to dynamically filter information and capture and retain long-range dependencies. Most recently, SSMs have demonstrated capacity for achieving state-of-the-art language modeling performance at scale \citep{lahoti2026mamba, zuo2024falcon, blakeman2025nemotron, team2024jamba}.

\textbf{Research Gap.} 
Despite these advancements, the deployment of SSMs in \textit{distributed} learning environments, where training is conducted across a set of agents, remains largely unexplored. While some studies have considered the intersection of SSMs with federated learning (FL) \citep{mcmahan2017communication}, most of these frameworks are \textit{architecture-agnostic}. Conversely, existing SSM-specific distributed learning studies either focus on classical, non-selective SSM settings \citep{11099331}, or use SSMs for specialized personalized-FL applications \citep{11195747}, leaving the convergence behavior of modern selective SSMs largely uncharacterized. Developing a foundational understanding of selective SSM convergence behavior in distributed settings can enable architecture-aware parameter tuning, protocol selection, and other control tasks for training performance optimization, uncovering another advantage of SSMs over transformer LLM approaches that rely on architecture-agnostic optimization tools~\citep{fang2026federated,soro2026ls}.
We thus pose the following question:

\begin{itemize}[topsep=0pt,itemsep=4pt,parsep=0pt, partopsep=0pt, leftmargin=*]
\item[] \textbf{\textit{How can we leverage the structural properties of modern selective SSMs to derive architecture-aware convergence bounds that rigorously characterize training behavior in terms of core state space parameters for distributed and federated learning techniques?}}
\end{itemize}

\textbf{Contributions.} 
In answering this question, our main contributions are summarized as follows:

\begin{itemize}[topsep=0pt,itemsep=4pt,parsep=0pt, partopsep=0pt, leftmargin=*]
    \item \textbf{Convergence analysis for distributed selective SSMs:} We provide an architecture-aware gradient norm bound for selective SSMs, revealing the impact of SSM parameters like core state-space parameter blocks. Then, we use this bound to establish the convergence rate of popular FL algorithms training selective SSMs, showing dependence on standard distributed learning factors like smoothness as well as architecture-specific parameters like block-wise spectral-norm envelopes.
    To the best of our knowledge, our work is the first to provide an architecture-aware convergence rate of FL algorithms for selective SSMs.
    \item \textbf{SSM theory validation and theory-motivated algorithm comparison}: We validate our theoretical results using teacher-model derived sequences and single-layer selective SSMs, observing no gradient-bound violations, increased client disagreement as $\rho$ approaches one, and increased drift with more local steps. We use the theoretical takeaways to rank federated algorithms on heterogeneous-domain datasets. Our results show that these rankings highly correlate with validation performance, supporting the practical relevance of the SSM-specific quantities identified by our analysis.

\end{itemize}

%% file: section_files/related_work.tex
\section{Related Work}\label{sec:related_work}

\textbf{State Space Models (SSMs).} SSMs originate from classical control theory concepts, where they model dynamic systems through the evolution of latent states over time~\citep{kalman1960new}. The modern SSMs were introduced with the pivotal Structured State Space (S4)~\citep{gu2022efficiently} and its variants~\citep{NEURIPS2022_e9a32fad, hasani2023liquid, NEURIPS2020_102f0bb6}, and leverage structured parameterizations to efficiently capture long-range dependencies. Selective SSMs introduced input-dependent state-space parameters to dynamically filter information~\citep{gu2024mamba} while the Structured State Space Duality (SSD) established connections with linear attention and enabled efficient implementations~\citep{dao2024transformers}, enabling better expressivity, training stability, and hardware-aware optimizations~\citep{NEURIPS2024_c7f795dc, kim2026ssdi, lahoti2026mamba}. These advances have driven growing adoption across language modeling~\citep{DBLP:journals/corr/abs-2406-07887, blakeman2025nemotron, he-etal-2025-densessm, glorioso2024zamba, ren2025samba}, computer vision~\citep{li2024mamba, zhu2024vision, xie2025pvmamba}, and other sequence modeling applications~\citep{ma2025hybridna}.

\textbf{Distributed Learning with SSMs.} Distributed learning enables model training across computational nodes, with federated learning (FL) addressing settings where data remain distributed among clients. Since the introduction of a foundational federated learning algorithm -  FedAvg~\citep{mcmahan2017communication} - numerous methods have been proposed to address statistical heterogeneity, communication costs, and client drift \citep{MLSYS2020_1f5fe839, pmlr-v119-karimireddy20a, wang2020tackling, reddi2021adaptive, acar2021federated, li2021model}. However, despite the growing adoption of SSMs, theoretical understanding of their behavior in these settings remains limited. The existing SSM-related approaches address specific challenges: FedAlign \citep{11099331} tackles the problem of models converging to different parameter basins in federated system identification, while FedSSM \citep{11195747} uses SSMs as a meta-learning tool within personalized FL. Related works also explore merging operations on hidden states, contexts, or representations \citep{liu2025picaso, pioro2024state}, rather than weight-space aggregation of independently trained selective SSMs. As a result, there remains an open question of how selective SSM dynamics affect optimization and aggregation guarantees in distributed learning settings. This motivates our study of connecting the architectural properties of selective SSMs to their behavior in distributed training.

%% file: section_files/system_model.tex
\section{System Model}\label{sec:system_model}
\begin{figure}[t]
    \centering
    \begin{subfigure}[b]{0.46\textwidth}
        \centering
        \includegraphics[width=\textwidth]{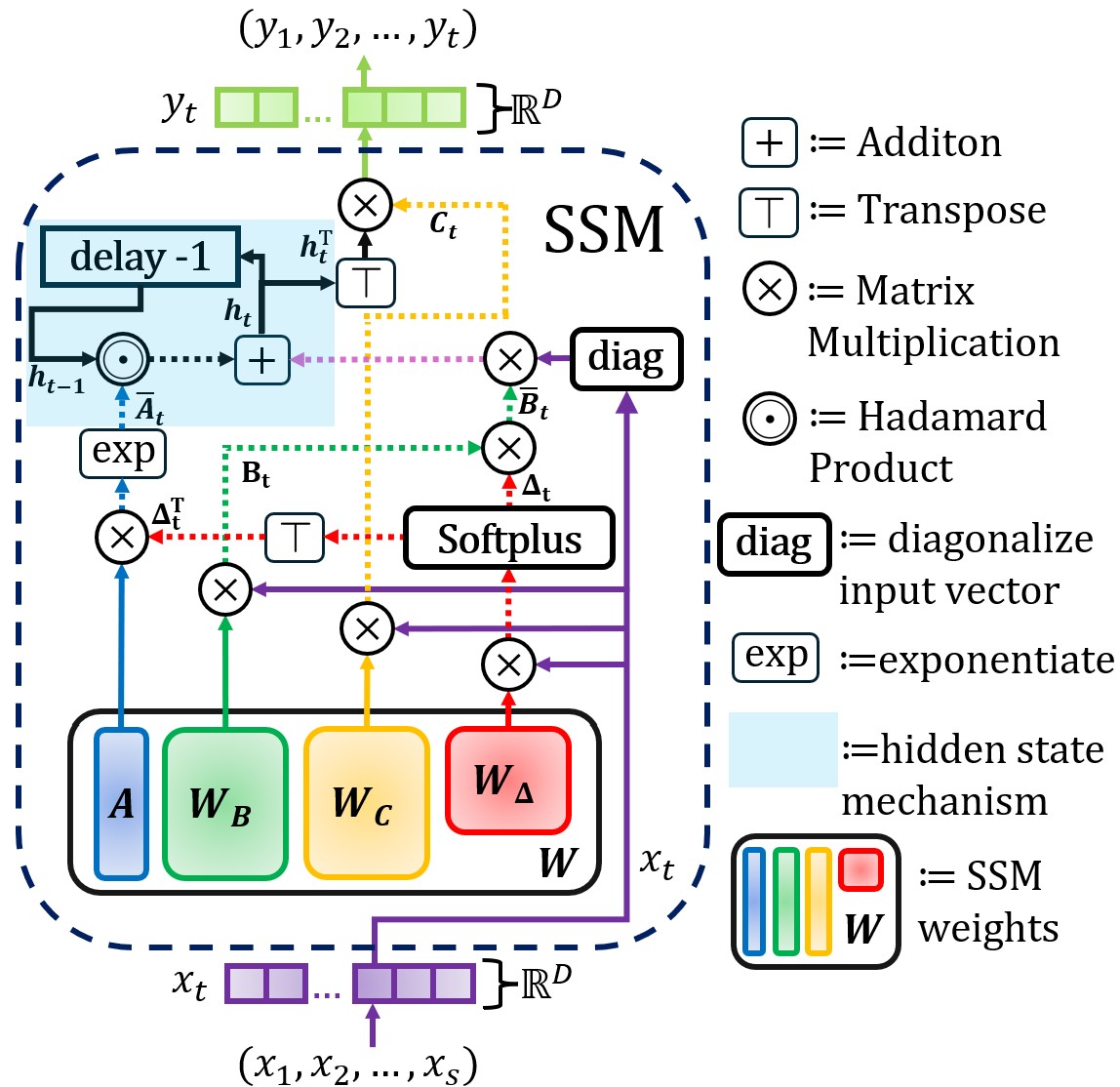}
        \caption{Selective SSM layer}
        \label{fig:ssm_layer_illustration}
    \end{subfigure}
    \hfill
    \begin{subfigure}[b]{0.53\textwidth}
        \centering
        \includegraphics[width=\textwidth]{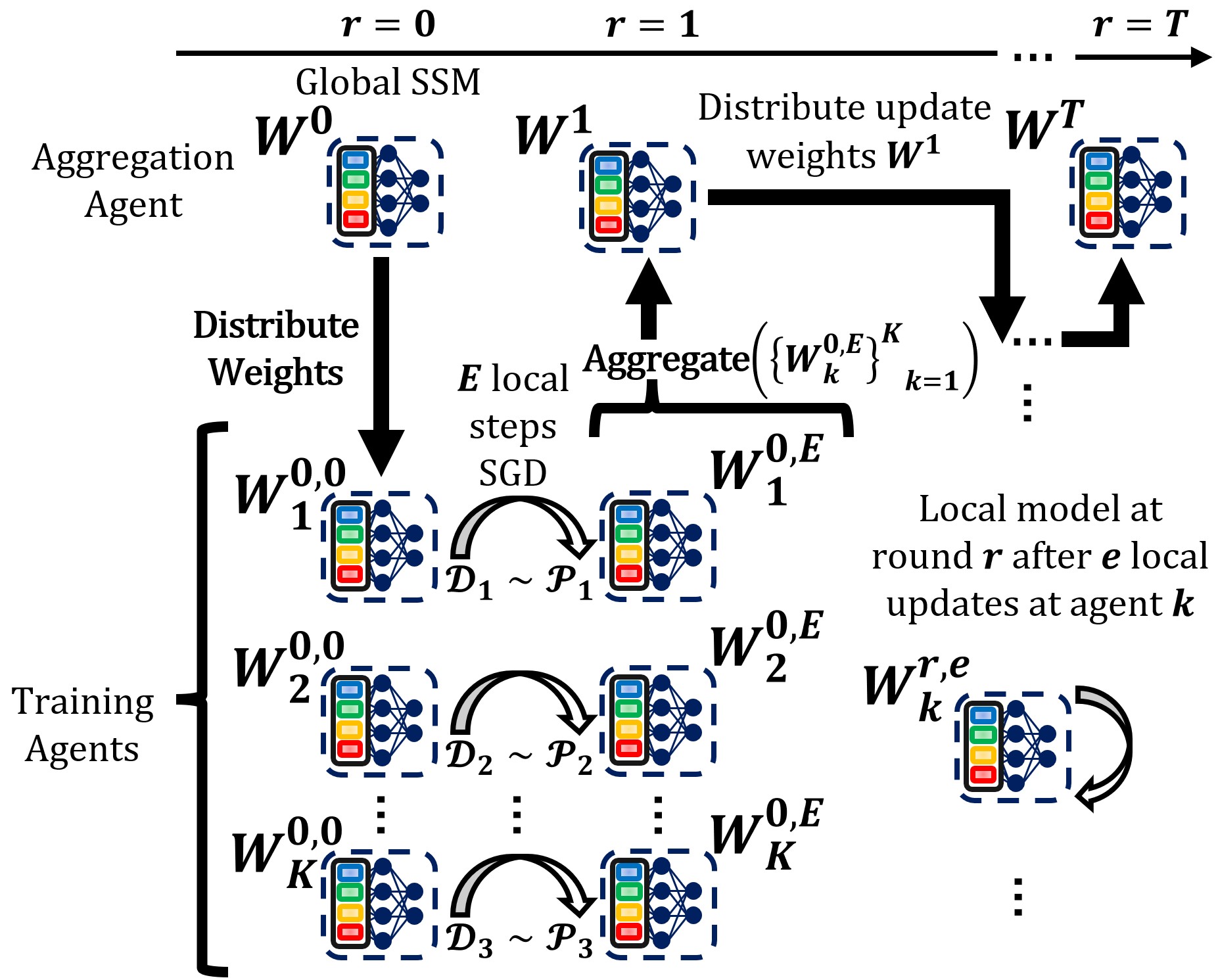}
        \caption{Distributed SSM training}
        \label{fig:distributed_learning_illustrated}
    \end{subfigure}
    \caption{(a) A selective SSM layer generates input-dependent discretization parameters from $x_t$ and uses them to update the hidden state $h_t$ and output $y_t$. (b) Federated SSM training proceeds over communication rounds: the server broadcasts $W^r$, clients perform $E$ local updates on local datasets $\mathcal D_k$, and the resulting models are aggregated into $W^{r+1}$. }
    \label{fig:overall_system_model}
\end{figure}

\textbf{State Space Model.} We consider a decentralized learning environment operating over a global selective state space model (SSM) parameterized by $W = \left(W_B, W_C, W_\Delta, A\right)$ where $W_\Delta \in \mathbb R^{D\times D}$, $W_B \in \mathbb R^{N\times D}$, $W_C \in \mathbb R^{N\times D}$, and $A\in \mathbb R^{N \times 1}$, with $A$ representing diagonal entries of a structured matrix $\mathcal A = \text{diag}(A)$. The SSM architecture is inspired by classical control theory concepts, in particular modeling dynamic systems through the evolution of latent states over time. These systems map a continuous input signal $x(t) \in \mathbb R$ to an output sequence $y(t) \in \mathbb R$ through an evolving, compressed latent state $h(t) \in \mathbb R^{N}$ using linear time-invariant (LTI) differential equations:

\begin{subequations}
\label{eq:system}
\noindent
\begin{minipage}{0.5\textwidth}
  \begin{equation}
     \textstyle h'(t) = \mathcal{A} h(t) + Bx(t) \label{eq:state}
  \end{equation}
\end{minipage}%
\begin{minipage}{0.5\textwidth}
  \begin{equation}
     \textstyle y(t) = C h(t) \label{eq:output}
  \end{equation}
\end{minipage}
\end{subequations}

However, since neural networks process discrete sequences, these continuous-time SSM dynamics are mapped to a discrete recurrence using a discretization rule. Under this discrete formulation, the model maps an input sequence $x = (x_1, \dots, x_S)$ where $x_t \in \mathbb R^{D \times 1}$ to a sequence of outputs $y_t \in \mathbb R^{D \times 1}$ through a sequence of compressed latent states $h_t \in \mathbb R^{N\times D}$. While earlier foundational models utilized various discretization rules \citep{gu2022efficiently}, we adopt the first-order exponential-Euler discretization used in the more recent architectures \citep{lahoti2026mamba}. This discretization scheme retains the exponential state transition while applying a first-order approximation to the input contribution. We analyze the discrete model defined by this recurrence.

In addition to this discrete foundation, modern SSMs introduce a selective mechanism that enables dynamic filtering of relevant information across the input sequence context. The selective SSMs achieve this by making the state-space and time-scale parameters depend on the current input. Specifically, given the model weights $W$, the input-dependent SSM quantities corresponding to an input token $x_t$ at the sequence position $t$ are defined as:
\begin{itemize}[leftmargin=*, noitemsep, topsep=2pt]
    \item \textbf{Time-scale parameter:} $\Delta_t = \text{Softplus}(W_\Delta x_t)$, where $\Delta_t \in \mathbb{R}^{D \times 1}$, acts as a gating mechanism, controlling the discretization step size to determine importance of the current token.
    \item \textbf{State transition matrix:} $\bar{A}_t = \exp(A \Delta_t^\top)$, where $\bar{A}_t \in \mathbb{R}^{N \times D}$, determines the retention and decay of the historical context, managing the system's memory across the sequence.
    \item \textbf{Input matrix:} $\bar{B}_t := B_t \Delta_t^\top = (W_B x_t)(\operatorname{Softplus}(W_\Delta x_t))^\top$, where $\bar{B}_t \in \mathbb{R}^{N \times D}$, and $B_t := W_B x_t$. Under the adopted exponential-Euler discretization, this matrix modulates how strongly the current input vector influences the newly updated latent state.
    \item \textbf{Output matrix:} $C_t = W_C x_t$, where $C_t \in \mathbb{R}^{N \times 1}$, linearly projects the compressed latent state back into the target output space.
\end{itemize}

Given these dynamically generated matrices, the discrete hidden state sequence $h_t$ and the corresponding output $y_t$ at each sequence step are computed according to the following recurrence:

\begin{subequations}
\vspace{-2mm}
\noindent
\begin{minipage}{0.5\textwidth}
  \begin{equation}
     \textstyle h_t = \bar{A}_t \odot h_{t-1} + \bar{B}_t \text{diag}(x_t) \label{ssm_dynamics_h_appendix}
  \end{equation}
\end{minipage}%
\begin{minipage}{0.5\textwidth}
  \begin{equation}
     \textstyle y_t = h_t^\top C_t, \label{eq:ssm_dynamics_y_appendix}
  \end{equation}
\end{minipage}
\end{subequations}
where $h_0 = \mathbf{0}_{N\times D}$, the exponential and Softplus functions are applied element-wise, and $\odot$ denotes the Hadamard product.

With the architecture specified above, we next define the local optimization objective. Let $\xi=(x, y^*)$ denote an input-target sequence pair, where $x=(x_1, \dots, x_S)$ is the input sequence  and $y^* = (y_1^*, \dots, y_S^*)$ is the corresponding target sequence. The sequence loss is then the average token loss, defined as:
\begin{equation}\label{eq:global_loss_definition}
    \textstyle \mathcal L^S\left(W;\xi\right) := \frac{1}{S}\sum_{t=1}^S \ell_t\left(y_t(W;x), y_t^*\right)
\end{equation}
where $y_t(W;x)$ is the model prediction at position $t$ parameterized by $W$, and $\ell_t$ is some differentiable loss function. When the sample $\xi$ is fixed, we write $\mathcal L^S(W)$ and $\ell_t(y_t)$ for notational simplicity. Figure \ref{fig:ssm_layer_illustration} illustrates the resulting architecture, the parameter transformations, and state recurrence.

\textbf{Distributed Optimization Environment.} Given the selective SSM setup established above, we consider the deployment of this architecture in a decentralized learning setting. In this framework, a central server initializes a global selective SSM model parameterized by weights $W^0$, which is collaboratively trained across a network of $K$ clients. Each client $k \in \{1, \dots, K\}$ possesses a local dataset $\mathcal D_k = \left\{\xi_{k,i} \sim \mathcal P_k\right\}_{i=1}^{|\mathcal D_k|}$, where $\mathcal P_k$ denotes the local data distribution of client $k$ and each sample $\xi_{k,i} = (x_{k,i},y_{k,i}^*)$ represents a discrete input sequence and its corresponding target sequence. For autoregressive next-token prediction, the target at each sequence position satisfies $y^*_t = x_{t+1}$. 

Given this setup, the framework aims to construct a global model that incorporates the combined knowledge of these diverse local distributions. Formally, the global objective is to find the optimal global model parameters $W^*$ that minimize the aggregated expected sequence modeling loss:
\begin{equation}\label{eq:init_global_objective}
    \textstyle W^* = \mathop{\arg\min}_{W} f(W), \quad  \textstyle \text{where} \quad f(W) =  \sum_{k=1}^K p_k \mathcal{L}_k(W)
\end{equation}
where $p_k$ are positive weights satisfying $\sum_{k=1}^Kp_k=1$, and $\mathcal{L}_k(W) = \mathbb{E}_{\xi \sim \mathcal P_k}\left[\mathcal{L}^S(W;\xi)\right]$ is the expected sequence loss for client $k$ over its local distribution.

To solve this optimization problem, the framework employs a decentralized training procedure initialized with the global model $W^0$. The optimization then proceeds over discrete communication rounds $r \in \mathcal T := \left\{0, 1, \dots, T-1\right\}$. At each round $r$, the server shares the current global parameters $W^r$ with the participating clients. Each client $k$ initializes its local model as $W_k^{r, 0} = W^r$ and performs $E$ local steps of stochastic gradient descent (SGD) with step size $\eta$ to minimize a local objective $F_k(W)$. The specific formulation of $F_k(W)$ depends on the distributed learning algorithm employed. For instance, in the classical FedAvg algorithm \citep{mcmahan2017communication}, the local objective of client $k$ is the expected sequence loss, i.e., $F_k(W) := \mathcal L_k(W)$. In contrast, FedProx \citep{MLSYS2020_1f5fe839} modifies the local objective as $F_k^r(W) := \mathcal L_k(W) + \frac{\mu}{2}\left\|W-W^r\right\|^2_F$, with $\mu > 0$, to mitigate local divergence caused by data heterogeneity. 

Through this local optimization step, the features of different distributions $\mathcal P_k$ are implicitly embedded into the updated local weights, resulting in a set of specialized models $W_k^{r, E}$. Following the completion of the $E$ local training steps, the server combines these localized models into a unified global representation using a generalized aggregation function:  $W^{r+1} = \text{Aggregate}\left(\left\{W_k^{r, E}\right\}_{k=1}^K; W^r\right)$.
In standard distributed baselines such as FedAvg and FedProx, this aggregation is implemented as a simple data-weighted linear average: 
$\textstyle W^{r+1} = \sum_{k=1}^K \frac{|\mathcal{D}_k|}{\sum_{j=1}^K |\mathcal{D}_j|}W_k^{r, E}$.
By repeating this process over $T$ communication rounds, the framework iteratively refines the global model toward an optimal parameter set $W^*$ by aggregating information learned across the participating clients. Figure~\ref{fig:distributed_learning_illustrated} illustrates this iterative training process.

%% file: section_files/theoretical_analysis.tex
\section{Theoretical Analysis}\label{sec:theoretical_analysis}
Given this decentralized SSM setup, we consider distributed training procedures that update a global model through repeated local adaptation and server-side aggregation. These procedures induce a global optimization trajectory $\{W^r\}_{r=0}^{T}$ over the objective defined in \eqref{eq:init_global_objective}. Since $f$ is generally nonconvex, we characterize convergence through the standard expected gradient stationarity criterion $\frac{1}{T}\sum_{r=0}^{T-1} \mathbb E\left\|\nabla f(W^r)\right\|_F^2.$

The central problem is to derive convergence bounds that explicitly capture the structure of selective SSMs, rather than relying solely on architecture-agnostic smoothness, variance, and heterogeneity constants. In particular, the recurrent state update induce dependencies between the local optimization trajectory and SSM-specific quantities including the stability of the discretized transition matrices, the input-dependent time scale parameter $\Delta_t=\mathrm{Softplus}(W_\Delta x_t)$, and the norms of the state-space projection parameters. Our convergence analysis characterizes how these SSM-specific quantities control gradient growth, client drift, and aggregation error. 

As such, in the remainder of this section, we establish theoretical guarantees for federated training of selective SSMs. We first derive an explicit gradient norm bound for a single-layer selective SSM and use it to characterize federated optimization in terms of SSM-specific parameters. We provide convergence guarantees for FedAvg and FedProx, and relate additional federated baselines to the resulting stability, parameter-sensitivity, and distribution-shift quantities. The main text highlights the key results and their implications, while full proofs, supplementary corollaries, and additional baseline analyses are provided in Appendices~\ref{appendix:gradient_norm_bound_proof}, \ref{appendix:fed_algs_convergence_bounds_proofs}, \ref{appendix:multi_layer_SSM_proofs},~\ref{sec:supplementary_corollaries},~\ref{appendix:additonal_distributed_learning_frameworks}.

\subsection{Gradient Norm Bound for Selective SSMs}\label{subsec:ssm_gradient_bounds}

The main technical challenge lies in controlling the parameter sensitivity induced by the input-dependent SSM recurrence. The loss gradient depends on the model parameters both directly through the output projection and indirectly through their influence on the recurrent hidden states. By bounding these two contributions separately and subsequently combining them, we derive a sequence-wise gradient bound that is independent of the sequence length.

We first introduce the notation used in our analysis. An input-target pair $\xi$ is admissible if $\xi \in \mathcal Z$, where $\mathcal Z$ is a fixed, parameter-independent measurable set of input-target sequences satisfying $\operatorname{Pr}_{\xi \sim \mathcal P_k}\left(\xi \in \mathcal Z\right) = 1$ for all clients $k$. All conditions below hold uniformly for admissible sequences, finite lengths $S \geq 1$, and token indices $t \leq S$, with finite constants independent of $S$, $\xi$, and $t$. Next, we state the necessary definitions and assumptions on the loss and the input sequence.

\begin{definition}[SSM Stability]
\label{def:ssm_stability}
For $\gamma, \delta > 0$, define $\mathcal S_{\gamma, \delta}$ as the set of parameter tuples $W$ satisfying $A_n \leq -\gamma$ for every $n \in \{1,\dots,N\}$, where $A_n$ is the $n$-th element of the transition parameter $A$, and $\left[\operatorname{Softplus}(W_\Delta x_t)\right]_d \geq \delta$ for every feature coordinate $d \in \{1,\dots,D\}$ and admissible input token. Consequently, every entry of $\bar A_t$ satisfies: $\left(\bar{A}_t\right)_{n,d} \leq \rho := \exp{(-\gamma\delta)} < 1$.
\end{definition}

\begin{assumption}[Output loss regularity]
\label{assum:lipschitz_cont_loss}
The token loss $\ell_t(\cdot, y_t^*)$ is twice continuously differentiable with respect to its prediction argument. For the parameter tuple $W$ under consideration, there exist finite constants $L_\ell,H_\ell>0$ such that, for every admissible $\xi$, the prediction $y=y_t(W;x)$ satisfies
\begin{equation}
    \textstyle \|\nabla_y\ell_t(y, y_t^*)\|_2\leq L_\ell, \qquad  \|\nabla_y^2\ell_t(y,y_t^*)\|_{\mathrm{op}}\leq H_\ell.
\end{equation}
\end{assumption}
\vspace{-3mm}

Note that for softmax cross-entropy over logits, the Hessian bound holds globally.

\begin{assumption}[Bounded Inputs]
\label{assum:bounded_inputs}
There exists $L_x > 0$ such that $\forall t$, $\|x_t\|_2 \leq L_x$ $\forall$ admissible $\xi$.  
\end{assumption}

\begingroup
\setlength{\topsep}{0.5pt}
\setlength{\partopsep}{0pt}
\setlength{\parsep}{0pt}
\setlength{\itemsep}{0pt}
\setlength{\parskip}{2.5pt} 
\setlength{\abovedisplayskip}{2pt}
\setlength{\belowdisplayskip}{2pt}
\setlength{\abovedisplayshortskip}{1pt}
\setlength{\belowdisplayshortskip}{1pt}
\begin{theorem}[SSM Gradient Bound] \label{thm:ssm_grad_bound} Suppose Assumptions~\ref{assum:lipschitz_cont_loss} and \ref{assum:bounded_inputs} hold for a parameter tuple $W \in \mathcal S_{\gamma, \delta}$. For a single-layer SSM defined in Section~\ref{sec:system_model}, the following bound holds for every finite integer $S \geq 1$ and every admissible length-$S$ sequence pair $\xi \in \mathcal Z$:
\begin{equation}
    \textstyle \left\|\nabla_W \mathcal L^S(W;\xi)\right\|_F \leq \mathcal G_{\text{SSM}}(W) := L_\ell L_x\left(K_B K_\Delta \frac{\sqrt{ND}}{1-\rho} + \|W_C\|_2 \frac{K_W}{1-\rho}\right)
\end{equation}
with
\begin{align}
    \textstyle K_\Delta & \textstyle := \left\|W_\Delta\right\|_2L_x + \sqrt{D}\log(2), ~~~~
    K_B := \left\|W_B\right\|_2 L_x^2, \\
    \textstyle K_W & \textstyle := ND K_BK_{\Delta} \left(K_\Delta + \sqrt{D} \left\|A\right\|_2L_x\right)
    \left(\frac{\rho}{1-\rho}\right) + L_x^2 \sqrt{N} K_\Delta + L_x \sqrt{D} K_B
\end{align}
where $\|\cdot\|_F$ is the Euclidean norm of the block Frobenius norms.
\end{theorem}
\endgroup

\textit{Discussion.} The key technical step is to bound both the hidden states and their parameter Jacobians under input-dependent transitions. The stability condition $\rho=\exp(-\gamma\delta)<1$ allows us to reduce the recurrent contributions to geometric sums, yielding an explicit parameter-gradient bound rather than requiring such a bound as an assumption. The resulting architecture-aware bound is independent of the sequence length while retaining explicit dependence on $\rho$ and the core selective SSM parameter blocks $W_\Delta$, $W_B$, $W_C$, and $A$. The complete proof is provided in Appendix~\ref{appendix:gradient_norm_bound_proof}. Appendix~\ref{appendix:multi_layer_SSM_proofs} extends the analysis to the multi-layer setting by combining stability with layerwise input-output sensitivities, resulting in the bound $\mathcal G_{\mathrm{SSM}}^{(Q)}(W)$ stated in Theorem~\ref{thm:ssm_multi_layer_grad_bound}. Together, these results provide architecture-specific control of the gradient norm required for the subsequent federated convergence analysis.

\subsection{Convergence Analysis}\label{subsec:theoretical_convergence_standard_federated_baselines}
To characterize federated optimization with selective SSMs, we introduce a local parameter neighborhood and impose the stochastic-gradient conditions required for the convergence analysis.

\begingroup
\setlength{\topsep}{0pt} 
\setlength{\partopsep}{0pt}
\setlength{\parsep}{0pt}
\setlength{\itemsep}{0pt}
\setlength{\parskip}{2.5pt}
\setlength{\abovedisplayskip}{2pt}
\setlength{\belowdisplayskip}{2pt}
\setlength{\abovedisplayshortskip}{1pt}
\setlength{\belowdisplayshortskip}{1pt}
\begin{definition}[$R$-Local Neighborhood]\label{def:r_local_neighborhood} For a reference model $\widetilde W$ and radius $R > 0$, define $\mathcal N_R(\widetilde W) := \{W:\|W-\widetilde W\|_F\leq R\}$, where the $\|\cdot\|_F$ is defined for a parameter tuple $W$ as $\|W\|_F:=\left(\|{W_\Delta} \|_F^2+\|{W_B} \|_F^2+\|{W_C} \|_F^2+\|A \|_F^2\right)^{1/2}$.
\end{definition}
\endgroup

In the following convergence analysis, the regularity conditions of Assumption~\ref{assum:lipschitz_cont_loss} are assumed to hold uniformly over each neighborhood $\mathcal N_R(W^r)$. In particular, the same finite constants $L_\ell$ and $H_\ell$ apply to every $W\in\mathcal N_R(W^r)$, every admissible input-target pair, and every token position. The stability margins $\gamma,\delta$ and input bound $L_x$ are likewise fixed throughout the neighborhood.

\begingroup
\setlength{\topsep}{1pt}
\setlength{\partopsep}{0pt}
\setlength{\parsep}{0pt}
\setlength{\itemsep}{0pt}
\setlength{\abovedisplayskip}{2pt}
\setlength{\belowdisplayskip}{2pt}
\setlength{\abovedisplayshortskip}{1pt}
\setlength{\belowdisplayshortskip}{1pt}
\begin{assumption}[Unbiased stochastic gradients]
\label{assum:unbiased_stochastic_gradients}
For every client $k$ and every stable local model neighborhood $\mathcal N_R(W^r)\subseteq\mathcal S_{\gamma,\delta}$, define
\begin{equation}
    g_k(W;\xi) := \nabla_W\mathcal L^S(W;\xi)
\end{equation}
as the stochastic sequence-level gradient for a sampled sequence
$\xi\sim \mathcal P_k$. Then, $\forall~W\in\mathcal N_R(W^r)$
\begin{equation}
    \mathbb E_{\xi\sim \mathcal P_k}\left[g_k(W;\xi)\right] = \nabla_W\mathcal L_k(W)
\end{equation}
holds. At each local update, conditional on all preceding training randomness, the mini-batch consists of $\Xi$ independent sequences drawn from $\mathcal P_k$. The mini-batch gradient is their average. Batches drawn by different clients at the same local step are conditionally independent.
\end{assumption}
\endgroup

\begingroup
\setlength{\topsep}{1pt}
\setlength{\partopsep}{0pt}
\setlength{\parsep}{0pt}
\setlength{\itemsep}{0pt}
\setlength{\parskip}{2.5pt}
\setlength{\abovedisplayskip}{2pt}
\setlength{\belowdisplayskip}{2pt}
\setlength{\abovedisplayshortskip}{1pt}
\setlength{\belowdisplayshortskip}{1pt}
\begin{lemma}[Local SSM smoothness]
\label{lem:ssm_local_smoothness}
Let $W^r$ be a fixed reference parameter tuple and let $0<R<\infty$ satisfy $\mathcal N_R(W^r)\subseteq\mathcal S_{\gamma,\delta}$. Suppose Assumptions~\ref{assum:lipschitz_cont_loss}, ~\ref{assum:bounded_inputs}, and ~\ref{assum:unbiased_stochastic_gradients} hold, with finite and differentiable expected local objectives $\mathcal L_k$ on this neighborhood. Furthermore, let the bounds in Assumption~\ref{assum:lipschitz_cont_loss} hold with the same finite constants $L_\ell$ and $H_\ell$ $\forall$ $W\in\mathcal N_R(W^r)$, every admissible input-target sequence pair $\xi\in\mathcal Z$, and every token position, independently of sequence length.

Then, for every client $k$, $\mathcal L_k$ is $L_{\mathrm{SSM},R}$-smooth on $\mathcal N_R(W^r)$. Specifically, for all
$W,V\in\mathcal N_R(W^r)$,
\begin{equation}
    \|\nabla_W\mathcal L_k(W)
      -\nabla_W\mathcal L_k(V)\|_F
    \leq
    L_{\mathrm{SSM},R}\|W-V\|_F,
\end{equation}
where $L_{\mathrm{SSM},R}$ is defined
in~\eqref{eq:L_ssm_definition}.
\end{lemma}
\endgroup

With these conditions in place, we characterize the convergence of FedAvg for selective SSMs. The corresponding analysis for FedProx and supplementary corollaries are available in Appendix~\ref{appendix:fed_algs_convergence_bounds_proofs},~\ref{sec:supplementary_corollaries}.

\begingroup
\setlength{\topsep}{0.5pt} 
\setlength{\partopsep}{0pt}
\setlength{\parsep}{0pt}
\setlength{\itemsep}{0pt}
\setlength{\parskip}{2.5pt} 
\setlength{\abovedisplayskip}{2pt}
\setlength{\belowdisplayskip}{2pt}
\setlength{\abovedisplayshortskip}{1pt}
\setlength{\belowdisplayshortskip}{1pt}
\begin{corollary}[Single-layer FedAvg convergence rate] \label{cor:supp_single_fedavg}
Under Assumptions~\ref{assum:lipschitz_cont_loss}-~\ref{assum:unbiased_stochastic_gradients}, the conditions specified in Definitions~\ref{def:ssm_stability}-~\ref{def:r_local_neighborhood}, and the conditions of
Theorem~\ref{thm:fedavg_ssm_convergence}, let $M_B,M_C,M_\Delta,$ and $M_A$ denote uniform upper bounds on $\left\|W_B\right\|_2,\left\|W_C\right\|_2,\left\|W_\Delta\right\|_2,$ and $\left\|A\right\|_2$, respectively, over the $R$-local neighborhoods $\mathcal N_R(W^r)$, and define
\begin{equation*}
    M_R := (1 + M_B)(1 + M_C)(1 + M_\Delta)(1 + M_\Delta + M_A).
\end{equation*}
We further define $\bar{\mathcal P}:=\sum_{k=1}^K p_k\mathcal P_k$, $\mathcal H_{\mathrm{TV}} := \sum_{k=1}^K p_k \operatorname{TV}(\mathcal P_k,\bar{\mathcal P})^2,$ where $\operatorname{TV}$ denotes the total variation distance. Under these conditions, for FedAvg applied to the single-layer selective SSM defined in Section~\ref{sec:system_model}, the global optimization trajectory $\{W^r\}_{r=0}^T$ satisfies
\begin{equation*} 
    \frac{1}{T}\sum_{r=0}^{T-1} \mathbb E \left\|\nabla f(W^r)\right\|_F^2 = \mathcal O\left(\frac{f\left(W^0\right)-f\left(W^*\right)}{\eta ET} + \frac{\eta M_R^4}{\Xi(1-\rho)^8} +  \frac{\eta^2 M_R^6(E-1) }{(1-\rho)^{12}}\left(\frac{1}{\Xi}+E\mathcal H_{\mathrm{TV}}\right)\right).
\end{equation*}
The quantities $N,D,L_x,L_\ell,$ and $H_\ell$ affect only the implicit constant and do not alter the displayed scaling with the optimization parameters, parameter envelopes, stability, or sequence length.
\end{corollary}
\endgroup

\textit{Proof sketch and discussion.} Standard FedAvg convergence bounds assume smoothness and bounds on stochastic gradient variance and client-gradient disagreement. In contrast, our analysis derives these quantities from the selective SSM recurrence. We first bound the hidden states and their parameter derivatives to obtain the gradient bound. The second-order analysis then establishes local smoothness, while the remaining lemmas translate the gradient bound into bounds on stochastic variance and client disagreement. Together, these results yield a convergence rate whose model dependence is explicit in the parameter envelopes $M_B,M_C,M_\Delta,M_A$ and the stability factor $\rho$, with dimensions and input and loss regularity constants fixed. As a result, the corollary identifies which SSM parameters enter the convergence guarantee without requiring a separately assumed smoothness or gradient-dissimilarity constant for the model.

The derived bound distinguishes sequence length from recurrent memory: the gradient and smoothness estimates contain no explicit sequence-length factor, but increase as the recurrent stability factor $\rho$ approaches one, allowing past state information to persist for longer. This dependence is stronger in the additional local-drift terms, which contain $M_R^6/(1-\rho)^{12}$, as opposed to the baseline stochastic term, which contains $M_R^4/(1-\rho)^8$. Consequently, the bound assigns a larger penalty to performing multiple local updates when $\rho$ is close to one. However, $\rho$ alone does not determine this penalty: even for the same $\rho$ and client distributions, larger parameter envelopes increase the coefficient multiplying $\mathcal H_{\mathrm{TV}}$. The analysis therefore separates the amount of distribution mismatch from the influence of the SSM parameters on its contribution to the convergence bound. This heterogeneity contribution vanishes at $E=1$, but does not decrease with batch size when $E>1$. Our synthetic experiments observe increasing client disagreement as $\rho$ approaches one, supporting this qualitative connection without establishing the worst-case exponents.

Furthermore, the radius $R$ determines how the broadcast model’s weight norms are enlarged to obtain bounds valid throughout local training: each block satisfies $\|W_X\|_2\leq\|W_X^r\|_2+R$. As such, allowing greater client displacement increases the parameter envelopes entering the convergence bound, even when $\rho$ remains unchanged. Additional local steps can therefore affect the guarantee through both the explicit dependence on $E$ and the larger envelopes needed to cover the resulting trajectories. These trajectories must also remain within a neighborhood satisfying $\mathcal N_R(W^r)\subseteq\mathcal S_{\gamma,\delta}$. The FedProx extension in Appendix~\ref{appendix:fed_algs_convergence_bounds_proofs} demonstrates how restricting client movement changes the guarantee: its effective radius $R_\mu\leq R$ permits smaller envelopes $M_X^\mu=M_X-R+R_\mu$, reducing the SSM-dependent factors multiplying noise and heterogeneity. However, the proximal term also increases local smoothness by \(\mu\). The resulting bound therefore exposes a trade-off between tighter control of the SSM parameters and the additional smoothness cost of proximal regularization.

%% file: section_files/experiments.tex
\section{Experiments}\label{sec:experiments}

In this section, we empirically evaluate the derived bounds and their implications for distributed SSM training in two parts: (i) direct evaluation of gradient and smoothness bounds using sequences generated by frozen SSM teachers, and (ii) Mamba2 language modeling experiments with nine federated learning algorithms. In the latter, we compare theory-motivated expectations with observed performance illustrating the practical utility of the derived bounds. Details on our experimental setup are included in Appendix~\ref{appendix:detailed_exp_setup}, with supplementary experimental results available in Appendix~\ref{appendix:additional_experimental_results}.

\textbf{Validation of Derived Bounds.} We consider autoregressive token sequences generated by a frozen one-layer selective SSM parameterized by $W=\left(W_\Delta,W_B,W_C,A\right)$ analyzed in Theorem~\ref{thm:ssm_grad_bound}. This setup follows the formulation of Section~\ref{sec:system_model} exactly, with the teacher SSM generating one-hot token ($L_x = 1$) sequences of length $S=64$ for different combinations of $D,N$ and stability limits $\rho$, yielding $72$ distinct datasets. For each such dataset, we train a learner SSM on $4096$ sequences using projected SGD with softmax cross-entropy loss ($L_\ell=\sqrt2$), with the model updates inside $\mathcal S_{\gamma,\delta}$. We check test-sequence gradients at five learner and one teacher checkpoints against the theorem's $\mathcal G_{\mathrm{SSM}}$. We also use $\mathcal G^{\mathrm{tight}}_{\mathrm{SSM}}$, tightened via realized stability, one-hot input bounds, and finite-history sums. We also evaluate the smoothness bound and examine the drift and disagreement quantities entering the FedAvg and FedProx analysis. Appendix~\ref{appendix:synthetic_results} provides additional validation results. 

\begin{table}[h]
\centering
\scriptsize
\caption{\textbf{Gradient and smoothness bounds on the synthetic bank: 1,411,776 evaluations, no violations detected.} Maximum utilization is the largest measured-to-derived ratio, with $a/b$ denoting configured/tightened constants. Evaluations count test sequences at each checkpoint (first row), sequence-checkpoint pairs (curvature rows), or consecutive-checkpoint pairs (trajectory rows).}
\label{tab:synthetic_bounds}
\begin{tabular}{llrrr}
\toprule
Statement & Inequality checked & Evaluations & Violations & Max.\ utilization \\
\midrule
Thm.~\ref{thm:ssm_grad_bound} & $\|\nabla\mathcal{L}\|_F\le\mathcal{G}_{\mathrm{SSM}}$, per sequence & 1,327,104 & 0 / 0 & $1.0\times10^{-3}$ / $8.3\times10^{-3}$ \\
Lemma~\ref{lem:ssm_local_smoothness} & $\widehat{\|\nabla^2\mathcal L\|_2}\le L_{\mathrm{SSM},0}$ & 41,472 & 0 & $7.0\times10^{-5}$ \\
Lemma~\ref{lem:ssm_local_smoothness} & finite-difference ratio $\le L_{\mathrm{SSM},0.001}$ & 41,472 & 0 & $8.2\times10^{-5}$ \\
Lemma~\ref{lem:ssm_local_smoothness} & $\|\nabla\mathcal{L}(W_a)-\nabla\mathcal{L}(W_b)\|_F\le L_{\mathrm{SSM},R}\|W_a-W_b\|_F$ & 864 & 0 & $1.0\times10^{-5}$ \\
Lemma~\ref{lem:local_lipschit_SSM_grad_bound} & $|\mathcal{G}_{\mathrm{SSM}}(W_a)-\mathcal{G}_{\mathrm{SSM}}(W_b)|\le L_{\mathcal{G}_\mathrm{SSM},R}\|W_a-W_b\|_F$ & 864 & 0 & $0.42$ \\
\bottomrule
\end{tabular}
\end{table}

\begin{figure}[t]
\centering
\begin{subfigure}[t]{0.28\linewidth}
  \centering\includegraphics[width=\linewidth]{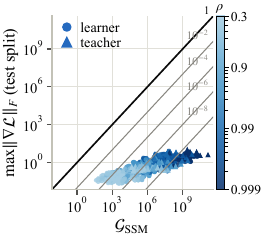}
  \caption{Gradient versus bound}\label{fig:synthetic_a}
\end{subfigure}\hfill
\begin{subfigure}[t]{0.225\linewidth}
  \centering\includegraphics[width=\linewidth]{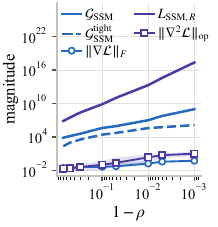}
  \caption{Growth with memory}\label{fig:synthetic_b}
\end{subfigure}\hfill
\begin{subfigure}[t]{0.225\linewidth}
  \centering\includegraphics[width=\linewidth]{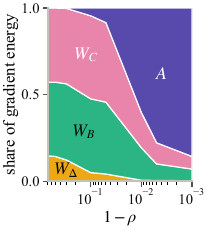}
  \caption{Grad. energy by $W_X$}\label{fig:synthetic_c}
\end{subfigure}\hfill
\begin{subfigure}[t]{0.25\linewidth}
  \centering\includegraphics[width=\linewidth]{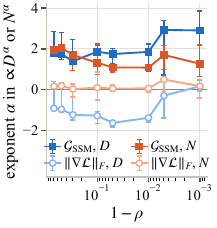}
  \caption{Exponents in $D$ and $N$}\label{fig:synthetic_d}
\end{subfigure}
\caption{\textbf{Single-layer bounds on the SSM teacher-generated datasets.} (\subref{fig:synthetic_a})~Across 1,296 learner and teacher parameter points, every measured per-sequence gradient norm lies below the configured $\mathcal G_{\mathrm{SSM}}$. (\subref{fig:synthetic_b})~At training end, $\mathcal G_{\mathrm{SSM}}$ and $L_{\mathrm{SSM},R}$ scale as $(1-\rho)^{-1.8}$ and $(1-\rho)^{-3.6}$, versus $(1-\rho)^{-0.5}$ and $(1-\rho)^{-1.0}$ for measured gradient and estimated Hessian norms. (\subref{fig:synthetic_c})~As $\rho\to1$, gradient energy shifts from $W_B,W_C$ toward the transition block $A$. (\subref{fig:synthetic_d})~One-factor fits estimate the power-law exponent $\alpha$ for the derived (filled) and measured (open) gradient norms as vocabulary size $D$ (blue) or state size $N$ (orange) varies, with bootstrap 95\% intervals.}
\label{fig:synthetic}
\vspace{-5mm}
\end{figure}

\textbf{Results and Discussion.} No violations were detected in 1,411,776 gradient, smoothness and bound-variation checks (Figure~\ref{fig:synthetic_a}, Table~\ref{tab:synthetic_bounds}), or in 103,200 federated checks across 1,200 FedAvg/FedProx runs on 72 teacher datasets. The federated runs explore step sizes beyond the theorems' range. The results also reveal informative patterns.

\textit{Memory amplifies gradient and curvature sensitivity.} Across all generated datasets, the measured gradient and estimated Hessian norms grow approximately as $(1-\rho)^{-0.5}$ and $(1-\rho)^{-1.0}$, respectively, while their bounds grow as $(1-\rho)^{-1.8}$ and $(1-\rho)^{-3.6}$, as seen in Figure~\ref{fig:synthetic_b}. Although the bounds grow more steeply, the curvature exponent is twice the gradient exponent for both the measured quantities and the certified constants. These trends are consistent with the recurrent-sensitivity structure of Theorem~\ref{thm:ssm_grad_bound} and Lemma~\ref{lem:ssm_local_smoothness}. The shared sensitivity factor $L_x(H_R+M_CJ_R)$ enters the gradient bound linearly and the first term of the smoothness bound quadratically in~\eqref{eq:L_ssm_definition}. The approximately doubled fitted curvature exponent mirrors this linear-versus-quadratic dependence, although the observed growth is substantially milder than that of the bounds.

\textit{Recurrent sensitivity and block-wise gradients.} The recurrent-Jacobian term, which propagates parameter perturbations through the transition $\bar A_t$, carries at least $99\%$ of $\mathcal G_{\mathrm{SSM}}$ at every $\rho$. The measured gradient energy (the share of $\|\nabla_W\mathcal L\|_F^2$ carried by each block) shifts from $W_B,W_C$ ($86\%$ at $\rho=0.3$) to $A$ ($85\%$ at $\rho=0.999$, Figure~\ref{fig:synthetic_c}). Dimension fits (Figure~\ref{fig:synthetic_d}) both vary with $\rho$ and give bound exponents of $1.4$-$2.9$ in $D$ and $1.1$-$2.0$ in $N$, reflecting the $N\,D^{2}$-type factors in $K_W$.

Appendix~\ref{appendix:synthetic_results} examines the quantities entering the convergence analysis. The observed $E^{0.87}$ growth of the $\|W_k^{r,E}-W^r\|_F$ drift is consistent with the linear accumulation bound $\eta E\mathcal G_{\mathrm{SSM},R}$ used in the FedAvg analysis. Increasing client disagreement as $\rho\to1$ likewise agrees qualitatively with the SSM-sensitivity factor $\mathcal G_{\mathrm{SSM},R}^2$ in Lemma~\ref{lem:ssm_gamma_tv_bound}, which bounds disagreement by $4\mathcal G_{\mathrm{SSM},R}^2\mathcal H_{\mathrm{TV}}$. Across the bank, variation in log-utilization $\log_{10}(\|\nabla\mathcal L\|_F/\mathcal G_{\mathrm{SSM}})$ is dominated by the memory factor $\rho/(1-\rho)^2$. Its signed covariance shares are $93\%$ for log-bound variance and $102\%$ for log-utilization variance, the latter partly offset by measured gradient growth. Dimension and norm factors contribute less (Figure~\ref{fig:app_drivers}). This identifies the main source of variation in tightness, not absolute looseness.

\textit{Limitations and outlook.} The experiments did demonstrate some limitations of our derived bounds, mostly how conservative their values are in absolute terms, and how parameter-envelope growth with radius $R$ over $T$ rounds and saturation of the sequence-level total-variation metric affect the looseness. Nevertheless, the observations agree qualitatively with the analysis's stability, sensitivity, and drift dependencies: no violations were detected, gradients and curvature grow with memory, blockwise gradients shift toward $A$, and drift grows with local steps. The recurrent factor $\rho/(1-\rho)^2$ dominates variation in tightness across the bank, with smaller contributions from projection and input factors. As such, they point to concrete future improvements: data-dependent constants that use the realized stability margin and the finite sequence length, envelopes that track the actual local trajectories rather than the worst case over the whole radius $R$, and a token-level dissimilarity that would make the heterogeneity term sensitive to the data partition. The Mamba2 experiments that follow examine whether the same drift and heterogeneity mechanisms help interpret practical federated training.

\textbf{Practical Federated Optimization.} We train randomly initialized Mamba2~\citep{dao2024transformers} models of depth $4$ and $16$ across $K=4$ clients for $T=50$ rounds with $E\in\{64,128,256\}$ local steps and heterogeneity $\beta\in\{0,0.7,1\}$ on six RedPajama~\citep{together2023redpajama} domains under four domain assignments ($\beta=0$ gives every client the same domain mix, $\beta=1$ one domain each). Nine algorithms are compared over $72$ configurations and three seeds using mean last-ten-round validation perplexity, which also selects their hyperparameters  (see Appendix~\ref{appendix:detailed_exp_setup} for details).

\textbf{Theory-Guided Expectations.} An update rule affects radius $R$ (in $\mathcal G_{\mathrm{SSM},R},L_{\mathrm{SSM},R}$), a heterogeneity term $\propto\eta^2L_{\mathrm{SSM},R}^2\mathcal G_{\mathrm{SSM},R}^2E(E-1)\mathcal H_{\mathrm{TV}}$, and the server step to $W^{r+1}$. For SGD, FedProx bounds drift by $R_\mu\leq\eta E\mathcal G_{\mathrm{SSM},R}$ (Theorem~\ref{cor:fedprox_ssm_convergence} in Appendix~\ref{appendix:fed_algs_convergence_bounds_proofs}), at the cost of $\mu$ in $L_{\mathrm{prox},R_\mu}$, motivating benefits from tuned proximal regularization. SCAFFOLD and FedDyn can reduce corrected client disagreement when corrections track client-gradient differences, though benefits may be smaller when disagreement is small. FedExP, FedAdam and DiLoCo rescale server updates and may enlarge neighborhoods, while the norm-preserving FedAlign and small-perturbation FedSAM may behave like FedAvg, without a shared convergence guarantee (Appendix~\ref{appendix:additonal_distributed_learning_frameworks}). Table~\ref{tab:practical_ranking} summarizes these expectations for the AdamW-based Mamba2 experiments motivated by the insights provided by our analysis. 

\begin{table}[t]
\centering
\footnotesize
\setlength{\tabcolsep}{3.5pt}
\begin{tabular}{llcccc}
\toprule
& & \multicolumn{2}{c}{Expected} & \multicolumn{2}{c}{$\Delta$PPL\,(\%)} \\
\cmidrule(lr){3-4}\cmidrule(lr){5-6}
Method & Theory-linked mechanism & $\beta=0$ & $\beta=1$ & $\beta=0$ & $\beta=1$ \\
\midrule
FedProx~\citep{MLSYS2020_1f5fe839} & radius, $R_\mu\leq R$ & $-$ & $-$ & $-17.0$ & $-34.4$\\
SCAFFOLD~\citep{pmlr-v119-karimireddy20a} & corrected disagreement & $+$ & $-$ & $+4.0$  & $-19.4$ \\
FedDyn~\citep{acar2021federated} & heterogeneity; radius (proximal $\alpha$) & $-$ & $-$ & $-2.2$ & $-35.4$ \\
FedAlign~\citep{11099331} & state-coordinate alignment & $=$ & $=$ & $-0.3$ & $-1.2$\\
FedSAM~\citep{pmlr-v162-qu22a} & expanded evaluation neighborhood & $=$ & $=$ & $-0.4$  & $-1.7$ \\
FedAvg~\citep{mcmahan2017communication} & reference & & & $0.0$ & $0.0$ \\
FedExP~\citep{jhunjhunwala2023fedexp} & adaptive extrapolation, $\eta_g\geq1$ & $=$ & $+$ & $-0.4$  & $+14.6$  \\
DiLoCo-style~\citep{douillard2024diloco} & server step, outer momentum & $+$ & $+$ & $+2.0$ & $+1.5$ \\
FedAdam~\citep{reddi2021adaptive} & server step $\alpha_s/\tau_s$ & $+$ & $+$ & $+17.2$  & $+0.6$ \\
\bottomrule
\end{tabular}
\caption{\textbf{Theorem-motivated expectations versus observed behavior on Mamba2.} $-$/$=$/$+$ denote expected lower/similar/higher validation perplexity relative to FedAvg, motivated by the listed mechanisms. $\Delta$PPL is the mean percentage change across the $24$ configurations at each $\beta$. }
\label{tab:practical_ranking}
\vspace{-6mm}
\end{table}

\textbf{Results and Discussion.} Across all $72$ configurations and all three random seeds, the Pearson correlation between the expected performance direction (encoded as $-1$, $0$, and $+1$ for lower, similar, and higher perplexity, respectively) and the observed log-perplexity change relative to FedAvg is $0.71$, increasing to $0.81$ when restricted to $\beta = 1.0$. For the two heterogeneity regimes $\beta\in\{0,1\}$, the percentage changes in validation perplexity relative to FedAvg reported in Table~\ref{tab:practical_ranking}, averaged over the $24$ configurations at each $\beta$ and all three seeds, agree with $13$ of the $16$ expected lower/similar/higher outcomes. FedProx's improvement over FedAvg is consistent with the potential benefit of tighter SSM envelopes identified in Theorem~\ref{cor:fedprox_ssm_convergence} in Appendix~\ref{appendix:fed_algs_convergence_bounds_proofs}. SCAFFOLD's improvement in $23/24$ configurations at $\beta=1$, compared with deterioration in $22/24$ at $\beta=0$, aligns with the disagreement-correction interpretation of the FedAvg bound developed in Appendix~\ref{appendix:additonal_distributed_learning_frameworks}. Together, these results support using the theorem's displacement and disagreement terms to interpret the observed dependence of algorithm performance on client heterogeneity.

%% file: appendix.tex
\section*{Appendix}

\startcontents[sections]
\printcontents[sections]{l}{1}{\setcounter{tocdepth}{2}}
\clearpage

\section{Selective SSM Gradient Bound Proof}\label{appendix:gradient_norm_bound_proof}
In this section, we derive an upper bound on the gradient norm of the loss function for some admissible input sequence processed by a single-layer selective state space model (SSM). 

Fix a parameter tuple $W \in \mathcal S_{\gamma, \delta}$ satisfying the assumptions of Theorem~\ref{thm:ssm_grad_bound}. Let $S \geq 1$ be an arbitrary integer and $\xi = (x, y^*) \in \mathcal Z$ an arbitrary admissible sequence pair of length $S$. We use the SSM dynamics from Section~\ref{sec:system_model} and the notational conventions from section~\ref{subsec:ssm_gradient_bounds}.

Given these definitions, we examine the sensitivity of the loss with respect to the learnable model parameter tuple
\begin{equation}
    W := (W_\Delta, W_B, W_C, A).
\end{equation}

Throughout the proof, gradients with respect to $W$ are understood as block gradients, i.e., 
\begin{equation}
    \|\nabla_W f\|_F := \left(\|\nabla_{W_\Delta} f\|_F^2+\|\nabla_{W_B} f\|_F^2+\|\nabla_{W_C} f\|_F^2+\|\nabla_A f\|_F^2\right)^{1/2}.
\end{equation}

Next, let $\theta = \operatorname{vec}(W)$ concatenate the vectorized parameter blocks. Scalar gradients are column vectors after vectorization:
\begin{equation*}
    \nabla_\theta f = \operatorname{vec}\left(\nabla_W f\right)
\end{equation*}
For a vector- or matrix-valued function $F$, $\partial F/\partial W$ denotes its flattened block Jacobian, with output entries indexing rows and scalar parameters indexing columns:
\begin{equation*}
    \left\|\frac{\partial F}{\partial W} \right\|_F^2 := \sum_{\alpha}\left\|\frac{\partial F}{\partial \theta_{\alpha}}\right\|_F^2
\end{equation*}
In product-rule expressions, products and transposes involving matrix-valued derivatives are applied separately to each scalar-parameter slice before flattening. Since $h_t$ does not depend on $W_C$, its $W_C$-Jacobian block is zero.

For clarity, in this section we omit the client index $k$ and focus on bounding the gradient norm for a single client processing a sequence of length $S$.

For the fixed sample $\xi$, using the simplified notation introduced in Section~\ref{sec:system_model}, the sequence loss is:
\begin{equation*}
    \mathcal L^S(W) = \frac{1}{S}\sum_{t=1}^S \ell_t(y_t)
\end{equation*}
where $\ell_t(y_t):=\ell_t(y_t(W;x), y^*_t)$ represents the instantaneous loss between the predicted and target tokens at the $t$-th place in the sequence $S$, and we omit the dependence on a fixed sample $\xi$. 

The goal of this derivation is to bound the norm of the gradient $\nabla_W \mathcal L^S(W)$. By the linearity of the gradient operator, it then follows that:
\begin{equation}
    \nabla_W \mathcal L^S(W) = \frac{1}{S} \sum_{t=1}^S \nabla_W \ell_t(y_t)
\end{equation}

We then apply the chain rule to the token loss $\ell_t(y_t)$, where $\nabla_y\ell_t(y_t)$ is a column vector. Noting that $y_t$ depends on $W$ both directly through the output projection vector $C_t$ and indirectly through the hidden state $h_t$, we obtain: 
\begin{align}
    \nabla_\theta\ell_t(y_t) &= \left(\frac{\partial y_t}{\partial W}\right)^\top\nabla_y\ell_t(y_t) \nonumber\\ &= \left(\frac{\partial(h_t^\top C_t)}{\partial W}
    \right)^\top\nabla_y\ell_t(y_t) \nonumber\\ &= \left(\frac{\partial h_t^\top}{\partial W}C_t + h_t^\top\frac{\partial C_t}{\partial W}\right)^\top\nabla_y\ell_t(y_t).
\end{align}
Here, the product-rule terms inside the parentheses are formed separately for each scalar-parameter slice and constructed as the columns of the output Jacobian before taking the outer transpose. The first term contributes only to the $W_\Delta$, $W_B$, and $A$ blocks, while the second term contributes only to the $W_C$ block.

Under the described block-gradient convention, the schematic chain-rule decomposition gives:
\begin{equation}
    \|\nabla_W\ell_t(y_t)\|_F = \|\operatorname{vec}(\nabla_W\ell_t(y_t))\|_2 = \|\nabla_\theta\ell_t(y_t)\|_2  =  \left\|\left(\frac{\partial h_t^\top}{\partial W}C_t + h_t^\top\frac{\partial C_t}{\partial W}\right)^\top\nabla_y\ell_t(y_t)\right\|_2.
\end{equation}

Next, applying the matrix-vector norm inequality to each scalar-parameter slice and summing the squared bounds gives:
\begin{equation} 
    \left\|\frac{\partial h_t^\top}{\partial W}C_t\right\|_F \leq \left\|\frac{\partial h_t}{\partial W}\right\|_F \|C_t\|_2.
\end{equation}

Using this bound, the triangle inequality, and the fact that $\|M^\top v\|_2\leq \|M\|_F\|v\|_2$, we get:
\begin{align}
    \|\nabla_W\ell_t(y_t)\|_F &\leq \left\|\left(\frac{\partial h_t^\top}{\partial W}C_t\right)^\top\nabla_y\ell_t(y_t)\right\|_2 +\left\|\left(\frac{\partial C_t}{\partial W}\right)^\top h_t\nabla_y\ell_t(y_t)\right\|_2 \nonumber \\ &\leq \|\nabla_y\ell_t(y_t)\|_2 \left\|\frac{\partial h_t}{\partial W}\right\|_F \|C_t\|_2 + \left\|\left(\frac{\partial C_t}{\partial W}\right)^\top h_t\nabla_y\ell_t(y_t) \right\|_2.
    \label{eq:norm_ineq_first_appendix}
\end{align}

Next, utilizing Assumptions \ref{assum:lipschitz_cont_loss} and \ref{assum:bounded_inputs}, we bound the two terms of this expression separately.

First, since $C_t = W_Cx_t$, we have:
\begin{equation}
    \|C_t\|_2 \leq \|W_C\|_2\|x_t\|_2 \leq \|W_C\|_2L_x
\end{equation}
Therefore,
\begin{equation}\label{eq:first_part_init_bound}
    \left\|\nabla_y\ell_t(y_t)\right\|_2
    \left\|\frac{\partial h_t}{\partial W}\right\|_F
    \|C_t\|_2
    \leq
    L_\ell L_x\|W_C\|_2
    \left\|\frac{\partial h_t}{\partial W}\right\|_F 
\end{equation}

For the second term, since $h_t$ does not depend on $W_C$ and $x_t$ is fixed, equation~\eqref{eq:ssm_dynamics_y_appendix} gives:
\begin{equation}
    \nabla_{W_C}\ell_t(y_t) = h_t\nabla_y\ell_t(y_t)x_t^\top.
\end{equation}

Furthermore, since $C_t=W_Cx_t$, the Jacobian $\partial C_t/\partial W$ is zero in the $W_\Delta$, $W_B$, and $A$ parameter blocks, and nonzero only for the $W_C$ block. As a result, $(\partial C_t/\partial W)^\top h_t\nabla_y\ell_t(y_t)$ contains the vectorized $W_C$-block gradient and zeros in all other parameter blocks, yielding:
\begin{align}
    \left\|\left(\frac{\partial C_t}{\partial W}\right)^\top h_t\nabla_y\ell_t(y_t) \right\|_2 &= \|\nabla_{W_C}\ell_t(y_t)\|_F = \|h_t\nabla_y\ell_t(y_t)x_t^\top\|_F.
\end{align}

This can be bounded as:
\begin{align}
    \left\| h_t\nabla_y\ell_t(y_t)x_t^\top\right\|_F &= \left\|h_t\nabla_y\ell_t(y_t)\right\|_2\|x_t\|_2  \leq \left\|h_t\nabla_y\ell_t(y_t)\right\|_2L_x \\ &\leq \left\|h_t\right\|_F\left\|\nabla_y\ell_t(y_t)\right\|_2L_x \leq L_xL_\ell \left\|h_t\right\|_F
\end{align}

By combining this result with \eqref{eq:first_part_init_bound}, we can rewrite inequality  \eqref{eq:norm_ineq_first_appendix} as:
\begin{equation}
    \left \|\nabla_W \ell_t(y_t)\right\|_F \leq L_\ell L_x\left(\left\|h_t\right\|_F+\left\|W_C\right\|_2\left\| \frac{\partial h_t}{\partial W}\right\|_F\right)
\end{equation}

Next, for the remaining two terms, we first evaluate the bound on the hidden state matrix $h_t \in \mathbb{R}^{N \times D}$. Given an initialization of $h_0 = \mathbf{0}$ (which is standard when processing independent training sequences), unrolling the recurrence from equation \eqref{ssm_dynamics_h_appendix} yields:
\begin{equation}
    h_t = \sum_{j=1}^t \left( \bigodot_{k=j+1}^t \bar{A}_k \right) \odot \left( \bar{B}_j \operatorname{diag}(x_j) \right)
\end{equation}
where we define the empty cumulative Hadamard product as the all-ones matrix,  $\bigodot_{k=t+1}^t \bar{A}_k = \mathbf{1}_{N \times D}$. Then, taking the Frobenius norm of both sides and applying the triangle inequality, we get:
\begin{equation}
    \|h_t\|_F \leq \sum_{j=1}^t \left\| \left( \bigodot_{k=j+1}^t \bar{A}_k \right) \odot \left( \bar{B}_j \operatorname{diag}(x_j) \right) \right\|_F = \sum_{j=1}^t \left\| \exp\left(A\sum_{k=j+1}^t\Delta_k^\top\right) \odot \left( \bar{B}_j \operatorname{diag}(x_j) \right) \right\|_F
\end{equation}
Utilizing the property of Hadamard product with Frobenius norms, the fact that $\left\|\bar{B}_j \operatorname{diag}(x_j)\right\|_F \leq \left\|  \bar{B}_j \right\|_2 \left\|\operatorname{diag}(x_j)\right\|_F \leq \left\|  \bar{B}_j \right\|_2L_x$, and the definition of $\bar{B}_t$, it follows that:
\begin{equation}
    \left\| \exp\left(A\sum_{k=j+1}^t\Delta_k^\top\right) \odot \left( \bar{B}_j \operatorname{diag}(x_j) \right) \right\|_F \leq \left\| \exp\left(A\sum_{k=j+1}^t\Delta_k^\top\right)\right\|_F \left\|  W_Bx_j \Delta_j^\top \right\|_2 L_x 
\end{equation}
Then:
\begin{equation}
    \left\|  W_Bx_j \Delta_j^\top \right\|_2 L_x  \leq \left\|  W_B\right\|_2\left\|x_j\right\|_2 \left\|\Delta_j^\top \right\|_2 L_x \leq \left\|  W_B\right\|_2 \left\|\text{Softplus}\left(W_\Delta x_j\right)\right\|_2 L_x^2
\end{equation}
And since $\text{Softplus}(a) \leq |a|+\log(2)$ so $\left\|\text{Softplus}(\mathbf{x})\right\|_2 \leq \left\||\mathbf{x}| + \log(2)\mathbf{1}\right\|_2\leq \|\mathbf{x}\|_2+ \sqrt{D}\log(2)$, it follows that $\left\|\text{Softplus}\left(W_\Delta x_j\right)\right\|_2 \leq \left\|\left|W_\Delta x_j \right|+ \log(2)\textbf{1}\right\|_2$, and from the triangle inequality, this results in $\left\|\left|W_\Delta x_j\right| + \log(2)\textbf{1}\right\|_2 \leq \left\|W_\Delta\right\|_2L_x + \sqrt{D}\log(2)$

So the bound becomes:
\begin{equation}
    \left\|  W_Bx_j \Delta_j^\top \right\|_2 L_x  \leq \left\|W_B\right\|_2 L_x^2 \left(\left\|W_\Delta\right\|_2L_x + \sqrt{D}\log(2)\right)
\end{equation}

Since these terms are independent of $j$, it follows that:
\begin{equation}
    \|h_t\|_F \leq \left\|W_B\right\|_2 L_x^2 \left(\left\|W_\Delta\right\|_2L_x + \sqrt{D}\log(2)\right) \sum_{j=1}^t\left\| \exp\left(A\sum_{k=j+1}^t\text{Softplus}\left(W_\Delta x_k\right)^\top\right)\right\|_F
\end{equation}

Next, the term $\frac{\partial h_t}{\partial W}$ can be expanded from equation \eqref{ssm_dynamics_h_appendix} as:
\begin{equation}
    \frac{\partial h_t}{\partial W} = \frac{\partial \bar A_t}{\partial W} \odot h_{t-1} + \bar{A}_t\odot \frac{\partial h_{t-1}}{\partial W} + \frac{\partial \bar B_t}{\partial W} \operatorname{diag}(x_t)
\end{equation}
We bound each term individually, using the fact that:
\begin{equation}
    \left\|\frac{\partial h_t}{\partial W}\right\|_F \leq \left\|\frac{\partial h_t}{\partial W_\Delta}\right\|_F + \left\|\frac{\partial h_t}{\partial W_B}\right\|_F + \left\|\frac{\partial h_t}{\partial W_C}\right\|_F + \left\|\frac{\partial h_t}{\partial A}\right\|_F
\end{equation}
where the $W_C$-dependent term is zero.

By the triangle inequality:
\begin{equation}
    \left\|\frac{\partial h_t}{\partial W}\right\|_F \leq \left\|\frac{\partial \bar A_t}{\partial W} \odot h_{t-1}\right\|_F + \left\|\bar{A}_t\odot \frac{\partial h_{t-1}}{\partial W}\right\|_F + \left\|\frac{\partial \bar B_t}{\partial W} \operatorname{diag}(x_t)\right\|_F
\end{equation}
Then, we bound each term. First:
\begin{align}
    \left\|\frac{\partial \bar B_t}{\partial W} \operatorname{diag}(x_t)\right\|_F &\leq \left\|\frac{\partial \bar B_t}{\partial W}\right\|_F\left\|\operatorname{diag}(x_t)\right\|_F \leq L_x \left(\left\|\frac{\partial \bar B_t}{\partial W_B}\right\|_F+\left\|\frac{\partial \bar B_t}{\partial W_\Delta}\right\|_F\right) \nonumber \\
    & \leq L_x\left(\sqrt{N}\left\|x_t\right\|_2\left\|\Delta_t\right\|_2 + \left\|W_B x_t\right\|_2 \left\|\frac{\partial \Delta_t}{\partial W_\Delta}\right\|_F \right) \nonumber \\
    & \leq L_x\left(\sqrt{N}L_x\left(\left\|W_\Delta\right\|_2L_x + \sqrt{D}\log(2)\right) + \sqrt{D}\left\|W_B\right\|_2 L_x^2 \right) \label{eq:B_t_bound},
\end{align}
where we used the fact that $\operatorname{Softplus}'(z)= \sigma(z) \leq 1$, so since $\Delta_t=\operatorname{Softplus}(W_\Delta x_t)$, it is true that $\left\|\frac{\partial \Delta_t}{\partial W_\Delta}\right\|_F \leq \sqrt D \|x_t\|_2 \leq \sqrt D L_x.$
    
Next:
\begin{align}
    \left\|\frac{\partial \bar A_t}{\partial W} \odot h_{t-1}\right\|_F &\leq \left\|\frac{\partial \bar A_t}{\partial W} \right\|_F\left\| h_{t-1}\right\|_F \leq \left( \left\|\frac{\partial \bar A_t}{\partial A} \right\|_F+ \left\|\frac{\partial \bar A_t}{\partial W_\Delta} \right\|_F\right)\left\| h_{t-1}\right\|_F\nonumber \\ &\leq \left(\left\|\bar{A}_t\right\|_F\left\|\Delta_t\right\|_2+ \left\|\frac{\partial \bar A_t}{\partial \Delta_t} \right\|_F\left\|\frac{\partial \Delta_t}{\partial W_\Delta} \right\|_F\right)\left\| h_{t-1}\right\|_F \nonumber \\ &\leq \left(\left\|\bar{A}_t\right\|_F\left\|\Delta_t\right\|_2+ \sqrt{D}\left\|A\right\|_2 \left\|\bar{A}_t\right\|_F L_x \right)\left\| h_{t-1}\right\|_F \nonumber \\ &\leq \left\| h_{t-1}\right\|_F\left\|\bar{A}_t\right\|_F \left(\left(\left\|W_\Delta\right\|_2L_x + \sqrt{D}\log(2)\right)+\sqrt{D}\left\|A\right\|_2 L_x\right)
\end{align}

Finally, by definition of $\|M\|_{\max}:=\max_{i,j}|M_{ij}|$, it follows that:
\begin{align}
     \left\|\bar{A}_t\odot \frac{\partial h_{t-1}}{\partial W}\right\|_F \leq  \left\|\bar{A}_t\right\|_{\max}\left\|\frac{\partial h_{t-1}}{\partial W}\right\|_F
\end{align}

Combining all the terms together, we get the following recurrence:
\begin{align}
    \left\|\frac{\partial h_t}{\partial W}\right\|_F &\leq \left\|W_B\right\|_2 L_x^2 \left(\left\|W_\Delta\right\|_2L_x + \sqrt{D}\log(2)\right) \sum_{j=1}^{t-1}\left\| \exp\left(A\sum_{k=j+1}^{t-1}\text{Softplus}\left(W_\Delta x_k\right)^\top\right)\right\|_F \nonumber \\ &\times \left\|\bar{A}_t\right\|_F \left( \left(\left\|W_\Delta\right\|_2L_x + \sqrt{D}\log(2)\right)+\sqrt{D}\left\|A\right\|_2 L_x\right) \nonumber \\ &+ L_x\left(\sqrt{N}L_x\left(\left\|W_\Delta\right\|_2L_x + \sqrt{D}\log(2)\right) + \sqrt{D}\left\|W_B\right\|_2 L_x^2 \right) + \left\|\bar{A}_t\right\|_{\max}\left\|\frac{\partial h_{t-1}}{\partial W}\right\|_F 
\end{align}

Next, define:
\begin{align}
    K_\Delta &:= \left\|W_\Delta\right\|_2L_x + \sqrt{D}\log(2) \\
    K_B &:= \left\|W_B\right\|_2 L_x^2 \\
    \Phi_{j,t} &:= \left\| \exp\left(A\sum_{k=j+1}^t\text{Softplus}\left(W_\Delta x_k\right)^\top\right)\right\|_F
\end{align}

So the recurrence becomes:
\begin{align}
    \left\|\frac{\partial h_t}{\partial W}\right\|_F \leq & \left\|\bar{A}_t\right\|_F K_BK_\Delta\left(K_\Delta + \sqrt{D}\left\|A\right\|_2L_x\right) \sum_{j=1}^{t-1} \Phi_{j,t-1} + L_x^2\sqrt{N}K_\Delta + L_x\sqrt{D}K_B \nonumber \\ &+ \left\|\bar{A}_t\right\|_{\max}\left\|\frac{\partial h_{t-1}}{\partial W}\right\|_F 
\end{align}

Since $h_0 = \mathbf{0}$ is parameter-independent, $\partial h_0 / \partial W = 0$, iterating the inequality gives:
\begin{align}
    \left\|\frac{\partial h_t}{\partial W}\right\|_F \leq \sum_{m=1}^t\Bigg[&\left\|\bar{A}_m\right\|_F K_BK_\Delta\left(K_\Delta + \sqrt{D}\left\|A\right\|_2L_x\right) \left(\sum_{j=1}^{m-1} \Phi_{j,m-1}\right) \nonumber \\ &+ L_x^2\sqrt{N}K_\Delta + L_x\sqrt{D}K_B\Bigg]\prod_{n=m+1}^t\left\|\bar{A}_n\right\|_{\max},
\end{align}
where empty sums are calculated as zero and empty scalar products are calculated as one.

Next, by Definition~\ref{def:ssm_stability}, each entry of $A\Delta_t^\top$ is at most $-\gamma\delta$. Since $\bar A_t=\exp(A\Delta_t^\top)$ element-wise, it follows that:
\begin{equation}
    \left\|\bar{A}_t\right\|_{\max} \leq \exp{\left(-\gamma\delta\right)}:= \rho < 1
\end{equation}
as well as
\begin{equation}
    \left\|\bar{A}_t\right\|_{F} \leq \sqrt{ND}\rho
\end{equation}
and in addition, by the definition of $\Phi_{j, t}$, the following inequality holds:
\begin{equation}
    \Phi_{j, m-1} \leq \sqrt{ND}\rho^{m-1-j}
\end{equation}
As a result, the inner summation becomes a simple geometric series:
\begin{equation}
    \sum_{j=1}^{m-1}\Phi_{j, m-1} \leq \sqrt{ND}\sum_{j=1}^{m-1}\rho^{m-1-j} = \sqrt{ND}\sum_{r=0}^{m-2}\rho^r \leq \frac{\sqrt{ND}}{1-\rho}
\end{equation}
As such, we can group all time-independent terms under a new constant:
\begin{equation}
    K_W := ND K_BK_{\Delta} \left(K_\Delta + \sqrt{D} \left\|A\right\|_2L_x\right)
    \left(\frac{\rho}{1-\rho}\right) + L_x^2 \sqrt{N} K_\Delta + L_x \sqrt{D} K_B
\end{equation}
reducing the recurrence to:
\begin{equation}
     \left\|\frac{\partial h_t}{\partial W}\right\|_F \leq K_W\sum_{m=1}^t \rho^{t-m} \leq \frac{K_W}{1-\rho}
\end{equation}

Finally, combining everything, we get the final bound on the gradient of loss at time $t$:
\begin{equation}
    \left \|\nabla_W \ell_t(y_t)\right\|_F \leq L_\ell L_x\left(K_BK_\Delta\frac{\sqrt{ND}}{1-\rho}+\left\|W_C\right\|_2\frac{K_W}{1-\rho}\right)
\end{equation}

Since $\nabla_W \mathcal L^S(W)=\frac{1}{S}\sum_{t=1}^S\nabla_W\ell_t(y_t)$, by the triangle inequality the same bound holds for the averaged empirical loss:
\begin{equation}
    \left \|\nabla_W \mathcal L^S(W)\right\|_F \leq L_\ell L_x\left(K_BK_\Delta\frac{\sqrt{ND}}{1-\rho}+\left\|W_C\right\|_2\frac{K_W}{1-\rho}\right) = \mathcal G_{\text{SSM}}(W)
\end{equation}

Since $S \geq 1$ and the admissible length-$S$ sequence pairs $\xi \in \mathcal Z$ were arbitrary, and the right-hand side of the inequality is independent of $S$ and $\xi$, the bound holds uniformly over all such lengths and admissible sequence pairs for the fixed parameter tuple $W$, proving Theorem~\ref{thm:ssm_grad_bound}. \qed

\section{Convergence of Federated SSM Baselines}\label{appendix:fed_algs_convergence_bounds_proofs}

In this section, we first define and prove a set of lemmas required to instantiate standard federated optimization convergence results in terms of selective SSM parameters. We then provide the proof of Theorem~\ref{thm:fedavg_ssm_convergence}, and state the corresponding FedProx theorem.

\subsection{Auxiliary Lemmas for SSM Federated Convergence}\label{appendix:aux_lemmas_ssm_federated_convergence}
Rather than proving the FedAvg or FedProx convergence bounds directly, we first show how the generic smoothness, gradient second-moment, stochastic variance, and client heterogeneity quantities appearing in those bounds can be replaced by SSM-specific constants. In particular, we establish local control of the SSM gradient bound in Lemma~\ref{lem:local_lipschit_SSM_grad_bound}, prove the local smoothness constant $L_{\mathrm{SSM},R}$ in Lemma~\ref{lem:ssm_local_smoothness}, and derive the architecture-dependent quantities $G_{\mathrm{SSM},R}^2$, $\sigma_{\mathrm{SSM},R,\Xi}^2$, and $\Gamma_{\mathrm{SSM},R}^2$ in Lemmas~\ref{lem:ssm_fedavg_g_constant}, \ref{lem:ssm_variance_bound}, and \ref{lem:ssm_gamma_tv_bound}, respectively. These lemmas are then used in the remainder of this section to state the final convergence bounds for FedAvg and FedProx distributed learning baselines.

\subsubsection{Local Lipschitz Continuity of the SSM Gradient Bound Lemma and Proof}

\begin{lemma}[Local Lipschitz Continuity of the SSM Gradient Bound]\label{lem:local_lipschit_SSM_grad_bound}
Let $W^r$ be some fixed reference parameter tuple and let $0 < R < \infty$ be a finite radius such that $\mathcal N_R(W^r)\subseteq\mathcal S_{\gamma,\delta}$. Suppose Assumptions~\ref{assum:lipschitz_cont_loss} and~\ref{assum:bounded_inputs} hold with the neighborhood-uniform constants specified in Section~\ref{subsec:theoretical_convergence_standard_federated_baselines}.

Let $\mathcal G_{\mathrm{SSM}}(W)$ be the bounding function defined in Theorem~\ref{thm:ssm_grad_bound}. Then $\mathcal G_{\mathrm{SSM}}$ is Lipschitz continuous on $\mathcal N_R(W^r)$. 

Specifically, with $N,D,L_\ell,L_x,\gamma,\delta$ fixed, there
exists a finite constant
$L_{\mathcal G_{\mathrm{SSM},R}}>0$, depending only
on $R$ and the block norms
$\|W_\Delta^r\|_2$, $\|W_B^r\|_2$,
$\|W_C^r\|_2$, and $\|A^r\|_2$, such that
for all $W,V\in\mathcal N_R(W^r)$,
\begin{equation}
    |\mathcal G_{\mathrm{SSM}}(W)
      -\mathcal G_{\mathrm{SSM}}(V)|
    \leq
    L_{\mathcal G_{\mathrm{SSM},R}}\|W-V\|_F.
\end{equation}
\end{lemma}

\textbf{Proof of Lemma~\ref{lem:local_lipschit_SSM_grad_bound}:} Fix the reference tuple $W^r$ and the radius $R$. Throughout this proof, $N,D,L_\ell,L_x,\gamma,\delta$ are constant as $W$ and $V$ vary over the neighborhood $\mathcal N_R(W^r)$. In particular, $\rho=\exp(-\gamma\delta)$ is the same for both
parameter tuples.

First, by Theorem~\ref{thm:ssm_grad_bound}, $\mathcal G_{\text{SSM}}(W)$ has an explicit definition:
\begin{equation}
    \mathcal G_{\text{SSM}}(W) = L_\ell L_x\left(K_BK_\Delta\frac{\sqrt{ND}}{1-\rho}+\left\|W_C\right\|_2\frac{K_W}{1-\rho}\right),
\end{equation}
where $W$ is the tuple of selective SSM weights $(W_\Delta, W_B,W_C,A)$.

Since $W=(W_\Delta, W_B, W_C, A)$ is a parameter tuple, given another parameter tuple $V = (V_\Delta, V_B, V_C, V_A)$ we use the block Frobenius norm:
\begin{equation}
    \|W-V\|_F^2:=\|W_\Delta-V_\Delta\|_F^2 + \|W_B-V_B\|_F^2+\|W_C-V_C\|_F^2+\|A-V_A\|_F^2.
\end{equation}
For each parameter block $X \in \{\Delta, B, C, A\}$, let $W_X$ denote the corresponding block of $W$, and let $W_X^r$ denote the corresponding block of the central model $W^r$ (for notational uniformity, we write $W_A:=A$ and $W_A^r:=A^r$). Given the central model $W^r$, we evaluate the maximum possible value of a norm of each parameter matrix within the neighborhood $\mathcal N_R(W^r)$. Since $W\in\mathcal N_R(W^r)$, each block satisfies
\begin{equation} 
    \|W_X-W_X^r\|_F \leq \|W-W^r\|_F \leq R.
\end{equation}

Therefore, by the triangle inequality and the fact that $\|\cdot\|_2\leq \|\cdot\|_F$,
\begin{equation}\label{eq:block_neighborhood_bound}
    \|W_X\|_2 \leq \|W_X^r\|_2+\|W_X-W_X^r\|_2 \leq \|W_X^r\|_2+\|W_X-W_X^r\|_F \leq \|W_X^r\|_2+R.
\end{equation}

Next, define:
\begin{align}
    M_\Delta &:= \left\|W_\Delta^r\right\|_2+R \label{eq:m_delta} \\ 
    M_B &:= \left\|W_B^r\right\|_2+R \label{eq:m_B_init} \\
    M_C &:= \left\|W_C^r\right\|_2+R \label{eq:m_C} \\
    M_A &:= \left\|A^r\right\|_2+R \label{eq:m_A}
\end{align}
Then, by \eqref{eq:block_neighborhood_bound}, it is true that for all individual matrices $W_B, W_C, W_\Delta, A$, the following holds:
\begin{align}
    \left\|W_\Delta\right\|_2  &\leq M_\Delta \\
    \left\|W_B\right\|_2 &\leq M_B \\
    \left\|W_C\right\|_2  &\leq M_C \\
    \left\|A\right\|_2  &\leq M_A 
\end{align}
The same bounds also hold for the corresponding blocks of any $V\in\mathcal N_R(W^r)$.

We further define:
\begin{align}
    M_{K_\Delta} &:= M_{\Delta} L_x + \sqrt{D}\log(2) \label{eq:m_k_delta} \\
    M_{K_B} &:= M_BL_x^2 \label{eq:m_k_b}
\end{align}

Next, we evaluate the Lipschitz constants for individual components of the expression for $\mathcal G_{\text{SSM}}(W)$. We'll utilize the property that: $\left|\left\|X\right\|_2-\left\|Y\right\|_2\right| \leq \left\|X-Y\right\|_2 \leq \left\|X-Y\right\|_F$.

First:
\begin{align}
    \left|K_\Delta(W)-K_\Delta(V)\right| &\leq L_x\left|\left\|W_\Delta\right\|_2-\left\|V_\Delta\right\|_2\right| \leq L_x\left\|W_\Delta-V_\Delta\right\|_F \leq  L_x \left\|W-V\right\|_F \\
    \left|K_B(W)-K_B(V)\right| &\leq L_x^2\left|\left\|W_B\right\|_2-\left\|V_B\right\|_2\right| \leq L_x^2\left\|W_B-V_B\right\|_F \leq L_x^2 \left\|W-V\right\|_F  
\end{align}
So the Lipschitz constants for the base functions are:
\begin{align}
    L_{K_\Delta} &= L_x \\
    L_{K_B} &= L_x^2 \\
    L_C &= 1 \text{ (for the term } \left\|W_C\right\|_2 \text{)} \\
    L_A &= 1 \text{ (for the term } \left\|A\right\|_2 \text{)}
\end{align}

Next, by the product rule for Lipschitz functions, if $|f| \leq M_f$ and $|g| \leq M_g$ and are Lipschitz with constants $L_f$ and $L_g$ on $\mathcal N_R(W^r)$, then $fg$ is Lipschitz with constant
\begin{equation}
    L_{fg} = M_fL_g + M_gL_f.
\end{equation}

Given that, let
\begin{equation}
    P_1(W) := K_B(W)K_\Delta(W).
\end{equation}
Then
\begin{equation}
    L_{P_1} = M_{K_B} L_{K_\Delta} + M_{K_\Delta} L_{K_B} = M_{K_B} L_x + M_{K_\Delta} L_x^2,
\end{equation}
and 
\begin{equation}
    M_{P_1} = M_{K_B} M_{K_\Delta}.
\end{equation}

Next, define
\begin{equation}
    P_2(W) := K_\Delta(W) + \sqrt{D}\|A\|_2L_x.
\end{equation}
Then, by the sum rule of Lipschitz functions:
\begin{equation}
    L_{P_2} = L_{K_\Delta} + \sqrt{D} L_A L_x = L_x + \sqrt{D}L_x,
\end{equation}
and
\begin{equation}
    M_{P_2} = M_{K_\Delta} + \sqrt{D}M_AL_x.
\end{equation}

Continuing, define
\begin{equation}
    P_3(W) := P_1(W)P_2(W).
\end{equation}
Then
\begin{equation}
    L_{P_3} = M_{P_1}L_{P_2} + M_{P_2}L_{P_1},
\end{equation}
and
\begin{equation}
    M_{P_3} = M_{P_1}M_{P_2}.
\end{equation}

Given these definitions, and using the definition of $K_W$ from Theorem~\ref{thm:ssm_grad_bound}, a valid Lipschitz constant for $K_W$ is:
\begin{equation}
    L_{K_W} = ND\frac{\rho}{1-\rho}L_{P_3} + L_x^3\sqrt{N} + L_x^3\sqrt{D},
\end{equation}
and the upper bound for $K_W(W)$ is:
\begin{equation}\label{eq:M_kw}
    M_{K_W} = ND\frac{\rho}{1-\rho}M_{P_3} + L_x^2\sqrt{N}M_{K_\Delta} + L_x\sqrt{D}M_{K_B}.
\end{equation}

Finally, define
\begin{equation}
    P_4(W) := \|W_C\|_2 K_W(W).
\end{equation}
Since $\|W_C\|_2$ is bounded by $M_C$ and has Lipschitz constant $L_C = 1$, the product rule gives
\begin{equation}
    L_{P_4} = M_C L_{K_W} + M_{K_W}.
\end{equation}

Combining everything together and using linearity, we get the final Lipschitz constant for $\mathcal G_{\text{SSM}}(W)$:
\begin{equation}
    L_{\mathcal G_{\mathrm{SSM},R}} =L_\ell L_x\left(\frac{\sqrt{ND}}{1-\rho}L_{P_1} + \frac{1}{1-\rho}L_{P_4}\right).
\end{equation}

Because $R < \infty$, the block norms of $W^r$ are finite, and $\rho < 1$, all constants $M_{K_\Delta}$, $M_{K_B}$, $M_{P_1}$, $M_{P_2}$, $M_{P_3}$, $M_{K_W}$, $L_{P_1}$, $L_{P_2}$, $L_{P_3}$, $L_{K_W}$, and $L_{P_4}$ are finite. As such, $L_{\mathcal G_{\mathrm{SSM},R}} < \infty$, and $\forall W, V \in \mathcal N_R(W^r)$,
\begin{equation}
    |\mathcal G_{\text{SSM}}(W)-\mathcal G_{\text{SSM}}(V)| \leq L_{\mathcal G_{\mathrm{SSM},R}}\|W-V\|_F.
\end{equation}

This proves that $\mathcal G_{\text{SSM}}(W)$ is locally Lipschitz continuous on $\mathcal N_R(W^r)$. \qed

\subsubsection{Local SSM smoothness Lemma Proof}
\textbf{Proof of Lemma~\ref{lem:ssm_local_smoothness}:} Fix the reference tuple $W^r$ and the radius $R$ from the lemma. Let $\xi = (x,y^*) \in \mathcal Z$ be an arbitrary admissible sequence pair of finite length $S \geq 1$. The inputs and targets are held fixed when differentiating with respect to the model parameters. For this fixed sequence pair, we use a simplified notation for $h_t(W;x)$, $y_t(W;x)$, and $\ell_t(y_t(W;x),y_t^*)$ as $h_t(W)$, $y_t(W)$, and $\ell_t(y_t(W))$, respectively.

Throughout the proof, $L_\ell, H_\ell, L_x, \gamma, \delta$ are fixed uniformly over $\mathcal N_R(W^r)$, with $\rho = \exp(-\gamma\delta) < 1$. The estimates below hold for every $W \in \mathcal N_R(W^r)$, every finite sequence length $S \geq 1$, every admissible length-$S$ pair $\xi \in \mathcal Z$, and every token position $t\leq S$, with constants independent of $S$, $\xi$, and $t$.

We use the parameter envelopes from \eqref{eq:m_delta}-\eqref{eq:m_A} and the induced bounds from~\eqref{eq:m_k_b},~\eqref{eq:m_k_delta}, and~\eqref{eq:M_kw}. 

For convenience, we define:
\begin{align}
    H_R &:= M_{K_B} M_{K_\Delta}\frac{\sqrt{ND}}{1-\rho} 
    \\ J_R &:= \frac{M_{K_W}}{1-\rho}
\end{align}

By the definition of $M_X$ as an upper-bound and from the Proof of Theorem~\ref{thm:ssm_grad_bound}, we know the following hold:
\begin{align}
    \left\|h_t(W)\right\|_F &\leq K_B(W)K_\Delta(W)\frac{\sqrt{ND}}{1-\rho} \leq H_R \\
    \left\|\frac{\partial h_t(W)}{\partial W}\right\|_F &\leq \frac{K_W(W)}{1-\rho }\leq J_R
\end{align}

Given this setup, we begin by bounding $D^2 h_t$ where $D$ is the Fréchet derivative defined:
\begin{align*}
    DF(W)[U] &:= \frac{d}{d\epsilon} F(W+\epsilon U)\Bigg|_{\epsilon=0} \\
    D^2F(W)[U,V] &:= \frac{\partial^2}{\partial \epsilon \partial \zeta} F(W+\epsilon U+\zeta V)\Bigg|_{\epsilon=\zeta=0}
\end{align*}
with $U, V$ being unit perturbations. We also utilize the operator norm defined $\|D\Psi(W)\|_{\mathrm{op}} := \sup_{\|U\|_F=1}\|D \Psi(W)[U]\|_F$. Note that after flattening the input and output tensors, this corresponds to the usual spectral norm of the corresponding Jacobian matrix. Parameter Hessians are represented as matrices using the vectorization of parameter blocks introduced in the single-layer proof. Their operator norms are taken with respect to the corresponding Euclidean norm, which equals the block Frobenius norm before vectorization.

Let $U = (U_\Delta,U_B, U_C,U_A)$ and $V=(V_\Delta,V_B,V_C,V_A)$ be arbitrary block perturbations satisfying 
\begin{equation}
    \|U\|_F=\|V\|_F=1.
\end{equation}
Each perturbation block therefore has Frobenius norm at most one. Since $\Delta_t$ depends only on $W_\Delta$, we write $D\Delta_t(W)[U_\Delta]$ as shorthand for $D\Delta_t(W)[(U_\Delta,0,0,0)]$, and use the analogous convention for its second derivative.

By the property of the operator norm (with $\Psi(W) = h_t(W)$), 
\begin{equation}
    \left\|D h_t(W)[U]\right \|_{F} \leq \|D h_t(W)\|_{\mathrm{op}}\|U\|_F = \|D h_t(W)\|_{\mathrm{op}} \leq \left\|\frac{\partial h_t(W)}{\partial W}\right\|_F \leq J_R.
\end{equation}
The same argument gives
\begin{equation}
    \|Dh_t(W)[V]\|_F\leq J_R.
\end{equation}
and therefore it follows that
\begin{equation}
    \|Dh_t(W)[U]\|_F\leq J_R, \qquad \|Dh_t(W)[V]\|_F\leq J_R.
\end{equation}

Next, let $b_t(W) := \bar{B}_t(W)\operatorname{diag}(x_t)$, then the hidden state recurrence equation becomes:
\begin{equation}
    h_t(W) = \bar{A}_t(W) \odot h_{t-1}(W)+b_t(W),
\end{equation}
and taking the first Fréchet derivative in direction $U$ gives:
\begin{equation}
    D h_t(W)[U] = D\bar A_t(W)[U]\odot h_{t-1}(W)+ \bar A_t(W)\odot Dh_{t-1}(W)[U] + Db_t(W)[U].
\end{equation}
Taking a second Fréchet derivative in direction $V$ gives
\begin{align}
    D^2h_t(W)[U,V]
    =&\ D^2\bar A_t(W)[U,V]\odot h_{t-1}(W) \nonumber\\
    &+D\bar A_t(W)[U]\odot Dh_{t-1}(W)[V] \nonumber\\
    &+D\bar A_t(W)[V]\odot Dh_{t-1}(W)[U] \nonumber\\
    &+\bar A_t(W)\odot D^2h_{t-1}(W)[U,V] \nonumber\\
    &+D^2b_t(W)[U,V].
\end{align}

Next, taking the Frobenius norms, and by the triangle inequality, it follows that
\begin{align}
    \|D^2h_t(W)[U,V]\|_F
    \leq &
    \|D^2\bar A_t(W)[U,V]\odot h_{t-1}(W)\|_F \nonumber\\
    &+
    \|D\bar A_t(W)[U]\odot Dh_{t-1}(W)[V]\|_F \nonumber\\
    &+
    \|D\bar A_t(W)[V]\odot Dh_{t-1}(W)[U]\|_F \nonumber\\
    &+
    \|\bar A_t(W)\odot D^2h_{t-1}(W)[U,V]\|_F \nonumber\\
    &+
    \|D^2b_t(W)[U,V]\|_F.
\end{align}

Using $\|X\odot Y\|_F\leq \|X\|_F\|Y\|_F$ and
$\|X\odot Y\|_F\leq \|X\|_{\max}\|Y\|_F$, we obtain:
\begin{align}
    \|D^2h_t(W)[U,V]\|_F
    \leq&
    \|D^2\bar A_t(W)[U,V]\|_F\|h_{t-1}(W)\|_F \nonumber\\
    &+
    \|D\bar A_t(W)[U]\|_F\|Dh_{t-1}(W)[V]\|_F \nonumber\\
    &+
    \|D\bar A_t(W)[V]\|_F\|Dh_{t-1}(W)[U]\|_F \nonumber\\
    &+
    \|\bar A_t(W)\|_{\max}\|D^2h_{t-1}(W)[U,V]\|_F \nonumber\\
    &+
    \|D^2b_t(W)[U,V]\|_F.
\end{align}

Then, we already derived the bounds:
\begin{align*}
    \left\|h_{t-1}(W)\right\|_F \leq H_R \qquad
    \|Dh_{t-1}(W)[U]\|_F \leq J_R \qquad 
    \|Dh_{t-1}(W)[V]\|_F \leq J_R,
\end{align*}
and by Definition~\ref{def:ssm_stability}:
\begin{align*}
    \left\|\bar{A}_t(W)\right\|_{\max} \leq \rho,
\end{align*}
so the bound becomes:
\begin{align}\label{eq:D_squared_ht_bound_init}
    \|D^2h_t(W)[U,V]\|_F
    \leq&
    H_R\|D^2\bar A_t(W)[U,V]\|_F \nonumber\\
    &+
    J_R\|D\bar A_t(W)[U]\|_F
    +
    J_R\|D\bar A_t(W)[V]\|_F \nonumber\\
    &+
    \rho\|D^2h_{t-1}(W)[U,V]\|_F
    +
    \|D^2b_t(W)[U,V]\|_F.
\end{align}

Next, we derive the three missing bounds in~\eqref{eq:D_squared_ht_bound_init}:
\begin{align}
    \left\|D \bar{A}_t(W)[U]\right\|_F, \quad \left\|D^2 \bar{A}_t(W)[U,V]\right\|_F, \quad \left\|D^2 b_t(W)[U,V ]\right\|_F
\end{align}
noting that the same upper bound applies to
$\|D\bar A_t(W)[U]\|_F$ and $\|D\bar A_t(W)[V]\|_F$.

We first derive bounds on the intermediate terms that comprise $\bar{A}_t$, and $\bar{B}_t$, starting with $\Delta_t(W)$.

From the definition of $\Delta_t(W)$:
\begin{equation}
    D \Delta_t(W)[U_\Delta] = \operatorname{diag}\left(\operatorname{Softplus}'(W_\Delta x_t)\right) U_\Delta x_t,
\end{equation}
and since $0 \leq \operatorname{Softplus}'(z) \leq 1$, it follows that:
\begin{equation}
    \left\|D \Delta_t(W)[U_\Delta]\right\|_2 \leq \left\|U_\Delta x_t \right\|_2 \leq \|U_\Delta\|_F\|x_t\|_2 \leq L_x.
\end{equation}
Similarly, 
\begin{equation}
    D^2 \Delta_t(W)[U_\Delta,V_\Delta] = \operatorname{Softplus}''(W_\Delta x_t) \odot (U_\Delta x_t) \odot (V_\Delta x_t).
\end{equation}
Since $0 \leq \operatorname{Softplus}''(z)\leq 1/4 $, the bound becomes:
\begin{equation}
    \left\|D^2 \Delta_t(W)[U_\Delta,V_\Delta]\right\|_2 \leq \frac{1}{4} \left\|(U_\Delta x_t)\odot (V_\Delta x_t)
    \right\|_2\nonumber \leq \frac{1}{4}
    \|U_\Delta x_t\|_2
    \|V_\Delta x_t\|_2 \leq
    \frac{1}{4} L_x^2.
\end{equation}

These bounds allow us to bound $D \bar{A}_t(W)[U]$. We define an intermediate matrix
\begin{equation}
    Z_t(W) := A \Delta_t(W)^\top,
\end{equation}
and calculate its derivative
\begin{equation}
    D Z_t(W)[U] = U_A \Delta_t(W)^\top + A\left(D \Delta_t(W)[U_\Delta]\right)^\top.
\end{equation}
Next, using the property of Frobenius norm for outer product $\|xy^\top\|_F=\|x\|_2\|y\|_2$, it follows that:
\begin{equation}
    \left\|D Z_t(W)[U] \right\|_F \leq \|U_A\|_2\|\Delta_t(W)\|_2 + \|A\|_2 \left\|D \Delta_t(W)[U_\Delta]\right\|_2 \leq M_{K_\Delta} + M_A L_x.
\end{equation}
This allows us to calculate:
\begin{align}
    \left\|D \bar{A}_t(W)[U]\right\|_F &= \left\|\bar{A}_t(W) \odot D Z_t(W)[U]\right\|_F \nonumber \\ &\leq \left\|\bar{A}_t(W)\right\|_{\max}\left\| D Z_t(W)[U]\right\|_F \leq \rho\left(M_{K_\Delta} + M_A L_x\right) := K_{A_{1,R}}
\end{align}

Next, we bound $\left\|D^2 \bar{A}_t(W)[U,V]\right\|_F$ by differentiating $D \bar{A}_t(W)[U]$ in direction $V$:
\begin{equation}
    D^2 \bar{A}_t(W)[U,V] = \bar{A}_t(W) \odot \left(DZ_t(W)[V] \odot D Z_t(W)[U] + D^2 Z_t(W)[U,V]\right).
\end{equation}
To do that we first bound:
\begin{equation}
    D^2 Z_t(W)[U,V] = U_A\left(D \Delta_t(W)[V_\Delta]\right)^\top + V_A\left(D\Delta_t(W)[U_\Delta]\right)^\top+ A \left(D^2 \Delta_t(W)[U_\Delta,V_\Delta]\right)^\top.
\end{equation}
We then apply previously derived bounds to the individual components:
\begin{align}
    \left\|D^2 Z_t(W)[U,V]\right\|_F \leq&  \|U_A\|_2\left\|D \Delta_t(W)[V_\Delta]\right\|_2 \nonumber \\ &+ \|V_A\|_2\left\|D\Delta_t(W)[U_\Delta]\right\|_2+ \|A\|_2 \left\|D^2 \Delta_t(W)[U_\Delta,V_\Delta]\right\|_2 \nonumber \\ \leq& 2 L_x+ \frac{1}{4}M_A L_x^2.
\end{align}
Plugging the result back in, we get the second component bound:
\begin{align}
    \left\|D^2 \bar{A}_t(W)[U,V]\right\|_F &\leq \left\|\bar{A}_t(W)\right\|_{\max} \left(\left\|DZ_t(W)[V]\right\|_F \left\| D Z_t(W)[U]\right\|_F + \left\|D^2 Z_t(W)[U,V]\right\|_F\right) \nonumber \\ &\leq \rho\left(\left(M_{K_{\Delta}}+M_AL_x\right)^2 + 2L_x + \frac{1}{4}M_AL_x^2\right) := K_{A_{2,R}}
\end{align}

Finally, we bound the third term $D^2 b_t(W)$. From definition, $b_t(W) = \bar{B}_t(W) \operatorname{diag}(x_t)$, where $\bar{B}_t(W) = (W_Bx_t)\Delta_t(W)^\top$. We first compute the second derivative of $\bar B_t(W)$:
\begin{align}
    &D^2 \bar{B}_t(W)[U,V] \nonumber  \\ &= (U_Bx_t)\left(D\Delta_t(W)[V_\Delta]\right)^\top + (V_Bx_t)(D\Delta_t(W)[U_\Delta])^\top + (W_Bx_t)(D^2\Delta_t(W)[U_\Delta,V_\Delta])^\top,
\end{align}
and using the previously derived bounds, it follows that:
\begin{equation}
    \left\|D^2 \bar{B}_t(W)[U,V]\right\|_F \leq 2L_x^2+ \frac{1}{4}M_B L_x^3.
\end{equation}
With that, we get:
\begin{align}
    \left\|D^2 b_t(W)[U,V]\right\|_F &\leq \left\|D^2 \bar{B}_t(W)[U,V]\right\|_F \left\|\operatorname{diag}(x_t)\right\|_F  \nonumber \\ &\leq \left(2L_x^2+ \frac{1}{4}M_B L_x^3\right) L_x = 2L_x^3+ \frac{1}{4}M_B L_x^4 := K_{B_{2,R}}
\end{align}

Given these derived bounds, the recurrence becomes:
\begin{equation}
    \left\|D^2h_t(W)[U,V]\right\|_F \leq \rho  \left\|D^2h_{t-1}(W)[U,V]\right\|_F + K_{A_{2,R}} H_R + 2K_{A_{1,R}}J_R + K_{B_{2,R}}.
\end{equation}
Since, as mentioned in the proof of~\ref{thm:ssm_grad_bound}, the initial state is initialized as $h_0=0$, we have:
\begin{equation}
    D^2 h_0(W)[U,V] = 0.
\end{equation}
As such, unrolling the recurrence we get:
\begin{align}
    \left\|D^2h_t(W)[U,V]\right\|_F &\leq \left(K_{A_{2,R}} H_R + 2K_{A_{1,R}}J_R  + K_{B_{2,R}}\right)\sum_{m=0}^{t-1}\rho^m \nonumber \\ &\leq 
    \frac{K_{A_{2,R}} H_R + 2K_{A_{1,R}}J_R + K_{B_{2,R}}}{1-\rho} := K_{H_{2,R}}
\end{align}
Therefore, for all $W\in\mathcal N_R(W^r)$ and all unit block perturbations
$U,V$,
\begin{equation}
    \left\|D^2h_t(W)[U,V]\right\|_F
    \leq K_{H_{2,R}}.
\end{equation}

Next, we use this to derive the first and second Fréchet derivatives of the output:
\begin{equation}
    y_t(W) = h_t(W)^\top(W_Cx_t).
\end{equation}
The first derivative is:
\begin{equation}
    D y_t(W)[U] = D h_t(W)[U]^\top C_t(W)+h_t(W)^\top U_C x_t,
\end{equation}
resulting in a bound
\begin{align}
    \left\|D y_t(W)[U]\right\|_2 &\leq \left\|D h_t(W)[U]\right\|_F\left\|W_C\right\|_2\|x_t\|_2 + \|h_t(W)\|_F\|U_Cx_t\|_2 \nonumber \\ &\leq J_RM_CL_x+H_RL_x:= K_{Y_{1,R}}.
\end{align}

Taking the Fréchet derivative in direction $V$, we get:
\begin{align}
    D^2 y_t(W)[U,V] = D^2h_t(W)[U,V]^\top C_t(W)+ D h_t(W)[U]^\top V_C x_t + D h_t(W)[V]^\top U_Cx_t,
\end{align}
where we also use the fact that $D^2(W_Cx_t)=0$. 
Then:
\begin{align}
    \left\|D^2 y_t(W)[U,V]\right\|_2 \leq& \|D^2 h_t(W)[U,V]\|_F \|(W_Cx_t)\|_2 \nonumber \\ &+ \|D h_t(W)[U]\|_F \|(V_Cx_t)\|_2 + \|D h_t(W)[V]\|_F \|(U_Cx_t)\|_2 \nonumber \\ \leq& K_{H_{2,R}}M_CL_x+2J_RL_x := K_{Y_{2,R}}.
\end{align}

Finally, consider the token loss $\ell_t(y_t(W))$. By the
second-order chain rule:
\begin{align}
    D^2_W \ell_t(y_t(W))[U,V] = D^2_y\ell_t(y_t)[Dy_t(W)[U], Dy_t(W)[V]] + D_y\ell_t(y_t) [D^2y_t(W)[U,V]],
\end{align}
where we use the notation $D^2_y\ell_t(y_t)[a,b] = a^\top \nabla_y^2 \ell_t(y_t)b$.

By Assumption~\ref{assum:lipschitz_cont_loss}, and utilizing all derived bounds, we get:
\begin{align}
    \left|D_W^2 \ell_t(y_t(W))[U,V]\right| & \leq \left|D_y^2\ell_t(y_t)[Dy_t(W)[U],Dy_t(W)[V]]\right| +\left|D_y\ell_t(y_t)[D^2y_t(W)[U,V]]\right| \nonumber \\ &\leq \|\nabla_y^2\ell_t(y_t)\|_{\mathrm{op}} \|Dy_t(W)[U]\|_2\|Dy_t(W)[V]\|_2 + \|\nabla_y\ell_t(y_t)\|_2\|D^2y_t(W)[U,V]\|_2 \nonumber \\ &\leq H_\ell K_{Y_{1,R}}^2 + L_\ell K_{Y_{2,R}} := L_{\mathrm{SSM},R}.
\end{align}

Substituting the definitions, we get:
\begin{align}\label{eq:L_ssm_definition}
    L_{\mathrm{SSM},R} =& H_\ell L_x^2\left(H_R + M_C J_R\right)^2 + L_\ell L_x \left(M_CK_{H_{2,R}} + 2J_R\right) \nonumber\\
    =& H_\ell L_x^2 \left(\frac{M_{K_B}M_{K_\Delta}\sqrt{ND} + M_C M_{K_W}}{1-\rho}\right)^2 \nonumber \\ &+ L_\ell L_x \Bigg[\frac{M_C}{1-\rho} \Bigg(\rho\left((M_{K_\Delta}+M_AL_x)^2 + 2L_x + \frac{1}{4}M_AL_x^2\right)\frac{M_{K_B}M_{K_\Delta}\sqrt{ND}}{1-\rho} \nonumber \\  & + 2\rho(M_{K_\Delta}+M_AL_x)\frac{M_{K_W}}{1-\rho} + 2L_x^3+\frac{1}{4}M_BL_x^4\Bigg) + \frac{2M_{K_W}}{1-\rho}\Bigg].
\end{align}

Since $U,V$ were arbitrary unit block perturbations, the previous bound implies
\begin{equation} 
    \left\|\nabla_W^2\ell_t(y_t(W))\right\|_{\mathrm{op}} \leq L_{\mathrm{SSM},R}.
\end{equation}

Finally, the sequence loss was defined in~\eqref{eq:global_loss_definition} as
\begin{equation*}
    \mathcal L^S(W) = \frac{1}{S}\sum_{t=1}^S \ell_t(y_t(W)),
\end{equation*}
using the simplified notation, where $\mathcal L^S(W) := \mathcal L^S(W;\xi)$ and $\ell_t(y_t(W)) = \ell_t(y_t(W;x), y_t^*)$, and we further incorporate $W$ to emphasize $y_t$'s dependence on the current weights. 

As a result,
\begin{align}
    \|\nabla_W^2\mathcal L^S(W)\|_{\mathrm{op}} &= \left\|\frac{1}{S}\sum_{t=1}^S\nabla_W^2\ell_t(y_t(W))\right\|_{\mathrm{op}} \nonumber \\ &\leq \frac{1}{S}\sum_{t=1}^S \|\nabla_W^2\ell_t(y_t(W))\|_{\mathrm{op}} \nonumber \\ & \leq L_{\mathrm{SSM},R}.\label{eq:appendix_hessian_bound_L_SSM_R_on_op_norm}
\end{align}

Then, for an client $k$, the local expected objective is
\begin{equation}
    \mathcal L_k(W) = \mathbb E_{\xi \sim\mathcal P_k}\left[\mathcal L^S(W;\xi)\right].
\end{equation}

For any $W,V \in \mathcal N_R(W^r)$, convexity of the neighborhood implies that
\begin{equation}
    W_\alpha: = V + \alpha(W-V), \qquad \alpha \in [0,1],
\end{equation}
lies in $\mathcal N_R(W^r)$. Applying the fundamental theorem of calculus to the fixed-sequence gradient and using the Hessian bound from~\eqref{eq:appendix_hessian_bound_L_SSM_R_on_op_norm} gives
\begin{align}
    \left\|\nabla_W\mathcal L^S(W;\xi)-\nabla_W\mathcal L^S(V;\xi)\right\|_F &= \left\|\operatorname{vec}\left(\nabla_W\mathcal L^S(W;\xi)-\nabla_W\mathcal L^S(V;\xi)\right)\right\|_2 \nonumber \\ &= \left\|\int_0^1 \nabla_W^2 \mathcal L^S(W_\alpha;\xi)\operatorname{vec}(W-V)d\alpha\right\|_2 \nonumber \\ & \overset{(a)}{\leq} \int_0^1 \left\|\nabla_W^2 \mathcal L^S(W_\alpha;\xi)\operatorname{vec}(W-V)\right\|_2 d\alpha \nonumber \\ & \overset{(b)}{\leq}\int_0^1 \left\|\nabla_W^2 \mathcal L^S(W_\alpha;\xi)\right\|_{\mathrm{op}}\left\|W-V\right\|_F d\alpha \nonumber \\ & \overset{(c)}{\leq}\int_0^1 L_{\mathrm{SSM},R}\|W-V\|_F\,d\alpha = L_{\mathrm{SSM},R}\|W-V\|_F.
\end{align}
where:
\begin{itemize}
    \item $(a)$ follows from the triangle inequality for an integral $\left\|\int v\right\| \leq \int \|v\|$.
    \item $(b)$ follows from the operator-norm inequality: $\|Hd\|_2 \leq \|H\|_{\mathrm{op}}\|d\|_2$ and the fact that $\left\|\operatorname{vec}(W-V)\right\|_2 = \left\|W-V\right\|_F$.
    \item $(c)$ follows from~\eqref{eq:appendix_hessian_bound_L_SSM_R_on_op_norm}.
\end{itemize}
This inequality holds uniformly over all admissible sequence pairs and finite sequence lengths.

Finally, by Assumption~\ref{assum:unbiased_stochastic_gradients} and Jensen's inequality,
\begin{align}
    \|\nabla_W\mathcal L_k(W) - \nabla_W\mathcal L_k(V)\|_F & =\left\|\mathbb E_{\xi\sim\mathcal P_k}\left[\nabla_W\mathcal L^S(W;\xi) - \nabla_W\mathcal L^S(V;\xi)\right] \right\|_F \nonumber \\
    &\leq \mathbb E_{\xi\sim\mathcal P_k}\left[\|\nabla_W\mathcal L^S(W;\xi) - \nabla_W\mathcal L^S(V;\xi)\|_F\right] \nonumber \\
    &\leq L_{\mathrm{SSM},R}\|W-V\|_F.
\end{align}

Therefore, $\mathcal L_k$ is $L_{\mathrm{SSM},R}$-smooth on
$\mathcal N_R(W^r)$. \qed

\subsubsection{SSM loss gradient second-moment constant Lemma and Proof}

In the following two lemmas, we consider some arbitrary client data distribution $\mathcal P$ and write $\mathcal L(W) = \mathbb E_{\xi \sim\mathcal P}\left[\mathcal L^S(W;
\xi)\right]$. Expectations written as $\mathbb E_\xi$ are taken with respect to the distribution $\mathcal P$.

\begin{lemma}[SSM loss gradient second-moment constant]\label{lem:ssm_fedavg_g_constant} 
Suppose the Assumptions~\ref{assum:lipschitz_cont_loss} and \ref{assum:bounded_inputs} hold. Let $\mathcal N_R(W^r) \subseteq \mathcal S_{\gamma,\delta}$ be a stable local model neighborhood around the round-$r$ broadcast model $W^r$. Let:
\begin{equation}
    g(W;\xi):=\nabla_W\mathcal L^S(W;\xi)
\end{equation}
denote the stochastic sequence-level gradient for a sampled input-target sequence pair $\xi = (x_{1:S}, y_{1:S}^*)$. Then, for every $W \in \mathcal N_R(W^r)$,
\begin{equation}
    \mathbb E_{\xi} \left\|g(W;\xi)\right\|_F^2 \leq G_{\mathrm{SSM},R}^2,
\end{equation}
where
\begin{equation}
    G_{\mathrm{SSM},R} := L_\ell L_x \left( M_{K_B}M_{K_\Delta}\frac{\sqrt{ND}}{1-\rho} + M_C\frac{M_{K_W}}{1-\rho}\right),
\end{equation}
where $M_{K_B}, M_{K_\Delta}, M_C, \text{ and } M_{K_W}$ are defined in~\eqref{eq:m_k_b},~\eqref{eq:m_k_delta},~\eqref{eq:m_C}, and ~\eqref{eq:M_kw} respectively.
\end{lemma}

\textbf{Proof of Lemma~\ref{lem:ssm_fedavg_g_constant}:} By Theorem~\ref{thm:ssm_grad_bound}, under the assumptions of the lemma, for every $W \in \mathcal N_R(W^r)$, every finite integer $S \geq 1$, and every admissible length-$S$ sequence pair $\xi \in \mathcal Z$, we have:
\begin{align}
    \left\|\nabla_W \mathcal L^S(W;\xi)\right\|_F \leq L_\ell L_x\left(K_BK_\Delta\frac{\sqrt{ND}}{1-\rho} + \|W_C\|_2\frac{K_W}{1-\rho}\right). \nonumber
\end{align}
Since $g(W;\xi) = \nabla_W\mathcal L^S(W;\xi)$, it follows that:
\begin{equation}
    \|g(W;\xi)\|_F \leq L_\ell L_x\left(K_BK_\Delta\frac{\sqrt{ND}}{1-\rho} + \|W_C\|_2\frac{K_W}{1-\rho}\right).
\end{equation}
Restricting to $W \in \mathcal N_R(W^r)$, by the local neighborhood bounds defined previously (~\eqref{eq:m_k_b},~\eqref{eq:m_k_delta},~\eqref{eq:m_C}, and ~\eqref{eq:M_kw}), 
\begin{align}
     K_B\leq M_{K_B},
    \qquad
    K_\Delta\leq M_{K_\Delta},
    \qquad
    K_W\leq M_{K_W},
    \qquad
    \|W_C\|_2\leq M_C.
\end{align}

Therefore, 
\begin{equation}\label{eq:g_ssm_definition}
    \|g(W;\xi)\|_F \leq L_\ell L_x \left( M_{K_B}M_{K_\Delta}\frac{\sqrt{ND}}{1-\rho} + M_C\frac{M_{K_W}}{1-\rho}\right) = G_{\mathrm{SSM},R}.
\end{equation}
Since both values are non-negative, squaring both sides gives:
\begin{equation}
    \|g(W;\xi)\|_F^2 \leq G_{\mathrm{SSM},R}^2 = L_\ell^2 L_x^2 \left( M_{K_B}M_{K_\Delta}\frac{\sqrt{ND}}{1-\rho} + M_C\frac{M_{K_W}}{1-\rho}\right)^2
\end{equation}

Finally, taking expectation over the sampled sequence $\xi$ preserves the bound:
\begin{equation}
    \mathbb E_\xi\left[\|g(W;\xi)\|_F^2\right] \leq G_{\mathrm{SSM},R}^2
\end{equation}
proving the lemma. \qed

\subsubsection{SSM stochastic gradient variance bound Lemma and Proof}
\begin{lemma}[SSM stochastic gradient variance bound]
\label{lem:ssm_variance_bound}
Suppose the assumptions of Lemma~\ref{lem:ssm_fedavg_g_constant} and Assumption~\ref{assum:unbiased_stochastic_gradients} hold. Let
\begin{equation*}
    \mathcal L(W):=\mathbb E_{\xi}\left[\mathcal L^S(W;\xi)\right]
\end{equation*}
denote the expected sequence loss over the local data distribution, and let
\begin{equation*}
    g(W;\xi):=\nabla_W\mathcal L^S(W;\xi)
\end{equation*}
denote the single-sequence stochastic gradient. For a mini-batch
$\Xi_\xi=\{\xi_1,\dots,\xi_\Xi\}$ of $\Xi$ independent sampled sequences, define
\begin{equation}
    g_\Xi(W;\Xi_\xi) := \frac{1}{\Xi}\sum_{j=1}^\Xi g(W;\xi_j).
\end{equation}
Then, for every $W\in\mathcal N_R(W^r)$ it follows that
\begin{equation}
    \mathbb E_{\Xi_\xi}\left\|g_\Xi(W;\Xi_\xi)-\nabla_W\mathcal L(W)\right\|_F^2 \leq \sigma_{\mathrm{SSM},R,\Xi}^2,
\end{equation}
where
\begin{equation}
   \sigma_{\mathrm{SSM},R,\Xi}^2 := \frac{1}{\Xi}\sup_{W\in\mathcal N_R(W^r)}\left(\mathbb E_{\xi}\|g(W;\xi)\|_F^2 - \|\nabla_W\mathcal L(W)\|_F^2\right). 
\end{equation}
Moreover, by Lemma~\ref{lem:ssm_fedavg_g_constant},
\begin{equation}
    \sigma_{\mathrm{SSM},R,\Xi}^2 \leq \frac{G_{\mathrm{SSM},R}^2}{\Xi}.
\end{equation}
\end{lemma}

\textbf{Proof of Lemma~\ref{lem:ssm_variance_bound}:} By unbiasedness and independence of the sampled sequences, the standard mini-batch variance identity gives:
\begin{align*}
    \mathbb E_{\Xi_\xi}\left[\left\|g_\Xi(W;\Xi_\xi)-\nabla_W\mathcal L(W)\right\|_F^2\right] = \frac{1}{\Xi}\mathbb E_{\xi}\left[\left\|g(W;\xi)-\nabla_W\mathcal L(W)\right\|_F^2\right].
\end{align*}
Using the identity
\begin{equation}
    \mathbb E_{\xi}\left[\left\|g(W;\xi)-\nabla_W\mathcal L(W)\right\|_F^2\right] = \mathbb E_{\xi}\|g(W;\xi)\|_F^2 - \|\nabla_W\mathcal L(W)\|_F^2,
\end{equation}
and taking the supremum over $W\in\mathcal N_R(W^r)$, we obtain
\begin{equation}
    \mathbb E_{\Xi_\xi}\left[\left\|g_\Xi(W;\Xi_\xi)-\nabla_W\mathcal L(W)\right\|_F^2\right] \leq \sigma_{\mathrm{SSM},R,\Xi}^2.
\end{equation}

Finally, Lemma~\ref{lem:ssm_fedavg_g_constant} gives
\begin{equation*}
    \mathbb E_{\xi}\|g(W;\xi)\|_F^2 \leq G_{\mathrm{SSM},R}^2,
\end{equation*}
and since $\|\nabla_W\mathcal L(W)\|_F^2\geq 0$,
\begin{equation}
    \sigma_{\mathrm{SSM},R,\Xi}^2 \leq \frac{G_{\mathrm{SSM},R}^2}{\Xi}.
\end{equation}
This proves the claim. \qed

\subsubsection{SSM client heterogeneity bound under total variation Lemma and Proof}
\begin{lemma}[SSM client heterogeneity bound under total variation]\label{lem:ssm_gamma_tv_bound}
Suppose the assumptions of Lemma~\ref{lem:ssm_fedavg_g_constant} and Assumption~\ref{assum:unbiased_stochastic_gradients} hold. Let
\begin{equation}
    f(W) := \sum_{k=1}^K p_k \mathcal L_k(W)
\end{equation}
be the global objective, where $p_k \geq 0$ and $\sum_{k=1}^K p_k=1$. Let $\mathcal P_k$ denote the local data distribution of client $k$, and define the global mixture distribution as:
\begin{equation}
    \bar{\mathcal P} := \sum_{k=1}^K p_k \mathcal P_k.
\end{equation}
Define the local client heterogeneity constant
\begin{equation}
    \Gamma_{\mathrm{SSM},R}^2:=\sup_{W \in \mathcal N_R(W^r)}\sum_{k=1}^K p_k \left\|\nabla_W \mathcal L_k(W) - \nabla_W f(W) \right\|_F^2.
\end{equation}
Then:
\begin{equation}
    \Gamma_{\mathrm{SSM},R}^2\leq 4G_{\mathrm{SSM},R}^2 \sum_{k=1}^K p_k \operatorname{TV}(\mathcal{P}_k, \bar{\mathcal P})^2,
\end{equation}
where
\begin{equation}
    \operatorname{TV}(P,Q):=\sup_A |P(A)-Q(A)|.
\end{equation}
\end{lemma}

\textbf{Proof of Lemma~\ref{lem:ssm_gamma_tv_bound}:} Recall that for each client $k$, the local objective is:
\begin{equation*}
    \mathcal L_k(W) =  \mathbb E_{\xi \sim \mathcal P_k}\left[\mathcal L^S(W;\xi)\right],
\end{equation*}
and, by Assumption~\ref{assum:unbiased_stochastic_gradients},
\begin{align}
    \nabla_W\mathcal L_k(W) = \mathbb E_{\xi\sim \mathcal P_k}\left[g(W;\xi)\right].
\end{align}
Similarly, since
\begin{equation*}
    f(W) := \sum_{k=1}^K p_k \mathcal L_k(W),
\end{equation*}
we have
\begin{equation}
    \nabla_W f(W) = \sum_{k=1}^Kp_k\nabla_W\mathcal L_k(W) = \sum_{k=1}^Kp_k\mathbb E_{\xi\sim \mathcal P_k}\left[g(W;\xi)\right].
\end{equation}

Next, by the definition of the mixture distribution $\bar{\mathcal P}$, we can write this as:
\begin{equation}
     \nabla_W f(W) = \mathbb E_{\xi \sim \bar{\mathcal P}}\left[g(W;\xi)\right].
\end{equation}
Therefore, it follows that
\begin{equation}
    \nabla_W\mathcal L_k(W)-\nabla_W f(W) = \mathbb E_{\xi\sim \mathcal P_k}[g(W;\xi)] - \mathbb E_{\xi\sim \bar{\mathcal P}}[g(W;\xi)].
\end{equation}

By the pointwise bound~\eqref{eq:g_ssm_definition} established in the proof of Lemma~\ref{lem:ssm_fedavg_g_constant}, for every $W \in \mathcal N_R(W^r)$ and every admissible sequence pair $\xi \in \mathcal Z$,
\begin{equation*}
    \left\|g(W;\xi)\right\|_F \leq G_{\mathrm{SSM},R}.
\end{equation*}
Since $\mathcal P_k(\mathcal Z) = \bar{\mathcal P}(\mathcal Z )= 1$, the total-variation inequality for bounded vector-valued functions gives:
\begin{equation}
    \left\|\mathbb E_{\xi\sim \mathcal P_k}[g(W;\xi)] - \mathbb E_{\xi\sim \bar{\mathcal P}}[g(W;\xi)]\right\|_F \leq 2G_{\mathrm{SSM},R}\operatorname{TV}(\mathcal P_k,\bar{\mathcal P}).
\end{equation}
So:
\begin{equation}
    \left\|\nabla_W\mathcal L_k(W)-\nabla_W f(W)\right\|_F^2 \leq 4G_{\mathrm{SSM},R}^2\operatorname{TV}(\mathcal P_k,\bar{\mathcal P})^2.
\end{equation}

Multiplying both sides by $p_k$ and summing over all clients gives:
\begin{equation}
    \sum_{k=1}^K p_k \left\|\nabla_W\mathcal L_k(W)-\nabla_W f(W) \right\|_F^2 \leq 4G_{\mathrm{SSM},R}^2\sum_{k=1}^K p_k \operatorname{TV}(\mathcal P_k,\bar{\mathcal P})^2.
\end{equation} 

Since the right-hand side of this inequality is independent of $W$, taking the supremum over $W \in \mathcal N_R(W^r)$ we get:
\begin{equation}
    \Gamma_{\mathrm{SSM},R}^2 := \sup_{W \in \mathcal N_{R}(W^r)} \sum_{k=1}^K p_k\left\|\nabla_W \mathcal L_k(W)-\nabla_W f(W)\right\|_F^2 \leq 4G_{\mathrm{SSM},R}^2\sum_{k=1}^Kp_k \operatorname{TV}(\mathcal P_k,\bar{\mathcal P})^2.
\end{equation}

This proves the SSM client heterogeneity bound. \qed

\subsection{Theorem~\ref{thm:fedavg_ssm_convergence}: FedAvg SSM Convergence and Proof}

\begin{theorem}[FedAvg SSM Convergence]
\label{thm:fedavg_ssm_convergence}
Consider a distributed learning framework established in Section~\ref{sec:system_model}, with the objective
\begin{equation*}
    f(W) := \sum_{k=1}^K p_k\mathcal L_k(W), \qquad p_k\geq0, \qquad \sum_{k=1}^K p_k=1,
\end{equation*}
where $W$ parameterizes a single-layer selective SSM defined in the same section and each $\mathcal L_k$ is finite and differentiable on the considered neighborhoods.

Let $W^0$ be deterministic with $f(W^0) < \infty$ and suppose $f$ is bounded below by a finite constant $f^*$. 

Consider full-participation FedAvg distributed training algorithm applied to $f(W)$ for $T \geq 1$ communication rounds, $E \geq 1$ local steps per round, $\Xi \geq 1$ mini-batch size, and constant deterministic local step size $\eta > 0$, with the global model evolution described as:
\begin{align*}
    W_k^{r,0} &= W^r, \\ W_k^{r,e+1} &= W_k^{r,e}-\eta g_{k,\Xi}^{r,e}~~ \text{ for }~~ e \in \{0, \dots, E-1\}, \\ W^{r+1} &= \sum_{k=1}^K p_k W_k^{r,E}
\end{align*}
where $g_{k,\Xi}^{r,e}$ is the average of $\Xi$ sequence-level gradients evaluated at $W_k^{r,e}$ with sampling and conditional unbiasedness, and $E, \Xi, T \in \mathbb{Z}_{>0}$.

Suppose this model evolution satisfies, almost surely:
\begin{equation*}
    W_{k}^{r,e} \in \mathcal N_{R}(W^r) \subseteq \mathcal{S}_{\gamma,\delta},~~~~ \left\|W_X^r\right\|_2 + R \leq M_X ~~\forall ~~ X \in \{\Delta,B,C,A\},
\end{equation*}
for every client $k$ and local step $0 \leq e \leq E$, where $W_A:=A$, and $R,\gamma,\delta>0$ and the finite deterministic envelopes $M_X$ are common across rounds. Let Assumptions~\ref{assum:lipschitz_cont_loss}-~\ref{assum:unbiased_stochastic_gradients} hold uniformly on these neighborhoods.

Define $\rho := \exp(-\gamma\delta)$ and auxiliary SSM-specific quantities as follows:
\begin{align*}
    M_{K_\Delta} & := M_\Delta L_x + \sqrt D\log 2, \\
    M_{K_B} & := M_B L_x^2 ,\\
    M_{K_W} & := ND M_{K_B}M_{K_\Delta}\left(M_{K_\Delta} + \sqrt D M_A L_x\right)\frac{\rho}{1-\rho} + L_x^2\sqrt N M_{K_\Delta}   + L_x\sqrt D M_{K_B}, \\ H_R & := \frac{M_{K_B}M_{K_\Delta}\sqrt{ND}}{1-\rho}, \qquad
    J_R := \frac{M_{K_W}}{1-\rho}, \\
    M_{A_{1,R}} & := \rho(M_{K_\Delta} + M_A L_x), \\
    M_{A_{2,R}} & := \rho\left[(M_{K_\Delta} + M_A L_x)^2 + 2L_x + \frac{1}{4}M_A L_x^2\right], \\
    M_{B_{2,R}} & := 2L_x^3 + \frac{1}{4}M_B L_x^4,\\
    M_{H_{2,R}} & := \frac{M_{A_{2,R}}H_R + 2M_{A_{1,R}}J_R + M_{B_{2,R}}}{1-\rho}.
\end{align*}
These quantities determine the common SSM-specific smoothness
and gradient bounds:
\begin{align*}
    L_{\mathrm{SSM},R} & := H_\ell L_x^2(H_R + M_C J_R)^2 + L_\ell L_x(M_C M_{H_{2,R}} + 2J_R), \\
    G_{\mathrm{SSM},R} & := L_\ell L_x(H_R + M_CJ_R). 
\end{align*}
For the client distributions, further define
\begin{equation*}
    \bar{\mathcal P} :=\sum_{k=1}^K p_k\mathcal P_k, \qquad \mathcal H_{\mathrm{TV}} := \sum_{k=1}^K p_k \operatorname{TV}\left(\mathcal P_k,\bar{\mathcal P}\right)^2,
\end{equation*}
where $\operatorname{TV}(P,Q) := \sup_{\mathcal A} \left|P(\mathcal A)-Q(\mathcal A)\right|$ is the total variation distance, with the supremum taken over measurable sets $\mathcal A$.

Then, for any constant deterministic step size $\eta > 0$ satisfying
\begin{equation*}
    \eta E L_{\mathrm{SSM},R}\leq\frac{1}{2}, \qquad 12\eta^2L_{\mathrm{SSM},R}^2E(E-1) \leq 1,
\end{equation*}
the following holds:
\begin{align*}
    \frac{1}{T}\sum_{r=0}^{T-1} \mathbb E\|\nabla f(W^r)\|_F^2 \leq& \frac{4\left(f(W^0)-f^*\right)}{\eta ET}
    + \frac{4\eta L_{\mathrm{SSM},R} G_{\mathrm{SSM},R}^2}{\Xi} \sum_{k=1}^K p_k^2 \nonumber \\ 
    &+ \frac{6\eta^2 L_{\mathrm{SSM},R}^2(E-1)G_{\mathrm{SSM},R}^2}{\Xi} + 48\eta^2 L_{\mathrm{SSM},R}^2 E(E-1) G_{\mathrm{SSM},R}^2\mathcal H_{\mathrm{TV}}.
\end{align*}
\end{theorem}

\textbf{Proof of Theorem~\ref{thm:fedavg_ssm_convergence}:} We start from a standard nonconvex FedAvg convergence bound applicable to the setting considered here. Previous works establish a general bound (see Supplemental Material~C.6, Equation~(96) of \cite{wang2020tackling}) for smooth local objectives, unbiased stochastic gradients with bounded variance, and bounded client gradient dissimilarity.

For a setting with full client participation, $E$ local SGD steps per client, and objective and aggregation weights $p_k$, the bound in \cite{wang2020tackling} takes the following form:
\begin{equation}
    \frac{1}{T}\sum_{r=0}^{T-1} \mathbb E\|\nabla f(W^r)\|_F^2 \leq \frac{4\left(f(W^0)-f^*\right)}{\eta ET} + 4\eta L\sigma^2\sum_{k=1}^Kp_k^2 + 6\eta^2 L^2 (E-1)\sigma^2 + 12\eta^2L^2E(E-1)\Gamma^2,
    \label{eq:generic_fedavg_bound_appendix}
\end{equation}
whenever $\eta>0$ satisfies
\begin{equation*}
    \eta LE\leq\frac{1}{2}, \qquad 12\eta^2L^2E(E-1) \leq 1.
\end{equation*}
Here, $L$ is a common smoothness constant, $\sigma^2$ bounds the conditional variance of each mini-batch gradient estimator, and $\Gamma^2$ bounds the weighted client gradient disagreement:
\begin{equation}
    \sum_{k=1}^K p_k \|\nabla\mathcal L_k(W)-\nabla f(W)\|_F^2 \leq \Gamma^2.
\end{equation}
Note that the inequality~\eqref{eq:generic_fedavg_bound_appendix} specializes the result in \cite{wang2020tackling} applied to full client participation, $E$ local SGD steps at every client, and aggregation weights equal to the objective weights $p_k$.

We apply this bound to the selective SSM setting by substituting the constants established in our lemmas. First,  Lemma~\ref{lem:ssm_local_smoothness}, evaluated with the common envelopes $M_X$, gives
\begin{equation}
    L=L_{\mathrm{SSM},R},
\end{equation}
with $L_{\mathrm{SSM},R}$ defined in Theorem~\ref{thm:fedavg_ssm_convergence}.
Since each neighborhood $\mathcal N_R(W^r)$ is convex and contains all local iterates in round $r$, it also contains their weighted averages and the segment joining $W^r$ to $W^{r+1}$. As such, all smoothness inequalities used in the convergence argument hold with the common constant $L_{\mathrm{SSM},R}$.

Next, Lemma~\ref{lem:ssm_fedavg_g_constant} derives the common stochastic gradient second-moment bound $G_{\mathrm{SSM},R}^2$. Together with the sampling conditions in Assumption~\ref{assum:unbiased_stochastic_gradients}, Lemma~\ref{lem:ssm_variance_bound}  establishes the choice
\begin{equation}
    \sigma^2 = \frac{G_{\mathrm{SSM},R}^2}{\Xi}.
\end{equation}
Finally, Lemma~\ref{lem:ssm_gamma_tv_bound} establishes the common heterogeneity bound
\begin{equation}
    \Gamma^2 = 4G_{\mathrm{SSM},R}^2 \sum_{k=1}^Kp_k \operatorname{TV}(\mathcal P_k,\bar{\mathcal P})^2 = 4G_{\mathrm{SSM},R}^2\mathcal H_{\mathrm{TV}}.
\end{equation}

Substituting these quantities into
\eqref{eq:generic_fedavg_bound_appendix}, we obtain the inequality stated in Theorem~\ref{thm:fedavg_ssm_convergence} whenever $\eta>0$ satisfies the stated conditions with $L=L_{\mathrm{SSM},R}$. \qed

\subsection{Theorem~\ref{cor:fedprox_ssm_convergence}: FedProx SSM Convergence and Proof}
The FedProx algorithm modifies the local client objective by adding a proximal penalty around the round-$r$ broadcast model $W^r$:
\begin{equation}
    F_k^r(W) := \mathcal L_k(W) + \frac{\mu}{2}\left\|W-W^r\right\|_F^2,\qquad \mu>0,
\end{equation}
thus changing the local optimization dynamics. 

In this section we introduce a theorem that quantifies the convergence behavior of a distributed learning framework that utilizes the FedProx with selective State Space Models, and provide a corresponding proof.

\begin{theorem}[FedProx convergence for selective SSMs]\label{cor:fedprox_ssm_convergence}
Consider a distributed learning framework established in Section~\ref{sec:system_model}, with the objective
\begin{equation*}
    f(W) := \sum_{k=1}^K p_k\mathcal L_k(W), \qquad p_k\geq0, \qquad \sum_{k=1}^K p_k=1,
\end{equation*}
where $W$ parameterizes a single-layer selective SSM defined in the same section and each $\mathcal L_k$ is finite and differentiable on the considered neighborhoods.

Let $W^0$ be deterministic with $f(W^0) < \infty$ and suppose $f$ is bounded below by a finite constant $f^*$. 

Consider full-participation FedProx distributed training algorithm applied to $f(W)$ for $T \geq 1$ communication rounds, $E \geq 1$ local steps per round, $\Xi \geq 1$ mini-batch size, constant deterministic local step size $\eta > 0$, and constant deterministic proximal coefficient $\mu > 0$, with the global model evolution described as:
\begin{align*}
    W_k^{r,0} &= W^r, \\ W_k^{r,e+1} &= W_k^{r,e}-\eta \left(g_{k,\Xi}^{r,e}+ \mu\left(W_{k}^{r,e}-W^r\right)\right)~~ \text{ for }~~ e \in \{0, \dots, E-1\}, \\ W^{r+1} &= \sum_{k=1}^K p_k W_k^{r,E}
\end{align*}
where $g_{k,\Xi}^{r,e}$ is the average of $\Xi$ sequence-level gradients evaluated at $W_k^{r,e}$ with sampling and conditional unbiasedness, and $E, \Xi, T \in \mathbb{Z}_{>0}$.

Suppose that, almost surely for every round $r < T$:
\begin{equation}
    W_{k}^{r,e} \in \mathcal N_{R}(W^r) \subseteq \mathcal{S}_{\gamma,\delta},~~~~ \left\|W_X^r\right\|_2 + R \leq M_X ~~\forall ~~ X \in \{\Delta,B,C,A\},
\end{equation}
for every client $k$ and local step $0 \leq e \leq E$, where $W_A:=A$, and $R,\gamma,\delta>0$ and the finite deterministic envelopes $M_X$ are common across rounds. Let Assumptions~\ref{assum:lipschitz_cont_loss}-~\ref{assum:unbiased_stochastic_gradients} hold uniformly on these neighborhoods.

Define $\rho := \exp(-\gamma\delta)$. Assume $0<\eta\mu\leq 1$, and define the effective radius $R_\mu$ by
\begin{equation*}
    R_\mu := \min\left\{R,\eta G_{\mathrm{SSM},R} \sum_{j=0}^{E-1}(1-\eta\mu)^j\right\} = \min\left\{R,\frac{G_{\mathrm{SSM},R}}{\mu}\left[1-(1-\eta\mu)^E\right]\right\},
\end{equation*}
where $G_{\mathrm{SSM},R}$ is defined in Theorem~\ref{thm:fedavg_ssm_convergence}, with the corresponding FedProx parameter envelopes defined as:
\begin{equation*}
    M_X^\mu := M_X - R + R_\mu  ~~\forall ~~ X \in \{\Delta,B,C,A\}.
\end{equation*}

The corresponding auxiliary SSM-specific quantities are then defined as:
\begin{align*}
    M_{K_\Delta}^\mu & := M_\Delta^\mu L_x + \sqrt D\log 2, \\
    M_{K_B}^\mu & := M_B^\mu L_x^2 ,\\
    M_{K_W}^\mu & := ND M_{K_B}^\mu M_{K_\Delta}^\mu\left(M_{K_\Delta}^\mu + \sqrt D M_A^\mu L_x\right)\frac{\rho}{1-\rho} + L_x^2\sqrt N M_{K_\Delta}^\mu   + L_x\sqrt D M_{K_B}^\mu, \\ 
    H_{R_\mu} & := \frac{M_{K_B}^\mu M_{K_\Delta}^\mu \sqrt{ND}}{1-\rho}, \qquad
    J_{R_\mu} := \frac{M_{K_W}^\mu}{1-\rho}, \\
    M_{A_{1,R}}^\mu & := \rho(M_{K_\Delta}^\mu + M_A^\mu L_x), \\
    M_{A_{2,R}}^\mu & := \rho\left[(M_{K_\Delta}^\mu + M_A^\mu L_x)^2 + 2L_x + \frac{1}{4}M_A^\mu L_x^2\right], \\
    M_{B_{2,R}}^\mu & := 2L_x^3 + \frac{1}{4}M_B^\mu L_x^4,\\
    M_{H_{2,R}}^\mu & := \frac{M_{A_{2,R}}^\mu H_{R_\mu} + 2M_{A_{1,R}}^\mu J_{R_\mu} + M_{B_{2,R}}^\mu}{1-\rho}.
\end{align*}

Let:
\begin{align*}
    L_{\mathrm{SSM},R_\mu} & := H_\ell L_x^2(H_{R_\mu} + M_C^\mu J_{R_\mu} )^2 + L_\ell L_x(M_C^\mu M^\mu_{H_{2,R}} + 2J_{R_\mu} ), \\
    L_{\mathrm{prox},R_\mu} & := L_{\mathrm{SSM},R_\mu}+\mu. \\
    G_{\mathrm{SSM},R_\mu} & := L_\ell L_x(H_{R_\mu}  + M_C^\mu J_{R_\mu} ), \\
\end{align*}
denote common deterministic SSM-specific smoothness
and gradient bounds constants evaluated on $\mathcal N_{R_\mu}(W^r)$. 

For the client distributions, further define
\begin{equation*}
    \bar{\mathcal P} :=\sum_{k=1}^K p_k\mathcal P_k, \qquad \mathcal H_{\mathrm{TV}} := \sum_{k=1}^K p_k \operatorname{TV}\left(\mathcal P_k,\bar{\mathcal P}\right)^2,
\end{equation*}
where $\operatorname{TV}(P,Q) := \sup_{\mathcal A} \left|P(\mathcal A)-Q(\mathcal A)\right|$ is the total variation distance, with the supremum taken over measurable sets $\mathcal A$.

Then, for any constant deterministic step size $\eta > 0$ satisfying
\begin{equation*}
    \eta E L_{\mathrm{SSM},R_\mu} \leq \frac{1}{2}, \qquad 12\eta^2L_{\mathrm{prox},R_\mu}^2E(E-1) \leq 1, 
\end{equation*}
the following bound holds:
\begin{align*}
    \frac{1}{T}\sum_{r=0}^{T-1} \mathbb E\left\|\nabla f(W^r)\right\|_F^2 \leq& \frac{3\left(f(W^0)-f^*\right)}{\eta ET} + \frac{3\eta L_{\mathrm{SSM},R_\mu} G_{\mathrm{SSM},R_\mu}^2}{\Xi} \sum_{k=1}^K p_k^2 \\ 
    &+ \frac{3\eta^2L_{\mathrm{prox},R_\mu}^2 (E-1)G_{\mathrm{SSM},R_\mu}^2}{\Xi} + 24\eta^2L_{\mathrm{prox},R_\mu}^2E(E-1)G_{\mathrm{SSM},R_\mu}^2\mathcal H_{\mathrm{TV}}.
\end{align*}
Moreover, almost surely, for every round $r < T$, client $k$,
and local step $0\leq e\leq E$,
\begin{equation*}
    \left\|W_k^{r,e}-W^r\right\|_F \leq \min\left\{R,\frac{G_{\mathrm{SSM},R}}{\mu}\left[1-(1-\eta\mu)^e\right]\right\} \leq R_\mu.
\end{equation*}
Consequently, all local iterates and the aggregated model $W^{r+1}$ lie in $\mathcal N_{R_\mu}(W^r)\subseteq\mathcal S_{\gamma,\delta}$.
\end{theorem}

\textbf{Proof of Theorem~\ref{cor:fedprox_ssm_convergence}:} The FedProx framework modifies the local update as:
\begin{equation}
    W_{k}^{r,e+1} = W_{k}^{r,e} - \eta \left(g_{k,\Xi}^{r,e} + \mu(W_{k}^{r,e}-W^r)\right), \qquad W_k^{r,0}=W^r.
\end{equation}

Subtracting the broadcast model $W^r$ gives
\begin{equation}
    W_{k}^{r,e+1}-W^r = (1-\eta\mu)(W_{k}^{r,e}-W^r) - \eta g_{k,\Xi}^{r,e}
\end{equation}

Taking Frobenius norms, using $0<\eta\mu\leq 1$, and by the triangle inequality,
\begin{equation}
    \left\|W_{k}^{r,e+1}-W^r\right\|_F \leq (1-\eta\mu)\left\|W_{k}^{r,e}-W^r\right\|_F + \eta\left\|g_{k,\Xi}^{r,e}\right\|_F.
\end{equation}
By the bound~\eqref{eq:g_ssm_definition} established in the proof of Lemma~\ref{lem:ssm_fedavg_g_constant}, every admissible sequence gradient satisfies $\left\|g_k(W;\xi)\right\|_F\leq G_{\mathrm{SSM},R}$ for all $W\in \mathcal{N}_R(W^r)$. As such, the mini-batch gradients satisfy:
\begin{align}
    \left\|g_{k,\Xi}^{r,e}\right\|_F = \left\|\frac{1}{\Xi}\sum_{i=1}^{\Xi}g_k(W^{r,e}_{k};\xi^{r,e}_{k,i})\right\|_F \leq\frac{1}{\Xi}\sum_{i=1}^{\Xi} \left\|g_k(W^{r,e}_{k};\xi^{r,e}_{k,i})\right\|_F \leq G_{\mathrm{SSM},R}
\end{align}
Therefore:
\begin{equation}
    \left\|W_{k}^{r,e+1}-W^r\right\|_F \leq (1-\eta\mu)\left\|W_{k}^{r,e}-W^r\right\|_F + \eta G_{\mathrm{SSM},R}.
\end{equation}
Unrolling the recursion from $W_k^{r,0}=W^r$ for every $0\leq e\leq E$, we get:
\begin{equation}
    \left\|W_{k}^{r,e}-W^r\right\|_F \leq \eta G_{\mathrm{SSM},R}\sum_{e'=0}^{e-1}(1-\eta\mu)^{e'} = \frac{G_{\mathrm{SSM},R}}{\mu}\left[1-(1-\eta\mu)^e\right] 
\end{equation}
Combining this with the established neighborhood bounds $\left\|W_{k}^{r,e}-W^r\right\|_F\leq R$, we get:
\begin{equation}
    \left\|W_{k}^{r,e}-W^r\right\|_F \leq \min{\left(R,\frac{G_{\mathrm{SSM},R}}{\mu}\left[1-(1-\eta\mu)^e\right] \right)}\leq R_\mu,
\end{equation}
and since $p_k\geq 0$ and $\sum_k p_k =1$, it follows that:
\begin{equation}
    \left\|W^{r+1}-W^r\right\|_F =  \left\|\sum_{k=1}^{K}p_k\left(W_{k}^{r,E}-W^r\right)\right\|_F \leq  \sum_{k=1}^{K}p_k\left\|W_{k}^{r,E}-W^r\right\|_F \leq R_\mu.
\end{equation}
This shows that all local iterates as well as $W^{r+1}$ belong to $\mathcal N_{R_\mu}(W^r) \subseteq \mathcal S_{\gamma, \delta}$. Since this neighborhood is convex, the segments used below in the smoothness arguments also remain inside it.

Next, for every $W \in \mathcal N_{R_\mu}(W^r)$ and parameter block $X\in\{\Delta,B,C,A\}$, 
\begin{equation}
    \left\|W_X\right\|_2 \leq \|W_X^r\|_2 + R_\mu \leq M_X - R + R_\mu =: M_X^\mu.
\end{equation}
These $M_X^\mu$ are uniform across rounds and deterministic. Let $L_{\mathrm{SSM},R_\mu}$ and $G_{\mathrm{SSM},R_\mu}$ denote the SSM smoothness and stochastic-gradient constants constructed using $M_X^\mu$, with the stability, input, and loss-regularity constants remaining unchanged. In particular:
\begin{align}
    M_{K_\Delta}^\mu & := M_\Delta^\mu L_x + \sqrt D\log 2, \\
    M_{K_B}^\mu & := M_B^\mu L_x^2 ,\\
    M_{K_W}^\mu & := ND M_{K_B}^\mu M_{K_\Delta}^\mu \left(M_{K_\Delta}^\mu + \sqrt D M_A^\mu L_x\right)\frac{\rho}{1-\rho} + L_x^2\sqrt N M_{K_\Delta}^\mu + L_x\sqrt D M_{K_B}^\mu, \\ 
    H_{R_\mu} & := \frac{M_{K_B}^\mu M_{K_\Delta}^\mu \sqrt{ND}}{1-\rho}, \qquad J_{R_\mu} := \frac{M_{K_W}^\mu}{1-\rho}, \\
    M_{A_{1,R}}^\mu & := \rho(M_{K_\Delta}^\mu + M_A^\mu L_x), \\
    M_{A_{2,R}}^\mu & := \rho\left[(M_{K_\Delta}^\mu + M_A^\mu L_x)^2 + 2L_x + \frac{1}{4}M_A^\mu L_x^2\right], \\
    M_{B_{2,R}}^\mu & := 2L_x^3 + \frac{1}{4}M_B^\mu L_x^4,\\
    M_{H_{2,R}}^\mu & := \frac{M_{A_{2,R}}^\mu H_{R_\mu} + 2M_{A_{1,R}}^\mu J_{R_\mu} + M_{B_{2,R}}^\mu}{1-\rho}, \\
    L_{\mathrm{SSM},R_\mu} & := H_\ell L_x^2(H_{R_\mu} + M_C^\mu J_{R_\mu})^2 + L_\ell L_x(M_C^\mu M_{H_{2,R}}^\mu + 2J_{R_\mu} ), \label{eq:fedprox_L_SSM_R}\\
    G_{\mathrm{SSM},R_\mu} & := L_\ell L_x(H_{R_\mu}  + M_C^\mu J_{R_\mu}), \label{eq:fedprox_G_SSM_R}
\end{align}
and in addition define:
\begin{equation}
    L_{\mathrm{prox},R_\mu} = L_{\mathrm{SSM},R_\mu}+\mu\label{eq:fedprox_L_prox_R}
\end{equation}

On this neighborhood, Lemma~\ref{lem:ssm_local_smoothness} implies that each $\mathcal L_k$ is $L_{\mathrm{SSM},R_\mu}$-smooth.

Then, for any $W,V\in\mathcal N_{R_\mu}(W^r)$, the gradients
of the local proximal objective satisfy:
\begin{align}
    &\biggl\|\bigl(\nabla\mathcal L_k(W)+\mu(W-W^r)\bigr) -\bigl(\nabla\mathcal L_k(V)+\mu(V-W^r)\bigr)\biggr\|_F \nonumber \\
    & \qquad \leq \|\nabla\mathcal L_k(W)-\nabla\mathcal L_k(V)\|_F + \mu\|W-V\|_F \nonumber \\ 
    & \qquad \leq  \left(L_{\mathrm{SSM},R_\mu}+\mu\right)\left\|W-V\right\|_F \nonumber \\ & \qquad = L_{\mathrm{prox},R_\mu}\|W-V\|_F.\label{eq:prox_smoothness_equation}
\end{align}
Therefore, the local proximal objective $\mathcal L_k(W)+\frac{\mu}{2}\|W - W^r\|_F^2$ is $L_{\mathrm{prox},R_\mu}$-smooth on $\mathcal N_{R_\mu}(W^r)$. Moreover, since $f = \sum_kp_k\mathcal L_k$ with $p_k\geq 0$ and $\sum_k p_k = 1$, the global objective $f$ remains $L_{\mathrm{SSM},R_\mu}$-smooth on the same neighborhood.

Next, define the following common deterministic variance constant:
\begin{equation}
    \sigma^2_{\mathrm{SSM},R_\mu} := \frac{G_{\mathrm{SSM},R_\mu}^2}{\Xi}. \label{eq:prox_sigma_Gamma_mu_bounds}
\end{equation}
By Lemma~\ref{lem:ssm_gamma_tv_bound}, for every round $r<T$,
\begin{equation}
    \sup_{W\in\mathcal N_{R_\mu}(W^r)} \sum_k p_k \|\nabla\mathcal L_k(W)-\nabla f(W)\|_F^2 \leq 4G_{\mathrm{SSM},R_\mu}^2\mathcal H_{\mathrm{TV}}. \label{eq:fedprox_uniform_heterogeneity}
\end{equation}
The right-hand side provides a common deterministic heterogeneity bound across rounds.

Next, fix a round $r < T$ and let $\mathbb E_r[\cdot]$ denote an expectation conditional on the training history at the beginning of that round, with $W^r$ being fixed under this conditioning. 

We begin by unrolling the local update and subtracting the starting weights:
\begin{align}
    W_k^{r,e} - W^r &= -\eta \sum_{j=0}^{e-1}\left[g^{r,j}_{k,\Xi}+\mu\left(W^{r,j}_k-W^r\right)\right] \nonumber \\ &= -\eta \sum_{j=0}^{e-1}\left[g^{r,j}_{k,\Xi}+\mu\left(W^{r,j}_k-W^r\right)\right] \nonumber \\ &\quad+ \eta\sum_{j=0}^{e-1}\left(\nabla \mathcal L_k\left(W^r\right)-\nabla \mathcal L_k\left(W^{r,j}_k\right)\right) - \eta\sum_{j=0}^{e-1}\left(\nabla \mathcal L_k\left(W^r\right)-\nabla \mathcal L_k\left(W^{r,j}_k\right)\right) \nonumber \\
    &= - \eta e\nabla \mathcal L_k\left(W^r\right) - \eta\sum_{j=0}^{e-1}\left[g^{r,j}_{k,\Xi}-\nabla \mathcal L_k\left(W^{r,j}_k\right)\right] \nonumber \\ 
    &\quad-\eta \sum_{j=0}^{e-1}\left[\mu\left(W^{r,j}_k-W^r\right)-\nabla \mathcal L_k\left(W^r\right) + \nabla \mathcal L_k\left(W^{r,j}_k\right)\right].\label{eq:fedprox_proof_unrolled_local_model_update}
\end{align}

Next, when evaluating the norms of the individual quantities on the right of~\eqref{eq:fedprox_proof_unrolled_local_model_update}, by Lemma~\ref{lem:ssm_variance_bound}, the following holds
\begin{equation}
    \mathbb E_r \left[\left\|\sum_{j=0}^{e-1}\left(g_{k, \Xi}^{r,j}-\nabla \mathcal L_k\left(W_k^{r,j}\right)\right)\right\|_F^2\right] \leq e\sigma_{\mathrm{SSM},R_\mu}^2\label{eq:fedprox_proof_diff_stochastic_grad_grad}
\end{equation}
for all $0 \leq e \leq E$. 

Then, by substituting $W_{k}^{r,j}$ and $W^r$ into~\eqref{eq:prox_smoothness_equation}, it follows that:
\begin{equation}
    \left\|\mu\left(W^{r,j}_k-W^r\right)-\nabla \mathcal L_k\left(W^r\right) + \nabla \mathcal L_k\left(W^{r,j}_k\right)\right\|_F \leq L_{\mathrm{prox},R_\mu}\left\|W^{r,j}_k-W^r\right\|_F
\end{equation}

Taking the norm of~\eqref{eq:fedprox_proof_unrolled_local_model_update}, applying the squared-normed inequality for a sum of three terms as well as the Cauchy-Schwarz inequality, and substituting these results, we get:
\begin{align}
    \mathbb E_r\left[\left\|W_{k}^{r,e}-W^r\right\|_F^2\right] &\leq 3\eta^2 e^2 \left\|\nabla \mathcal L_k\left(W^r\right)\right\|_F^2 + 3\eta^2 e\sigma_{\mathrm{SSM},R_\mu}^2 \nonumber \\ &\quad + 3\eta^2 e L_{\mathrm{prox},R_\mu}^2 \sum_{j=0}^{e-1}\mathbb E_r\left[\left\|W^{r,j}_k-W^r\right\|_F^2\right].\label{eq:fedprox_proof_cond_exp_model_diff_bound_init}
\end{align}

By multiplying~\eqref{eq:fedprox_proof_cond_exp_model_diff_bound_init} by $p_k$ and summing over all clients, we get:
\begin{align}
    \sum_{k=1}^K p_k\mathbb E_r\left[\left\|W_{k}^{r,e}-W^r\right\|_F^2\right] &\leq 3\eta^2 e\sigma_{\mathrm{SSM},R_\mu}^2 + 3\eta^2 e^2 \sum_{k=1}^K p_k\left\|\nabla \mathcal L_k\left(W^r\right)\right\|_F^2  \nonumber \\ &\quad + 3\eta^2 e L_{\mathrm{prox},R_\mu}^2 \sum_{j=0}^{e-1}\sum_{k=1}^K p_k\mathbb E_r\left[\left\|W^{r,j}_k-W^r\right\|_F^2\right].\label{eq:fedprox_proof_cond_exp_model_diff_bound_clients_averaged}
\end{align}

Next, note that:
\begin{align}
    \sum_{k=1}^K p_k\left\|\nabla \mathcal L_k\left(W^r\right)\right\|_F^2 &= \sum_{k=1}^K p_k\left\|\nabla \mathcal L_k\left(W^r\right) +\nabla f(W^r)-\nabla f(W^r)\right\|_F^2 \nonumber \\ 
    &\leq \left\|\nabla f(W^r)\right\|_F^2+ \sum_{k=1}^K p_k\left\|\nabla \mathcal L_k\left(W^r\right)-\nabla f(W^r)\right\|_F^2 \nonumber \\
    &\leq \left\|\nabla f(W^r)\right\|_F^2 + 4G_{\mathrm{SSM},R_\mu}^2\mathcal H_{\mathrm{TV}}.
\end{align}
where the last inequality follows from~\eqref{eq:fedprox_uniform_heterogeneity}. Substituting this result into~\eqref{eq:fedprox_proof_cond_exp_model_diff_bound_clients_averaged}, we get:
\begin{align}
    \sum_{k=1}^K p_k\mathbb E_r\left[\left\|W_{k}^{r,e}-W^r\right\|_F^2\right] &\leq 3\eta^2 e\sigma_{\mathrm{SSM},R_\mu}^2 + 3\eta^2 e^2\left(\left\|\nabla f(W^r)\right\|_F^2 + 4G_{\mathrm{SSM},R_\mu}^2\mathcal H_{\mathrm{TV}}\right)
    \nonumber \\ &\quad + 3\eta^2 e L_{\mathrm{prox},R_\mu}^2 \sum_{j=0}^{e-1}\sum_{k=1}^K p_k\mathbb E_r\left[\left\|W^{r,j}_k-W^r\right\|_F^2\right].\label{eq:fedprox_proof_cond_exp_model_diff_bound_clients_averaged_with_global_gradient}
\end{align}

Averaging~\eqref{eq:fedprox_proof_cond_exp_model_diff_bound_clients_averaged_with_global_gradient} over $e=0, \dots, E-1$, where we utilize $\frac{1}{E}\sum_{e=0}^{E-1}e = \frac{E-1}{2}$ and $\frac{1}{E}\sum_{e=0}^{E-1}e^2 \leq E(E-1)$, yields:
\begin{align}
    \frac{1}{E}\sum_{e=1}^{E-1}\sum_{k=1}^K p_k\mathbb E_r\left[\left\|W_{k}^{r,e}-W^r\right\|_F^2\right] &\leq \frac{3}{2}\eta^2(E-1)\sigma_{\mathrm{SSM},R_\mu}^2 \nonumber \\ &\quad + 3\eta^2E(E-1)\left(\left\|\nabla f(W^r)\right\|_F^2 + 4G_{\mathrm{SSM},R_\mu}^2\mathcal H_{\mathrm{TV}}\right)
    \nonumber \\ &\quad + 3\eta^2 E(E-1) L_{\mathrm{prox},R_\mu}^2 \left[\frac{1}{E}\sum_{e=0}^{E-1}\sum_{k=1}^K p_k\mathbb E_r\left[\left\|W^{r,e}_k-W^r\right\|_F^2\right]\right],\label{eq:fedprox_proof_cond_exp_model_diff_bound_clients_averaged_with_global_gradient_rounds_averaged}
\end{align}
where for the last term, we used:
\begin{align}
    \frac{1}{E}\sum_{e=0}^{E-1}e\sum_{j=0}^{e-1}\sum_{k=1}^K p_k\mathbb E_r\left[\left\|W^{r,j}_k-W^r\right\|_F^2\right] \leq E(E-1) \left[\frac{1}{E}\sum_{e=0}^{E-1}\sum_{k=1}^K p_k\mathbb E_r\left[\left\|W^{r,e}_k-W^r\right\|_F^2\right] \right]
\end{align}

Rearranging the terms in~\eqref{eq:fedprox_proof_cond_exp_model_diff_bound_clients_averaged_with_global_gradient_rounds_averaged} and dividing out by $1-3\eta^2 E(E-1) L_{\mathrm{prox},R_\mu}^2$, yields:
\begin{align}
    \frac{1}{E}\sum_{e=1}^{E-1}\sum_{k=1}^K p_k\mathbb E_r\left[\left\|W_{k}^{r,e}-W^r\right\|_F^2\right] &\leq \frac{1}{1-3\eta^2 E(E-1) L_{\mathrm{prox},R_\mu}^2}\Bigg(\frac{3}{2}\eta^2(E-1)\sigma_{\mathrm{SSM},R_\mu}^2 \nonumber \\ &\quad + 3\eta^2E(E-1)\left(\left\|\nabla f(W^r)\right\|_F^2 + 4G_{\mathrm{SSM},R_\mu}^2\mathcal H_{\mathrm{TV}}\right)\Bigg).\label{eq:fedprox_proof_inequality_difference_rearanged}
\end{align}
Next, note that by the step-size condition $12\eta^2 E(E-1) L_{\mathrm{prox},R_\mu}^2 \leq 1$, it is true that $3\eta^2 E(E-1) L_{\mathrm{prox},R_\mu}^2 \leq \frac{1}{4}$, and therefore $\frac{1}{1-3\eta^2 E(E-1) L_{\mathrm{prox},R_\mu}^2} \leq \frac{4}{3}$. Substituting that result into~\eqref{eq:fedprox_proof_inequality_difference_rearanged} yields:
\begin{align}
    \frac{1}{E}\sum_{e=1}^{E-1}\sum_{k=1}^K p_k\mathbb E_r\left[\left\|W_{k}^{r,e}-W^r\right\|_F^2\right] &\leq 2\eta^2(E-1)\sigma_{\mathrm{SSM},R_\mu}^2 \nonumber \\ &\quad + 4\eta^2E(E-1)\left(\left\|\nabla f(W^r)\right\|_F^2 + 4G_{\mathrm{SSM},R_\mu}^2\mathcal H_{\mathrm{TV}}\right).\label{eq:fedprox_proof_final_inequality_local_model_differnce}
\end{align}

Next, we evaluate $W^{r+1}-W^r$. By the update rule:
\begin{equation}
    W^{r+1}-W^r = -\eta \sum_{e=0}^{E-1}\sum_{k=1}^K p_k\left[g^{r,e}_{k,\Xi}+\mu\left(W^{r,e}_k-W^r\right)\right].
\end{equation}

Since the segment joining $W^r$ and $W^{r+1}$ lies in $\mathcal N_{R_\mu}(W^r)$, by the smoothness of $f$ established in Lemma~\ref{lem:ssm_local_smoothness}, it is true that:
\begin{align}
    \mathbb E_r\left[f(W^{r+1})\right] \leq& f(W^r) \nonumber \\ &-\eta\left\langle \nabla f\left(W^r\right), \mathbb E_r \left[\sum_{e=0}^{E-1}\sum_{k=1}^K p_k \left(\nabla \mathcal L_k\left(W^{r,e}_k\right)+\mu \left(W_{k}^{r,e}-W^r\right)\right)\right]\right\rangle_F \nonumber \\ &+ \frac{\eta^2L_{\mathrm{SSM},R_\mu}}{2}\mathbb E_r\left[\left\|\sum_{e=0}^{E-1}\sum_{k=1}^Kp_k\left(g^{r,e}_{k,\Xi}+\mu\left(W^{r,e}_k-W^r\right)\right)\right\|_F^2\right].\label{eq:fedprox_proof_expected_next_round_loss_init}
\end{align}
Here, the tower property and because of conditional unbiasedness, the stochastic gradients can be replaced by local gradients inside the expected inner product. 

Next, note that:
\begin{align}
   \sum_{e=0}^{E-1}\sum_{k=1}^Kp_k\left(g^{r,e}_{k,\Xi}+\mu\left(W^{r,e}_k-W^r\right)\right) &= \sum_{e=0}^{E-1}\sum_{k=1}^Kp_k\left(g^{r,e}_{k,\Xi}+\mu\left(W^{r,e}_k-W^r\right)+\nabla \mathcal L_k\left(W_{k}^{r,e}\right)-\nabla \mathcal L_k\left(W_{k}^{r,e}\right)\right) \nonumber \\ &= \sum_{e=0}^{E-1}\sum_{k=1}^Kp_k\left(\nabla \mathcal L_k\left(W_{k}^{r,e}\right)+\mu\left(W^{r,e}_k-W^r\right)\right) \nonumber \\ &\quad+\sum_{e=0}^{E-1}\sum_{k=1}^Kp_k\left(g^{r,e}_{k,\Xi}-\nabla \mathcal L_k\left(W_{k}^{r,e}\right)\right).
\end{align}
When taking norms, for a sum of two matrices, the squared norm is at most twice the sum of their squared norms. As such, it follows that:
\begin{align}
    \mathbb E_r\left[\left\|\sum_{e=0}^{E-1}\sum_{k=1}^Kp_k\left(g^{r,e}_{k,\Xi}+\mu\left(W^{r,e}_k-W^r\right)\right)\right\|_F^2\right] &\leq 2\mathbb E_r\left[\left\|\sum_{e=0}^{E-1}\sum_{k=1}^Kp_k\left(g^{r,e}_{k,\Xi}-\nabla \mathcal L_k\left(W_{k}^{r,e}\right)\right)\right\|_F^2\right] \nonumber \\ &\quad + 2\mathbb E_r\left[\left\|\sum_{e=0}^{E-1}\sum_{k=1}^Kp_k\left(\nabla \mathcal L_k\left(W_{k}^{r,e}\right)+\mu\left(W^{r,e}_k-W^r\right)\right)\right\|_F^2\right]
\end{align}

Next, the centered gradient-noise differences have conditional mean zero given the complete history before each local step. Conditional independence across clients eliminates the same-step cross terms, while conditional unbiasedness and the tower property eliminate the cross terms between different local steps. Applying Lemma~\ref{lem:ssm_variance_bound} therefore gives
\begin{equation}
    \mathbb E_r\left[\left\|\sum_{e=0}^{E-1}\sum_{k=1}^Kp_k\left(g^{r,e}_{k,\Xi}-\nabla \mathcal L_k\left(W_{k}^{r,e}\right)\right)\right\|_F^2\right] \leq \sum_{e=0}^{E-1}\sum_{k=1}^K p_k^2  \sigma_{\mathrm{SSM},R_\mu}^2 =E\sigma_{\mathrm{SSM},R_\mu}^2\sum_{k=1}^Kp_k^2,
\end{equation}
and as a result:
\begin{align}
    \mathbb E_r\left[\left\|\sum_{e=0}^{E-1}\sum_{k=1}^Kp_k\left(g^{r,e}_{k,\Xi}+\mu\left(W^{r,e}_k-W^r\right)\right)\right\|_F^2\right] &\leq 2E\sigma_{\mathrm{SSM},R_\mu}^2\sum_{k=1}^Kp_k^2 \nonumber \\ &\quad + 2\mathbb E_r\left[\left\|\sum_{e=0}^{E-1}\sum_{k=1}^Kp_k\left(\nabla \mathcal L_k\left(W_{k}^{r,e}\right)+\mu\left(W^{r,e}_k-W^r\right)\right)\right\|_F^2\right]
\end{align}

When it comes to the inner-product term, we note that, since $f(W^r) = \sum_{k=1}^Kp_k \mathcal L_{k}(W^r)$, it is true that:
\begin{equation}
    E \nabla f(W^r)=\sum_{e=0}^{E-1}\sum_{k=1}^K p_k \nabla \mathcal L_k(W^r).
\end{equation}

Using this result, and the $-2\langle U,V\rangle_F=\|U-V\|_F^2-\|U\|_F^2-\|V\|_F^2$ equality, the following holds:
\begin{align}
    -&\eta\left\langle \nabla f\left(W^r\right), \mathbb E_r \left[\sum_{e=0}^{E-1}\sum_{k=1}^K p_k \left(\nabla \mathcal L_k\left(W^{r,e}_k\right)+\mu \left(W_{k}^{r,e}-W^r\right)\right)\right]\right\rangle_F \\ &= -\frac{\eta E}{2}\left\|\nabla f(W^r)\right\|_F^2 \nonumber \\ &-\frac{\eta}{2E}\mathbb E_r \left[\left\|\sum_{e=0}^{E-1}\sum_{k=1}^K p_k \left(\nabla \mathcal L_k\left(W^{r,e}_k\right)+\mu \left(W_{k}^{r,e}-W^r\right)\right)\right\|_F^2\right] \nonumber \\ &+\frac{\eta}{2E}\mathbb E_r \left[\left\|\sum_{e=0}^{E-1}\sum_{k=1}^K p_k \left(\nabla \mathcal L_k\left(W^{r,e}_k\right)+\mu \left(W_{k}^{r,e}-W^r\right)-\nabla \mathcal L_k\left(W^r\right)\right)\right\|_F^2\right],
\end{align}
where the $E$-terms appeared from evaluating the inner-product with $E \nabla f(W^r)$ and dividing out by $E$.

Then, the positive term is bounded as:
\begin{align}
    \mathbb E_r &\left[\left\|\sum_{e=0}^{E-1}\sum_{k=1}^K p_k \left(\nabla \mathcal L_k\left(W^{r,e}_k\right)+\mu \left(W_{k}^{r,e}-W^r\right)-\nabla \mathcal L_k\left(W^r\right)\right)\right\|_F^2\right] \nonumber \\
    &\leq E\sum_{e=0}^{E-1}\sum_{k=1}^Kp_k \mathbb E_r \left[\left\|\nabla \mathcal L_k\left(W^{r,e}_k\right)+\mu \left(W_{k}^{r,e}-W^r\right)-\nabla \mathcal L_k\left(W^r\right)\right\|_F^2\right] \nonumber \\ &\leq E L_{\mathrm{prox},R_\mu}^2 \sum_{e=0}^{E-1}\sum_{k=1}^Kp_k \mathbb E_r\left[\left\|W_k^{r,e}-W^r\right\|_F^2\right]
\end{align}
where the first inequality is true by Jensen's inequality with weights $p_k/E$, and the second inequality follows from~\eqref{eq:prox_smoothness_equation}.

Substituting these results back into~\eqref{eq:fedprox_proof_expected_next_round_loss_init}, we obtain:
\begin{align}
    \mathbb E_r\left[f\left(W^{r+1}\right)\right] \leq& f\left(W^r\right) -\frac{\eta E}{2}\left\|\nabla f(W^r)\right\|_F^2 \nonumber \\
    &+ \left(\eta^2 L_{\mathrm{SSM},R_\mu}-\frac{\eta}{2E} \right)\mathbb E_r \left[\left\|\sum_{e=0}^{E-1}\sum_{k=1}^K p_k \left(\nabla \mathcal L_k\left(W^{r,e}_k\right)+\mu \left(W_{k}^{r,e}-W^r\right)\right)\right\|_F^2\right] \nonumber \\
    &+ \frac{\eta L_{\mathrm{prox},R_\mu}^2 }{2}\sum_{e=0}^{E-1} \sum_{k=1}^K p_k\mathbb E_r\left[\left\|W_k^{r,e}-W^r\right\|_F^2\right] \nonumber \\ 
    &+ \eta^2 EL_{\mathrm{SSM},R_\mu} \sigma_{\mathrm{SSM},R_\mu}^2\sum_{k=1}^Kp_k^2.
\end{align}

The step-size condition $\eta E L_{\mathrm{SSM},R_\mu}\leq 1/2$ implies that $\eta^2 L_{\mathrm{SSM},R_\mu}-\frac{\eta}{2E} \leq 0$, and therefore the term multiplied by this factor is never positive. As such, the following holds:
\begin{align}
    \mathbb E_r\left[f\left(W^{r+1}\right)\right] \leq& f\left(W^r\right) -\frac{\eta E}{2}\left\|\nabla f(W^r)\right\|_F^2 \nonumber \\
    &+ \frac{\eta L_{\mathrm{prox},R_\mu}^2 }{2} \sum_{e=0}^{E-1} \sum_{k=1}^{K}  p_k \mathbb E_r\left[\left\|W_k^{r,e}-W^r\right\|_F^2\right] \nonumber \\ 
    &+ \eta^2 EL_{\mathrm{SSM},R_\mu} \sigma_{\mathrm{SSM},R_\mu}^2\sum_{k=1}^Kp_k^2.
\end{align}

Substituting the result from~\eqref{eq:fedprox_proof_final_inequality_local_model_differnce} into this inequality, we get:
\begin{align}
    \mathbb E_r\left[f\left(W^{r+1}\right)\right] \leq& f\left(W^r\right) \nonumber \\ &+ \left(2\eta^3E^2(E-1)L_{\mathrm{prox},R_\mu}^2 -\frac{\eta E}{2}\right)\left\|\nabla f(W^r)\right\|_F^2 \nonumber \\
    &+ \eta^3 E(E-1) L_{\mathrm{prox},R_\mu}^2  \sigma_{\mathrm{SSM},R_\mu}^2 \nonumber \\ &+ 8\eta^3E^2(E-1)L_{\mathrm{prox},R_\mu}^2G_{\mathrm{SSM},R_\mu}^2\mathcal H_{\mathrm{TV}}\nonumber  \\ 
    &+ \eta^2 EL_{\mathrm{SSM},R_\mu} \sigma_{\mathrm{SSM},R_\mu}^2\sum_{k=1}^Kp_k^2.\label{eq:fedprox_proof_one_round_descent_init}
\end{align}

Next, by the second step size condition $12\eta^2L_{\mathrm{prox},R_\mu}^2E(E-1) \leq 1$ it follows that:
\begin{align}
    2\eta^3E^2(E-1)L_{\mathrm{prox},R_\mu}^2 -\frac{\eta E}{2} &= \eta E\left(2\eta^2 E(E-1)L_{\mathrm{prox},R_\mu}^2 - \frac{1}{2}\right) \nonumber \\ &\leq \eta E \left(\frac{1}{6}-\frac{1}{2}\right) = -\frac{\eta E}{3}.
\end{align}
Substituting in~\eqref{eq:fedprox_proof_one_round_descent_init}, we get:
\begin{align}
    \mathbb E_r\left[f\left(W^{r+1}\right)\right] \leq& f\left(W^r\right) \nonumber \\ &-\frac{\eta E}{3}\left\|\nabla f(W^r)\right\|_F^2 \nonumber \\
    &+ \eta^3 E(E-1) L_{\mathrm{prox},R_\mu}^2  \sigma_{\mathrm{SSM},R_\mu}^2 \nonumber \\ &+ 8\eta^3E^2(E-1)L_{\mathrm{prox},R_\mu}^2G_{\mathrm{SSM},R_\mu}^2\mathcal H_{\mathrm{TV}}\nonumber  \\ 
    &+ \eta^2 EL_{\mathrm{SSM},R_\mu} \sigma_{\mathrm{SSM},R_\mu}^2\sum_{k=1}^Kp_k^2.\label{eq:fedprox_proof_one_round_descent}
\end{align}

Taking total expectations in~ \eqref{eq:fedprox_proof_one_round_descent}, rearranging, and summing over $r=0,\ldots,T-1$ makes the objective differences telescope. Using $\mathbb E[f(W^T)]\geq f^*$ and dividing by $\eta ET/3$, we obtain:
\begin{align}
    \frac{1}{T}\sum_{r=0}^{T-1}\mathbb E\left\|\nabla f(W^r)\right\|_F^2 \leq & \frac{3\left(f(W^0)-f^*\right)}{\eta ET} + 3\eta L_{\mathrm{SSM},R_\mu} \sigma_{\mathrm{SSM},R_\mu}^2\sum_{k=1}^Kp_k^2 \nonumber \\
    & + 3\eta^2 (E-1) L_{\mathrm{prox},R_\mu}^2  \sigma_{\mathrm{SSM},R_\mu}^2 + 24\eta^2E(E-1)L_{\mathrm{prox},R_\mu}^2G_{\mathrm{SSM},R_\mu}^2\mathcal H_{\mathrm{TV}}.\label{eq:fedprox_proof_final_inequality_with_sigma}
\end{align}

Finally, substituting $\sigma_{\mathrm{SSM},R_\mu}^2
=G_{\mathrm{SSM},R_\mu}^2/\Xi$ (which follows from Lemma~\ref{lem:ssm_variance_bound}) into~\eqref{eq:fedprox_proof_final_inequality_with_sigma}, gives the stated bound, completing the proof. \qed

\section{Multi-Layer Selective SSM proofs}\label{appendix:multi_layer_SSM_proofs}
In this section we derive the corresponding Gradient, FedAvg, and FedProx convergence bounds for a Selective SSM with of an arbitrary finite depth $Q$, extending our derivations in Sections~\ref{appendix:gradient_norm_bound_proof}, and~\ref{appendix:fed_algs_convergence_bounds_proofs} to a multi-layer selective State Space Model (SSM) case. 

\subsection{Multi-Layer Selective SSM Gradient Bound Proof}\label{appendix:multi_layer_selective_ssm}

Similar to the single-layer setting, we first derive an upper bound on the gradient norm of the loss function for an arbitrary sequence of length $S$ processed by a selective state space model (SSM) composed of $Q$ sequential selective SSM layers. The analysis once again utilizes definitions from Section~\ref{sec:system_model} and an input sequence $x = (x_1, \dots, x_S)$ where $x_t \in \mathbb R^{D\times 1}$. 

The multi-layer selective SSM consists of $Q$ layers arranged sequentially. Each layer $q \in \left\{1,\dots,Q\right\}$ is parameterized by the learnable parameter tuple
\begin{equation}
    W^{(q)} = \left(W_\Delta^{(q)}, W_B^{(q)}, W_C^{(q)}, A^{(q)}\right),
 \end{equation}
where $W_\Delta^{(q)} \in \mathbb R^{D\times D}, W_B^{(q)}\in \mathbb R^{N\times D}, W_C^{(q)}\in \mathbb R^{N\times D}, A^{(q)}\in \mathbb R^{N\times 1}$. Note that we therefore consider a multi-layer SSM where all of its $Q$ layers have the same input/output dimension $D$ and state dimension $N$. 

The full set of model parameters is then given by:
\begin{equation}
    W = \left(W^{(1)}, \dots, W^{(Q)}\right).
\end{equation}

In this multi-layer formulation, we denote the input to layer $q$ at time step $t$ by $x_t^{(q)}$, where $x_t^{(1)} = x_t$ and
\begin{equation}
    x_t^{(q+1)} = y_t^{(q)}
\end{equation}
with $y_t^{(q)}$ denoting the output of layer $q$.

Given this formulation, each layer $q$ of the multi-layer selective SSM operates according to the following dynamics:
\begin{align}
    \Delta_t^{(q)} &= \text{Softplus}\left(W_\Delta^{(q)}x_t^{(q)}\right) , \\
    \bar{A}^{(q)}_t &= \exp\left(A^{(q)}\left(\Delta_t^{(q)}\right)^\top\right), \\
    \bar{B}_t^{(q)} &= \left(W_B^{(q)}x_t^{(q)}\right)\left(\text{Softplus}\left(W_\Delta^{(q)}x_t^{(q)}\right)\right)^\top, \\
    h_t^{(q)} &= \bar{A}^{(q)}_t \odot h_{t-1}^{(q)} + \bar{B}_t^{(q)} \text{diag}\left(x_t^{(q)}\right), \label{eq:hidden_state_definition_multi_layer}\\
    y_t^{(q)} &= \left(h_t^{(q)}\right)^\top W_C^{(q)}x_t^{(q)} = \left(h_t^{(q)}\right)^\top C_t^{(q)} 
\end{align}
where the hidden state of each layer $q$ is initialized as $h_0^{(q)} = \mathbf{0}_{N\times D}$, just as in the single-layer setting. 

Finally, the empirical loss for processing a sequence of length $S$ for such defined multi-layer selective SSM is:
\begin{equation}
    \mathcal L^{S,(Q)}(W;\xi) = \frac{1}{S} \sum_{t=1}^S \ell_t\left(y_t^{(Q)}\right) = \frac{1}{S} \sum_{t=1}^S \ell_t\left(x_t^{(Q+1)}\right).
\end{equation}
Here, $\xi=(x,y^*)$ is an admissible input-target pair of length $S\geq 1$, and $\ell_t\left(y_t^{(Q)}\right)$ abbreviates $\ell_t(y_t^{(Q)}(W;x),y_t^*)$. For fixed $\xi$, we write $\mathcal L^{S,(Q)}(W)$ in place of $\mathcal L^{S,(Q)}(W;\xi)$.

In this multi-layer setting, the full gradient is the
block tuple:
\begin{equation}
    \nabla_W\mathcal L^{S,(Q)}(W) = \left(\nabla_{W^{(1)}}\mathcal L^{S,(Q)}(W), \dots, \nabla_{W^{(Q)}}\mathcal L^{S,(Q)}(W)\right).
\end{equation}

Similar to the single-layer proof, in the multi-layer formulation, gradients with respect to the full parameter tuple $W = \left(W^{(1)}, \dots, W^{(Q)}\right)$ are understood as block gradients across layers. As such, for any function $f$,
\begin{equation}\label{eq:def_func_gradient}
    \left\|\nabla_W f\right\|_F^2 := \sum_{q=1}^Q \left\|\nabla_{W^{(q)}} f\right\|_F^2,
\end{equation}
where, for each layer $q$, the gradient is further decomposed into parameter blocks as:
\begin{equation}
    \left\|\nabla_{W^{(q)}} f\right\|_F^2 := \left\|\nabla_{W_\Delta^{(q)}} f\right\|_F^2 + \left\|\nabla_{W_B^{(q)}} f\right\|_F^2 + \left\|\nabla_{W_C^{(q)}} f\right\|_F^2 + \left\|\nabla_{A^{(q)}} f\right\|_F^2.
\end{equation}

As in the single-layer proof, parameter derivatives with respect to $W^{(q)}$ are understood as flattened block Jacobians, with rows representing the output coordinates and columns corresponding to the scalar parameters of layer $q$. We use the same vectorization and scalar-parameter-slice conventions for products and transposes. 

For this multi-layer formulation, we introduce the following new assumption and definition, that extend the single-layer analysis assumptions to the multi-layer setting:

\begin{assumption}[Bounded Layer Inputs]\label{assum:bounded_layer_inuts}
For every layer $q\in \{1, \dots, Q\}$ of a multi-layer selective SSM of depth $Q$ parameterized by a fixed set of weights $W=\left(W^{(1)},\dots, W^{(Q)}\right)$, there exists a constant $L_{x,q} > 0$ such that:
\begin{equation}
    \left\|x_t^{(q)}\left(W;\xi\right)\right\|_2 \leq L_{x,q} \text{ } \forall t\in \{1, \dots, S\}.
\end{equation}
This holds for every finite integer length $S \geq 1$, every admissible input-target pair $\xi \in \mathcal Z$, and every token index $t \in \{1, \dots, S\}$, with finite constants independent of $S$, $\xi$, and $t$.
\end{assumption}

\begin{definition}[Multi-layer SSM Stability]\label{def:multilayer_ssm_stability}
For each layer $q \in \left\{1, \dots, Q\right\}$, let $\gamma_q > 0$ and $\delta_q > 0$. We define ${\mathcal S}_{\boldsymbol{\gamma}, \boldsymbol{\delta}}^{(Q)}$ as follows:
\begin{align}
   {\mathcal S}_{\boldsymbol{\gamma}, \boldsymbol{\delta}}^{(Q)} := \bigg\{ &W \bigg | A_n^{{(q)}} \leq -\gamma_q, \left[\operatorname{Softplus}\left(W_{\Delta}^{(q)}x_{t}^{(q)}\left(W;\xi\right)\right)\right]_d \geq \delta_q \nonumber \\ &\forall q\in\{1,\dots,Q\}, \forall n\in\{1,\dots,N\} ,\forall d\in\{1,\dots,D\}, \forall t\in\{1,\dots,S\}, \nonumber \\ &\text{finite integer } S \geq 1, \text{and every admissible length}-S \text{ pair } \xi \in \mathcal Z \bigg\},\label{eq:S_Q_gamma_delta}
\end{align}
where $\boldsymbol{\gamma} = \left(\gamma_1, \dots, \gamma_Q\right)$ and $\boldsymbol{\delta} = \left(\delta_1, \dots, \delta_Q\right)$

We then say that the full parameter tuple $W$ is stable with these margins if $W \in{\mathcal S}_{\boldsymbol{\gamma}, \boldsymbol{\delta}}^{(Q)}$.

For each layer $q$, we further define:
\begin{equation}\label{eq:rho_q}
    \rho_q := \exp(-\gamma_q\delta_q) < 1.
\end{equation}
\end{definition}

As a result of Definition~\ref{def:multilayer_ssm_stability}, for a layer $q$, it follows that:
\begin{equation}\label{eq:rho_q_relationship_A_max}
    \left\|\bar A_t^{(q)}\right\|_{\max} \leq \rho_q.
\end{equation}

Given such setup, we state the following theorem:

\begin{theorem}[Multi-Layer SSM Gradient Bound] \label{thm:ssm_multi_layer_grad_bound}
Let $W = \left(W^{(1)}, \dots, W^{(Q)}\right)$ be a full parameter tuple satisfying Assumptions~\ref{assum:lipschitz_cont_loss} and~\ref{assum:bounded_layer_inuts}, with $W \in{\mathcal S}_{\boldsymbol{\gamma}, \boldsymbol{\delta}}^{(Q)}$, that parametrizes a multi-layer selective SSM model of finite integer depth $Q \geq 1$. Then, for such $W$, for every finite integer $S \geq 1$ and every admissible input-target pair $\xi \in \mathcal Z$ of length $S$, the following bound holds:
\begin{equation}
    \textstyle \left\|\nabla_W \mathcal L^{S,(Q)}(W;\xi)\right\|_F \leq  L_\ell \left  (\sum_{q=1}^Q L_{x,q}^2\left(H_{q} + \left\|W_C^{(q)}\right\|_2J_{q}\right)^2 \prod _{q'=q+1}^Q \kappa_{q'}^2\right)^{1/2} := \mathcal G_{\text{SSM}}^{(Q)}(W)
\end{equation}
with
\begin{align}
    K_{\Delta, q} &:= \left\|W_\Delta^{(q)}\right\|_2L_{x, q} + \sqrt{D}\log(2), 
    ~~~~ K_{B,q} := \left\|W_B^{(q)}\right\|_2 L_{x,q}^2, \\ 
    K_{W,q} &:= ND K_{B,q}K_{\Delta,q} \left(K_{\Delta,q} + \sqrt{D} \left\|A^{(q)}\right\|_2L_{x,q}\right)
    \left(\frac{\rho_q}{1-\rho_q}\right) + L_{x,q}^2 \sqrt{N} K_{\Delta,q} + L_{x,q} \sqrt{D} K_{B,q},\\
    H_q &:= \frac{K_{B,q} K_{\Delta,q}
    \sqrt{ND}}{1-\rho_q} , ~~~~ J_{q}:= \frac{K_{W,q}}{1-\rho_q}, ~~~~ \kappa_{q}:=\left\|W_C^{(q)}\right\|_2\left(H_{{q}}+ \frac{L_{x,q}\Lambda_{x,q}}{1-\rho_{q}}\right), \\ 
    \Lambda_{x,q} &:= \rho_{q} \left\| A^{(q)}\right\|_2 \left\|W_\Delta^{(q)}\right\|_2 H_{{q}} + \left\|W_B^{(q)}\right\|_2  L_{x,q} \left(2K_{\Delta,q} + L_{x,q}\left\|W_\Delta^{(q)}\right\|_2\right).
\end{align}
\end{theorem}

\textbf{Proof of Theorem~\ref{thm:ssm_multi_layer_grad_bound}}: Fix an arbitrary full parameter tuple $W$ satisfying Assumptions~\ref{assum:lipschitz_cont_loss} and~\ref{assum:bounded_layer_inuts} such that $W \in{\mathcal S}_{\boldsymbol{\gamma}, \boldsymbol{\delta}}^{(Q)}$. Let $\xi \in \mathcal Z$ be some fixed admissible sequence pair of length $S$ ($S \geq 1$ is some fixed finite integer). 

By the definition of a gradient with respect to a tuple in equation~\eqref{eq:def_func_gradient}, it follows that:
\begin{equation}\label{eq:desired_bound_on_seq_S}
    \left\|\nabla_W \mathcal L^{S, (Q)}(W)\right\|_F = \left(\sum_{q=1}^Q \left\|\nabla_{W^{(q)}} \mathcal L^{S, (Q)}(W)\right\|_F^2\right)^{1/2}
\end{equation}

As a result, to derive a bound on the gradient norm with respect to the full model parameter tuple $W$, it is sufficient to derive the bound on:
\begin{equation}
    \left\|\nabla_{W^{(q)}} \mathcal L^{S, (Q)}(W)\right\|_F^2
\end{equation}
for an arbitrary layer $q\in\{1,\dots,Q\}$, and combine the resulting layer-wise bounds. 

This gradient can be expanded as:
\begin{equation}
    \nabla_{W^{(q)}} \mathcal L^{S, (Q)}(W) = \nabla_{W^{(q)}}\left(\frac{1}{S}\sum_{t=1}^S \ell_t\left(y_t^{(Q)}\right)\right) = \frac{1}{S}\sum_{t=1}^S\nabla_{W^{(q)}} \ell_t\left(y_{t}^{(Q)}\right)
\end{equation}

Applying the chain rule again, we get: 
\begin{equation}\label{eq:multi_layer_loss_gradient_expression}
    \operatorname{vec}\left(\nabla_{W^{(q)}}\mathcal L^{S,(Q)}(W)\right) =\frac{1}{S}\sum_{t=1}^S\left(\frac{\partial y_t^{(Q)}}{\partial W^{(q)}}\right)^\top\nabla_y\ell_t\left(y_t^{(Q)}\right).
\end{equation}

For $Q\geq2$, we first consider the case $q=Q-1$ and focus on $\partial y_t^{(Q)}/\partial W^{(q)}$, since $\nabla_y\ell_t\left(y_t^{(Q)}\right)$ is independent of $q$. By the product rule and using the SSM dynamics, it follows that:
\begin{equation}\label{eq:multi_layer_dyQ_dWq_less_1}
    \frac{\partial y_t^{(Q)}}{\partial W^{(Q-1)}} = \frac{\partial \left(h_t^{(Q)}\right)^\top}{\partial W^{{(Q-1)}}}C_t^{(Q)}+\left(h_t^{(Q)}\right)^\top \frac{\partial C_t^{(Q)}}{\partial W^{(Q-1)}}
\end{equation}

Next, we evaluate~\eqref{eq:multi_layer_dyQ_dWq_less_1} coordinate-wise, as this avoids any ambiguity associated with calculating the matrix-valued derivatives with respect to the parameter tuples. Let $w^{(Q-1)}$ denote an arbitrary scalar coordinate of $W^{(Q-1)}$, and let $d \in \{1, \dots, D\}$ denote an arbitrary coordinate of $y^{(Q)}_t$. Then:
\begin{equation}
    y_{t,d}^{(Q)} = \sum_{n=1}^N h_{t,n,d}^{(Q)} C_{t,n}^{(Q)},
\end{equation}
and by the product rule:
\begin{equation}\label{eq:multi_layer_dyQ_dWq_less_1_coordinate}
    \frac{\partial y_{t,d}^{(Q)}}{\partial w^{(Q-1)}} = \sum_{n=1}^N \left[C_{t,n}^{(Q)} \frac{\partial h_{t,n,d}^{(Q)}}{\partial w^{(Q-1)}} + h_{t,n,d}^{(Q)}\frac{\partial C_{t,n}^{(Q)}}{\partial w^{(Q-1)}}\right]
\end{equation}

Next, since the weight $w^{(Q-1)}$ is not a parameter of layer $Q$, it influences $h_t^{(Q)}$ at any coordinate only through the inputs to layer $Q$ (i.e. $h_t^{(Q)}$ is a function of all  $x^{(Q)}_1, \dots, x^{(Q)}_j, \dots, x^{(Q)}_t$, which depend on the parameters of the previous layer through $x^{(Q)}_j=y^{(Q-1)}_j$). As a result:
\begin{equation}
    \frac{\partial h_{t,n,d}^{(Q)}}{\partial w^{(Q-1)}} = \sum_{j=1}^t\sum_{d'=1}^D \frac{\partial h_{t,n,d}^{(Q)}}{\partial x_{j,d'}^{(Q)}} \frac{\partial x_{j,d'}^{(Q)}}{\partial w^{(Q-1)}}.
\end{equation}

Furthermore, since $C_{t,n}^{(Q)} = \sum_{d'=1}^D \left(W_C^{(Q)}\right)_{n, d'} x_{t, d'}^{(Q)}$, and $W_C^{(Q)}$ is independent of any parameter of $W^{(Q-1)}$, it follows that:
\begin{equation}
    \frac{\partial C_{t,n}^{(Q)}}{\partial w^{(Q-1)}} = \sum_{d'=1}^D \left(W_C^{(Q)}\right)_{n, d'} \frac{\partial x_{t, d'}^{(Q)}}{\partial w^{(Q-1)}}
\end{equation}

Substituting these expressions into~\eqref{eq:multi_layer_dyQ_dWq_less_1_coordinate} results in:
\begin{align}
    \frac{\partial y_{t,d}^{(Q)}}{\partial w^{(Q-1)}} &= \sum_{n=1}^N \left[C_{t,n}^{(Q)}\sum_{j=1}^t\sum_{d'=1}^D \frac{\partial h_{t,n,d}^{(Q)}}{\partial x_{j,d'}^{(Q)}} \frac{\partial x_{j,d'}^{(Q)}}{\partial w^{(Q-1)}} + h_{t,n,d}^{(Q)}\sum_{d'=1}^D \left(W_C^{(Q)}\right)_{n, d'} \frac{\partial x_{t, d'}^{(Q)}}{\partial w^{(Q-1)}}\right] \nonumber \\ &= \sum_{n=1}^N \left[C_{t,n}^{(Q)}\sum_{j=1}^t\sum_{d'=1}^D \frac{\partial h_{t,n,d}^{(Q)}}{\partial x_{j,d'}^{(Q)}} \frac{\partial x_{j,d'}^{(Q)}}{\partial w^{(Q-1)}} + h_{t,n,d}^{(Q)}\sum_{j=1}^t \sum_{d'=1}^D \mathbbm{1}_{j=t} \left(W_C^{(Q)}\right)_{n, d'} \frac{\partial x_{j, d'}^{(Q)}}{\partial w^{(Q-1)}}\right] \nonumber \\ &= \sum_{j=1}^t \sum_{d'=1}^D\left[\sum_{n=1}^N \left( C_{t,n}^{(Q)}\frac{\partial h_{t,n,d}^{(Q)}}{\partial x_{j,d'}^{(Q)}} + \mathbbm{1}_{j=t} h_{t, n, d}^{(Q)} \left(W_C^{(Q)}\right)_{n, d'}\right) \right]\frac{\partial x_{j, d'}^{(Q)}}{\partial w^{(Q-1)}}
\end{align}
where $\mathbbm{1}_{j=t}$ is $1$ if $j =t$ and $0$ otherwise. 

Next, we note that:
\begin{align}
    \frac{\partial y_{t, d}^{(Q)}}{ \partial x_{j, d'}^{(Q)}} &=   \frac{\partial \left(\sum_{n=1}^N h_{t, n, d}^{(Q)} C_{t, n}^{(Q)}\right)}{ \partial x_{j, d'}^{(Q)}} =  \sum_{n=1}^N \frac{\partial \left(h_{t, n, d}^{(Q)} C_{t, n}^{(Q)}\right)}{\partial x_{j, d'}^{(Q)}} = \sum_{n=1}^N \left[C_{t, n}^{(Q)} \frac{\partial h_{t, n, d}^{(Q)}}{\partial x_{j, d'}^{(Q)}} + h_{t, n, d}^{(Q)} \frac{\partial C_{t, n}^{(Q)}}{\partial x_{j, d'}^{(Q)}}\right] \nonumber \\ &= \sum_{n=1}^N \left[C_{t, n}^{(Q)} \frac{\partial h_{t, n, d}^{(Q)}}{\partial x_{j, d'}^{(Q)}} + h_{t, n, d}^{(Q)} \frac{\partial \left(\sum_{d''=1}^D \left(W_C^{(Q)}\right)_{n, d''} x_{t, d''}^{(Q)}\right)}{\partial x_{j, d'}^{(Q)}}\right] \nonumber \\ &= \sum_{n=1}^N \left( C_{t,n}^{(Q)}\frac{\partial h_{t,n,d}^{(Q)}}{\partial x_{j,d'}^{(Q)}} + \mathbbm{1}_{j=t} h_{t, n, d}^{(Q)} \left(W_C^{(Q)}\right)_{n, d'}\right) 
\end{align}
where we use the fact that if $j \neq t$, 
\begin{equation}
    \frac{\partial \left(\sum_{d''=1}^D \left(W_C^{(Q)}\right)_{n, d''} x_{t, d''}^{(Q)}\right)}{\partial x_{j, d'}^{(Q)}} = 0
\end{equation}
and if $j=t$, 
\begin{equation}
    \frac{\partial \left(\sum_{d''=1}^D \left(W_C^{(Q)}\right)_{n, d''} x_{j, d''}^{(Q)}\right)}{\partial x_{j, d'}^{(Q)}} =  \left(W_C^{(Q)}\right)_{n, d'}
\end{equation}

As a result, it follows that:
\begin{equation}
    \frac{\partial y_{t,d}^{(Q)}}{\partial w^{(Q-1)}} = \sum_{j=1}^t \sum_{d'=1}^D\frac{\partial y_{t, d}^{(Q)}}{ \partial x_{j, d'}^{(Q)}} \frac{\partial x_{j, d'}^{(Q)}}{\partial w^{(Q-1)}}
\end{equation}
and finally, since $x_{j, d'}^{(Q)} = y_{j, d'}^{(Q-1)}$, 
\begin{equation}\label{eq:recursion_multi_layer_coordinate_less_1}
    \frac{\partial y_{t,d}^{(Q)}}{\partial w^{(Q-1)}} = \sum_{j=1}^t \sum_{d'=1}^D\frac{\partial y_{t, d}^{(Q)}}{ \partial x_{j, d'}^{(Q)}} \frac{\partial y_{j, d'}^{(Q-1)}}{\partial w^{(Q-1)}}
\end{equation}

For $Q \geq 3$, the same coordinate-wise chain rule argument applies to $w^{(Q-2)}$. That is, since $w^{(Q-2)}$ is not a parameter of layer $Q$, it influences $y^{(Q)}_t$ only through the inputs sequence $x_1^{(Q)}, \dots, x_t^{(Q)}$. As such, it follows that:
\begin{equation}\label{eq:recursion_multi_layer_coordinate_less_2}
    \frac{\partial y_{t,d}^{(Q)}}{\partial w^{(Q-2)}} = \sum_{j=1}^t \sum_{d'=1}^D\frac{\partial y_{t, d}^{(Q)}}{ \partial x_{j, d'}^{(Q)}} \frac{\partial y_{j, d'}^{(Q-1)}}{\partial w^{(Q-2)}}
\end{equation}

Furthermore, using the same argument applied to layer $Q-1$, we can say that:
\begin{equation}
    \frac{\partial y_{j, d'}^{(Q-1)}}{\partial w^{(Q-2)}} = \sum_{j'=1}^j \sum_{d''=1}^D\frac{\partial y_{j, d'}^{(Q-1)}}{ \partial x_{j', d''}^{(Q-1)}} \frac{\partial y_{j', d''}^{(Q-2)}}{\partial w^{(Q-2)}}.
\end{equation}

Substituting this result into~\eqref{eq:recursion_multi_layer_coordinate_less_2}, we obtain:
\begin{equation}
    \frac{\partial y_{t,d}^{(Q)}}{\partial w^{(Q-2)}} = \sum_{j=1}^t \sum_{d'=1}^D\frac{\partial y_{t, d}^{(Q)}}{ \partial x_{j, d'}^{(Q)}} \sum_{j'=1}^j \sum_{d''=1}^D\frac{\partial y_{j, d'}^{(Q-1)}}{ \partial x_{j', d''}^{(Q-1)}} \frac{\partial y_{j', d''}^{(Q-2)}}{\partial w^{(Q-2)}} = \sum_{j=1}^t\sum_{j'=1}^{j} \sum_{d'=1}^D\sum_{d''=1}^D \frac{\partial y_{t,d}^{(Q)}}{\partial x_{j,d'}^{(Q)}} \frac{\partial y_{j,d'}^{(Q-1)}}{\partial x_{j',d''}^{(Q-1)}} \frac{\partial y_{j',d''}^{(Q-2)}}{\partial w^{(Q-2)}}
\end{equation}

This result means that, compared to the case of $w^{(Q-1)}$, the derivative with respect to $w^{(Q-2)}$ contains one additional input-output derivative, which corresponds to its propagation through layer $Q-1$. Applying this argument recursively, we derive the final derivative of the $d$-th coordinate of the final output with respect to some parameter $w^{(q)}$ of the layer $q\in\{1, \dots, Q-1\}$ as:
\begin{equation}\label{eq:multi_layer_derived_recurrence_coordinate}
    \frac{\partial y_{t,d}^{(Q)}}{\partial w^{(q)}} = \sum_{j_{Q-1}=1}^{j_Q}\cdots\sum_{j_q=1}^{j_{q+1}}\sum_{d_{Q-1}=1}^D\cdots\sum_{d_q=1}^D\left[\prod_{q'=q+1}^{Q}\frac{\partial y_{j_{q'},d_{q'}}^{({q'})}}{\partial x_{j_{{q'}-1},d_{{q'}-1}}^{(q')}}\right] \frac{\partial y_{j_q,d_q}^{(q)}}{\partial w^{(q)}},
\end{equation}
where we also have $j_Q=t$, $d_Q=d$, and $j_q\leq j_{q+1}\leq\cdots\leq j_{Q-1}\leq j_Q=t$.

We next collect the coordinate derivatives in~\eqref{eq:multi_layer_derived_recurrence_coordinate} into matrices, which transforms sums over feature coordinates into matrix products, while the sums over time indices remain explicit. 

For each layer $q$, define the input-output Jacobian matrix as:
\begin{equation}
    \mathbf{K}_{t,j}^{(q)} := \frac{\partial y_t^{(q)}}{\partial x_j^{(q)}} = \begin{pmatrix}
        \dfrac{\partial y_{t,1}^{(q)}}{\partial x_{j,1}^{(q)}} & \cdots & \dfrac{\partial y_{t,1}^{(q)}}{\partial x_{j,D}^{(q)}} \\ \vdots & \ddots & \vdots \\ \dfrac{\partial y_{t,D}^{(q)}}{\partial x_{j,1}^{(q)}} & \dots & \dfrac{\partial y_{t,D}^{(q)}}{\partial x_{j,D}^{(q)}} 
    \end{pmatrix} \in \mathbb R^{D \times D}.
\end{equation}
The Jacobian is computed with the parameters of layer $q$ and all input tokens other than $x_j^{(q)}$ held fixed, and evaluated at the input sequence received by that layer. Additionally, by causality, $\mathbf{K}_{t,j}^{(q)} = \mathbf{0}_{D\times D}$ for $j > t$.

Then, for any layer $q\in \{1, \dots, Q-1\}$, for any subsequent layer $q' \in \{q+1, \dots, Q\}$, the parameters $W^{(q)}$ affect $y^{(q')}$ only through the input sequence to layer $q'$. Since $x_j^{(q')} = y_j^{(q'-1)}$, the chain rule can be written in the matrix form as:
\begin{equation}\label{eq:multilayer_jacobian_recursion}
     \frac{\partial y_t^{(q')}}{\partial W^{(q)}} = \sum_{j=1}^t \frac{\partial y_t^{(q')}}{\partial x_j^{(q')}}\frac{\partial x_j^{(q')}}{\partial W^{(q)}} = \sum_{j=1}^t \mathbf K_{t,j}^{(q')}\frac{\partial y_j^{(q'-1)}}{\partial W^{(q)}}.
\end{equation}

Applying this recursively, we get:
\begin{equation}\label{eq:multi_layer_derived_recurrence_matrix}
    \frac{\partial y_t^{(Q)}}{\partial W^{(q)}} = \sum_{j_{Q-1}=1}^{j_Q}\cdots\sum_{j_q=1}^{j_{q+1}} \left[\mathbf{K}_{j_Q, j_{Q-1}}^{(Q)} \times  \cdots \times \mathbf{K}_{j_{q+1}, j_{q}}^{(q+1)} \right] \frac{\partial y_{j_q}^{(q)}}{\partial W^{(q)}}
\end{equation}
for any $q\in \{1, \dots, Q-1\}$, the matrix product contains only $\mathbf{K}_{t, j_{Q-1}}^{(Q)}$ for the case $q = Q-1$, the $\times$ operator denotes matrix multiplication, and $j_Q=t$. 

Note that equation~\eqref{eq:multi_layer_derived_recurrence_matrix} provides an explicit expression for the parameter Jacobian $\frac{\partial y_t^{(Q)}}{\partial W^{(q)}}$ appearing in~\eqref{eq:multi_layer_loss_gradient_expression}, in terms of the own-layer parameter Jacobians $\frac{\partial y_{j_q}^{(q)}}{\partial W^{(q)}}$ and the input-output Jacobians of the subsequent layers. 

We next use this decomposition to bound the norm of the parameter Jacobian. We use the Frobenius norm $\|\cdot\|_F$ for the flattened parameter Jacobians and the spectral norm $\|\cdot\|_2$ (where for $M \in \mathbb R^{D \times D}$, $\|M\|_2 := \sup_{\substack{v \in \mathbb R^D \\ \|v\|_2=1}}\|Mv\|_2$) for the input-output matrices. Applying the triangle inequality and the fact that $\|AB\|_F \leq \|A\|_2 \|B\|_F$, it follows that:
\begin{align}
    \left\|\frac{\partial y_t^{(Q)}}{\partial W^{(q)}}\right\|_F &= \left\|\sum_{j_{Q-1}=1}^{j_Q}\cdots\sum_{j_q=1}^{j_{q+1}} \left[\mathbf{K}_{j_Q, j_{Q-1}}^{(Q)} \times  \cdots \times \mathbf{K}_{j_{q+1}, j_{q}}^{(q+1)} \right] \frac{\partial y_{j_q}^{(q)}}{\partial W^{(q)}}\right\|_F \nonumber \\ &\leq \sum_{j_{Q-1}=1}^{j_Q}\cdots\sum_{j_q=1}^{j_{q+1}}\left\|\left[\mathbf{K}_{j_Q, j_{Q-1}}^{(Q)} \times  \cdots \times \mathbf{K}_{j_{q+1}, j_{q}}^{(q+1)} \right] \frac{\partial y_{j_q}^{(q)}}{\partial W^{(q)}}\right\|_F \nonumber \\
    & \leq \sum_{j_{Q-1}=1}^{j_Q}\cdots\sum_{j_q=1}^{j_{q+1}}\left[ \prod_{q'=q+1}^{Q} \left\|\mathbf{K}_{j_{q'}, j_{q'-1}}^{(q')}\right\|_2\right]\left\|\frac{\partial y_{j_q}^{(q)}}{\partial W^{(q)}}\right\|_F, \label{eq:output_jacobian_norm_bound}
\end{align}
where $j_Q=t$ and $q\in \{1, \dots, Q-1\}$. 

Next, we bound the terms of this expression separately. Thanks to Assumption~\ref{assum:bounded_layer_inuts} and Definition~\ref{def:multilayer_ssm_stability}, we can apply the single-layer gradient proof argument from Appendix~\ref{appendix:gradient_norm_bound_proof} to each layer $q$, bounding the term $\left\|\frac{\partial y_{j_q}^{(q)}}{\partial W^{(q)}}\right\|_F$.

First, for each $q\in \{1, \dots, Q\}$, let:
\begin{align}
    K_{\Delta, q} &:= \left\|W_\Delta^{(q)}\right\|_2L_{x, q} + \sqrt{D}\log(2), ~~~~
    K_{B,q} := \left\|W_B^{(q)}\right\|_2 L_{x,q}^2, \\ K_{W,q} &:= ND K_{B,q}K_{\Delta,q} \left(K_{\Delta,q} + \sqrt{D} \left\|A^{(q)}\right\|_2L_{x,q}\right)
    \left(\frac{\rho_q}{1-\rho_q}\right) + L_{x,q}^2 \sqrt{N} K_{\Delta,q} + L_{x,q} \sqrt{D} K_{B,q}.
\end{align}

Given these definitions, we apply the single-layer argument from Appendix~\ref{appendix:gradient_norm_bound_proof} for every layer $q \in \{1, \dots, Q\}$. Since the input to any layer $q$ depends only on the parameters of the preceding layers, it remains fixed when differentiating with respect to $W^{(q)}$. It therefore follows that:
\begin{align}
    \left\|h_j^{(q)}\right\|_F &\leq \frac{K_{B,q} K_{\Delta,q}
    \sqrt{ND}}{1-\rho_q} := H_{q}, \label{eq:multi_layer_state_bound} \\ \left\|\frac{\partial h_j^{(q)}}{\partial W^{(q)}}\right\|_F &\leq \frac{K_{W,q}}{1-\rho_q} := J_{q}, \label{eq:multi_layer_state_jacobian_bound}
\end{align}
for every $q \in \{1, \dots, Q\}$ and $j \in \{1, \dots, S\}$. Here, the state Jacobian is flattened as in the single-layer analysis. 

Then, by definition of $C_j^{(q)}$, Assumption~\ref{assum:bounded_layer_inuts} gives:
\begin{equation}
    \left\|C_j^{(q)}\right\|_2 = \left\|W_C^{(q)} x_j^{(q)}\right\|_2 \leq \left\|W_C^{(q)}\right\|_2 L_{x,q}.
\end{equation}

Finally, by applying the same product-rule argument to $y_j^{(q)} = \left(h_j^{(q)}\right)^\top C_j^{(q)}$ as in Appendix~\ref{appendix:gradient_norm_bound_proof}, and using~\eqref{eq:multi_layer_state_bound} and~\eqref{eq:multi_layer_state_jacobian_bound}, we arrive at:
\begin{align}
    \left\|\frac{\partial y_{j}^{(q)}}{\partial W^{(q)}}\right\|_F & \leq \left\|C_j^{(q)}\right\|_2 \left \|\frac{\partial h_j^{(q)}}{\partial W^{(q)}}\right\|_F + \left\|h_j^{(q)}\right\|_F \left\|x_j^{(q)}\right\|_2 \nonumber \\ &\leq\left\|W_C^{(q)}\right\|_2 L_{x,q}J_{q} + H_{q} L_{x,q} = L_{x,q} \left(H_{q} + \left\|W_C^{(q)}\right\|_2J_{q}\right) \label{eq:multi_layer_own_output_jacobian_bound}.
\end{align}
This bound holds for every $q \in \{1, \dots, Q\}$ and $j \in \{1, \dots, S\}$ uniformly over all admissible input sequences. 

Substituting~\eqref{eq:multi_layer_own_output_jacobian_bound} into~\eqref{eq:output_jacobian_norm_bound}, we get:
\begin{equation}\label{eq:output_jacobian_norm_bound_with_single_layer}
    \left\|\frac{\partial y_t^{(Q)}}{\partial W^{(q)}}\right\|_F \leq L_{x,q} \left(H_{q} + \left\|W_C^{(q)}\right\|_2 J_{q}\right) \left(\sum_{j_{Q-1}=1}^{j_Q}\cdots\sum_{j_q=1}^{j_{q+1}}\left[ \prod_{q'=q+1}^{Q} \left\|\mathbf{K}_{j_{q'}, j_{q'-1}}^{(q')}\right\|_2\right] \right)
\end{equation}

Next, we derive a bound on the term:
\begin{equation}\label{eq:backprop_term_additon}
    \sum_{j_{Q-1}=1}^{j_Q}\cdots\sum_{j_q=1}^{j_{q+1}}\left[ \prod_{q'=q+1}^{Q} \left\|\mathbf{K}_{j_{q'}, j_{q'-1}}^{(q')}\right\|_2\right] 
\end{equation}

To derive the bound, for an arbitrary layer $q' \in \{q+1, \dots, Q\}$, we begin by expressing $\mathbf{K}_{t, j}^{(q')}$ in terms of the layer's hidden-state sensitivity. Then, by definition:
\begin{equation}
    \left(\mathbf{K}_{t,j}^{(q')}\right)_{d,d'} = \frac{\partial y_{t,d}^{(q')}}{\partial x_{j,d'}^{(q')}} = \sum_{n=1}^N C_{t,n}^{(q')} \frac{\partial h_{t,n,d}^{(q')}}{\partial x_{j,d'}^{(q')}} + \mathbbm{1}_{j=t} \sum_{n=1}^N h_{t,n,d}^{(q')}\left(W_C^{(q')}\right)_{n,d'}, \label{eq:multi_layer_K_matrix_coordinate_expression}
\end{equation}
for $1 \leq j \leq t \leq S$ and $d, d' \in \{1, \dots, D\}$.

Next, by definition of the spectral norm, we know that:
\begin{equation}\label{eq:multi_layer_K_matrix_spectral_def}
    \left\|\mathbf{K}_{t,j}^{(q')}\right\|_2 = \sup_{\substack{u \in \mathbb R^D \\ \|u\|_2 = 1}} \left\|\mathbf{K}_{t,j}^{(q')}u\right\|_2.
\end{equation}

As such, we first evaluate the value of $\mathbf{K}_{t,j}^{(q')}u$ for some arbitrary unit vector $u \in \mathbb R^D$. Multiplying~\eqref{eq:multi_layer_K_matrix_coordinate_expression} by $u_{d'}$, summing over $d'$, and collecting the output coordinates, we get:
\begin{equation} 
    \mathbf{K}_{t,j}^{(q')}u = \left(\sum_{d'=1}^D u_{d'}\frac{\partial h_t^{(q')}}{\partial x_{j,d'}^{(q')}}\right)^\top C_t^{(q')} + \mathbbm{1}_{j=t} \left(h_t^{(q')}\right)^\top W_C^{(q')}u.
\end{equation}

Note that, since $u$ is a unit vector, the sum in the first term is, by definition, the directional derivative of the hidden state with respect to $x_{j}^{(q')}$ in direction $u$. We denote it as:
\begin{equation}\label{eq:multi_layer_directional_derivative_hidden_state}
    D_{x_j^{(q')}} h_t^{(q')}[u] := \sum_{d'=1}^D u_{d'}\frac{\partial h_t^{(q')}}{\partial x_{j,d'}^{(q')}} \in \mathbb R^{N \times D}.
\end{equation}
As such, it follows that:
\begin{equation}\label{eq:multi_layer_K_matrix_directional_derivative_expression}
    \mathbf{K}_{t,j}^{(q')}u = \left(D_{x_j^{(q')}} h_t^{(q')}[u]\right)^\top C_t^{(q')} + \mathbbm{1}_{j=t} \left(h_t^{(q')}\right)^\top W_C^{(q')}u
\end{equation}

We then bound both terms of the sum in~\eqref{eq:multi_layer_K_matrix_directional_derivative_expression} separately. 

We evaluate the term $\left(D_{x_j^{(q')}} h_t^{(q')}[u]\right)^\top C_t^{(q')}$ first. Recall that by~\eqref{eq:hidden_state_definition_multi_layer}:
\begin{equation*}
    h_t^{(q')} = \bar{A}_t^{(q')} \odot h_{t-1}^{(q')} + \bar{B}_t^{(q')} \operatorname{diag}\left(x_t^{(q')}\right).
\end{equation*}
Since, for fixed layer $q'$ weights, $\bar{A}_t^{(q')}$ and $\bar{B}_t^{(q')}$ depend only on the input $x_t^{(q')}$, their derivatives with respect to $x_j^{(q')}$ vanish when $t\neq j$. Furthermore, $D_{x_j^{(q')}} x_t^{(q')}[u] = \mathbbm{1}_{j=t}u$. As such, applying these results and the product rule, we get:
\begin{align}
    D_{x_j^{(q')}} h_t^{(q')}[u] = \bar{A}_t^{(q')} \odot D_{x_j^{(q')}} h_{t-1}^{(q')}[u] &+ \mathbbm{1}_{j=t}\bigg[\left(D_{x_t^{(q')}} \bar{A}_t^{(q')}[u]\right)\odot h_{t-1}^{(q')} + \nonumber \\ & \left(D_{x_t^{(q')}} \bar{B}_t^{(q')}[u]\right)\operatorname{diag}\left(x_t^{(q')}\right) + \bar{B}_t^{(q')} \operatorname{diag}(u) \bigg],
\end{align}
where we also used the fact that $D_{x_j^{(q')}} \operatorname{diag}\left(x_t^{(q')}\right)[u] = \mathbbm{1}_{j=t}\operatorname{diag}\left(u\right)$. 

Next, we consider three cases. When $0 \leq t < j$, because the hidden state is initialized to $0$ and its evolution is causal, then:
\begin{equation}
    D_{x_j^{(q')}} h_t^{(q')}[u] = \mathbf{0}, \qquad \forall t \in \{0, \dots, j-1\}.
\end{equation}

When, $t=j$, $D_{x_j^{(q')}} h_{t-1}^{(q')}[u] = D_{x_t^{(q')}} h_{t-1}^{(q')}[u] = \mathbf{0}$, so:
\begin{equation}
    D_{x_t^{(q')}} h_{t}^{(q')}[u] = \left(D_{x_t^{(q')}} \bar{A}_t^{(q')}[u]\right)\odot h_{t-1}^{(q')} + \left(D_{x_t^{(q')}} \bar{B}_t^{(q')}[u]\right)\operatorname{diag}\left(x_t^{(q')}\right) + \bar{B}_t^{(q')} \operatorname{diag}(u).
\end{equation}

Finally, for $t > j$, the recurrence becomes:
\begin{equation}
    D_{x_j^{(q')}} h_t^{(q')}[u] = \bar{A}_t^{(q')} \odot D_{x_j^{(q')}} h_{t-1}^{(q')}[u].
\end{equation}

Next, by equation~\eqref{eq:rho_q_relationship_A_max}, which follows from definition~\ref{def:multilayer_ssm_stability}, we obtain:
\begin{equation}\label{eq:recursive_definition_of_D_h_d_x}
    \left\|D_{x_j^{(q')}} h_t^{(q')}[u]\right\|_F \leq \rho_{q'}\left\|D_{x_j^{(q')}} h_{t-1}^{(q')}[u]\right\|_F \leq \rho_{q'}^{t-j} \left\|D_{x_j^{(q')}} h_j^{(q')}[u]\right\|_F, \qquad \forall t > j.
\end{equation}

As a result, we just need to bound $\left\|D_{x_j^{(q')}} h_{t}^{(q')}[u] \right\|_F$ for the case when $t=j$, and  apply the recursive relationship. By the triangle inequality, when $t=j$, it follows that:
\begin{align}
    \left\|D_{x_j^{(q')}} h_{j}^{(q')}[u] \right\|_F \leq &\left\|\left(D_{x_j^{(q')}} \bar{A}_j^{(q')}[u]\right)\odot h_{j-1}^{(q')}\right\|_F \nonumber \\ &+ \left\|\left(D_{x_j^{(q')}} \bar{B}_j^{(q')}[u]\right)\operatorname{diag}\left(x_j^{(q')}\right)\right\|_F + \left\|\bar{B}_j^{(q')} \operatorname{diag}(u)\right\|_F
\end{align}

Then, evaluating it term by term
\begin{align}
    D_{x_j^{(q')}} \bar{A}_j^{(q')}[u] &= D_{x_j^{(q')}} \left(\exp\left(A^{(q')}\left(\Delta_j^{(q')}\right)^\top\right)\right)[u] \nonumber \\ &= \exp\left(A^{(q')}\left(\Delta_j^{(q')}\right)^\top\right) \odot \left[A^{(q')}\left(D_{x_j^{(q')}}\Delta_j^{(q')}[u]\right)^\top\right] \nonumber \\ &= \bar{A}_j^{(q')} \odot \left[A^{(q')}\left(D_{x_j^{(q')}}\Delta_j^{(q')}[u]\right)^\top\right].
\end{align}

Using the definition:
\begin{align}
    D_{x_j^{(q')}}\Delta_j^{(q')}[u] = D_{x_j^{(q')}}\left(\operatorname{Softplus}\left(W_\Delta^{(q')}x_j^{(q')}\right) \right)[u] = \operatorname{Softplus}'\left(W_\Delta^{(q')}x_j^{(q')}\right) \odot \left(W_\Delta^{(q')}u\right).
\end{align}

Next, since $\operatorname{Softplus}'(z) = \frac{1}{1+e^{-z}}\in(0,1)$, we get:
\begin{equation}\label{eq:d_delta_dx_bound}
    \left\|D_{x_j^{(q')}}\Delta_j^{(q')}[u]\right\|_2 \leq \left\|W_\Delta^{(q')}\right\|_2 \|u\|_2 = \left\|W_\Delta^{(q')}\right\|_2 .
\end{equation}
Then, the fact that $\|ab^\top\|_F = \|a\|_2\|b\|_2$ (which follows from the definition of the Frobenius norm), the following holds: 
\begin{align}
    \left\|D_{x_j^{(q')}} \bar{A}_j^{(q')}[u] \right\|_F &\leq  \left\|\bar{A}_j^{(q')}\right\|_{\max} \left\| A^{(q')}\left(D_{x_j^{(q')}}\Delta_j^{(q')}[u]\right)^\top\right\|_F \nonumber \\ &=\left\|\bar{A}_j^{(q')}\right\|_{\max} \left\| A^{(q')}\right\|_2\left\|D_{x_j^{(q')}}\Delta_j^{(q')}[u]\right\|_2 \leq \rho_{q'} \left\| A^{(q')}\right\|_2 \left\|W_\Delta^{(q')}\right\|_2,\label{eq:d_A_x_j_result}
\end{align}
Finally, combining equations~\eqref{eq:d_A_x_j_result} and~\eqref{eq:multi_layer_state_bound} yields:
\begin{equation}\label{eq:d_A_x_h_j_result}
    \left\|\left(D_{x_j^{(q')}} \bar{A}_j^{(q')}[u]\right)\odot h_{j-1}^{(q')}\right\|_F \leq \left\|D_{x_j^{(q')}} \bar{A}_j^{(q')}[u]\right\|_F\left\| h_{j-1}^{(q')}\right\|_F \leq \rho_{q'} \left\| A^{(q')}\right\|_2 \left\|W_\Delta^{(q')}\right\|_2 H_{{q'}}.
\end{equation}

For the second term, from the definition of $\bar{B}_{j}^{(q')}$, and using the product rule, we have:
\begin{equation}
    D_{x_j^{(q')}} \bar{B}_j^{(q')}[u] = \left(W_B^{(q')}u\right)\left(\Delta_j^{(q')}\right)^\top + \left(W_B^{(q')}x_j^{(q')}\right) \left(D_{x_j^{(q')}}\Delta_j^{(q')}[u]\right)^\top
\end{equation}
Then, combining the result of~\eqref{eq:d_delta_dx_bound} with 
\begin{equation}
    \left\|\Delta_{j}^{(q')}\right\|_2 \leq  \left\|W_\Delta^{(q')}x_{j}^{(q')}\right\|_2 + \sqrt{D}\log(2) \leq K_{\Delta,q'},
\end{equation}
we get:
\begin{align}
    \left\|D_{x_j^{(q')}} \bar{B}_j^{(q')}[u]\right\|_F & \leq \left\|W_B^{(q')}u\right\|_2\left\|\Delta_j^{(q')}\right\|_2 + \left\|W_B^{(q')}x_j^{(q')}\right\|_2 \left\|D_{x_j^{(q')}}\Delta_j^{(q')}[u]\right\|_2 \nonumber \\ &\leq \left\|W_B^{(q')}\right\|_2 K_{\Delta,q'} + \left\|W_B^{(q')}\right\|_2 L_{x,q'}\left\|W_\Delta^{(q')}\right\|_2.
\end{align}

Finally, since $\left\|\operatorname{diag}\left(v\right)\right\|_2 = \max_{d}\left|v_d\right|= \left\|v\right\|_{\infty} \leq \left\|v\right\|_2$ (and therefore $\left\|\operatorname{diag}\left(x_j^{(q')}\right)\right\|_2 \leq L_{x,q'}$), the second term is bounded as:
\begin{align}
    \left\|\left(D_{x_j^{(q')}} \bar{B}_j^{(q')}[u]\right)\operatorname{diag}\left(x_j^{(q')}\right)\right\|_F  &\leq \left\|D_{x_j^{(q')}} \bar{B}_j^{(q')}[u]\right\|_F\left\|\operatorname{diag}\left(x_j^{(q')}\right)\right\|_2 \nonumber \\ & \leq \left\|W_B^{(q')}\right\|_2  L_{x,q'} \left(K_{\Delta,q'} + L_{x,q'}\left\|W_\Delta^{(q')}\right\|_2\right) \label{eq:d_B_diag_x_result}
\end{align}

For the remaining term, the bound is as follows:
\begin{align}
    \left\|\bar{B}_j^{(q')} \operatorname{diag}(u)\right\|_F &\leq \left\|\bar{B}_j^{(q')}\right\|_F \left\| \operatorname{diag}(u)\right\|_2 \nonumber \\ &\leq \left\|W_B^{(q')}x_j^{(q')}\right\|_2 \left\|\Delta_j^{(q')}\right\|_2 \left\| u\right\|_2 \nonumber \\ &\leq \left\|W_B^{(q')}\right\|_2 L_{x,q'}K_{\Delta,q'} \label{eq:d_b_u__result}
\end{align}

Combining the three bounds, we get:
\begin{align}
    \left\|D_{x_j^{(q')}} h_{j}^{(q')}[u] \right\|_F \leq \rho_{q'} \left\| A^{(q')}\right\|_2 \left\|W_\Delta^{(q')}\right\|_2 H_{{q'}} + \left\|W_B^{(q')}\right\|_2  L_{x,q'} \left(2K_{\Delta,q'} + L_{x,q'}\left\|W_\Delta^{(q')}\right\|_2\right).
\end{align}

For ease of notation, for each layer $q'$, we define:
\begin{equation}\label{eq:Lambda_x_q_prime}
    \Lambda_{x,q'} := \rho_{q'} \left\| A^{(q')}\right\|_2 \left\|W_\Delta^{(q')}\right\|_2 H_{{q'}} + \left\|W_B^{(q')}\right\|_2  L_{x,q'} \left(2K_{\Delta,q'} + L_{x,q'}\left\|W_\Delta^{(q')}\right\|_2\right).
\end{equation}

Using this notation, we have the following bound:
\begin{equation}
    \left\|D_{x_j^{(q')}} h_{j}^{(q')}[u] \right\|_F \leq \Lambda_{x,q'}, \qquad \|u\|_2 = 1.
\end{equation}

Substituting this bound into the recurrence established in~\eqref{eq:recursive_definition_of_D_h_d_x}, we get:
\begin{equation}\label{eq:final_D_ht_dxj_u}
    \left\|D_{x_j^{(q')}} h_t^{(q')}[u]\right\|_F \leq \rho_{q'}^{t-j} \left\|D_{x_j^{(q')}} h_j^{(q')}[u]\right\|_F \leq \Lambda_{x,q'}\rho_{q'}^{t-j}, \qquad 1 \leq j \leq t \leq S.
\end{equation}
given $\|u\|_2 = 1$, and $\left\|D_{x_j^{(q')}} h_t^{(q')}[u]\right\|_F=0$ for $t < j$.

Returning to the expression for $\mathbf{K}_{t,j}^{(q')}u$ in~\eqref{eq:multi_layer_K_matrix_directional_derivative_expression}, we can bound the first term of the sum using the bound derived in~\eqref{eq:final_D_ht_dxj_u} (and the fact that $\left\|Z^\top c\right\|_2 \leq \|Z\|_F\|c\|_2$) as follows:
\begin{equation}
    \left\|\left(D_{x_j^{(q')}} h_t^{(q')}[u]\right)^\top C_t^{(q')}\right\|_2 \leq \left\|D_{x_j^{(q')}} h_t^{(q')}[u]\right\|_F\left\|C_t^{(q')}\right\|_2 \leq \left\|W_C^{(q')}\right\|_2 L_{x,q'}\Lambda_{x,q'}\rho_{q'}^{t-j}.
\end{equation}

For the second term, we can bound this as:
\begin{equation}
    \left\|\mathbbm{1}_{j=t}\left(h_t^{(q')}\right)^\top W_C^{(q')}u\right\|_2 \leq \mathbbm{1}_{j=t} \left\|h_t^{(q')}\right\|_F\left\|W_C^{(q')}\right\|_2\left\|u\right\|_2 = \mathbbm{1}_{j=t} \left\|h_t^{(q')}\right\|_F\left\|W_C^{(q')}\right\|_2
\end{equation}

Using the triangle inequality, and combining the derived bounds on the terms of the sum in~\eqref{eq:multi_layer_K_matrix_directional_derivative_expression}, we get:
\begin{equation}\label{eq:final_ku_2_norm_bound}
    \left\|\mathbf{K}_{t,j}^{(q')}u\right\|_2 \leq \left\|W_C^{(q')}\right\|_2L_{x,q'}\Lambda_{x,q'}\rho_{q'}^{t-j} + \mathbbm{1}_{j=t} \left\|W_C^{(q')}\right\|_2H_{{q'}}.
\end{equation}

Since the right hand side of~\eqref{eq:final_ku_2_norm_bound} is independent of $u$, and therefore holds uniformly for all unit vectors $u$, taking the supremum over all $u$, we get:
\begin{equation}
    \left\|\mathbf{K}_{t,j}^{(q')}\right\|_2 \leq \left\|W_C^{(q')}\right\|_2L_{x,q'}\Lambda_{x,q'}\rho_{q'}^{t-j} + \mathbbm{1}_{j=t} \left\|W_C^{(q')}\right\|_2H_{{q'}}, \qquad 1 \leq j \leq t \leq S.
\end{equation}

Next, for every fixed position $t$, summing $\left\|\mathbf{K}_{t,j}^{(q')}\right\|_2 $ over $j$ yields:
\begin{align}
    \sum_{j=1}^t\left\|\mathbf{K}_{t,j}^{(q')}\right\|_2  &\leq \sum_{j=1}^t\left(\left\|W_C^{(q')}\right\|_2L_{x,q'}\Lambda_{x,q'}\rho_{q'}^{t-j} + \mathbbm{1}_{j=t} \left\|W_C^{(q')}\right\|_2H_{{q'}}\right) \nonumber \\ &= \left\|W_C^{(q')}\right\|_2L_{x,q'}\Lambda_{x,q'}\sum_{j=1}^t\rho_{q'}^{t-j} + \left\|W_C^{(q')}\right\|_2H_{{q'}} \nonumber \\ 
    &= \left\|W_C^{(q')}\right\|_2L_{x,q'}\Lambda_{x,q'}\frac{1-\rho_{q'}^t}{1-\rho_{q'}} + \left\|W_C^{(q')}\right\|_2H_{{q'}} \nonumber \\ &\leq \left\|W_C^{(q')}\right\|_2\left(H_{{q'}}+ \frac{L_{x,q'}\Lambda_{x,q'}}{1-\rho_{q'}}\right) = \kappa_{q'},
\end{align}
where, for each layer $q'$, we define the uniform layerwise input-output sensitivity bound 
\begin{equation}\label{eq:kappa_definition}
    \kappa_{q'}:=\left\|W_C^{(q')}\right\|_2\left(H_{{q'}}+ \frac{L_{x,q'}\Lambda_{x,q'}}{1-\rho_{q'}}\right) 
\end{equation}

Next, since $\kappa_{q'}$ is independent of $t$, it follows that:
\begin{equation}\label{eq:multi_layer_K_temporal_sum_bound_uniform}
    \max_{1 \leq t \leq S}\sum_{j=1}^t\left\|\mathbf{K}_{t,j}^{(q')}\right\|_2 \leq \kappa_{q'}.
\end{equation}
Since $W$ satisfies the stated stability and bounded-input conditions, the bound holds uniformly over all admissible input-target sequences and token indices. 

Now, returning to~\eqref{eq:backprop_term_additon}, consider the two inner-most sums (for $q \leq Q-2$):
\begin{align}
    \sum_{j_{q+1}=1}^{j_{q+2}}\sum_{j_{q}=1}^{j_{q+1}}\left[\prod_{q'=q+1}^Q \left\|\mathbf{K}_{j_{q'},j_{q'-1}}^{(q')}\right\|_2\right] &= \sum_{j_{q+1}=1}^{j_{q+2}}\sum_{j_{q}=1}^{j_{q+1}}\left\|\mathbf{K}_{j_{q+1},j_{q}}^{(q+1)}\right\|_2\left\|\mathbf{K}_{j_{q+2},j_{q+1}}^{(q+2)}\right\|_2\cdots\left\|\mathbf{K}_{j_{Q},j_{Q-1}}^{(Q)}\right\|_2 \nonumber \\ &= \sum_{j_{q+1}=1}^{j_{q+2}}\left[\sum_{j_{q}=1}^{j_{q+1}}\left\|\mathbf{K}_{j_{q+1},j_{q}}^{(q+1)}\right\|_2\left\|\mathbf{K}_{j_{q+2},j_{q+1}}^{(q+2)}\right\|_2\cdots\left\|\mathbf{K}_{j_{Q},j_{Q-1}}^{(Q)}\right\|_2\right] \nonumber \\ &= \sum_{j_{q+1}=1}^{j_{q+2}}\left\|\mathbf{K}_{j_{q+2},j_{q+1}}^{(q+2)}\right\|_2\cdots\left\|\mathbf{K}_{j_{Q},j_{Q-1}}^{(Q)}\right\|_2\left[\sum_{j_{q}=1}^{j_{q+1}}\left\|\mathbf{K}_{j_{q+1},j_{q}}^{(q+1)}\right\|_2\right] \nonumber \\ &\leq  \sum_{j_{q+1}=1}^{j_{q+2}}\left\|\mathbf{K}_{j_{q+2},j_{q+1}}^{(q+2)}\right\|_2\cdots\left\|\mathbf{K}_{j_{Q},j_{Q-1}}^{(Q)}\right\|_2\kappa_{q+1} \nonumber \\ & = \kappa_{q+1}\sum_{j_{q+1}=1}^{j_{q+2}}\left\|\mathbf{K}_{j_{q+2},j_{q+1}}^{(q+2)}\right\|_2\cdots\left\|\mathbf{K}_{j_{Q},j_{Q-1}}^{(Q)}\right\|_2,
\end{align}
where the inequality follows from~\eqref{eq:multi_layer_K_temporal_sum_bound_uniform}. Since $\kappa_{q+1}$ is independent of all remaining sums, and we can apply the same argument successively to each remaining innermost sum, it follows that:
\begin{equation}\label{eq:nested_sum_bound_backprop_contribution}
    \sum_{j_{Q-1}=1}^{j_Q}\cdots\sum_{j_q=1}^{j_{q+1}}\left[ \prod_{q'=q+1}^{Q} \left\|\mathbf{K}_{j_{q'}, j_{q'-1}}^{(q')}\right\|_2\right]  \leq \kappa_{q+1} \kappa_{q+2} \cdots\kappa_{Q} = \prod_{q'=q+1}^Q \kappa_{q'}.
\end{equation}
for the fixed $W$ and every admissible input-target sequence pair. For $q=Q-1$, there is only one sum, and holds directly from definition: $\sum_{j_{Q-1}=1}^t \left\|\mathbf{K}_{t,j_{Q-1}}^{(Q)}\right\|_2 \leq \kappa_Q$.

Finally, combining the nested-sum bound in~\eqref{eq:nested_sum_bound_backprop_contribution} with~\eqref{eq:output_jacobian_norm_bound} and~\eqref{eq:multi_layer_own_output_jacobian_bound}, we obtain:
\begin{equation}\label{eq:multi_layer_final_output_jacobian_bound}
    \left\|\frac{\partial y_t^{(Q)}}{\partial W^{(q)}}\right\|_F \leq L_{x,q}\left(H_{q}+\left\|W_C^{(q)}\right\|_2J_{q}\right)\prod_{q'=q+1}^Q \kappa_{q'}.
\end{equation}

This bound holds for every $q \in \{1, \dots, Q\}$ and $t \in \{1, \dots, S\}$, where an empty product is interpreted as one. For the $q=Q$ case, the inequality follows directly from the own-layer bound in~\eqref{eq:multi_layer_own_output_jacobian_bound}. 

We now return to the loss gradient in~\eqref{eq:multi_layer_loss_gradient_expression}. Under Assumption~\ref{assum:lipschitz_cont_loss}, applied at the fixed full parameter tuple $W$ to final-layer predictions uniformly over admissible pairs and token indices, we have:
\begin{equation}\label{eq:nabla_ell_y_bound}
    \left\|\nabla_y \ell_t \left(y_t^{(Q)}\right)\right\|_2 \leq L_\ell.
\end{equation}

Finally, taking the Euclidean norm of the vectorized gradient in~\eqref{eq:multi_layer_loss_gradient_expression}, which equals its block Frobenius norm, gives:
\begin{align}
    \left\|\nabla_{W^{(q)}}\mathcal L^{S,(Q)}(W)\right\|_F &= \left\|\frac{1}{S}\sum_{t=1}^S\left(\frac{\partial y_t^{(Q)}}{\partial W^{(q)}}\right)^\top \nabla_y\ell_t\left(y_t^{(Q)}\right)\right\|_2 \nonumber \\ &\leq \frac{1}{S}\sum_{t=1}^S \left\|\nabla_y \ell_t \left(y_t^{(Q)}\right)\right\|_2 \left\|\frac{\partial y_t^{(Q)}}{\partial W^{(q)}}\right\|_F \nonumber \\ &\leq  \frac{1}{S}\sum_{t=1}^S L_{\ell} L_{x,q}\left(H_{q}+\left\|W_C^{(q)}\right\|_2J_{q}\right)\prod_{q'=q+1}^Q \kappa_{q'} \nonumber \\ &= L_{\ell} L_{x,q}\left(H_{q}+\left\|W_C^{(q)}\right\|_2J_{q}\right)\prod_{q'=q+1}^Q \kappa_{q'} \label{eq:multi_layer_loss_gradient_block_bound},
\end{align}
where we use~\eqref{eq:nabla_ell_y_bound}, ~\eqref{eq:multi_layer_final_output_jacobian_bound}, the triangle inequality, and the fact that $\|M^\top v\|_2\leq\|M\|_F\|v\|_2$. The bound is independent of the token index $t$.

Substituting the result of~\eqref{eq:multi_layer_loss_gradient_block_bound} back into~\eqref{eq:desired_bound_on_seq_S}, we get:
\begin{equation}
    \left\|\nabla_{W}\mathcal L^{S,(Q)}(W)\right\|_F \leq L_\ell \left  (\sum_{q=1}^Q L_{x,q}^2\left(H_{q} + \left\|W_C^{(q)}\right\|_2J_{q}\right)^2 \prod _{q'=q+1}^Q \kappa_{q'}^2\right)^{1/2} = \mathcal G_{\text{SSM}}^{(Q)}(W).
\end{equation}

Since $S\geq1$ and the admissible length-$S$ pair $\xi\in\mathcal Z$ were arbitrary, and the right-hand side is independent of $S$ and $\xi$, the bound holds uniformly over all such lengths and admissible pairs for the fixed full parameter tuple $W$, proving Theorem~\ref{thm:ssm_multi_layer_grad_bound}. \qed

\subsection{Auxiliary Lemmas for Multi-Layer SSM Federated Convergence}\label{appendix:multilayer_auxiliary_lemmas}

First, we describe the following setup that will be utilized in the remainder of the proofs of the multi-layer SSM federated convergence. 

Fix a reference full-model parameter tuple $W^r = (W^{r,(1)},\dots,W^{r,(Q)})$. For each layer $q\in\{1, \dots, Q\}$ of the $Q$-layer deep SSM defined in~\ref{appendix:multi_layer_selective_ssm}, choose some radius $R_q > 0$ and define:
\begin{equation}\label{eq:multi_layer_neighborhood_definition}
    \mathcal N_{R_q}^{(q)}\left(W^{r,(q)}\right) := \left\{V^{(q)} : \left\|V^{(q)}-W^{r, (q)}\right\|_F \leq R_q\right\}
\end{equation}
as the layerwise parameter neighborhood of radius $R_q$, centered at the parameter tuple of layer $q$ in the central model $W^{r}$.

Furthermore, define:
\begin{align}
    \mathcal N_{\mathbf{R}_{Q}} &:= \mathcal N_{R_1}^{(1)}\left(W^{r, (1)}\right) \times \dots \times \mathcal N_{R_Q}^{(Q)}\left(W^{r, (Q)}\right) \\
    \mathbf{R}_Q &:= \left(R_1, \dots, R_Q\right).
\end{align}

The radii $R_1, \dots, R_Q$ are chosen such that:
\begin{equation}
    \mathcal N_{\mathbf R_Q} \subseteq \mathcal S_{\boldsymbol\gamma,\boldsymbol\delta}^{(Q)}.
\end{equation}
Furthermore, we assume that Assumption~\ref{assum:bounded_layer_inuts} holds uniformly over this neighborhood with the same finite constants $L_{x,q}$. Furthermore, we let Assumption~\ref{assum:lipschitz_cont_loss} hold at the final-layer predictions $y_t^{(Q)}(W;x)$ with the same finite constants $L_\ell$ and $H_\ell$ throughout this neighborhood.

All stability margins, layer-input bounds, and loss-regularity constants are fixed independently of the parameter tuple $W \in \mathcal N_{\mathbf R_Q}$, the admissible sequence pair, its finite length, and the token position. In particular, $\rho_q = \exp(-\gamma_q\delta_q)<1$ is fixed for each layer.

For each $q\in \{1, \dots, Q\}$, define the parameter envelopes as:
\begin{align}
    M_{X, q} &:= \left\|W_X^{r,(q)}\right\|_2 + R_q \qquad X \in\{\Delta, B, C\}, \label{eq:multi_layer_parameter_envelopes}  \\ 
    M_{A,q}&:= \left\|A^{r,(q)}\right\|_2 + R_q \label{eq:multi_layer_A_envelope}.
\end{align}
By the definition of the layerwise neighborhoods, and the triangle inequality, every $W\in \mathcal N_{\mathbf{R}_{Q}}$ satisfies:
\begin{align}\label{eq:multi_layer_parameter_norm_bounds}
    \left\|W_X^{(q)}\right\|_2 & \leq M_{X,q} \qquad X \in\{\Delta, B, C\},  \nonumber \\ 
    \left\|A^{(q)}\right\|_2  &\leq  M_{A,q}
\end{align}

Given these envelopes, let:
\begin{align}
    M_{K_\Delta, q} & := M_{\Delta,q} L_{x,q} + \sqrt{D}\log(2), \label{eq:multi_layer_M_KDelta} \\ 
    M_{K_B,q} &:= M_{B,q}L_{x,q}^2, \label{eq:multi_layer_M_KB} \\
    M_{K_W,q} &:= ND M_{K_{B},q}M_{K_{\Delta},q}\left(M_{K_{\Delta},q} + \sqrt{D} M_{A,q}L_{x,q}\right)
     \left(\frac{\rho_q}{1-\rho_q}\right) \nonumber \\ &\qquad + L_{x,q}^2 \sqrt{N} M_{K_{\Delta},q} + L_{x,q} \sqrt{D} M_{K_B,q}. \label{eq:multi_layer_M_KW}
\end{align}
These are the layerwise counterparts of the single-layer envelopes introduced in Appendix~\ref{appendix:aux_lemmas_ssm_federated_convergence}, with the neighborhood radius $R$ replaced by $R_q$. 

Then, by the definitions of $K_{\Delta, q}$, $K_{B, q}$, and $K_{W,q}$, it follows that:
\begin{equation}\label{eq:multi_layer_K_envelope_bounds}
    K_{\Delta,q}  \leq M_{K_\Delta,q}, \qquad K_{B,q} \leq M_{K_B,q}, \qquad K_{W,q} \leq M_{K_W,q}
\end{equation}
over $W\in \mathcal N_{\mathbf{R}_Q}$.

Further, define:
\begin{align}
    H_{R_q} &:= \frac{M_{K_B,q} M_{K_{\Delta},q}
    \sqrt{ND}}{1-\rho_q}, \label{eq:multi_layer_H_R_definition} \\ 
    J_{R_q} &:= \frac{M_{K_W,q}}{1-\rho_q}. \label{eq:multi_layer_J_R_definition} \\
    M_{\Lambda_x,q} &:= \rho_q M_{A,q}M_{\Delta,q}H_{R_q} + M_{B,q}L_{x,q}\left(2M_{K_\Delta,q} + L_{x,q}M_{\Delta,q}\right), \label{eq:multi_layer_M_Lambda_x} \\
    M_{\kappa,q} &:= M_{C,q}\left(H_{R_q} + \frac{L_{x,q}M_{\Lambda_x,q}}{1-\rho_q}\right).\label{eq:multi_layer_M_kappa}
\end{align}

By the definitions of $\Lambda_{x,q}$ and $\kappa_q$ in the multi-layer gradient analysis, the following bounds hold
\begin{equation}
    \Lambda_{x,q}(W) \leq M_{\Lambda_x,q},
    \qquad
    \kappa_q(W)\leq M_{\kappa,q}
\end{equation}
for every $W \in \mathcal N_{\mathbf R_Q}$. These bounds follow by replacing each parameter norm by its corresponding envelope, since all factors are nonnegative.

\subsubsection{Local multi-layer SSM smoothness Lemma And Proof}
\begin{lemma}[Local multi-layer SSM smoothness]\label{lem:multilayer_local_smoothness}
Under the neighborhood assumptions and parameter envelopes introduced above, there exists a finite constant $L_{\mathrm{SSM},\mathbf R_Q}^{(Q)}$ defined:
\begin{equation}
    L_{\mathrm{SSM},\mathbf R_Q}^{(Q)} := H_\ell\Theta_Q^2+L_\ell\Omega_Q,
\end{equation}
with $\Theta_Q$ and $\Omega_Q$ defined in~\eqref{eq:Theta_Q_def_lemma} and~\eqref{eq:lem_multilayer_second_sensitivity_recursion} respectively, such that, for every admissible sequence pair $\xi$ of finite length $S$ and every $W,V\in\mathcal N_{\mathbf R_Q}$,
\begin{equation}
    \left\|\nabla_W\mathcal L^{S,(Q)}(W;\xi) - \nabla_W\mathcal L^{S,(Q)}(V;\xi)\right\|_F \leq L_{\mathrm{SSM},\mathbf R_Q}^{(Q)}\|W-V\|_F,
\end{equation}
independent of $S$ and $\xi$.

Furthermore, if each client objective $\mathcal L_k$ is finite and differentiable on $\mathcal N_{\mathbf R_Q}$ and satisfies:
\begin{equation}
    \nabla_W\mathcal L_k(W) = \mathbb E_{\xi\sim\mathcal P_k}\left[\nabla_W\mathcal L^{S,(Q)}(W;\xi)\right],
\end{equation}
then every $\mathcal L_k$, and therefore $f=\sum_k p_k\mathcal L_k$, is also $L_{\mathrm{SSM},\mathbf R_Q}^{(Q)}$-smooth on $\mathcal N_{\mathbf R_Q}$.

$\Theta_Q$ and $\Omega_Q$ have the following recursive definitions:
\begin{align}
    \Theta_0 &:= 0, \qquad \Theta_q : =\left(L_{x,q}^2\left(H_{R_q}+M_{C,q}J_{R_q}\right)^2+M_{\kappa,q}^2\Theta_{q-1}^2\right)^{1/2} \label{eq:Theta_Q_def_lemma}, \\
    \Omega_0 &:= 0, \qquad 
    \Omega_q := \frac{M_{C,q}L_{x,q}\mathcal E_q}{1 - \rho_q} + 2\left(L_{x,q}+M_{C,q}\Theta_{q-1}\right)\left(J_{R_q} + \frac{M_{\Lambda_x,q}\Theta_{q-1}}{1-\rho_q}\right) \nonumber \\ &~~~~~~~~~~~~~~~~~~~~~\qquad + 2H_{R_q}\Theta_{q-1} + M_{\kappa,q}\Omega_{q-1}. \label{eq:lem_multilayer_second_sensitivity_recursion}
\end{align}
where:
\begin{align}
    \mathcal E_q :=& \rho_q H_{R_q} \Bigg[\left(M_{K_\Delta,q} + M_{A,q}\left(L_{x,q} + M_{\Delta,q}\Theta_{q-1}\right)\right)^2 + 2\left(L_{x,q} + M_{\Delta,q}\Theta_{q-1}\right) \nonumber \\
    &+ M_{A,q}\left(\frac{1}{4}\left(L_{x,q} + M_{\Delta,q}\Theta_{q-1}\right)^2 + 2\Theta_{q-1}\right)\Bigg] \nonumber \\
    &+ 2\rho_q\left(M_{K_\Delta,q} + M_{A,q}\left(L_{x,q} + M_{\Delta,q}\Theta_{q-1}\right)\right)\left(J_{R_q} + \frac{M_{\Lambda_x,q}\Theta_{q-1}}{1-\rho_q}\right)\nonumber \\
    & + 2\Theta_{q-1}M_{K_\Delta,q}L_{x,q}\nonumber \\
    & + 2\left(L_{x,q} + M_{B,q}\Theta_{q-1}\right)\left[L_{x,q}\left(L_{x,q} + M_{\Delta,q}\Theta_{q-1}\right) + M_{K_\Delta,q}\Theta_{q-1}\right] \nonumber \\
    & + M_{B,q}L_{x,q}\Bigg[L_{x,q}\left(\frac{1}{4}\left(L_{x,q} + M_{\Delta,q}\Theta_{q-1}\right)^2 + 2\Theta_{q-1}\right) + 2\left(L_{x,q} + M_{\Delta,q}\Theta_{q-1}\right)\Theta_{q-1}\Bigg].\label{eq:lem_multilayer_second_state_forcing}
\end{align}
\end{lemma}

\textbf{Proof of Lemma~\ref{lem:multilayer_local_smoothness}:} Let $W \in \mathcal N_{\mathbf R_Q}$ be some fixed arbitrary parameter tuple, and $\xi=(x,y^*)$ be some admissible fixed input-target sequence pair of length $S \geq 1$. Note that, while $x$ and $y^*$ remain fixed throughout the differentiation, the intermediate inputs and outputs $x_t^{(q)}=y_t^{(q-1)}$ depend on the parameters of the preceding layers for $q \geq 2$. 

We introduce:
\begin{equation}
    U = \left(U^{(1)}, \dots, U^{(Q)}\right), \qquad V = \left(V^{(1)}, \dots, V^{(Q)}\right)
\end{equation}
as two arbitrary full-model perturbations such that:
\begin{equation}
    U^{(q)} = \left(U_\Delta^{(q)}, U_B^{(q)}, U_C^{(q)}, U_A^{(q)}\right), \qquad V^{(q)} = \left(V_\Delta^{(q)}, V_B^{(q)}, V_C^{(q)}, V_A^{(q)}\right)
\end{equation}
and:
\begin{equation}
    \|U\|_F^2 := \sum_{q=1}^Q\left\|U^{(q)}\right\|_F^2 = 1, \qquad \|V\|_F^2 := \sum_{q=1}^Q\left\|V^{(q)}\right\|_F^2 = 1.
\end{equation}
Note that each layerwise norm such as $\left\|U^{(q)}\right\|_F$ is defined as the Euclidean combination of the Frobenius norms of its four parameter blocks. As a result, every layerwise direction and every individual parameter-block direction has norm at most one. 

Then, for a twice continuously differentiable function $F$ of the full parameter tuple, we define the first and second directional derivatives by
\begin{align}
    DF(W)[U] &:= \frac{d}{d\epsilon} F(W+\epsilon U) \bigg|_{\epsilon = 0}, \\
    D^2F(W)[U,V] &:= \frac{\partial^2}{\partial\epsilon\partial\zeta}F(W + \epsilon U+\zeta V)\bigg|_{\epsilon = \zeta = 0}.
\end{align}
Unless otherwise specified, all derivatives in the remaining part of this proof are taken with respect to the full parameter tuple $W$.

The first-order chain rule gives
\begin{equation}
    D\left(\ell_t\left(y_t^{(Q)}\right)\right)(W)[U] = \left\langle \nabla_y\ell_t\left(y_t^{(Q)}(W)\right), Dy_t^{(Q)}(W)[U] \right\rangle.
\end{equation}
where we used the notation $\ell_t(y) = \ell_t(y,y_t^*)$. By differentiating this expression in $V$ direction, we get:
\begin{align}
    D^2\left(\ell_t\left(y_t^{(Q)}\right)\right)(W)[U,V] =& \left(Dy_t^{(Q)}(W)[U]\right)^\top\nabla_y^2\ell_t\left(y_t^{(Q)}(W)\right)Dy_t^{(Q)}(W)[V] \nonumber \\ 
    &+ \left\langle\nabla_y\ell_t\left(y_t^{(Q)}(W)\right),D^2y_t^{(Q)}(W)[U,V]\right\rangle.
    \label{eq:multilayer_token_loss_second_derivative}
\end{align}

By the neighborhood-uniform loss-regularity bounds introduced in Assumption~\ref{assum:lipschitz_cont_loss} and the Cauchy-Schwarz inequality, it follows that
\begin{align}
    \Bigg|D^2\left(\ell_t\left(y_t^{(Q)}\right)\right)(W)[U,V]\Bigg| \leq & H_\ell \left\|Dy_t^{(Q)}(W)[U]\right\|_2 \left\|Dy_t^{(Q)}(W)[V]\right\|_2 + L_\ell \left\|D^2y_t^{(Q)}(W)[U,V]\right\|_2. \label{eq:multilayer_token_loss_hessian_reduction}
\end{align}

As such, to bound the second derivative of the loss, it suffices to derive uniform bounds on the first and second directional derivatives of the final-layer output. 

We first establish the first-order output sensitivity bounds needed for the second-order calculation.

For ease of notation, for a full-model direction $U$ we write $U^{(\leq q)}=\left(U^{(1)},\dots,U^{(q)}\right)$ and $U^{(<q)}=(U^{(1)},\dots,U^{(q-1)})$. We further define the following recursion:
\begin{equation}
    \Theta_0 := 0, \qquad \Theta_q : =\left(L_{x,q}^2\left(H_{R_q}+M_{C,q}J_{R_q}\right)^2+M_{\kappa,q}^2\Theta_{q-1}^2\right)^{1/2}
\end{equation}

First, by~\eqref{eq:multi_layer_own_output_jacobian_bound} and the definition of the neighborhood envelopes, we have:
\begin{equation}
    \left\|\frac{\partial y_t^{(q)}}{\partial W^{(q)}}\right\|_F \leq L_{x,q}\left(H_{R_q} + M_{C,q}J_{R_q}\right).
\end{equation}

Then, by~\eqref{eq:multi_layer_K_temporal_sum_bound_uniform}
and the bound $\kappa_q(W) \leq M_{\kappa,q}$, for every
$t \in \{1,\dots,S\}$,
\begin{equation}
    \sum_{j=1}^{t} \left\|\frac{\partial y_t^{(q)}}{\partial x_j^{(q)}}\right\|_2 = \sum_{j=1}^{t}\left\|\mathbf K_{t,j}^{(q)}\right\|_2 \leq \kappa_q(W) \leq  M_{\kappa,q}.
\end{equation}

We next prove by induction on $q$ that
\begin{equation}\label{eq:multilayer_first_sensitivity_bound}
    \left\|Dy_t^{(q)}(W)[U]\right\|_2 \leq \Theta_q\left\|U^{(\leq q)}\right\|_F
\end{equation}
for every $q \in \{1,\dots,Q\}$ and every $t \in \{1,\dots,S\}$.

When $q=1$, the input sequence is independent of $W$, and the output of the first layer depends only on $W^{(1)}$. As such:
\begin{equation}
    Dy_t^{(1)}(W)[U] = \frac{\partial y_t^{(1)}}{\partial W^{(1)}}\operatorname{vec}\left(U^{(1)}\right),
\end{equation}
Using $\|Az\|_2\leq\|A\|_F\|z\|_2$ and $\|\operatorname{vec}(U^{(1)})\|_2=\|U^{(1)}\|_F$ gives
\begin{align}
    \left\|Dy_t^{(1)}(W)[U]\right\|_2 \leq L_{x,1}\left(H_{R_1}+M_{C,1}J_{R_1}\right)\left\|U^{(1)}\right\|_F=\Theta_1\left\|U^{(\leq1)}\right\|_F,
\end{align}
where the equality follows from $\Theta_0=0$ and the definition of $\Theta_1$.

Next, for a fixed $q \geq 2$, Since $U^{(\leq q-1)} = U^{(<q)}$, suppose (the induction hypothesis):
\begin{equation}\label{eq:multilayer_first_induction_hypothesis}
    \left\|Dy_j^{(q-1)}(W)[U]\right\|_2 \leq \Theta_{q-1}\left\|U^{(<q)}\right\|_F, \qquad j\in\{1,\dots,S\}.
\end{equation}

For a fixed $t$, the output $y_t^{(q)}$ depends on the current-layer parameters $W^{(q)}$ and the input tokens $x_1^{(q)}, \dots, x_t^{(q)}$, which depend on the preceding layers through $x_j^{(q)} = y_j^{(q-1)}$. The chain rule therefore gives
\begin{equation}
    Dy_t^{(q)}(W)[U] = \frac{\partial y_t^{(q)}}{\partial W^{(q)}} \operatorname{vec}\left(U^{(q)}\right) + \sum_{j=1}^{t}\frac{\partial y_t^{(q)}}{\partial x_j^{(q)}} Dy_j^{(q-1)}(W)[U].
\end{equation}
The first term accounts for the direct parameter perturbation in layer $q$, while the sum accounts for changes propagated through its input sequence.

Taking norms and applying the established Jacobian bounds and the induction hypothesis yields
\begin{align}
    \left\|Dy_t^{(q)}(W)[U]\right\|_2 \leq& L_{x,q}\left(H_{R_q} + M_{C,q}J_{R_q}\right)\left\|U^{(q)}\right\|_F + \sum_{j=1}^{t}\left\|\frac{\partial y_t^{(q)}}{\partial x_j^{(q)}}\right\|_2\left\|Dy_j^{(q-1)}(W)[U]\right\|_2 \nonumber \\
    \leq& L_{x,q}\left(H_{R_q} + M_{C,q}J_{R_q}\right)\left\|U^{(q)}\right\|_F + M_{\kappa,q}\Theta_{q-1}\left\|U^{(<q)}\right\|_F.
\end{align}
Applying Cauchy-Schwarz to the two terms gives
\begin{align}
    \left\|Dy_t^{(q)}(W)[U]\right\|_2 &\leq \left(L_{x,q}^2\left(H_{R_q} + M_{C,q}J_{R_q}\right)^2 + M_{\kappa,q}^2\Theta_{q-1}^2\right)^{1/2} \times \left(\left\|U^{(q)}\right\|_F^2 + \left\|U^{(<q)}\right\|_F^2\right)^{1/2} \nonumber \\ & = \Theta_q \left\|U^{(\leq q)}\right\|_F,
\end{align}
where the equality follows from the definition of $\Theta_q$ and the block Frobenius norm. This completes the induction.

The same argument applies to $V$. Finally, since $\|U^{(\leq q)}\|_F\leq\|U\|_F = 1$ and $\|V^{(\leq q)}\|_F\leq\|V\|_F=1$, we obtain the following first-derivative bounds:
\begin{equation}\label{eq:multilayer_first_sensitivity_unit}
    \left\|Dy_t^{(q)}(W)[U]\right\|_2\leq \Theta_q, \qquad \left\|Dy_t^{(q)}(W)[V]\right\|_2\leq \Theta_q,
\end{equation}
uniformly over all layers $q \in \{1,\dots,Q\}$ and token positions $t \in \{1,\dots,S\}$.

Finally, applying this result to $q = Q$, we get the resulting first-derivative bounds:
\begin{equation}\label{eq:Q_multilayer_first_sensitivity_unit}
    \left\|Dy_t^{(Q)}(W)[U]\right\|_2\leq \Theta_Q, \qquad \left\|Dy_t^{(Q)}(W)[V]\right\|_2\leq \Theta_Q,
\end{equation}
and substituting back into~\eqref{eq:multilayer_token_loss_hessian_reduction} we get: 
\begin{align}
    \Bigg|D^2\left(\ell_t\left(y_t^{(Q)}\right)\right)(W)[U,V]\Bigg| \leq & H_\ell \Theta_Q^2 + L_\ell \left\|D^2y_t^{(Q)}(W)[U,V]\right\|_2. \label{eq:multilayer_token_loss_hessian_reduction_with_P_Q_for_first_order_derivative}
\end{align}

We next bound the second output derivative term $\left\|D^2y_t^{(Q)}(W)[U,V]\right\|_2$. For any layer $q$, the second-order product rule gives:
\begin{align}
    D^2y_t^{(q)}(W)[U,V] =& \left(D^2h_t^{(q)}(W)[U,V]\right)^\top C_t^{(q)} + \left(h_t^{(q)}\right)^\top D^2C_t^{(q)}(W)[U,V] \nonumber\\
    &+ \left(Dh_t^{(q)}(W)[U]\right)^\top DC_t^{(q)}(W)[V] + \left(Dh_t^{(q)}(W)[V]\right)^\top DC_t^{(q)}(W)[U]. \label{eq:multilayer_output_second_derivative}
\end{align}
We bound each term separately. 

By the first-order bounds already established:
\begin{equation}
    \left\|Dx_t^{(q)}(W)[U]\right\|_2 = \left\|Dy_t^{(q-1)}(W)[U]\right\|_2 \leq \Theta_{q-1},
\end{equation}
vanishing for $q=1$, consistent with $\Theta_0 = 0$. The same holds for $\|Dx_t^{(q)}(W)[V]\|_2$. As a result, by the product rule:
\begin{align}
    \left\|DC_t^{(q)}(W)[U]\right\|_2 &\leq L_{x,q} + M_{C,q}\Theta_{q-1}, \\
    \left\|DC_t^{(q)}(W)[V]\right\|_2 &\leq L_{x,q} + M_{C,q}\Theta_{q-1},
    \\
    \left\|D^2C_t^{(q)}(W)[U,V]\right\|_2 &\leq 2\Theta_{q-1} + M_{C,q}\left\|D^2x_t^{(q)}(W)[U,V]\right\|_2.
\end{align}
Furthermore, by~\eqref{eq:multi_layer_state_jacobian_bound} and~\eqref{eq:final_D_ht_dxj_u},
\begin{align}
    \left\|Dh_t^{(q)}(W)[U]\right\|_F &\leq J_{R_q}  + M_{\Lambda_x,q}\Theta_{q-1} \sum_{j=1}^t\rho_q^{t-j} \leq J_{R_q} +\frac{M_{\Lambda_x,q}\Theta_{q-1}}{1 - \rho_q},
\end{align}
with the same inequality holding for $\left\|Dh_t^{(q)}(W)[V]\right\|_F$.

Since $\left\|C_t^{(q)}\right\|_2 \leq M_{C,q}L_{x,q}$ and $\left\|h_t^{(q)}\right\|_F\leq H_{R_q}$, equation~\eqref{eq:multilayer_output_second_derivative} yields
\begin{align}
    \left\|D^2y_t^{(q)}(W)[U,V]\right\|_2 \leq&  M_{C,q}L_{x,q}\left\|D^2h_t^{(q)}(W)[U,V]\right\|_F + M_{C,q}H_{R_q}\left\|D^2x_t^{(q)}(W)[U,V]\right\|_2 \nonumber\\
    &+ 2\left(L_{x,q} + M_{C,q}\Theta_{q-1}\right)\left(J_{R_q} +\frac{M_{\Lambda_x,q}\Theta_{q-1}}{1 - \rho_q}\right) + 2H_{R_q}\Theta_{q-1}.\label{eq:multilayer_output_second_reduction}
\end{align}

As a result, we need to bound the second derivative of the hidden state. By the definition of $h_t^{(q)}(W)$, differentiating it twice yields:
\begin{align}
    D^2h_t^{(q)}(W)[U,V] = & \bar A_t^{(q)}\odot D^2h_{t-1}^{(q)}(W)[U,V] + D^2\bar A_t^{(q)}(W)[U,V]\odot h_{t-1}^{(q)} \nonumber \\
    &+ D\bar A_t^{(q)}(W)[U] \odot Dh_{t-1}^{(q)}(W)[V] + D\bar A_t^{(q)}(W)[V]\odot Dh_{t-1}^{(q)}(W)[U] \nonumber \\ &+ D^2\left[B_t^{(q)}\left(\Delta_t^{(q)}\odot x_t^{(q)}\right)^\top\right](W)[U,V].
\end{align}

Then, since the absolute values of the first and second derivatives of $\operatorname{Softplus}$ are bounded above by $1$ and $1/4$, respectively, the product and chain rule give:
\begin{align}
    \left\|D\Delta_t^{(q)}(W)[U]\right\|_2 &\leq L_{x,q} + M_{\Delta,q}\Theta_{q-1}, \\
    \left\|D\Delta_t^{(q)}(W)[V]\right\|_2 &\leq L_{x,q} + M_{\Delta,q}\Theta_{q-1}, \\
    \left\|D^2\Delta_t^{(q)}(W)[U,V]\right\|_2 &\leq \frac{1}{4}\left(L_{x,q}+M_{\Delta,q}\Theta_{q-1}\right)^2 + 2\Theta_{q-1} + M_{\Delta,q}\left\|D^2x_t^{(q)}(W)[U,V]\right\|_2.
\end{align}

Then, applying the product rule and the exponential chain rule (per every entry of $A$), and using the established state and first-derivative bounds, we get:
\begin{align}
    \left\|D^2h_t^{(q)}(W)[U,V]\right\|_F \leq & \rho_q\left\|D^2h_{t-1}^{(q)}(W)[U,V]\right\|_F + \mathcal E_q + M_{\Lambda_x,q} \left\|D^2x_t^{(q)}(W)[U,V]\right\|_2, \label{eq:multilayer_second_state_recurrence}
\end{align}
where we collected all the terms independent of the second input derivative in $\mathcal E_q$ defined as
\begin{align}
    \mathcal E_q :=& \rho_q H_{R_q} \Bigg[\left(M_{K_\Delta,q} + M_{A,q}\left(L_{x,q} + M_{\Delta,q}\Theta_{q-1}\right)\right)^2 + 2\left(L_{x,q} + M_{\Delta,q}\Theta_{q-1}\right) \nonumber \\
    &+ M_{A,q}\left(\frac{1}{4}\left(L_{x,q} + M_{\Delta,q}\Theta_{q-1}\right)^2 + 2\Theta_{q-1}\right)\Bigg] \nonumber \\
    &+ 2\rho_q\left(M_{K_\Delta,q} + M_{A,q}\left(L_{x,q} + M_{\Delta,q}\Theta_{q-1}\right)\right)\left(J_{R_q} + \frac{M_{\Lambda_x,q}\Theta_{q-1}}{1-\rho_q}\right)\nonumber \\
    & + 2\Theta_{q-1}M_{K_\Delta,q}L_{x,q}\nonumber \\
    & + 2\left(L_{x,q} + M_{B,q}\Theta_{q-1}\right)\left[L_{x,q}\left(L_{x,q} + M_{\Delta,q}\Theta_{q-1}\right) + M_{K_\Delta,q}\Theta_{q-1}\right] \nonumber \\
    & + M_{B,q}L_{x,q}\Bigg[L_{x,q}\left(\frac{1}{4}\left(L_{x,q} + M_{\Delta,q}\Theta_{q-1}\right)^2 + 2\Theta_{q-1}\right) + 2\left(L_{x,q} + M_{\Delta,q}\Theta_{q-1}\right)\Theta_{q-1}\Bigg].\label{eq:multilayer_second_state_forcing}
\end{align}
The first term in $\mathcal E_q$ comes from the second transition derivative multiplied entrywise by the state:
\begin{equation*}
    D^2\bar A_t^{(q)}(W)[U,V]\odot h_{t-1}^{(q)}.
\end{equation*}
The second comes from the two mixed transition-state terms,
\begin{equation*}
    D\bar A_t^{(q)}(W)[U]\odot Dh_{t-1}^{(q)}(W)[V] + D\bar A_t^{(q)}(W)[V]\odot Dh_{t-1}^{(q)}(W)[U].
\end{equation*}
The remaining three terms come from differentiating the
additive state term
\begin{equation*}
    B_t^{(q)}\left(\Delta_t^{(q)}\odot x_t^{(q)}\right)^\top
\end{equation*}
twice. The contributions involving the second input derivative
are collectively bounded by
\begin{equation*}
    M_{\Lambda_x,q}\left\|D^2x_t^{(q)}(W)[U,V]\right\|_2,
\end{equation*}
where $M_{\Lambda_x,q}$ was defined in~\eqref{eq:multi_layer_M_Lambda_x}.

Since the initial state and its derivatives equal $\mathbf{0}$, unrolling~\eqref{eq:multilayer_second_state_recurrence} and summing the geometric series gives
\begin{equation}
    \left\|D^2h_t^{(q)}(W)[U,V]\right\|_F \leq \frac{\mathcal E_q}{1-\rho_q} + \frac{M_{\Lambda_x,q}}{1-\rho_q}\max_{1\leq j\leq t}\left\|D^2x_j^{(q)}(W)[U,V]\right\|_2.\label{eq:multilayer_second_state_bound}
\end{equation}

Next, we note that since $\left\|D^2x_t^{(q)}(W)[U,V]\right\|_2 \leq \max_{1\leq j\leq t}\left\|D^2x_j^{(q)}(W)[U,V]\right\|_2$, substituting~\eqref{eq:multilayer_second_state_bound} and using this fact, the new coefficient of the $\max_{1\leq j\leq t}\left\|D^2x_j^{(q)}(W)[U,V]\right\|_2$ term becomes:
\begin{equation}
    M_{C,q}H_{R_q} + \frac{M_{C,q}L_{x,q}M_{\Lambda_x,q}}{1-\rho_q} = M_{\kappa,q}.
\end{equation}

Finally, we define a new recursive term $\Omega_q$ as follows:
\begin{align}
    \Omega_0 &:= 0, \nonumber \\
    \Omega_q &:= \frac{M_{C,q}L_{x,q}\mathcal E_q}{1 - \rho_q} + 2\left(L_{x,q}+M_{C,q}\Theta_{q-1}\right)\left(J_{R_q} + \frac{M_{\Lambda_x,q}\Theta_{q-1}}{1-\rho_q}\right) \nonumber \\ &\qquad + 2H_{R_q}\Theta_{q-1} + M_{\kappa,q}\Omega_{q-1}. \label{eq:multilayer_second_sensitivity_recursion}
\end{align}
For $q=1$, second input derivative vanishes. Then, for $q \geq 2$, and using the similar induction argument as in the first-order derivative, as well as the fact $x_j^{(q)}=y_j^{(q-1)}$, the induction yields:
\begin{equation}\label{eq:multilayer_second_sensitivity_bound}
    \left\|D^2y_t^{(q)}(W)[U,V]\right\|_2\leq\Omega_q
\end{equation}
for every layer and token position.
All constants are finite and independent of $W$, $\xi$, $S$, and $t$, because the neighborhood envelopes are fixed, $\rho_q<1$, and $Q$ is finite.

Finally, applying this bound to $q=Q$, and substituting the result into~\eqref{eq:multilayer_token_loss_hessian_reduction} it follows that:
\begin{equation}
    \Bigg|D^2\left(\ell_t\left(y_t^{(Q)}\right)\right)(W)[U,V]
    \Bigg|
    \leq H_\ell\Theta_Q^2+L_\ell\Omega_Q =: L_{\mathrm{SSM},\mathbf R_Q}^{(Q)}.
\end{equation}

Averaging over the token losses preserves this bound. Since $U,V$ are arbitrary unit directions, it is true that:
\begin{equation}
    \left\|\nabla_W^2\mathcal L^{S,(Q)}(W;\xi) \right\|_{\mathrm{op}} \leq L_{\mathrm{SSM},\mathbf R_Q}^{(Q)},
\end{equation}
where the operator norm is induced by the block Frobenius norm.

Since $\mathcal N_{\mathbf R_Q}$ is convex, integrating the Hessian along the segment joining any $W_0, W_1 \in\mathcal N_{\mathbf R_Q}$ gives
\begin{equation}
    \left\|\nabla_W\mathcal L^{S,(Q)}(W_1;\xi) - \nabla_W\mathcal L^{S,(Q)}(W_0;\xi)\right\|_F \leq L_{\mathrm{SSM},\mathbf R_Q}^{(Q)}\|W_1-W_0\|_F.
\end{equation}

Taking expectations, by Assumption~\ref{assum:unbiased_stochastic_gradients} and Jensen's inequality, we get:
\begin{align}
    \left\|\nabla_W\mathcal L_k(W_1) - \nabla_W\mathcal L_k(W_0)\right\|_F &= \left\|\mathbb E_{\xi \sim \mathcal P_k}\left[\nabla_W\mathcal L^{S,(Q)}(W_1;\xi) - \nabla_W\mathcal L^{S,(Q)}(W_0;\xi)\right]\right\|_F \nonumber \\ 
    &\leq \mathbb E_{\xi \sim \mathcal P_k}\left[\left\|\nabla_W\mathcal L^{S,(Q)}(W_1;\xi) - \nabla_W\mathcal L^{S,(Q)}(W_0;\xi)\right\|_F\right] \nonumber \\ &\leq L_{\mathrm{SSM},\mathbf R_Q}^{(Q)}\|W_1-W_0\|_F,
\end{align}
demonstrating that every client objective is $L_{\mathrm{SSM},\mathbf R_Q}^{(Q)}$-smooth on $\mathcal N_{\mathbf R_Q}$.

Finally, since this inequality holds for every client objective, and $p_k\geq0$ and $\sum_kp_k=1$, it also holds for $f = \sum_kp_k\mathcal L_k$, demonstrating that $f$ is also $L_{\mathrm{SSM},\mathbf R_Q}^{(Q)}$-smooth on $\mathcal N_{\mathbf R_Q}$, completing the proof. \qed.

\subsubsection{Multi-layer SSM gradient second-moment constant Lemma and Proof}

\begin{lemma}[Multi-layer SSM gradient second-moment constant]\label{lem:multilayer_ssm_fedavg_g_constant} 
Suppose the Assumptions~\ref{assum:lipschitz_cont_loss} and \ref{assum:bounded_layer_inuts} hold. Let $\mathcal N_{\mathbf{R}_Q}(W^r) \subseteq \mathcal S^{(Q)}_{\boldsymbol{\gamma},\boldsymbol{\delta}}$ be a stable local model neighborhood around the round-$r$ broadcast $Q$-deep model $W^r$. Let:
\begin{equation}
    g(W;\xi) := \nabla_W\mathcal L^{S,(Q)}(W;\xi)
\end{equation}
denote the stochastic sequence-level gradient for a sampled input-target sequence pair $\xi = (x_{1:S}, y_{1:S}^*)$. Then, for every $W \in \mathcal N_{\mathbf{R}_Q}(W^r)$,
\begin{equation}
    \mathbb E_{\xi} \left\|g(W;\xi)\right\|_F^2 \leq \left(G_{\mathrm{SSM},\mathbf{R}_Q}^{(Q)}\right)^2,
\end{equation}
where
\begin{equation}
    G_{\mathrm{SSM},\mathbf{R}_Q}^{(Q)} := L_\ell \left  (\sum_{q=1}^Q L_{x,q}^2\left(H_{R_q} + M_{C,q}J_{R_q}\right)^2 \prod _{q'=q+1}^Q M_{\kappa,{q'}}^2\right)^{1/2} = L_\ell\Theta_Q.
\end{equation}
where $H_{R_q}, J_{R_q}, M_{C,q} \text{ and } M_{\kappa,q}$ are defined in~\eqref{eq:multi_layer_H_R_definition},~\eqref{eq:multi_layer_J_R_definition},~\eqref{eq:multi_layer_parameter_envelopes}, and ~\eqref{eq:multi_layer_M_kappa} respectively.
\end{lemma}

\textbf{Proof of Lemma~\ref{lem:multilayer_ssm_fedavg_g_constant}:} By Theorem~\ref{thm:ssm_multi_layer_grad_bound}, under the assumptions of the lemma, for every $W \in \mathcal N_{\mathbf{R}_Q}(W^r) \subseteq \mathcal S^{(Q)}_{\boldsymbol{\gamma},\boldsymbol{\delta}}$, every finite integer $S \geq 1$, and every admissible length-$S$ sequence pair $\xi \in \mathcal Z$, we have:
\begin{align}
    \left\|\nabla_W \mathcal L^{S,(Q)}(W;\xi)\right\|_F \leq  L_\ell \left  (\sum_{q=1}^Q L_{x,q}^2\left(H_{q} + \left\|W_C^{(q)}\right\|_2J_{q}\right)^2 \prod _{q'=q+1}^Q \kappa_{q'}^2\right)^{1/2}. \nonumber
\end{align}
Since $g(W;\xi) = \nabla_W\mathcal L^{S,(Q)}(W;\xi)$, it follows that:
\begin{equation}
    \|g(W;\xi)\|_F \leq  L_\ell \left  (\sum_{q=1}^Q L_{x,q}^2\left(H_{q} + \left\|W_C^{(q)}\right\|_2J_{q}\right)^2 \prod _{q'=q+1}^Q \kappa_{q'}^2\right)^{1/2}.
\end{equation}
Restricting to $W  \in \mathcal N_{\mathbf{R}_Q}(W^r)$, by the local neighborhood bounds defined previously (~\eqref{eq:multi_layer_H_R_definition},~\eqref{eq:multi_layer_J_R_definition},~\eqref{eq:multi_layer_parameter_envelopes}, and ~\eqref{eq:multi_layer_M_kappa}), 
\begin{align}
    H_q\leq H_{R_q},
    \qquad
    \left\|W_C^{(q)}\right\|_2 \leq M_{C,q},
    \qquad
    J_q\leq J_{R_q},
    \qquad
    \kappa_q \leq M_{\kappa,q}.
\end{align}

Therefore, 
\begin{equation}\label{eq:multilayer_g_ssm_definition}
    \left\|g(W;\xi)\right\|_F \leq L_\ell \left  (\sum_{q=1}^Q L_{x,q}^2\left(H_{R_q} + M_{C,q}J_{R_q}\right)^2 \prod _{q'=q+1}^Q M_{\kappa,{q'}}^2\right)^{1/2} = G_{\mathrm{SSM},\mathbf{R}_Q}^{(Q)}.
\end{equation}
Since both values are non-negative, squaring both sides gives:
\begin{equation}
    \|g(W;\xi)\|_F^2 \leq \left(G_{\mathrm{SSM},\mathbf{R}_Q}^{(Q)}\right)^2 = L_\ell^2 \left  (\sum_{q=1}^Q L_{x,q}^2\left(H_{R_q} + M_{C,q}J_{R_q}\right)^2 \prod _{q'=q+1}^Q M_{\kappa,{q'}}^2\right)
\end{equation}

Finally, taking expectation over the sampled sequence $\xi$ preserves the bound:
\begin{equation}
    \mathbb E_\xi\left[\|g(W;\xi)\|_F^2\right] \leq \left(G_{\mathrm{SSM},\mathbf{R}_Q}^{(Q)}\right)^2
\end{equation}
proving the lemma. \qed

\subsubsection{Multi-layer SSM stochastic gradient variance bound Lemma and Proof}
\begin{lemma}[Multi-layer SSM stochastic gradient variance bound]
\label{lem:multilayer_ssm_variance_bound}
Suppose the assumptions of Lemma~\ref{lem:multilayer_ssm_fedavg_g_constant} and Assumption~\ref{assum:unbiased_stochastic_gradients} hold. Let
\begin{equation*}
    \mathcal L(W):=\mathbb E_{\xi\sim\mathcal P}\left[\mathcal L^{S, (Q)}(W;\xi)\right]
\end{equation*}
denote the expected sequence loss over some local data distribution $\mathcal P$, for a $Q$-layer deep SSM parameterized by $W$. Define
\begin{equation*}
    g(W;\xi):=\nabla_W\mathcal L^{S, (Q)}(W;\xi)
\end{equation*}
as the single-sequence stochastic gradient. For a mini-batch
$\Xi_\xi=\{\xi_1,\dots,\xi_\Xi\}$ of $\Xi$ independent sampled sequences, define
\begin{equation}
    g_\Xi(W;\Xi_\xi) := \frac{1}{\Xi}\sum_{j=1}^\Xi g(W;\xi_j).
\end{equation}
Then, for every $W\in\mathcal N_{\mathbf{R}_Q}(W^r)$ it follows that
\begin{equation}
    \mathbb E_{\Xi_\xi}\left[\left\|g_\Xi(W;\Xi_\xi)-\nabla_W\mathcal L(W)\right\|_F^2\right] \leq \left(\sigma_{\mathrm{SSM},\mathbf{R}_Q,\Xi}^{(Q)}\right)^2,
\end{equation}
where
\begin{equation}
   \left(\sigma_{\mathrm{SSM},\mathbf{R}_Q,\Xi}^{(Q)}\right)^2 := \frac{1}{\Xi}\sup_{W\in\mathcal N_{\mathbf R_Q}(W^r)}\left(\mathbb E_{\xi\sim\mathcal P}\|g(W;\xi)\|_F^2 - \|\nabla_W\mathcal L(W)\|_F^2\right). 
\end{equation}
Moreover, by Lemma~\ref{lem:multilayer_ssm_fedavg_g_constant},
\begin{equation}
    \left(\sigma_{\mathrm{SSM},\mathbf{R}_Q,\Xi}^{(Q)}\right)^2 \leq \frac{\left(G_{\mathrm{SSM},\mathbf{R}_Q}^{(Q)}\right)^2}{\Xi}.
\end{equation}
\end{lemma}

\textbf{Proof of Lemma~\ref{lem:multilayer_ssm_variance_bound}:} By unbiasedness and independence of the sampled sequences, the standard mini-batch variance identity gives:
\begin{align*}
    \mathbb E_{\Xi_\xi}\left[\left\|g_\Xi(W;\Xi_\xi)-\nabla_W\mathcal L(W)\right\|_F^2\right] = \frac{1}{\Xi}\mathbb E_{\xi\sim\mathcal P}\left[\left\|g(W;\xi)-\nabla_W\mathcal L(W)\right\|_F^2\right].
\end{align*}
Using the identity
\begin{equation}
    \mathbb E_{\xi\sim\mathcal P} \left[\left\|g(W;\xi)-\nabla_W\mathcal L(W)\right\|_F^2\right] = \mathbb E_{\xi\sim\mathcal P} \|g(W;\xi)\|_F^2 - \|\nabla_W\mathcal L(W)\|_F^2,
\end{equation}
and taking the supremum over $W\in \mathcal N_{\mathbf{R}_Q}(W^r)$, we obtain
\begin{equation}
    \mathbb E_{\Xi_\xi}\left[\left\|g_\Xi(W;\Xi_\xi)-\nabla_W\mathcal L(W)\right\|_F^2\right] \leq \left(\sigma_{\mathrm{SSM},\mathbf{R}_Q,\Xi}^{(Q)}\right)^2.
\end{equation}

Finally, Lemma~\ref{lem:multilayer_ssm_fedavg_g_constant} gives
\begin{equation*}
    \mathbb E_{\xi\sim\mathcal P}\|g(W;\xi)\|_F^2 \leq \left(G_{\mathrm{SSM},\mathbf{R}_Q}^{(Q)}\right)^2
\end{equation*}
and since $\|\nabla_W\mathcal L(W)\|_F^2\geq 0$,
\begin{equation}
    \left(\sigma_{\mathrm{SSM},\mathbf{R}_Q,\Xi}^{(Q)}\right)^2 \leq \frac{\left(G_{\mathrm{SSM},\mathbf{R}_Q}^{(Q)}\right)^2}{\Xi}.
\end{equation}
This proves the claim. \qed

\subsubsection{Multi-layer SSM client heterogeneity bound under total variation Lemma and Proof}
\begin{lemma}[Multi-layer SSM client heterogeneity bound under total variation]\label{lem:multilayer_ssm_gamma_tv_bound}
Suppose the assumptions of Lemma~\ref{lem:multilayer_ssm_fedavg_g_constant} and Assumption~\ref{assum:unbiased_stochastic_gradients} hold. Let
\begin{equation}
    f(W) := \sum_{k=1}^K p_k \mathcal L_k(W)
\end{equation}
be the global objective, where $p_k \geq 0$ and $\sum_{k=1}^K p_k=1$. Let $\mathcal P_k$ denote the local data distribution of client $k$, and define the global mixture distribution as:
\begin{equation}
    \bar{\mathcal P} := \sum_{k=1}^K p_k \mathcal P_k.
\end{equation}
Define the local client heterogeneity constant
\begin{equation}
    \left(\Gamma_{\mathrm{SSM},\mathbf{R}_Q}^{(Q)}\right)^2:=\sup_{W \in \mathcal N_{\mathbf{R}_Q}(W^r)}\sum_{k=1}^K p_k \left\|\nabla_W \mathcal L_k(W) - \nabla_W f(W) \right\|_F^2.
\end{equation}
Then:
\begin{equation}
    \left(\Gamma_{\mathrm{SSM},\mathbf{R}_Q}^{(Q)}\right)^2 \leq 4\left(G_{\mathrm{SSM},\mathbf{R}_Q}^{(Q)}\right)^2 \sum_{k=1}^K p_k \operatorname{TV}(\mathcal{P}_k, \bar{\mathcal P})^2,
\end{equation}
where
\begin{equation}
    \operatorname{TV}(P,Q):=\sup_A |P(A)-Q(A)|.
\end{equation}
\end{lemma}

\textbf{Proof of Lemma~\ref{lem:multilayer_ssm_gamma_tv_bound}:} Recall that for each client $k$, the local objective is:
\begin{equation*}
    \mathcal L_k(W) =  \mathbb E_{\xi \sim \mathcal P_k}\left[\mathcal L^{S,(Q)}(W;\xi)\right],
\end{equation*}
and, by Assumption~\ref{assum:unbiased_stochastic_gradients},
\begin{align}
    \nabla_W\mathcal L_k(W) = \mathbb E_{\xi\sim \mathcal P_k}\left[g(W;\xi)\right].
\end{align}
Similarly, since
\begin{equation*}
    f(W) := \sum_{k=1}^K p_k \mathcal L_k(W),
\end{equation*}
we have
\begin{equation}
    \nabla_W f(W) = \sum_{k=1}^Kp_k\nabla_W\mathcal L_k(W) = \sum_{k=1}^Kp_k\mathbb E_{\xi\sim \mathcal P_k}\left[g(W;\xi)\right].
\end{equation}

Next, by the definition of the mixture distribution $\bar{\mathcal P}$, we can write this as:
\begin{equation}
     \nabla_W f(W) = \mathbb E_{\xi \sim \bar{\mathcal P}}\left[g(W;\xi)\right].
\end{equation}
Therefore, it follows that
\begin{equation}
    \nabla_W\mathcal L_k(W)-\nabla_W f(W) = \mathbb E_{\xi\sim \mathcal P_k}[g(W;\xi)] - \mathbb E_{\xi\sim \bar{\mathcal P}}[g(W;\xi)].
\end{equation}

By the pointwise bound~\eqref{eq:multilayer_g_ssm_definition} established in the proof of Lemma~\ref{lem:multilayer_ssm_fedavg_g_constant}, for every $W \in \mathcal N_{\mathbf{R}_Q}(W^r)$ and every admissible sequence pair $\xi \in \mathcal Z$,
\begin{equation*}
    \left\|g(W;\xi)\right\|_F \leq G_{\mathrm{SSM},\mathbf{R}_Q}^{(Q)}.
\end{equation*}
Since $\mathcal P_k(\mathcal Z) = \bar{\mathcal P}(\mathcal Z) = 1$, the total-variation inequality for bounded vector-valued functions gives:
\begin{equation}
    \left\|\mathbb E_{\xi\sim \mathcal P_k}[g(W;\xi)] - \mathbb E_{\xi\sim \bar{\mathcal P}}[g(W;\xi)]\right\|_F \leq 2G_{\mathrm{SSM},\mathbf{R}_Q}^{(Q)}\operatorname{TV}(\mathcal P_k,\bar{\mathcal P}).
\end{equation}
So:
\begin{equation}
    \left\|\nabla_W\mathcal L_k(W)-\nabla_W f(W)\right\|_F^2 \leq 4\left(G_{\mathrm{SSM},\mathbf{R}_Q}^{(Q)}\right)^2\operatorname{TV}(\mathcal P_k,\bar{\mathcal P})^2.
\end{equation}

Multiplying both sides by $p_k$ and summing over all clients gives:
\begin{equation}
    \sum_{k=1}^K p_k \left\|\nabla_W\mathcal L_k(W)-\nabla_W f(W) \right\|_F^2 \leq 4\left(G_{\mathrm{SSM},\mathbf{R}_Q}^{(Q)}\right)^2\sum_{k=1}^K p_k \operatorname{TV}(\mathcal P_k,\bar{\mathcal P})^2.
\end{equation} 

Since the right-hand side of this inequality is independent of $W$, taking the supremum over $W \in \mathcal N_{\mathbf{R}_Q}(W^r)$ we get:
\begin{equation}
    \left(\Gamma_{\mathrm{SSM},\mathbf{R}_Q}^{(Q)}\right)^2 := \sup_{W \in \mathcal N_{\mathbf{R}_Q}(W^r)} \sum_{k=1}^K p_k\left\|\nabla_W \mathcal L_k(W)-\nabla_W f(W)\right\|_F^2 \leq 4\left(G_{\mathrm{SSM},\mathbf{R}_Q}^{(Q)}\right)^2\sum_{k=1}^Kp_k \operatorname{TV}(\mathcal P_k,\bar{\mathcal P})^2.
\end{equation}

This proves the multi-layer SSM client heterogeneity bound. \qed

\subsection{Corollary~\ref{thm:multilayer_fedavg_ssm_convergence}:
FedAvg Multi-Layer SSM Convergence and Proof}

\begin{corollary}[FedAvg convergence for multi-layer selective SSMs]
\label{thm:multilayer_fedavg_ssm_convergence}

Consider the distributed learning framework established in Section~\ref{sec:system_model}, with objective
\begin{equation*}
    f(W):=\sum_{k=1}^K p_k\mathcal L_k(W), \qquad p_k\geq0, \qquad \sum_{k=1}^K p_k=1,
\end{equation*}
where the weights $p_k$ are fixed and
\begin{equation*}
        W=\left(W^{(1)},\ldots,W^{(Q)}\right)
\end{equation*}
parameterizes the selective SSM of finite depth $Q\geq1$ defined in Appendix~\ref{appendix:multi_layer_selective_ssm}.

Each client objective is finite and differentiable on the neighborhoods under consideration, and satisfies
\begin{equation*}
    \nabla_W\mathcal L_k(W) = \mathbb E_{\xi\sim\mathcal P_k} \left[\nabla_W\mathcal L^{S,(Q)}(W;\xi)\right].
\end{equation*}
where $S$ is the length of the sampled sequence pair $\xi$, and all clients use the same sequence-loss map.

Let $W^0$ be deterministic with $f(W^0)<\infty$, and suppose that $f$ is bounded below by a finite constant $f^*$.

Consider full-participation FedAvg for $T$ communication rounds, with $E$ local steps per round, mini-batch size $\Xi$, and constant deterministic local step size $\eta>0$, where $T,E,\Xi \in \mathbb Z_{>0}$:
\begin{align*}
    W_k^{r,0} &= W^r,\\
    W_k^{r,e+1} &= W_k^{r,e}-\eta g_{k,\Xi}^{r,e}, \qquad e=0,\ldots,E-1,\\
    W^{r+1} & =\sum_{k=1}^K p_kW_k^{r,E}.
\end{align*}
The mini-batch gradient is
\begin{equation*}
    g_{k,\Xi}^{r,e} := \frac{1}{\Xi}\sum_{i=1}^{\Xi} \nabla_W\mathcal L^{S_i,(Q)}(W_k^{r,e};\xi_{k,i}^{r,e}),
\end{equation*}
where $S_i$ denotes the length of $\xi_{k,i}^{r,e}$.

Conditional on the complete training history before each local step, each client draws $\Xi$ fresh independent samples from $\mathcal P_k$, and the mini-batches drawn by different clients are conditionally independent.

For fixed deterministic radii $R_q > 0$, define
\begin{align*}
    \mathcal N_{\mathbf R_Q}(W^r) := \prod_{q=1}^Q \mathcal N_{R_q}^{(q)}(W^{r,(q)}), \qquad \mathbf R_Q:=(R_1,\ldots,R_Q).
\end{align*}
    
Suppose that, almost surely, for every round $r<T$, client $k$, and local step $0\leq e\leq E$,
\begin{equation*}
    W_k^{r,e} \in \mathcal N_{\mathbf R_Q}(W^r) \subseteq \mathcal S_{\boldsymbol\gamma,\boldsymbol\delta}^{(Q)}.
\end{equation*}
Suppose also that finite deterministic envelopes $M_{X,q}$ satisfy, almost surely for every round,
\begin{equation*}
    \left\|W_X^{r,(q)}\right \|_2+R_q \leq M_{X,q}, \qquad X\in\{\Delta,B,C,A\}, \qquad q\in\{1,\ldots,Q\},
\end{equation*}
where $W_A^{r,(q)}:=A^{r,(q)}$.

The stability margins $\gamma_q,\delta_q>0$, the layer-input bounds $L_{x,q}$, the loss-regularity constants $L_\ell,H_\ell$, and the parameter envelopes $M_{X,q}$ are deterministic and common across rounds.

Let Assumption~\ref{assum:bounded_layer_inuts} hold uniformly on all these neighborhoods, and for each layer, define
\begin{align*}
    \rho_q:=\exp(-\gamma_q\delta_q)<1.   
\end{align*}
Define:
\begin{align*}
    M_{K_\Delta, q} & := M_{\Delta,q} L_{x,q} + \sqrt{D}\log(2), \\ M_{K_B,q} &:= M_{B,q}L_{x,q}^2, \\
    M_{K_W,q} &:= ND M_{K_{B},q}M_{K_{\Delta},q}\left(M_{K_{\Delta},q} + \sqrt{D} M_{A,q}L_{x,q}\right)
     \left(\frac{\rho_q}{1-\rho_q}\right) \\ &\qquad + L_{x,q}^2 \sqrt{N} M_{K_{\Delta},q} + L_{x,q} \sqrt{D} M_{K_B,q}, \\
     H_{R_q} &:= \frac{M_{K_B,q} M_{K_{\Delta},q}
    \sqrt{ND}}{1-\rho_q}, \\ 
    J_{R_q} &:= \frac{M_{K_W,q}}{1-\rho_q},  \\
    M_{\Lambda_x,q} &:= \rho_q M_{A,q}M_{\Delta,q}H_{R_q} + M_{B,q}L_{x,q}\left(2M_{K_\Delta,q} + L_{x,q}M_{\Delta,q}\right), \\
    M_{\kappa,q} &:= M_{C,q}\left(H_{R_q} + \frac{L_{x,q}M_{\Lambda_x,q}}{1-\rho_q}\right),\\
    \Theta_q &: =\left(L_{x,q}^2\left(H_{R_q}+M_{C,q}J_{R_q}\right)^2+M_{\kappa,q}^2\Theta_{q-1}^2\right)^{1/2}, \qquad \Theta_0 := 0,  \\
    \mathcal E_q :=& \rho_q H_{R_q} \Bigg[\left(M_{K_\Delta,q} + M_{A,q}\left(L_{x,q} + M_{\Delta,q}\Theta_{q-1}\right)\right)^2 + 2\left(L_{x,q} + M_{\Delta,q}\Theta_{q-1}\right) \nonumber \\
    &+ M_{A,q}\left(\frac{1}{4}\left(L_{x,q} + M_{\Delta,q}\Theta_{q-1}\right)^2 + 2\Theta_{q-1}\right)\Bigg] \nonumber \\
    &+ 2\rho_q\left(M_{K_\Delta,q} + M_{A,q}\left(L_{x,q} + M_{\Delta,q}\Theta_{q-1}\right)\right)\left(J_{R_q} + \frac{M_{\Lambda_x,q}\Theta_{q-1}}{1-\rho_q}\right)\nonumber \\
    & + 2\Theta_{q-1}M_{K_\Delta,q}L_{x,q}\nonumber \\
    & + 2\left(L_{x,q} + M_{B,q}\Theta_{q-1}\right)\left[L_{x,q}\left(L_{x,q} + M_{\Delta,q}\Theta_{q-1}\right) + M_{K_\Delta,q}\Theta_{q-1}\right] \nonumber \\
    & + M_{B,q}L_{x,q}\Bigg[L_{x,q}\left(\frac{1}{4}\left(L_{x,q} + M_{\Delta,q}\Theta_{q-1}\right)^2 + 2\Theta_{q-1}\right) + 2\left(L_{x,q} + M_{\Delta,q}\Theta_{q-1}\right)\Theta_{q-1}\Bigg],\\
    \Omega_q &:= \frac{M_{C,q}L_{x,q}\mathcal E_q}{1 - \rho_q} + 2\left(L_{x,q}+M_{C,q}\Theta_{q-1}\right)\left(J_{R_q} + \frac{M_{\Lambda_x,q}\Theta_{q-1}}{1-\rho_q}\right) + 2H_{R_q}\Theta_{q-1} + M_{\kappa,q}\Omega_{q-1}, \\ \Omega_0 &= 0
\end{align*}

Let:
\begin{align*}
    L_{\mathrm{SSM},\mathbf R_Q}^{(Q)} &:= H_\ell\Theta_Q^2+L_\ell\Omega_Q,\\
    G_{\mathrm{SSM},\mathbf R_Q}^{(Q)} &:= L_\ell\Theta_Q = L_\ell \left(\sum_{q=1}^Q L_{x,q}^2\left(H_{R_q}+M_{C,q}J_{R_q}\right)^2\prod_{q'=q+1}^Q M_{\kappa,q'}^2\right)^{1/2},
\end{align*}
define the common deterministic smoothness and gradient bounds on all the considered neighborhoods.

For the client distributions, define
\begin{equation*}
    \bar{\mathcal P} := \sum_{k=1}^Kp_k\mathcal P_k, \qquad \mathcal H_{\mathrm{TV}} := \sum_{k=1}^Kp_k \operatorname{TV}(\mathcal P_k,\bar{\mathcal P})^2,
\end{equation*}
where
\begin{equation*}
    \operatorname{TV}(P,Q) := \sup_{\mathcal A}|P(\mathcal A)-Q(\mathcal A)|.
\end{equation*}

If the step size satisfies
\begin{equation*}
    \eta E L_{\mathrm{SSM},\mathbf R_Q}^{(Q)} \leq\frac{1}{2}, \qquad 12\eta^2\left(L_{\mathrm{SSM},\mathbf R_Q}^{(Q)}\right)^2E(E-1) \leq 1, 
\end{equation*}
then
\begin{align*}
    \frac{1}{T}\sum_{r=0}^{T-1} \mathbb E\|\nabla f(W^r)\|_F^2 \leq& \frac{4(f(W^0)-f^*)}{\eta ET} + \frac{4\eta L_{\mathrm{SSM},\mathbf R_Q}^{(Q)}\left(G_{\mathrm{SSM},\mathbf R_Q}^{(Q)}\right)^2}{\Xi} \sum_{k=1}^Kp_k^2 \nonumber \\
    &+ \frac{6\eta^2\left(L_{\mathrm{SSM},\mathbf R_Q}^{(Q)}\right)^2(E-1)\left(G_{\mathrm{SSM},\mathbf R_Q}^{(Q)}\right)^2}{\Xi} \nonumber \\ &+ 48\eta^2\left(L_{\mathrm{SSM},\mathbf R_Q}^{(Q)}\right)^2E(E-1)\left(G_{\mathrm{SSM},\mathbf R_Q}^{(Q)}\right)^2\mathcal H_{\mathrm{TV}}.
\end{align*}
\end{corollary}

\textbf{Proof of Corollary~\ref{thm:multilayer_fedavg_ssm_convergence}:} We apply the federated optimization argument from the proof of Theorem~\ref{thm:fedavg_ssm_convergence} to the full parameter tuple $W=\left(W^{(1)},\ldots,W^{(Q)}\right)$. After vectorization, the Euclidean norm of this tuple is exactly its block Frobenius norm.

The preceding neighborhood estimates remain valid when the reference-dependent parameter envelopes are replaced by the common deterministic upper envelopes $M_{X,q}$, because their defining expressions are nondecreasing in those envelopes.

Lemma~\ref{lem:multilayer_local_smoothness} therefore supplies the common smoothness constant $L=L_{\mathrm{SSM},\mathbf R_Q}^{(Q)}$. 

Lemma~\ref{lem:multilayer_ssm_variance_bound}, applied conditionally on the history before each mini-batch,
supplies the common variance bound $\sigma^2 = \frac{\left(G_{\mathrm{SSM},\mathbf R_Q}^{(Q)}\right)^2}{\Xi}.$

Conditional independence across clients gives the corresponding weighted-noise variance bound $\sigma^2\sum_kp_k^2$.

Lemma~\ref{lem:multilayer_ssm_gamma_tv_bound} supplies the common client-gradient disagreement bound $\Gamma^2 = 4\left(G_{\mathrm{SSM},\mathbf R_Q}^{(Q)}\right)^2\mathcal H_{\mathrm{TV}}.$

Each neighborhood $\mathcal N_{\mathbf R_Q}(W^r)$ is a product of convex balls. Since it contains all local iterates in round $r$, it also contains their weighted aggregate and the segments used in the smoothness argument.

Substituting these constants into \eqref{eq:generic_fedavg_bound_appendix} yields the desired inequality in Corollary~\ref{thm:multilayer_fedavg_ssm_convergence} under the specified step size conditions.
\qed

\subsection{Corollary~\ref{thm:multilayer_fedprox_ssm_convergence}:
FedProx Multi-Layer SSM Convergence and Proof}

\begin{corollary}[FedProx convergence for multi-layer selective SSMs]\label{thm:multilayer_fedprox_ssm_convergence}

Consider the distributed learning framework established in Section~\ref{sec:system_model}, with objective
\begin{equation*}
    f(W) := \sum_{k=1}^K p_k\mathcal L_k(W), \qquad p_k \geq 0, \qquad \sum_{k=1}^K p_k = 1,
\end{equation*}
where the weights $p_k$ are fixed and $W = \left(W^{(1)},\ldots,W^{(Q)}\right)$ parameterizes the selective SSM of finite integer depth $Q \geq 1$ defined in Appendix~\ref{appendix:multi_layer_selective_ssm}.

Each client objective is finite and differentiable on the neighborhoods under consideration and satisfies
\begin{equation*}
    \nabla_W\mathcal L_k(W) =\mathbb E_{\xi\sim\mathcal P_k} \left[\nabla_W\mathcal L^{S,(Q)}(W;\xi)\right],
\end{equation*}
where $S$ is the length of the sampled sequence pair $\xi$, and all clients use the same sequence-loss map. Let $W^0$ be deterministic with $f(W^0) < \infty$, and suppose that $f$ is bounded below by a finite constant $f^*$.

Consider full-participation FedProx for $T$ communication rounds, with $E$ local steps per round, mini-batch size $\Xi$, constant deterministic local step size $\eta > 0$, and constant deterministic proximal coefficient $\mu > 0$, where $T, E, \Xi \in \mathbb Z_{>0}$:
\begin{align*}
    W_k^{r,0} &= W^r,\\
    W_k^{r,e+1} &= W_k^{r,e} - \eta\left(g_{k,\Xi}^{r,e} + \mu(W_k^{r,e}-W^r)\right), \qquad e=0,\ldots,E-1,\\
    W^{r+1} &= \sum_{k=1}^K p_kW_k^{r,E}, \qquad r=0,\ldots,T-1.
\end{align*}
The mini-batch gradient is
\begin{equation*}
    g_{k,\Xi}^{r,e} := \frac{1}{\Xi}\sum_{i=1}^{\Xi} \nabla_W\mathcal L^{S_i,(Q)}\left(W_k^{r,e};\xi_{k,i}^{r,e}\right),
\end{equation*}
where $S_i$ denotes the length of $\xi_{k,i}^{r,e}$. Conditional on the complete training history before each local step, each client draws $\Xi$ fresh independent samples from $\mathcal P_k$, and the mini-batches drawn by different clients are conditionally independent.

For fixed deterministic radii $R_q > 0$, define
\begin{equation*}
    \mathcal N_{\mathbf R_Q}(W^r) := \prod_{q=1}^Q \mathcal N_{R_q}^{(q)}(W^{r,(q)}), \qquad \mathbf R_Q:=(R_1,\ldots,R_Q).
\end{equation*}
Suppose that, almost surely, for every round $r < T$, client $k$, and local step $0\leq e\leq E$,
\begin{equation*}
    W_k^{r,e} \in \mathcal N_{\mathbf R_Q}(W^r) \subseteq \mathcal S_{\boldsymbol\gamma,\boldsymbol\delta}^{(Q)}.
\end{equation*}
Suppose also that finite deterministic envelopes $M_{X,q}$ satisfy, almost surely for every round,
\begin{equation*}
    \left\|W_X^{r,(q)}\right\|_2+R_q\leq M_{X,q}, \qquad X\in\{\Delta,B,C,A\}, \qquad q=1,\ldots,Q,
\end{equation*}
where $W_A^{r,(q)}:=A^{r,(q)}$.

Let Assumption~\ref{assum:bounded_layer_inuts} hold uniformly on all these neighborhoods, and let Assumption~\ref{assum:lipschitz_cont_loss} hold uniformly at
the final-layer predictions. The stability margins $\gamma_q,\delta_q>0$, layer-input bounds $L_{x,q}$, loss-regularity constants $L_\ell,H_\ell$, and parameter envelopes $M_{X,q}$ are deterministic and common across rounds, admissible sequence pairs, finite sequence lengths, and token positions.

For each layer, define
\begin{equation*}
    \rho_q:=\exp(-\gamma_q\delta_q)<1.
\end{equation*}

Let $G_{\mathrm{SSM},\mathbf R_Q}^{(Q)}$ be the common gradient bound defined in Corollary~\eqref{thm:multilayer_fedavg_ssm_convergence}, constructed using the original envelopes $M_{X,q}$.

Assume $0 < \eta\mu \leq 1$, and define
\begin{align*}
    R_{\mu,q} &:= \min\left\{R_q, \frac{G_{\mathrm{SSM},\mathbf R_Q}^{(Q)}}{\mu}\left[1-(1-\eta\mu)^E\right]\right\}, \qquad q=1,\ldots,Q, \\
    \mathbf R_{Q,\mu} &:= (R_{\mu,1},\ldots,R_{\mu,Q}), \\
    M_{X,q}^{\mu} &:= M_{X,q}-R_q+R_{\mu,q}, \qquad X\in\{\Delta,B,C,A\}.
\end{align*}

Using these reduced parameter envelopes, define
\begin{align*}
    M_{K_\Delta,q}^{\mu} &:= M_{\Delta,q}^{\mu}L_{x,q} + \sqrt D\log2,\\
    M_{K_B,q}^{\mu} &:= M_{B,q}^{\mu}L_{x,q}^2, \\
    M_{K_W,q}^{\mu} &:=ND M_{K_B,q}^{\mu}M_{K_\Delta,q}^{\mu}\left(M_{K_\Delta,q}^{\mu} + \sqrt D M_{A,q}^{\mu}L_{x,q}\right)\frac{\rho_q}{1-\rho_q} \\ 
    &\quad + L_{x,q}^2\sqrt N M_{K_\Delta,q}^{\mu} + L_{x,q}\sqrt D M_{K_B,q}^{\mu},\\
    H_{R_{\mu,q}} &:= \frac{M_{K_B,q}^{\mu}M_{K_\Delta,q}^{\mu}\sqrt{ND}}{1-\rho_q},\\
    J_{R_{\mu,q}} &:= \frac{M_{K_W,q}^{\mu}}{1-\rho_q},\\
    M_{\Lambda_x,q}^{\mu} &:= \rho_q M_{A,q}^{\mu}M_{\Delta,q}^{\mu}H_{R_{\mu,q}} + M_{B,q}^{\mu}L_{x,q}\left(2M_{K_\Delta,q}^{\mu} + L_{x,q}M_{\Delta,q}^{\mu}\right),\\
    M_{\kappa,q}^{\mu} &:= M_{C,q}^{\mu}\left(H_{R_{\mu,q}} + \frac{L_{x,q}M_{\Lambda_x,q}^{\mu}}{1 - \rho_q}\right).
\end{align*}
Initialize $\Theta_0^\mu = \Omega_0^\mu = 0$, and recursively define
\begin{align*}
    \Theta_q^\mu &:= \left[L_{x,q}^2\left(H_{R_{\mu,q}} + M_{C,q}^{\mu}J_{R_{\mu,q}}\right)^2 + (M_{\kappa,q}^{\mu})^2(\Theta_{q-1}^\mu)^2\right]^{1/2}, \\
    \mathcal E_q^\mu &:= \rho_qH_{R_{\mu,q}}
    \Bigg[\left(M_{K_\Delta,q}^{\mu} + M_{A,q}^{\mu}\left(L_{x,q} + M_{\Delta,q}^{\mu}\Theta_{q-1}^{\mu}\right)\right)^2 \\
    &\qquad + 2\left(L_{x,q} + M_{\Delta,q}^{\mu}\Theta_{q-1}^{\mu}\right) + M_{A,q}^{\mu}\left(\frac{1}{4}\left(L_{x,q} + M_{\Delta,q}^{\mu}\Theta_{q-1}^{\mu}\right)^2 + 2\Theta_{q-1}^{\mu}\right)\Bigg] \\
    &\quad + 2\rho_q \left(M_{K_\Delta,q}^{\mu} + M_{A,q}^{\mu}\left(L_{x,q} + M_{\Delta,q}^{\mu}\Theta_{q-1}^{\mu}\right)\right)
    \left(J_{R_{\mu,q}} + \frac{M_{\Lambda_x,q}^{\mu}\Theta_{q-1}^{\mu}}{1 - \rho_q}\right)\\
    &\quad + 2\Theta_{q-1}^{\mu}M_{K_\Delta,q}^{\mu}L_{x,q}\\
    &\quad + 2\left(L_{x,q} + M_{B,q}^{\mu}\Theta_{q-1}^{\mu}\right)\left[L_{x,q}\left(L_{x,q} + M_{\Delta,q}^{\mu}\Theta_{q-1}^{\mu}\right) + M_{K_\Delta,q}^{\mu}\Theta_{q-1}^{\mu}\right] \\
    &\quad + M_{B,q}^{\mu}L_{x,q}\Bigg[L_{x,q}\left(\frac{1}{4}\left(L_{x,q} + M_{\Delta,q}^{\mu}\Theta_{q-1}^{\mu}\right)^2 + 2\Theta_{q-1}^{\mu}\right)\\
    &\qquad\qquad + 2\left(L_{x,q} + M_{\Delta,q}^{\mu}\Theta_{q-1}^{\mu}\right)\Theta_{q-1}^{\mu}\Bigg],\\
    \Omega_q^\mu &:= \frac{M_{C,q}^{\mu}L_{x,q}\mathcal E_q^\mu}{1-\rho_q}\\
    &\quad + 2\left(L_{x,q} + M_{C,q}^{\mu}\Theta_{q-1}^{\mu}\right)\left(J_{R_{\mu,q}} + \frac{M_{\Lambda_x,q}^{\mu}\Theta_{q-1}^{\mu}}{1-\rho_q}\right)\\
    &\quad + 2H_{R_{\mu,q}}\Theta_{q-1}^{\mu} + M_{\kappa,q}^{\mu}\Omega_{q-1}^{\mu}.
\end{align*}

Let
\begin{align*}
    L_{\mathrm{SSM},\mathbf R_{Q,\mu}}^{(Q)} &:= H_\ell(\Theta_Q^\mu)^2+L_\ell\Omega_Q^\mu,\\
    L_{\mathrm{prox},\mathbf R_{Q,\mu}}^{(Q)} &:= L_{\mathrm{SSM},\mathbf R_{Q,\mu}}^{(Q)}+\mu,\\
    G_{\mathrm{SSM},\mathbf R_{Q,\mu}}^{(Q)} &:= L_\ell\Theta_Q^\mu \\
    &= L_\ell\left[\sum_{q=1}^Q L_{x,q}^2\left(H_{R_{\mu,q}} + M_{C,q}^{\mu}J_{R_{\mu,q}}\right)^2\prod_{q'=q+1}^Q(M_{\kappa,q'}^\mu)^2\right]^{1/2},
\end{align*}
where empty products equal one.
These are common deterministic smoothness and gradient
bounds on the reduced neighborhoods
\begin{equation*}
    \mathcal N_{\mathbf R_{Q,\mu}}(W^r) := \prod_{q=1}^Q \mathcal N_{R_{\mu,q}}^{(q)}(W^{r,(q)}).
\end{equation*}

For the client distributions, define
\begin{equation*}
    \bar{\mathcal P}:=\sum_{k=1}^Kp_k\mathcal P_k, \qquad \mathcal H_{\mathrm{TV}} := \sum_{k=1}^Kp_k \operatorname{TV}(\mathcal P_k,\bar{\mathcal P})^2,
\end{equation*}
where $\operatorname{TV}(P,Q):=\sup_{\mathcal A}|P(\mathcal A)-Q(\mathcal A)|$, with the supremum taken over measurable sets $\mathcal A$.

If the step size satisfies
\begin{equation*}
    \eta E L_{\mathrm{SSM},\mathbf R_{Q,\mu}}^{(Q)} \leq\frac{1}{2}, \qquad
    12\eta^2\left(L_{\mathrm{prox},\mathbf R_{Q,\mu}}^{(Q)}\right)^2E(E-1) \leq 1,
\end{equation*}
then
\begin{align*}
    \frac{1}{T}\sum_{r=0}^{T-1} \mathbb E\|\nabla f(W^r)\|_F^2 \leq& \frac{3(f(W^0)-f^*)}{\eta ET} \\ 
    &+ \frac{3\eta L_{\mathrm{SSM},\mathbf R_{Q,\mu}}^{(Q)}\left(G_{\mathrm{SSM},\mathbf R_{Q,\mu}}^{(Q)}\right)^2}{\Xi}\sum_{k=1}^Kp_k^2  \\ 
    &+ \frac{3\eta^2\left(L_{\mathrm{prox},\mathbf R_{Q,\mu}}^{(Q)}\right)^2(E-1)\left(G_{\mathrm{SSM},\mathbf R_{Q,\mu}}^{(Q)}\right)^2}{\Xi} \\
    &+ 24\eta^2\left(L_{\mathrm{prox},\mathbf R_{Q,\mu}}^{(Q)}\right)^2 E(E-1)\left(G_{\mathrm{SSM},\mathbf R_{Q,\mu}}^{(Q)}\right)^2\mathcal H_{\mathrm{TV}}.
\end{align*}

Moreover, almost surely, for every round $r < T$, client $k$, layer $q$, and local step $0\leq e\leq E$,
\begin{equation*}
    \left\|W_k^{r,e,(q)}-W^{r,(q)}\right\|_F \leq \min\left\{R_q,\frac{G_{\mathrm{SSM},\mathbf R_Q}^{(Q)}}{\mu}\left[1-(1-\eta\mu)^e\right]\right\}\leq R_{\mu,q}.
\end{equation*}
Consequently, all local iterates and the aggregated model $W^{r+1}$ lie in
\begin{equation*}
    \mathcal N_{\mathbf R_{Q,\mu}}(W^r) \subseteq \mathcal N_{\mathbf R_Q}(W^r) \subseteq \mathcal S_{\boldsymbol\gamma,\boldsymbol\delta}^{(Q)}.
\end{equation*}
\end{corollary}

\textbf{Proof of Corollary~\ref{thm:multilayer_fedprox_ssm_convergence}:} We apply the federated optimization argument from the proof of Theorem~\ref{cor:fedprox_ssm_convergence} to the full parameter tuple $W=\left(W^{(1)},\ldots,W^{(Q)}\right)$. After vectorization, the Euclidean norm of this tuple is exactly its block Frobenius norm.

First, the pointwise multi-layer gradient bound implies $\left\|g_{k,\Xi}^{r,e}\right\|_F \leq G_{\mathrm{SSM},\mathbf R_Q}^{(Q)}$, which also bounds each layer block of the mini-batch gradient. Applying the displacement recursion from the single-layer FedProx
proof separately to each layer, and using $0<\eta \mu\leq 1$, gives
\begin{equation}
    \left\|W_k^{r,e,(q)}-W^{r,(q)}\right\|_F \leq \frac{G_{\mathrm{SSM},\mathbf R_Q}^{(Q)}}{\mu}\left[1-(1-\eta\mu)^e\right].
\end{equation}
Combining this with the assumed containment in the original neighborhood yields the stated bound by $R_{\mu,q}$.

Consequently, all local iterates lie in $\mathcal N_{\mathbf R_{Q,\mu}}(W^r)$. This neighborhood is a product of convex balls, so it also contains their weighted aggregate and the segments used in the smoothness argument.

For every $W\in\mathcal N_{\mathbf R_{Q,\mu}}(W^r)$,
\begin{equation*}
    \left\|W_X^{(q)}\right\|_2 \leq \left\|W_X^{r,(q)}\right\|_2 + R_{\mu,q} \leq M_{X,q}-R_q+R_{\mu,q} = M_{X,q}^{\mu}.
\end{equation*}
The preceding neighborhood estimates therefore apply with the common deterministic envelopes $M_{X,q}^{\mu}$, because their defining expressions are nondecreasing in those envelopes.

Lemma~\ref{lem:multilayer_local_smoothness} supplies the common smoothness constant $L_{\mathrm{SSM},\mathbf R_{Q,\mu}}^{(Q)}$ for each client objective and for $f$. Adding the proximal penalty makes each local proximal objective $L_{\mathrm{prox},\mathbf R_{Q,\mu}}^{(Q)}$-smooth, where
\begin{equation*}
    L_{\mathrm{prox},\mathbf R_{Q,\mu}}^{(Q)} = L_{\mathrm{SSM},\mathbf R_{Q,\mu}}^{(Q)} + \mu.
\end{equation*}

Lemma~\ref{lem:multilayer_ssm_variance_bound}, applied conditionally on the complete history before each mini-batch, supplies
\begin{equation*}
    \sigma^2 = \frac{\left(G_{\mathrm{SSM},\mathbf R_{Q,\mu}}^{(Q)}\right)^2}{\Xi}.
\end{equation*}
The proximal term is deterministic under this conditioning and does not alter the gradient-noise variance. Conditional independence across clients gives the corresponding weighted-noise variance bound $\sigma^2\sum_kp_k^2$.

Lemma~\ref{lem:multilayer_ssm_gamma_tv_bound} supplies the common client-gradient disagreement bound
\begin{equation*}
    \Gamma^2 = 4\left(G_{\mathrm{SSM},\mathbf R_{Q,\mu}}^{(Q)}\right)^2 \mathcal H_{\mathrm{TV}}.
\end{equation*}
The common proximal term cancels when comparing client gradients with their weighted average at the same parameter tuple.

The remaining local-drift and descent calculations in the proof of Theorem~\ref{cor:fedprox_ssm_convergence} depend only on these smoothness, variance, and disagreement bounds, the conditional sampling assumptions, and convexity of the neighborhoods. Applying those calculations with the multi-layer constants above under the state learning-rate conditions and substituting the defined values of $\sigma^2$ and $\Gamma^2$ yields the desired inequality in Corollary~\ref{thm:multilayer_fedprox_ssm_convergence}, completing the proof. 
\qed

\section{Supplementary Corollaries}\label{sec:supplementary_corollaries}
In this section, we provide supplementary corollaries to the theorems presented in the preceding sections and in the main manuscript. These corollaries simplify the preceding results and make explicit their dependence on SSM-specific quantities, including stability, parameter norms, network depth, and related architectural properties.

\subsection{Single-layer convergence rate corollaries}
In this section, we provide convergence rate guarantees for FedAvg~\citep{mcmahan2017communication} and FedProx~\citep{MLSYS2020_1f5fe839} learning frameworks. Every convergence corollary retains the assumptions of its referenced result. Furthermore, we establish and retain the following conventions:
\begin{align}
    F_0 &:=f(W^0)-f^*, \\
    P_2 &:=\sum_{k=1}^K p_k^2, \\
    \mathcal H_{\mathrm{TV}} & :=\sum_kp_k
\operatorname{TV}(\mathcal P_k,\bar{\mathcal P})^2
\end{align}
All $\mathcal O$ expressions denote upper bounds on nonnegative quantities.

Note that Corollary~\ref{cor:supp_single_fedavg} is stated in Section~\ref{sec:theoretical_analysis}, with this section containing its proof. 

\subsubsection{Corollary~\ref{cor:supp_single_fedavg}: Single-layer FedAvg convergence rate Proof}
\textbf{Proof of Corollary~\ref{cor:supp_single_fedavg}:}

By Lemmas~\ref{lem:ssm_local_smoothness} and~\ref{lem:ssm_fedavg_g_constant}:
\begin{equation}
    G_{\mathrm{SSM},R}=\mathcal O\left(\frac{M_R}{(1-\rho)^2}\right), \qquad L_{\mathrm{SSM},R}=\mathcal O\left(\frac{M_R^2}{(1-\rho)^4}\right)
\end{equation}

Substituting these bounds into Theorem~\ref{thm:fedavg_ssm_convergence}, collecting the two local-drift terms, and using the fact that $P_2 \leq 1$, yields the inequality in Corollary~\ref{cor:supp_single_fedavg}. \qed

\subsubsection{Corollary~\ref{cor:supp_single_fedprox}: Single-layer FedProx convergence rate and Proof}
\begin{corollary}[Single-layer FedProx convergence rate]
\label{cor:supp_single_fedprox}
Under the assumptions and step-size conditions of Theorem~\ref{cor:fedprox_ssm_convergence}, given its reduced envelopes $M_X^\mu=M_X-R+R_\mu$, define:
\begin{equation*}
    M_{R_\mu} := (1 + M_B^\mu)(1 + M_C^\mu)(1 + M_\Delta^\mu)(1 + M_\Delta^\mu + M_A^\mu).
\end{equation*}
Then
\begin{equation}
\label{eq:supp_single_fedprox}
    \frac{1}{T}\sum_{r=0}^{T-1} \mathbb E \left\|\nabla f(W^r)\right\|_F^2 = \mathcal O\left(\frac{F_0}{\eta ET} + \frac{\eta M_{R_\mu}^4}{\Xi (1-\rho)^8} + \eta^2 \frac{M_{R_\mu}^2}{(1-\rho)^4}\left(\frac{M_{R_\mu}^2}{(1-\rho)^4}+\mu\right)^2 (E-1) \left(\frac{1}{\Xi}+E\mathcal H_{\mathrm{TV}}\right)
    \right).
\end{equation}
The implicit constant has the same dependencies as in Corollary~\ref{cor:supp_single_fedavg} and is independent of $\mu$. 
\end{corollary}

\textbf{Proof of Corollary~\ref{cor:supp_single_fedprox}:} The single-layer scaling argument, evaluated with $M_X^\mu$, gives $G_{\mathrm{SSM},R_\mu}=\mathcal O\left(\frac{M_{R_\mu}}{(1-\rho)^2}\right)$ and $L_{\mathrm{SSM},R_\mu}=\mathcal O\left(\frac{M_{R_\mu}^2}{(1-\rho)^4}\right)$.

Consequently $L_{\mathrm{prox},R_\mu} = \mathcal O\left(\frac{M_{R_\mu}^2}{(1-\rho)^4}+\mu\right)$. Substituting these results into Theorem~\ref{cor:fedprox_ssm_convergence} and using the fact that $P_2 \leq 1$ proves the bound.
\qed

Note that although $\frac{M_{R_\mu}}{(1-\rho)^2} \leq \frac{M_{R}}{(1-\rho)^2}$ for the same original envelopes, the bound in Corollary~\ref{cor:supp_single_fedprox} also contains additional $\mu$ terms. As such, these bounds expose a trade-off and do not imply that every positive proximal coefficient improves the convergence guarantee. 

\subsection{Multi-Layer SSM Gradient Scaling Bounds}
In this section we provide simplified bounds on the multi-layer gradient, where all per-layer quantities are scaled to a single, common value. For the remainder of this section, we will utilize the following notation: for a fixed model in Theorem~\ref{thm:ssm_multi_layer_grad_bound}, let
\begin{equation}
\label{eq:supp_common_envelopes}
    \begin{aligned}
    L_x &:= \max_{1\leq q\leq Q}L_{x,q},\qquad \rho := \max_{1\leq q\leq Q}\rho_q<1,\\
    M_X &:= \max_{1\leq q\leq Q}\left\|W_X^{(q)}\right\|_2, \qquad X\in\{\Delta,B,C,A\},
    \end{aligned}
\end{equation}
where $W_A^{(q)}:=A^{(q)}$.

For neighborhood-uniform statements, take $M_X:=\max_q M_{X,q}$ using the deterministic envelopes in the multi-layer convergence results. The bounds on $L_x$ and $\rho$ are then uniform on those neighborhoods as well. Note that common upper bounds suffice - the layers do not need to have identical weights or identical stability factors.

Using these common quantities, define
\begin{align}
    M_{K_\Delta} &:= M_\Delta L_x + \sqrt D\log2, \label{eq:supp_common_layer_constants_start}\\
    M_{K_B} &:= M_BL_x^2,\\
    M_{K_W} &:= NDM_{K_B}M_{K_\Delta}(M_{K_\Delta} + \sqrt D M_A L_x) \frac{\rho}{1-\rho}\nonumber\\
        &\quad + L_x^2\sqrt N M_{K_\Delta} + L_x\sqrt D M_{K_B},\\
    H &:= \frac{M_{K_B}M_{K_\Delta}\sqrt{ND}}{1-\rho},\\
    J &:= \frac{M_{K_W}}{1-\rho},\\
    \Lambda &:= \rho M_AM_\Delta H +M_BL_x(2M_{K_\Delta}+L_xM_\Delta),\\
    \kappa &:=M_C\left(H+\frac{L_x\Lambda}{1-\rho}\right). 
    \label{eq:supp_common_layer_constants_end}
\end{align}

\subsubsection{Corollary~\ref{cor:supp_common_gradient}: Multi-layer gradient under common layer bounds and Proof}
\begin{corollary}[Multi-layer gradient under common layer bounds] \label{cor:supp_common_gradient}
Under the assumptions of Theorem~\ref{thm:ssm_multi_layer_grad_bound}, the quantities in \eqref{eq:supp_common_layer_constants_start}-\eqref{eq:supp_common_layer_constants_end} satisfy, for every admissible sequence pair $\xi$ of finite length $S$,
\begin{equation*}
    \left\|\nabla_W\mathcal L^{S,(Q)}(W;\xi)\right\|_F \leq L_\ell L_x(H + M_CJ)\sqrt{\sum_{q=0}^{Q-1}\kappa^{2q}}.
\end{equation*}

Furthermore, the following inequalities hold for different values of $\kappa$:
\begin{equation*}
     \left\|\nabla_W\mathcal L^{S,(Q)}(W;\xi)\right\|_F\leq  L_\ell L_x(H + M_CJ) \begin{cases}
     \frac{1}{\sqrt{1-\kappa^2}}, &0\leq\kappa<1,\\
     \sqrt Q, &\kappa=1,\\
     \sqrt{(\kappa^{2Q}-1)/(\kappa^2-1)}, &\kappa>1.
    \end{cases}
\end{equation*}
\end{corollary}

\textbf{Proof of Corollary~\ref{cor:supp_common_gradient}:} All defining expressions~\ref{eq:supp_common_layer_constants_start}-~\ref{eq:supp_common_layer_constants_end} are nondecreasing in the nonnegative parameter, input, and stability upper bounds. As a result, each own-layer factor in Theorem~\ref{thm:ssm_multi_layer_grad_bound} is at most $L_x(H + M_CJ)$, and each
downstream sensitivity is at most $\kappa$. Substitution into the inequality from Theorem~\ref{thm:ssm_multi_layer_grad_bound} gives the right-hand side bounded by $L_\ell L_x(H + M_CJ) \left(\sum_{q=1}^Q\kappa^{2(Q-q)}\right)^{1/2}$. Evaluating the geometric sum proves both claimed inequalities. \qed

For common envelopes independent of $Q$, the upper bound is uniform in depth when $\kappa<1$, grows at most as $\sqrt Q$ when $\kappa=1$, and can grow exponentially when $\kappa>1$. Temporal stability $\rho<1$ alone does not imply $\kappa<1$, and an exponentially growing upper bound does not establish that actual gradients grow exponentially.

\subsubsection{Corollary~\ref{cor:supp_gradient_scaling}: Multi-layer gradient rate under common layer bounds and Proof}
\begin{corollary}[Multi-layer gradient rate under common layer bounds]\label{cor:supp_gradient_scaling}
Under the assumptions of Theorem~\ref{thm:ssm_multi_layer_grad_bound}, define
\begin{equation*}
    \mathfrak M_q := \left(1+\|W_B^{(q)}\|_2\right)\left(1+\|W_C^{(q)}\|_2\right)\left(1+\|W_\Delta^{(q)}\|_2\right)\left(1+\|W_\Delta^{(q)}\|_2+\|A^{(q)}\|_2\right).
\end{equation*}
Then, given the $\kappa_q$ defined in that theorem, it holds that:
\begin{equation*}
    \left\|\nabla_W\mathcal L^{S,(Q)}(W;\xi)\right\|_F =\mathcal O\left[ \left(\sum_{q=1}^Q \frac{\mathfrak M_q^2}{(1-\rho_q)^4} \prod_{j=q+1}^Q\kappa_j^2\right)^{1/2}\right].
\end{equation*}
Furthermore, with the common envelopes defined above, let $\mathfrak M := (1+M_B)(1+M_C)(1+M_\Delta)(1+M_\Delta+M_A)$. Then the following holds:
\begin{align*}
    \left\|\nabla_W\mathcal L^{S,(Q)}(W;\xi)\right\|_F &= \mathcal O\left(
    \frac{\mathfrak M}{(1-\rho)^2}\sqrt{\sum_{q=0}^{Q-1}\kappa^{2q}}\right) \\ &= \mathcal O\left(
    \frac{\mathfrak M}{(1-\rho)^2}\begin{cases}
     \frac{1}{\sqrt{1-\kappa^2}}, &0\leq\kappa<1,\\
     \sqrt Q, &\kappa=1,\\
     \sqrt{(\kappa^{2Q}-1)/(\kappa^2-1)}, &\kappa>1.
    \end{cases} \right) .
\end{align*}
The implicit constants depend only on $N,D,L_x,L_\ell$, where $L_ x\geq \max_q L_{x,q}$, and are independent of $Q$, sequence length, parameter norms, and stability factors.
\end{corollary}

\textbf{Proof of Corollary~\ref{cor:supp_gradient_scaling}}: The single-layer estimates give $L_{x,q}\left(H_q+\left\|W_C^{(q)}\right\|_2J_q\right) \leq C\mathfrak M_q/(1-\rho_q)^2$ for some common constant $C$. Substituting this into Theorem~\ref{thm:ssm_multi_layer_grad_bound} we obtain the first equality of in Corollary~\ref{cor:supp_gradient_scaling}.

Additionally, substituting $L_x(H + M_CJ)=\mathcal O\left(\mathfrak M/(1-\rho)^2\right)$ into Corollary~\ref{cor:supp_common_gradient} proves the second equality. \qed

\subsection{Multi-layer convergence under common layer bounds}

In the following corollaries and their proof, we utilize the neighborhood-uniform common envelopes defined in~\eqref{eq:supp_common_envelopes} and~\ref{eq:supp_common_layer_constants_start}-~\ref{eq:supp_common_layer_constants_end}, with the same $L_\ell,H_\ell$ throughout all considered neighborhoods. The following reduction replaces the original layer-dependent sensitivity recursions by three finite geometric sums.

We further define the common auxiliary constants
\begin{align}
    d_* &:= L_x + M_\Delta \label{eq:supp_common_second_constants_start}, \\
    b_* &:= \frac{1}{4}d_*^2+2,\\
    c_* &:= M_{K_\Delta} + M_A d_*, \\
    j_* &:= J + \frac{\Lambda}{1-\rho},\\
    \mathcal E_* &:= \rho H(c_*^2+2d_*+M_Ab_*) + 2 \rho c_*j_*+2M_{K_\Delta} L_x\nonumber\\
       &\quad + 2(L_x + M_B)(L_x d_* + M_{K_\Delta}) + M_BL_x(L_xb_* + 2d_*),\\
    B_*& := \frac{M_CL_x\mathcal E_*}{1-\rho} +2(L_x + M_C)j_*+2H. \label{eq:supp_common_second_constants_end}
\end{align}
For $Q\geq1$, define
\begin{align}
    \widehat G_Q &:= L_\ell L_x(H + M_CJ) \sqrt{\sum_{q=0}^{Q-1}\kappa^{2q}}, \label{eq:G_Q_hat}\\
    \widehat L_Q  &:= H_\ell L_x^2(H + M_CJ)^2  \sum_{q=0}^{Q-1}\kappa^{2q} \nonumber\\
    &\quad + L_\ell B_*\Bigg[\sum_{q=0}^{Q-1}\kappa^q + L_x^2(H+M_CJ)^2\sum_{q=1}^Q\kappa^{Q-q}\left(\sum_{j=0}^{q-1}\kappa^{2j}\right)\Bigg].\label{eq:L_Q_hat}
\end{align}

In particular, for $\kappa \neq1$, 
\begin{equation}
    \sum_{q=0}^{Q-1}\kappa^q = \frac{1-\kappa^Q}{1-\kappa}, \qquad \sum_{q=1}^Q\kappa^{Q-q} \left(\sum_{j=0}^{q-1}\kappa^{2j}\right) = \frac{(1-\kappa^Q)(1-\kappa^{Q+1})}{(1-\kappa)(1-\kappa^2)},\qquad \kappa \neq1.
\end{equation}
For $\kappa=1$, the two sums reduce to $Q$ and $Q(Q+1)/2$, respectively. 

\subsubsection{Corollary~\ref{cor:supp_common_smoothness}: Common-layer smoothness and depth dependence and Proof}
\begin{corollary}[Common-layer smoothness and depth dependence] \label{cor:supp_common_smoothness}
Under the neighborhood assumptions of Lemma~\ref{lem:multilayer_local_smoothness}, the constants $G_{\mathrm{SSM},\mathbf R_Q}^{(Q)}$ and $L_{\mathrm{SSM},\mathbf R_Q}^{(Q)}$ defined in Corollary~\ref{thm:multilayer_fedavg_ssm_convergence} satisfy:
\begin{equation*}
    G_{\mathrm{SSM},\mathbf R_Q}^{(Q)} \leq \widehat G_Q,
    \qquad L_{\mathrm{SSM},\mathbf R_Q}^{(Q)} \leq \widehat L_Q.
\end{equation*}
If all common envelopes and regularity constants are fixed independently of $Q$, their depth dependence is
\begin{equation*}
    \left(\widehat G_Q,\widehat L_Q\right)=
    \begin{cases}
      \left(\mathcal O(1),\mathcal O(1)\right), &0\leq\kappa < 1,\\
      \left(\mathcal O(\sqrt Q),\mathcal O(Q^2)\right), &\kappa = 1,\\
      \left(\mathcal O(\kappa^Q),\mathcal O(\kappa^{2Q})\right), &\kappa > 1.
    \end{cases}
\end{equation*}
Here only $Q$ varies: the implicit constants may depend on the fixed common envelopes and on $\kappa$.
\end{corollary}

\textbf{Proof of Corollary~\ref{cor:supp_common_smoothness}}: After replacing the layerwise quantities by their common upper bounds, the first-sensitivity recursion yields, for every $q\geq1$,
\begin{equation}
    \Theta_q^2 \leq L_x^2(H + M_CJ)^2\sum_{j=0}^{q-1}\kappa^{2j}.
\end{equation}
Hence,
\begin{equation}
    G_{\mathrm{SSM},\mathbf R_Q}^{(Q)} =L_\ell\Theta_Q \leq L_\ell L_x(H+M_CJ)\sqrt{\sum_{j=0}^{Q-1}\kappa^{2j}}  =\widehat G_Q.
\end{equation}

For the second-sensitivity recursion, the terms independent of $\Omega_{q-1}$ are bounded by
\begin{equation}
    B_*\left(1+\Theta_{q-1}^2\right),
\end{equation}
This follows because, after replacing the layerwise quantities by their common upper bounds, the terms independent of $\Omega_{q-1}$ form a polynomial of degree at most two in $\Theta_{q-1}$ with nonnegative coefficients. The sum of these coefficients is $B_*$, obtained by setting $\Theta_{q-1}=1$. Since each of $1$, $\Theta_{q-1}$, and $\Theta_{q-1}^2$ is at most $1+\Theta_{q-1}^2$, the displayed bound follows.
Therefore, unrolling the recursion gives
\begin{align}
    \Omega_Q &\leq B_*\sum_{q=1}^Q \kappa^{Q-q}\left(1+\Theta_{q-1}^2\right) \nonumber \\ 
    &\leq B_*\left[\sum_{q=0}^{Q-1}\kappa^q + L_x^2(H + M_CJ)^2 \sum_{q=1}^Q \kappa^{Q-q} \left(\sum_{j=0}^{q-1}\kappa^{2j}\right)\right],
\end{align}
where the first-sensitivity bound was used in the second inequality. Substituting the bounds on $\Theta_Q$ and $\Omega_Q$ into
\begin{equation}
    L_{\mathrm{SSM},\mathbf R_Q}^{(Q)} = H_\ell\Theta_Q^2+L_\ell\Omega_Q,
\end{equation}
we get:
\begin{equation} 
    L_{\mathrm{SSM},\mathbf R_Q}^{(Q)} \leq \widehat L_Q.
\end{equation}
Finally, evaluating the geometric sums gives
\begin{equation}
    \left(\widehat G_Q,\widehat L_Q\right) =    
        \begin{cases}   
        \left(\mathcal O(1),\mathcal O(1)\right), &0\leq\kappa<1,\\
        \left(\mathcal O(\sqrt Q),\mathcal O(Q^2)\right), &\kappa=1,\\
        \left(\mathcal O(\kappa^Q),\mathcal O(\kappa^{2Q})\right), &\kappa>1,
        \end{cases}
\end{equation}
which proves the stated depth dependence. \qed

\subsubsection{Corollary~\ref{cor:supp_common_fedavg}: Multi-layer FedAvg under common layer bounds and Proof}
\begin{corollary}[Multi-layer FedAvg under common layer bounds]
\label{cor:supp_common_fedavg}
Under the assumptions of Corollary~\ref{thm:multilayer_fedavg_ssm_convergence}, let $\widehat G_Q$ and $\widehat L_Q$ be the common upper bounds in~\eqref{eq:G_Q_hat} and~\eqref{eq:L_Q_hat}, evaluated with the neighborhood-uniform common envelopes introduced above. If
\begin{equation*}
    \eta E\widehat L_Q\leq\frac12, \qquad 12\eta^2\widehat L_Q^2E(E-1)\leq1,
\end{equation*}
then
\begin{equation*}
\begin{aligned}
    \frac{1}{T}\sum_{r=0}^{T-1}\mathbb E\left\|\nabla f(W^r)\right\|_F^2 = \mathcal O\Bigg(&\frac{F_0}{\eta ET} + \frac{\eta\widehat L_Q\widehat G_Q^2P_2}{\Xi}\\
    &+\eta^2\widehat L_Q^2\widehat G_Q^2(E-1)\left[\frac1\Xi+E\mathcal H_{\mathrm{TV}}\right]\Bigg).
\end{aligned}
\end{equation*}
The implicit constant is numerical. For common envelopes and loss-regularity constants fixed independently of $Q$, the depth dependence of $\widehat G_Q$ and $\widehat L_Q$ is explicit in~\eqref{eq:G_Q_hat} and~\eqref{eq:L_Q_hat} and Corollary~\ref{cor:supp_common_smoothness}.
\end{corollary}

\textbf{Proof of Corollary~\ref{cor:supp_common_fedavg}:} Corollary~\ref{cor:supp_common_smoothness} gives
\begin{equation*}
    G_{\mathrm{SSM},\mathbf R_Q}^{(Q)}\leq\widehat G_Q, \qquad L_{\mathrm{SSM},\mathbf R_Q}^{(Q)}\leq\widehat L_Q.
\end{equation*}
Consequently, the displayed step-size conditions imply those of Corollary~\ref{thm:multilayer_fedavg_ssm_convergence}. Substituting these upper bounds into that corollary and grouping the two local-drift terms proves the stated claim. \qed

\subsubsection{Corollary~\ref{cor:supp_common_fedprox}:
Multi-layer FedProx under common layer bounds and Proof}
\begin{corollary}[Multi-layer FedProx under common layer bounds]\label{cor:supp_common_fedprox}
Under the assumptions of Corollary~\ref{thm:multilayer_fedprox_ssm_convergence}, retain its original radii $R_q$, effective radii $R_{\mu,q}$, and reduced envelopes $M_{X,q}^{\mu} = M_{X,q}-R_q+R_{\mu,q}$.

For this corollary and its proof, evaluate the existing common bounds $\widehat G_Q$ and $\widehat L_Q$ in~\eqref{eq:G_Q_hat} and~\eqref{eq:L_Q_hat} using
\begin{equation*}
\begin{aligned}
    L_x&:=\max_{1\leq q\leq Q}L_{x,q}, \qquad \rho:=\max_{1\leq q\leq Q}\rho_q,\\ 
    M_X&:=\max_{1\leq q\leq Q}M_{X,q}^{\mu}, \qquad X\in\{\Delta,B,C,A\}.
\end{aligned}
\end{equation*}
All auxiliary quantities in
\eqref{eq:supp_common_layer_constants_start}-\eqref{eq:supp_common_layer_constants_end} and \eqref{eq:supp_common_second_constants_start}-\eqref{eq:supp_common_second_constants_end}, including $\kappa$ and $B_*$, are evaluated with these reduced common envelopes. The loss-regularity constants $L_\ell,H_\ell$ are unchanged.

If $0<\eta\mu\leq1$ and
\begin{equation*}
    \eta E\widehat L_Q\leq\frac{1}{2}, \qquad 12\eta^2\left(\widehat L_Q+\mu\right)^2E(E-1)\leq1,
\end{equation*}
then
\begin{equation*}
\begin{aligned}
    \frac{1}{T}\sum_{r=0}^{T-1} \mathbb E\left\|\nabla f(W^r)\right\|_F^2 = \mathcal O\Bigg(&\frac{F_0}{\eta ET} + \frac{\eta\widehat L_Q\widehat G_Q^2P_2}{\Xi}\\
    &+\eta^2(\widehat L_Q+\mu)^2\widehat G_Q^2(E-1)\left[\frac1\Xi+E\mathcal H_{\mathrm{TV}}\right]\Bigg).
\end{aligned}
\end{equation*}
The implicit constant is numerical and independent of $\mu$.
\end{corollary}

\textbf{Proof of Corollary~\ref{cor:supp_common_fedprox}:}
The displacement argument in the proof of Corollary~\ref{thm:multilayer_fedprox_ssm_convergence} places all local iterates in
$\mathcal N_{\mathbf R_{Q,\mu}}(W^r)$. This product of convex balls also contains their weighted aggregate and the segments used in the smoothness argument.

Since $M_{X,q}^{\mu}\leq M_X$ for every layer, the common-envelope construction of Corollary~\ref{cor:supp_common_smoothness} gives
\begin{equation*}
    G_{\mathrm{SSM},\mathbf R_{Q,\mu}}^{(Q)} \leq\widehat G_Q, \qquad L_{\mathrm{SSM},\mathbf R_{Q,\mu}}^{(Q)} \leq\widehat L_Q.
\end{equation*}
Therefore,
\begin{equation*}
    L_{\mathrm{prox},\mathbf R_{Q,\mu}}^{(Q)} = L_{\mathrm{SSM},\mathbf R_{Q,\mu}}^{(Q)}+\mu \leq\widehat L_Q+\mu.
\end{equation*}
The displayed step-size conditions consequently imply those of Corollary~\ref{thm:multilayer_fedprox_ssm_convergence}. Substituting these upper bounds into that corollary and grouping the two local-drift terms proves the stated claim. \qed

\subsection{SSM Convergence under Distributed Learning}
In this section we provide and prove Corollary~\ref{cor:supp_one_local_step}: SSM Convergence under Distributed Learning and the corresponding rate corollary - a setting where every client runs only one step of gradient descent between model weights exchange. 

\begin{corollary}[SSM Convergence under Distributed Learning]
\label{cor:supp_one_local_step}
Under the assumptions of Theorem~\ref{thm:fedavg_ssm_convergence}, let $E=1$. If $\eta L_{\mathrm{SSM},R}\leq1/2$, then
\begin{equation*}
    \frac{1}{T}\sum_{r=0}^{T-1} \mathbb E\left\|\nabla f(W^r)\right\|_F^2 \leq\frac{4F_0}{\eta T} + \frac{4\eta L_{\mathrm{SSM},R} G_{\mathrm{SSM},R}^2P_2}{\Xi}.
\end{equation*}

Similarly, under the assumptions of Corollary~\ref{thm:multilayer_fedavg_ssm_convergence} with $E=1$, if $\eta L_{\mathrm{SSM},\mathbf R_Q}^{(Q)}\leq1/2$, then
\begin{equation*}
    \frac{1}{T}\sum_{r=0}^{T-1}\mathbb E\left\|\nabla f(W^r)\right\|_F^2 \leq \frac{4 F_0}{\eta T} + \frac{4\eta L_{\mathrm{SSM},\mathbf R_Q}^{(Q)}\left(G_{\mathrm{SSM},\mathbf R_Q}^{(Q)}\right)^2 P_2}{\Xi}.
\end{equation*}

Both bounds also apply to FedProx with one local step, under the same FedAvg assumptions and step-size conditions,
for every $\mu>0$. No restriction on $\eta\mu$ is needed for this one-step conclusion.
\end{corollary}

\textbf{Proof of Corollary~\ref{cor:supp_one_local_step}:} Setting $E=1$ in Theorem~\ref{thm:fedavg_ssm_convergence} or Corollary~\ref{thm:multilayer_fedavg_ssm_convergence} eliminates both terms containing $E-1$ and makes the second step-size condition automatic. 

For FedProx, initialization at the broadcast model gives
\begin{equation}
    W_k^{r,1} = W^r-\eta\left( g_{k,\Xi}^{r,0}+\mu(W^r-W^r)\right) = W^r-\eta g_{k,\Xi}^{r,0}.
\end{equation}
As such, for every $\mu>0$, its one-step update and weighted aggregate coincide with those of FedAvg. Applying the corresponding FedAvg result proves the claim. \qed

\begin{corollary}[SSM Convergence Rate under Distributed Learning] \label{cor:distributed_learning_rate}
Under the assumptions of Corollary~\ref{cor:supp_single_fedavg}, let $E=1$. Under $\eta L_{\mathrm{SSM},R}\leq1/2$, let $M_B,M_C,M_\Delta,$ and $M_A$ denote uniform upper bounds on $\left\|W_B\right\|_2,\left\|W_C\right\|_2,\left\|W_\Delta\right\|_2,$ and $\left\|A\right\|_2$, respectively, over the $R$-local neighborhoods $\mathcal N_R(W^r)$ and define:
\begin{equation*}
    M_R := (1 + M_B)(1 + M_C)(1 + M_\Delta)(1 + M_\Delta + M_A).
\end{equation*}
Under these conditions, the global optimization trajectory $\{W^r\}_{r=0}^T$ satisfies
\begin{equation*} 
    \frac{1}{T}\sum_{r=0}^{T-1} \mathbb E \left\|\nabla f(W^r)\right\|_F^2 = \mathcal O\left(\frac{F_0}{\eta ET} + \frac{\eta M_R^4}{\Xi(1-\rho)^8}\right).
\end{equation*}
The quantities $N,D,L_x,L_\ell,$ and $H_\ell$ affect only the implicit constant and do not alter the displayed scaling with the optimization parameters, parameter envelopes, stability, or sequence length.
\end{corollary}

\textbf{Proof of Corollary~\ref{cor:distributed_learning_rate}:} Setting $E=1$ in Corollary~\ref{cor:supp_single_fedavg} eliminates both terms containing $E-1$, yielding the equality in the Corollary. \qed

\section{Discussion on Additional Distributed Learning Frameworks}\label{appendix:additonal_distributed_learning_frameworks}
The FedAvg and FedProx results (Theorems~\ref{thm:fedavg_ssm_convergence}
and~\ref{cor:fedprox_ssm_convergence}) relate convergence of selective SSM training to recurrent stability, parameter norms, local client drift, and data heterogeneity. In this section, we use these quantities to interpret additional federated learning methods, identifying how their local updates and server aggregation interact with the SSM-specific bounds. While we do not establish convergence guarantees for these methods, the discussion provides a qualitative, theory-guided characterization of their convergence behavior, offering intuition and motivating expectations about their performance under different training conditions. 

The discussion below assumes SGD-style updates, containment in the stated neighborhood, and regularity assumptions. The connection to the AdamW-based Mamba2 experiments is therefore qualitative rather than a direct application of the convergence results.

\textbf{SCAFFOLD.} SCAFFOLD~\citep{pmlr-v119-karimireddy20a} introduces server and client control variates (correction terms), which act as drift-correction terms for local client updates. In its classical SGD form, the local update is
\begin{equation}
    W_{k}^{r,e+1} = W_{k}^{r,e} - \eta\left(g_k(W_{k}^{r,e};\xi_{k}^{r,e}) + c^r - c_k^r\right),
\end{equation}
where $c^r =\sum_k p_k c_k^r$ approximates the global update direction and $c_k^r$ approximates the client-specific drift direction. More generally, allowing a correction scale $\lambda_c > 0$, define the corrected local direction
\begin{equation}
    d_{k}^{r,e} := g_k(W_{k}^{r,e};\xi_{k}^{r,e}) + \lambda_c(c^r-c_k^r).
\end{equation}
The correction term $c^r-c_k^r$ is intended to reduce the component of the local update that points toward the client-specific objective rather than the global objective. 

Assume that the correction term satisfies:
\begin{equation}
    \|c^r-c_k^r\|_F \leq C_R,
\end{equation}
then the corrected direction satisfies: 
\begin{equation}
    \|d_{k}^{r,e}\|_F \leq G_{\mathrm{SSM},R} + \lambda_c C_R := G_{\mathrm{SCAF},R},
\end{equation}
and the corresponding local drift radius can be bounded by
\begin{equation}
    R_{\mathrm{SCAF}} \leq \eta E G_{\mathrm{SCAF},R}.
\end{equation}
While this mechanism does not directly change the SSM smoothness constant $L_{\mathrm{SSM},R}$, and could even increase the worst-case local drift radius through the additional correction term, the main effect of SCAFFOLD is to reduce the component of each client update that is biased toward the local objective rather than the global objective. In other words, SCAFFOLD can be interpreted as modifying the effective heterogeneity term. Specifically, define the corrected heterogeneity measure
\begin{equation}
    \Gamma_{\mathrm{SCAF},R}^{2} := \sup_{W\in\mathcal N_R(W^r)} \sum_{k=1}^{K} p_k \left\|\nabla\mathcal L_k(W) - \nabla f(W) + \lambda_c(c^r-c_k^r)\right\|_F^2.
\end{equation}
When the control terms accurately estimate the client-specific drift directions, the correction $\lambda_c(c^r-c_k^r)$ partially cancels the difference $\nabla \mathcal L_k(W)-\nabla f(W)$. In this case, one expects
\begin{equation}
    \Gamma_{\mathrm{SCAF},R}^{2} < \Gamma_{\mathrm{SSM},R}^{2}.
\end{equation}
This motivates a potentially reduced client disagreement, although a convergence bound must also account for the evolution and accuracy of the control variates.

This does not imply that SCAFFOLD has a smaller drift radius than FedAvg. In fact, in nearly IID settings, the uncorrected client gradients are already approximately aligned with the global update direction. As a result, the reduction from $\Gamma_{\mathrm{SSM},R}^{2}$ to $\Gamma_{\mathrm{SCAF},R}^{2}$ may be limited, while the additional correction term can increase the effective update radius and potentially degrade performance. In contrast, under non-IID data, the heterogeneity correction can provide a meaningful benefit by reducing the statistical drift term, which is amplified in the SSM-specific bound through the architecture-dependent gradient and smoothness quantities.

\textbf{FedAdam.} FedAdam~\citep{reddi2021adaptive} keeps the same FedAvg-style local client training phase, but replaces the direct server averaging step with an adaptive server optimizer. Let
\begin{equation}
    \bar W^{r+1} = \sum_{k=1}^{K}p_k W_{k}^{r,E}
\end{equation}
denote the standard FedAvg aggregated parameters, and define the server pseudo-gradient:
\begin{equation}
    \Delta^r := \bar W^{r+1}-W^r.
\end{equation}
FedAdam tracks moving averages of the server pseudo-gradient and its coordinatewise square as:
\begin{align}
    m^r &= \beta_1 m^{r-1} + (1-\beta_1)\Delta^r,\\
    v^r &= \beta_2 v^{r-1} + (1-\beta_2)(\Delta^r\odot \Delta^r),
\end{align}
and updates the server model coordinatewise as
\begin{equation}
    W^{r+1} = W^r + \alpha_s\frac{m^r}{\sqrt{v^r}+\tau_s},
\end{equation}
where $\alpha_s$ is the server learning rate and $\tau_s > 0$ is the denominator offset. 

Since the client-side training phase is unchanged from FedAvg, FedAdam does not directly reduce the local SSM drift radius. The local client trajectories remain controlled by the same worst-case bound
\begin{equation}
    R_{\mathrm{local}} \leq \eta E G_{\mathrm{SSM},R}.
\end{equation}
The main change is instead introduced at the server: since $\sqrt{v^r}+\tau_s \geq \tau_s$, the server update satisfies the coarse bound
\begin{equation}
    \|W^{r+1}-W^r\|_F \leq \frac{\alpha_s}{\tau_s}\|m^r\|_F.
\end{equation}

As such, even though the adaptive server mechanism can be beneficial when it rescales poorly conditioned coordinates in a useful way, it can also increase the effective global update radius when $\frac{\alpha_s}{\tau_s}$ is large or when the moment estimates are not well aligned with a stable descent direction. 

This can be particularly detrimental for selective SSMs - an aggressive server-side adaptive step can increase the effective neighborhood radius, resulting in larger potential values of $L_{\mathrm{SSM},R}$ and $G_{\mathrm{SSM},R}$, and making the global trajectory more sensitive to non-IID client updates. Compared to FedProx, FedAdam does not explicitly constrain local client drift, and compared to SCAFFOLD, FedAdam does not correct the client heterogeneity term. Instead, it rescales the aggregated update after local drift has already occurred.

\textbf{FedAlign.} The FedAlign baseline we consider is an adaptation of the state-alignment method proposed in~\citep{11099331}. The original FedAlign algorithm was developed for federated system identification with classical linear state-space models of the form
\begin{align}
    \tilde x_{k+1}=\tilde A\tilde x_k+\tilde B u_k,  \qquad \tilde y_k=\tilde C\tilde x_k+\tilde D u_k
\end{align}
where the system matrices $\tilde A,\tilde B,\tilde C,\tilde D$ are fixed model parameters. In that setting, the main alignment issue is that different clients may learn dynamically equivalent SSMs in different state-coordinate bases, so directly averaging their parameters can distort the global system dynamics. FedAlign addresses this by using similarity transformations to map local SSMs into a common parameter basin before aggregation.

Our setting differs from this original SYSID formulation because Mamba-style selective SSMs do not have a single fixed set of state-space matrices applied uniformly across the sequence. Instead, the effective transition and input terms are coupled to learned projections and input-dependent discretization quantities. Therefore, our FedAlign implementation adapts the same state-alignment principle to Mamba2 by aligning coupled SSM state, group, and head-indexed parameter blocks before aggregation.

As such, we can write the adapted FedAlign aggregation as
\begin{equation}
    W^{r+1} = \sum_{k=1}^{K} p_k \operatorname{Align}_k(W_k^{r,E}),
\end{equation}
where $\operatorname{Align}_k$ is the client-specific alignment map applied to the coupled Mamba2 SSM blocks. 

Next, as long as $\operatorname{Align}_k$ is a norm-preserving permutation or matching of coupled coordinates, then 
\begin{equation}
    \left\|\operatorname{Align}_k(W_k^{r,E})-\operatorname{Align}_k(W^r)\right \|_F = \left\|W_k^{r,E}-W^r\right \|_F.
\end{equation}

The identity preserves distance when the same alignment is applied to both models, not distance to the fixed broadcast model $W^r$. FedAlign can reduce aggregation error from mismatched state coordinates, but norm preservation alone does not control local drift, statistical heterogeneity, or stability-sensitive gradient growth in selective SSMs. FedAvg-like behavior is therefore expected when coordinate mismatch is limited. Containment of aligned models requires separate verification.

\textbf{FedDyn.} FedDyn~\citep{acar2021federated} combines proximal regularization with a dynamic gradient correction. In the full-participation, weighted formulation, client $k$ maintains a state $\nu_k^r$, initialized at zero, and optimizes
\begin{equation}
    g_k^r(W) := \mathcal L_k(W) - \langle \nu_k^r,W\rangle_F + \frac{\alpha}{2}\|W-W^r\|_F^2, \qquad \nu_k^{r+1} := \nu_k^r-\alpha(W_k^{r,E}-W^r),
\end{equation}
where $\alpha>0$. Writing $\nu^r:=\sum_k p_k\nu_k^r$ and $\bar W^{r+1}:=\sum_k p_kW_k^{r,E}$, the server updates are $W^{r+1}=\bar W^{r+1}-\nu^{r+1}/\alpha$.

For local SGD, under the stable-neighborhood conditions of Theorem~\ref{cor:fedprox_ssm_convergence}, given the bound $\|\nu_k^r\|_F\leq C_\nu$, and with the learning rate satisfying $0<\eta\alpha\leq1$, the proximal displacement recursion gives:
\begin{equation}
    \max_{e\in \{0, \dots, E\}}\left\|W_k^{r,e}-W^r\right\|_F \leq R_{\mathrm{dyn}} := \min\left\{R,\frac{G_{\mathrm{SSM},R}+C_\nu}{\alpha}\left[1-(1-\eta\alpha)^E\right]\right\}.
\end{equation}
The geometric factor therefore damps accumulated local displacement. While this bound includes $C_\nu$, it ignores cancellation between the local gradient and the dual correction. Accurate tracking can on the other hand produce smaller corrected updates than the inequality suggests, and given a tighter local neighborhood, the SSM parameter envelopes could be reduced.

Furthermore, note that the client-gradient disagreement becomes:
\begin{equation}
    \Gamma_{\mathrm{dyn},R}^{2} := \sup_{W\in\mathcal N_R(W^r)} \sum_k p_k \left\|\nabla\mathcal L_k(W)-\nabla f(W) - (\nu_k^r-\nu^r) \right\|_F^2.
\end{equation}
As a result, when the centered states accurately track the client-gradient differences, this residual can be smaller than $\Gamma_{\mathrm{SSM},R}^{2}$. Therefore, FedDyn combines local-displacement damping with a mechanism for reducing the disagreement highlighted by the FedAvg bound.

However, these benefits compete with the additional local smoothness term $\alpha$ and the server displacement:
\begin{equation}
    \|W^{r+1}-W^r\|_F \leq R_{\mathrm{dyn}} + \frac{\|\nu^{r+1}\|_F}{\alpha}.
\end{equation}

As such, suitable regularization and accurate tracking can make the damping and disagreement beneficial while preserving stable containment, if these improvements outweigh the smoothness constant increase provided the server update also remains within the stable neighborhood. As such, for a well-tuned FedDyn, we expect it to improve on FedAvg, with the benefits being especially highlighted under heterogeneous data conditions through disagreement correction.

\textbf{FedSAM.} FedSAM~\citep{pmlr-v162-qu22a} retains FedAvg aggregation but uses a sharpness-aware local SGD step with two gradient evaluations:
\begin{equation}
    u_k^{r,e} := \mathfrak R_{\mathrm{SAM}}\frac{g_k(W_k^{r,e};\xi_k^{r,e})}{\|g_k(W_k^{r,e};\xi_k^{r,e})\|_F}, \qquad W_k^{r,e+1} = W_k^{r,e} - \eta\,g_k(W_k^{r,e}+u_k^{r,e};\xi_k^{r,e}),
\end{equation}
where $\mathfrak R_{\mathrm{SAM}}>0$ is the perturbation radius, with $u_k^{r,e}=0$ when the first gradient vanishes.

Suppose the local iterates lie in $\mathcal N_R(W^r)$ and the SSM assumptions hold uniformly on $\mathcal N_{R+\mathfrak R_{\mathrm{SAM}}}(W^r) \subseteq\mathcal S_{\gamma,\delta}$. The expanded neighborhood contains the perturbed evaluation points, resulting in the local-displacement bound
\begin{equation}
    R_{\mathrm{SAM}} := \max_{e\in \{0, \dots, E\}}\|W_k^{r,e}-W^r\|_F  \leq \eta E G_{\mathrm{SSM},R+\mathfrak R_{\mathrm{SAM}}}.
\end{equation}
The sequence-wise smoothness argument in the proof of Lemma~\ref{lem:ssm_local_smoothness} also gives
\begin{equation}
    \left\|g_k(W_k^{r,e}+u_k^{r,e};\xi_k^{r,e}) - g_k(W_k^{r,e};\xi_k^{r,e})\right\|_F \leq L_{\mathrm{SSM},R+\mathfrak R_{\mathrm{SAM}}} \mathfrak R_{\mathrm{SAM}}.
\end{equation}
As a result, relative to an SGD step from the same iterate and sample, SAM changes the update by at most $\eta L_{\mathrm{SSM},R+\mathfrak R_{\mathrm{SAM}}} \mathfrak R_{\mathrm{SAM}}$. When the perturbation changes the SSM constants only slightly and its accumulated effect over training remains small, we expect FedSAM to perform similarly compared to FedAvg under both nearly IID and heterogeneous data. Note that this expectation relies on the small-perturbation approximation, and possible generalization benefits from flatter solutions are not quantified by these worst-case bounds.

\textbf{FedExP.} FedExP~\citep{jhunjhunwala2023fedexp} retains the FedAvg local updates but rescales the aggregated update:
\begin{equation}
\begin{aligned}
        W^{r+1} &= W^r + \eta_g^r \sum_{k=1}^K p_k(W_k^{r,E}-W^r), \\ \eta_g^r &:= \max\left\{1,\frac{\sum_{k=1}^K p_k\|W_k^{r,E} - W^r\|_F^2}{2\left\|\sum_{k=1}^K p_k(W_k^{r,E}-W^r)\right\|_F^2 + \epsilon_g}\right\},
\end{aligned}
\end{equation}
where $\epsilon_g > 0$ stabilizes the denominator. By the weighted variance identity, the server scale factor $\eta_g^r$ exceeds $1$ when
\begin{equation}
    \sum_{k=1}^K p_k \left\|W_k^{r,E}-\sum_{j=1}^K p_jW_j^{r,E}\right\|_F^2 > \left\|\sum_{k=1}^K p_k(W_k^{r,E}-W^r)\right\|_F^2 + \epsilon_g.
\end{equation}
As a result, rescaling is activated when disagreement among client updates is sufficiently large relative to the squared norm of their average. Otherwise, $\eta_g^r=1$ and the server update reduces to FedAvg.

For local SGD under the same stable-neighborhood assumptions, neighborhood radius remains $R_{\mathrm{local}}\leq\eta E G_{\mathrm{SSM},R}$, however, the server step satisfies:
\begin{equation}
    \|W^{r+1}-W^r\|_F = \eta_g^r \left\|\sum_{k=1}^K p_k(W_k^{r,E}-W^r)\right\|_F \leq \eta_g^r\eta E G_{\mathrm{SSM},R},
\end{equation}
Since $\eta_g^r\geq1$, this server-step upper bound is no smaller than the corresponding FedAvg bound, although the actual server displacement need not exceed the local radius. Therefore, server scaling may require a larger neighborhood to cover the server update, with larger parameter envelopes entering the SSM gradient and smoothness bounds. Containment of this expanded neighborhood in $\mathcal S_{\gamma,\delta}$ is further not guaranteed. 

The frequency of rescaling depends on disagreement among the client updates, which can reflect both data heterogeneity and stochastic sampling. For IID data, we expect sufficiently similar client updates for the server scaling to remain mostly inactive, resulting in FedAvg-like behavior. However, given heterogeneous data, greater dispersion relative to the average update can activate scaling more often. The resulting server steps also amplify any local-drift error in the averaged update, without correcting client disagreement. For the considered SSM setting, we therefore expect worse performance than FedAvg when these displacement and stability costs outweigh the acceleration benefit. This is a qualitative expectation, not a convergence comparison.

\textbf{DiLoCo-style.} The original DiLoCo framework~\citep{douillard2024diloco} combines AdamW local optimization with a Nesterov momentum server update. Our DiLoCo-style baseline retains the shared local phase and updates the server as follows:
\begin{equation}
\begin{aligned}
    m_o^r &= \beta_o m_o^{r-1} + \sum_{k=1}^K p_k(W_k^{r,E}-W^r), \\
    W^{r+1} &= W^r+\alpha_o\left(\beta_o m_o^r + \sum_{k=1}^K p_k(W_k^{r,E}-W^r)\right),
\end{aligned}
\end{equation}
where $\alpha_o>0$, $\beta_o\in[0,1)$, and $m_o^{-1}=0$. Setting $\beta_o=0$ and $\alpha_o=1$ reduces to FedAvg.

For the SGD-based comparison, the same stable-neighborhood conditions with a common gradient bound $G_{\mathrm{SSM},R}$ across rounds hold. Each averaged client update then has norm at most $\eta E G_{\mathrm{SSM},R}$. 

As such, unrolling the server momentum recursion we get $\|m_o^r\|_F\leq\eta E G_{\mathrm{SSM},R}/(1-\beta_o)$, and therefore
\begin{equation}
    \|W^{r+1}-W^r\|_F \leq  \frac{\alpha_o}{1-\beta_o}\eta E G_{\mathrm{SSM},R}.
\end{equation}
When $\alpha_o/(1-\beta_o)>1$, this upper bound exceeds the corresponding FedAvg bound. Furthermore, unlike FedExP, momentum can accumulate even when clients agree, because it combines averaged updates across rounds. Persistent update directions can increase server displacement, while changes in direction can leave the momentum contribution misaligned with the current averaged update. The server trajectory may therefore require a larger neighborhood and larger SSM parameter envelopes, with stable containment requiring separate verification. The within-round local update rule remains unchanged, and outer momentum does not directly correct client disagreement.

As such, depending on the momentum parameters, the displacement and history effects could outweigh the acceleration benefit of momentum, leading to worse performance than FedAvg, although the effects could potentially balance out. 

\textbf{Summary}. The FedAvg bound (Theorem~\ref{thm:fedavg_ssm_convergence}) connects convergence to $L_{\mathrm{SSM},R}$, $G_{\mathrm{SSM},R}$, and $\mathcal H_{\mathrm{TV}}$ through local displacement, stable containment, and client disagreement. Under the stated assumptions, FedProx (Theorem~\ref{cor:fedprox_ssm_convergence}) and FedDyn's proximal term damp displacement bounds, at a smoothness cost. SCAFFOLD and FedDyn can reduce corrected disagreement when their control states accurately track client-gradient differences, though IID data leave less disagreement to correct. FedAdam, FedExP, and DiLoCo-style retain the local phase but modify server updates through coordinatewise adaptation, disagreement-dependent extrapolation, and momentum, respectively. These updates may require larger stable neighborhoods. FedAlign targets coordinate mismatch through norm-preserving permutations, while FedSAM perturbs gradient evaluation points, so FedAvg-like behavior is expected when coordinate mismatch or accumulated perturbation effects, respectively, are small. These theorem-motivated expectations are collected in Table~\ref{tab:practical_ranking} and compared with observations in Section~\ref{sec:experiments}.

\section{Discussion on Model Merging}\label{appendix:model_mergin_discussion}
\subsection{Model Merging with SSMs Motivation}\label{appendix_model_merging_motivation_subseq}

Section~\ref{sec:system_model} formulates distributed learning as the minimization of the aggregated expected sequence loss in~\eqref{eq:init_global_objective} through repeated local adaptation and parameter aggregation. The optimization proceeds over $T$ communication rounds, each consisting of $E$ local SGD steps followed by aggregation of the updated model parameters. This framework also provides a natural formulation of one-shot model merging~\citep{pmlr-v162-wortsman22a, ilharco2023editing}.

Consider the case of task or domain experts adapted from a common pretrained initialization. This setting shares the central operation studied in our distributed framework: combining parameter updates learned from heterogeneous data. Specifically, let $M$ experts share the same selective SSM architecture, parameterization, and pretrained initialization $W^0$. Expert $m$ adapts to a local dataset $\mathcal D_m$ associated with a distribution $\mathcal P_m$, with all experts using the sequence-loss formulation of Section~\ref{sec:system_model}. To make the connection explicit, suppose each expert starts from $W_m^{0,0}=W^0$ and performs $E$ local SGD steps following the update rule in that section. Denote its final parameters and task vector by
\begin{equation}
    W_m := W_m^{0,E}, \qquad \mathbf T_m := W_m-W^0.
\end{equation}
The independently adapted experts are then combined using a merging operator,
\begin{equation}
    W_{\mathrm{merge}} = \operatorname{Merge}\!\left(\{W_m\}_{m=1}^{M};W^0\right),
\end{equation}
which may use both the expert parameters and the pretrained reference. If the merging operator is chosen to be weighted parameter averaging, the final model parameters are obtained as
\begin{equation}
    W_{\mathrm{avg}} = \sum_{m=1}^{M}p_m W_m = W^0 + \sum_{m=1}^{M}p_m\mathbf T_m, \qquad p_m>0,\quad \sum_{m=1}^{M}p_m=1,
\end{equation}
which reduces exactly to one round of FedAvg with $K=M$ and $T=1$ when the weights in the objective and aggregation are kept the same. 

This demonstrates that averaging task vectors around a shared initialization provides a direct connection between one-shot model merging and the distributed optimization setting considered in this work. More general merging operators retain this sequence of independent adaptation followed by aggregation, although the differences in how the expert updates are combined limit the applicability of the current theoretical results. 

Additionally, practical model merging often involves large pretrained networks and extensive independent fine-tuning. Although our formulation permits pretrained initialization and includes a multilayer extension in Appendix~\ref{appendix:multi_layer_SSM_proofs}, applying the results to complete pretrained architectures and expert-specific optimization protocols requires verifying additional assumptions. Applying our federated convergence results to one-shot merging requires a neighborhood around $W^0$ containing the full independent adaptation trajectories, without intermediate synchronization. This may require a large radius $R$ and correspondingly large parameter envelopes, increasing $L_{\mathrm{SSM},R}$ and $G_{\mathrm{SSM},R}$ and making the bounds loose or potentially uninformative. The required stability and step-size conditions may also become harder to satisfy. When a common stable neighborhood exists, convex averaging preserves containment within it, but stability alone does not guarantee predictive performance across the contributing domains.

Nevertheless, the shared structure of independent adaptation followed by aggregation motivates including model merging as a related setting for interpreting our results. In the remainder of this section, we discuss connections between representative merging methods and our analysis, clarify their applicability and limitations, and highlight potential takeaways for merging selective SSMs. We further provide supplementary model merging results in Appendix~\ref{sec:appendix_mamba2_merging}.

\subsection{Model Merging Methods: Connections to Selective SSM Analysis}\label{appendix:ssm_merging_methods}
For a merged model parameterized by $W_{\mathrm{merge}}$ with $W_{\mathrm{merge}}\in\mathcal N_R(W^0) \subseteq \mathcal S_{\gamma,\delta}$, the uniform regularity conditions of Lemma~\ref{lem:ssm_local_smoothness} imply
\begin{align}
    f\left(W_{\mathrm{merge}}\right) &= f\left(W^0 + \left(W_{\mathrm{merge}} - W^0\right)\right) \nonumber \\ &\leq f\left(W^0\right) + \left\langle\nabla f\left(W^0\right),W_{\mathrm{merge}} - W^0\right\rangle_F + \frac{L_{\mathrm{SSM},R}}{2}\left\|W_{\mathrm{merge}} - W^0\right\|_F^2.\label{eq:ssm_merging_local_loss_control}
\end{align}
Similarly, for a $Q$-layer model with
$W_{\mathrm{merge}}\in\mathcal N_{\mathbf R_Q}(W^0) \subseteq\mathcal S_{\boldsymbol\gamma,\boldsymbol\delta}^{(Q)}$, the assumptions of Lemma~\ref{lem:multilayer_local_smoothness} give
\begin{equation}\label{eq:multilayer_ssm_merging_local_loss_control}
    f\left(W_{\mathrm{merge}}\right)\leq f\left(W^0\right) +\left\langle\nabla f\left(W^0\right),W_{\mathrm{merge}}-W^0\right\rangle_F +\frac{L_{\mathrm{SSM},\mathbf R_Q}^{(Q)}}{2}\left\|W_{\mathrm{merge}}-W^0\right\|_F^2,
\end{equation}
where the inner product and norm are taken over the full parameter tuple, and $L_{\mathrm{SSM},\mathbf R_Q}^{(Q)} = H_\ell\Theta_Q^2 + L_\ell\Omega_Q$ is the existing multi-layer smoothness constant. 

These inequalities highlight important factors that determine the loss bound after merging. Firstly, the alignment of $W_{\mathrm{merge}}-W^0$ with $-\nabla f(W^0)$ indicates whether the update initially moves toward a lower aggregated loss. In addition, the merging operator determines the direction and magnitude of this update - different merging methods might rescale or modify the model parameters during aggregation. Finally, the SSM smoothness constant bounds the error in this local loss prediction and depends on the neighborhood over which it is evaluated. The radius $R$ therefore plays two roles: it specifies the maximum displacement of the model parameters $\left\|W_{\mathrm{merge}}-W^0\right\|_F\leq R$ (or $\left\|W_{\mathrm{merge}}^{(q)}-W^{0, (q)}\right\|_F \leq R_q \forall q\in\{1,\dots, Q\}$) covered by the analysis, and determines the parameter range over which the smoothness bound holds. A larger radius can require larger parameter envelopes $M_X$, leading to larger smoothness constants $L_{\mathrm{SSM},R}$ (or $L_{\mathrm{SSM},\mathbf R_Q}^{(Q)}$) and a looser loss bound. 

This motivates examining how different merging methods affect the merged model's displacement from $W^0$ and the neighborhood needed to contain it. We discuss several representative methods below.

\textbf{Weighted parameter averaging.} If all the experts lie in a common stable neighborhood $\mathcal N_R(W^0)$, their weighted average satisfies
\begin{equation}
    \left\|W_{\mathrm{avg}}-W^0\right\|_F \leq \sum_{m=1}^{M}p_m\left\|W_m-W^0\right\|_F \leq R.
\end{equation}
So an averaging merge operation does not enlarge the neighborhood containing the experts, and the same local smoothness bound applies.  This property also holds for the multi-layer product neighborhood. Furthermore, Lemma~\ref{lem:ssm_gamma_tv_bound} bounds the weighted squared gradient disagreement at a common parameter tuple by $4G_{\mathrm{SSM},R}^2\mathcal H_{\mathrm{TV}}$, showing how data heterogeneity and SSM sensitivity jointly determine this bound.

\textbf{Task Arithmetic.} Task Arithmetic~\citep{ilharco2023editing} scales combined task vectors. In a weighted, normalized form,
\begin{equation}\label{eq:ssm_merging_task_arithmetic}
    W_{\mathrm{TA}}=W^0+\lambda\overline{\mathbf T}, \qquad \overline{\mathbf T}:=\sum_{m=1}^{M}p_m\mathbf T_m, \qquad \lambda\geq0.
\end{equation}
Here $\lambda=1$ gives parameter averaging, while $0\leq\lambda\leq1$ preserves containment in the common neighborhood. Larger scales require a separate containment check. In~\eqref{eq:ssm_merging_local_loss_control}, the directional term scales with $\lambda$ and the quadratic remainder with $\lambda^2$. Thus, scaling controls displacement relative to SSM sensitivity, but shrinking the update can also suppress useful adaptation and cannot recover directions lost through cancellation.

\textbf{TIES-Merging.} TIES-Merging~\citep{NEURIPS2023_1644c9af} trims small task-vector entries, elects each coordinate's sign according to the total retained movement, and averages the nonzero entries agreeing with that sign. The resulting vector is scaled and added to $W^0$. Because coordinates may select different expert subsets, the result need not remain in the experts' common neighborhood. Moreover, raw magnitude and sign do not account for how changes to $A,W_\Delta,W_B,$ and $W_C$ interact through the SSM recurrence. Our bounds therefore motivate sensitivity-informed trimming and conflict resolution, but do not establish that TIES improves stability or reduces the merged model's loss.

\textbf{Practical Takeaways and Limitations.} Overall, the model merging discussion motivates jointly considering expert adaptation, merge strength, and SSM sensitivity. The dependence of the bounds on neighborhood radii, recurrent stability, and interactions between layers motivates considering SSM-specific properties when evaluating and designing merging methods. Additionally, the derived relationships remain conditional on the stated assumptions - the local loss bounds require uniform regularity over a common convex stable neighborhood, while applying the federated convergence results additionally requires containment of the full adaptation trajectories. Extensive fine-tuning may therefore require large parameter envelopes and yield loose bounds. These limitations motivate extending the analysis to complete pretrained architectures, practical optimizers, and general merging operators as a promising direction for future work. We complement this discussion with supplementary model-merging experiments and an analysis of their results in Appendix~\ref{sec:appendix_mamba2_merging}.

\section{Detailed Experimental Setup}\label{appendix:detailed_exp_setup} 
In this section, we describe the experimental setup in detail. The experiments fall into three groups: (i) synthetic dataset validation of the derived theoretical bounds, (ii) federated language modeling with practical selective state-space models on RedPajama, (iii) one-shot model merging supplementary experiments with pretrained Large Mamba2-based models. We detail the datasets, models, as well as training and validation procedures needed for reproducibility.

\subsection{Datasets}

\textbf{Synthetic Validation.} The synthetic datasets are autoregressive token sequences sampled from frozen teacher selective SSMs that implement the recurrence of Section~\ref{sec:system_model} exactly, so that every assumption of Theorem~\ref{thm:ssm_grad_bound} holds for the data-generating model (rather than being assumed to hold). A teacher consists only of the four parameter blocks $W=\left(W_\Delta, W_B, W_C, A\right)$ of a layer (no embedding, convolution, residual path, normalization, MLP, or Transformer components). The inputs are one-hot encoded: $x_t$ is one of the standard basis vectors $e_1,\dots,e_D$ of $\mathbb R^{D}$ for every $t$. This formulation makes the input dimension $D$ also act as the vocabulary size (since there are $D$ different basis vectors). Since $\|x_t\|_2=1$, Assumption~\ref{assum:bounded_inputs} holds with $L_x = 1$.

The teacher SSM is turned into a generative model as follows. For a partially generated sequence $x_1,\dots,x_t$, running it through the recurrence, starting from $h_0=\mathbf 0_{N\times D}$, produces an output $y_t\in\mathbb R^{D}$, which has one coordinate per token in the vocabulary. We therefore treat the $d$-th coordinate $y_{t,d}$ as the logit of the event that the next input is token $d$, i.e., that $x_{t+1}=e_d$. We then take a softmax with temperature $\tau>0$ to convert these $D$ logits into a probability vector over the tokens, and draw the next input as follows:
\begin{equation*}
    \Pr\left[x_{t+1}=e_d \,\middle|\, x_1,\dots,x_t\right]=\left[\operatorname{softmax}\left(y_t/\tau\right)\right]_d=\frac{\exp\left(y_{t,d}/\tau\right)}{\sum_{d'=1}^{D}\exp\left(y_{t,d'}/\tau\right)}, \qquad d=1,\dots,D.
\end{equation*}
Table~\ref{tab:synthetic_generation} summarizes the single-layer dataset generation over different settings.

\begin{table}[h]
\centering
\footnotesize
\caption{Generation settings of the synthetic teacher datasets. Grid entries are searched per configuration in the fixed order given in the text.}
\label{tab:synthetic_generation}
\begin{tabular}{lp{0.55\linewidth}}
\toprule
Setting & Value \\
\midrule
Input/vocabulary dimension $D$ & $\{4,8,16\}$ \\
State dimension $N$ & $\{4,8,16\}$ \\
Stability limit $\rho$ & $\{0.3,0.5,0.7,0.9,0.95,0.99,0.995,0.999\}$ \\
Configurations (datasets) & $72$ \\
Sequence length $S$ & $64$ stored tokens ($16$ burn-in tokens discarded, $80$ generated) \\
Training / test sequences & $4{,}096$ / $1{,}024$ \\
Sampling seeds (train / test) & $1{,}000$ / $1{,}001{,}003$ \\
First input $x_1$ & uniform over $\{e_1,\dots,e_D\}$ \\
Teacher seeds $s$ & $\{0,1,2\}$ (grid) \\
Gain $g$, temperature $\tau$ & grid \\
$W_\Delta$, $W_C$ & i.i.d.\ $\mathcal N(0,0.75^2)$ entries \\
$W_B$ & i.i.d.\ $\mathcal N(0,(0.75\,g)^2)$ entries \\
$A$ & $A_n=-\gamma(1+0.25|z_n|)$, $z_n\sim\mathcal N(0,1)$ \\
Step margin $\delta$ & $\lfloor10^{5}\hat\delta\rfloor/10^{5}$ (realized minimum step of $W_\Delta$) \\
Transition margin $\gamma$ & $-\ln(\rho)/\delta$ \\
Parameter precision & single precision (float32) \\
\bottomrule
\end{tabular}
\end{table}

\textit{Single-layer synthetic datasets.} To generate the single-layer datasets, we instantiate one frozen teacher SSM for every combination of the input dimension $D\in\{4,8,16\}$, the state dimension $N\in\{4,8,16\}$, and the stability limit $\rho\in\{0.3,0.5,0.7,0.9,0.95,0.99,0.995,0.999\}$ (from Definition~\ref{def:ssm_stability}), which gives $3\times3\times 8=72$ different frozen teachers. Each such teacher, with its specific combination of $D$, $N$, and $\rho$, then generates $4{,}096$ training and $1{,}024$ test sequences of length $S=64$ according to the procedure described above.

\textit{Teacher initialization and stability targeting.} For a teacher SSM corresponding to a combination of $D$, $N$, and $\rho$, its weights are drawn with i.i.d. Gaussian entries as follows: $\left(W_B\right)_{n,d} \sim \mathcal N\big(0,(0.75g)^2\big)$,  $\left(W_C\right)_{n,d} \sim \mathcal N\left(0,0.75^2\right)$,  $\left(W_\Delta \right)_{d,d'} \sim\mathcal N\left(0,0.75^2\right)$, where $g$ is a gain multiplier on the input projection described below. The margins $(\gamma,\delta)$ are chosen such that the teacher belongs to $\mathcal S_{\gamma,\delta}$ with $\rho=\exp(-\gamma\delta)$ equal to the specified value while satisfying both constraints of Definition~\ref{def:ssm_stability}. Since the inputs are one-hot, $W_\Delta x_t$ is simply a column of $W_\Delta$, so the time-scale constraint of Definition~\ref{def:ssm_stability} reduces to the entry-wise condition $\operatorname{Softplus}\big(\left(W_\Delta\right)_{d,d'}\big) \geq \delta$ for all $d,d'$. We therefore draw $W_\Delta$ first and set $\delta$ to the smallest step it produces,
\begin{equation*}
    \hat\delta:=\min_{d,d'}\operatorname{Softplus}\left(\left(W_\Delta\right)_{d,d'}\right)=\operatorname{Softplus}\left(\min_{d,d'}\left(W_\Delta\right)_{d,d'}\right), \qquad \delta:=\left\lfloor 10^{5}\,\hat\delta\right\rfloor/10^{5},
\end{equation*}
rounded down so that $\delta \leq \hat\delta$ holds exactly, and then $\gamma := -\ln(\rho)/\delta$ for the designated $\rho$. The parameters of the transition block are then generated as follows:
\begin{equation*}
    A_n = -\gamma \left(1+0.25\,|z_n|\right), \qquad z_n\sim\mathcal N(0,1),\quad n=1,\dots,N,
\end{equation*}
so that $A_n \leq -\gamma$ holds with a multiplicative margin. With this construction, the realized stability limit $\hat\rho :=\max_{n,d,d'}\left(\bar A_t(e_{d'})\right)_{n,d} = \exp(-\hat\gamma\hat\delta)$ with $\hat \gamma:= \min_n(-A_n)$ satisfies $\hat\rho=\rho^{(\hat\delta/\delta)(1+0.25\min_n|z_n|)}$ and tracks $\rho$, including as $\rho\to1$. Across all generated datasets $\rho-\hat\rho\in[0,0.03]$. By construction, $A_n\le-\gamma$, $\operatorname{Softplus}\big((W_\Delta)_{d,d'}\big)\ge\delta$, and $\hat\rho<1$ hold. Since the admissible inputs form the finite set $\{e_1,\dots,e_D\}$, these conditions are also verified exhaustively over all $D$ one-hot inputs before a teacher is accepted. A clamping projection onto $\mathcal S_{\gamma,\delta}$ is available as a safeguard, but it was not active for any of the $72$ accepted teachers.

\textit{Sampling.} Given the frozen teacher with a combination of $D$, $N$, and $\rho$ and the frozen weights constructed as described above, a sequence is generated by drawing $x_1$ uniformly from $\{e_1,\dots,e_D\}$, starting from $h_0=\mathbf 0_{N\times D}$, and sampling $x_{t+1}$ as described above for $79$ further steps. The first $16$ tokens are discarded as burn-in and the remaining $S=64$ tokens are stored as a single sequence. The burn-in aims to reduce the stored prefix's dependence on the first input. The recurrent state at the start of the stored window is not persisted, and every subsequent teacher-forced or learner evaluation restarts from $h_0=\mathbf 0_{N\times D}$ at the first stored token, which is the convention of the theory. 

The training and test splits are sampled with separate fixed seeds ($1{,}000$ and $1{,}001{,}003$), reused in every configuration. Each dataset stores the token indices as 8-bit integers together with the frozen teacher parameters, the generation settings, and the stability certificate, and is never modified afterwards. The learner only sees the stored tokens and never accesses the teacher.

\textit{Acceptance criteria.} A generated candidate dataset is kept only if it passes the fixed criteria of Table~\ref{tab:synthetic_acceptance}, which exclude degenerate teachers (whose sequences collapse into a dominant token, or become deterministic, constant, or uninformative), while ensuring that the next-token task carries predictive information. All diagnostics of the table are computed on the union of the training and test splits, unless stated otherwise. The central quantity is the teacher-forced negative log-likelihood (NLL): the stored tokens of a sequence are fed to the teacher as inputs, starting from $h_0=\mathbf 0_{N\times D}$, instead of re-sampling from it, and at every position $t=1,\dots,S-1$ we record $-\ln\left(\Pr[x_{t+1}\mid x_1,\dots,x_t]\right)$, i.e., minus the log-probability that the teacher (at the sampling temperature $\tau$) assigns to the stored next token. The NLL is the average of these values over all positions and sequences and measures how predictable the generated data are for the very model that produced them. A model that assigns probability $1/D$ to every token has NLL $\ln D$, which serves as the uniform baseline. The informativeness criterion therefore requires the sequences to be predictable beyond chance, while the perplexity criterion $\exp(\mathrm{NLL})\geq1.20$ excludes near-deterministic teachers. The mutual information is estimated from the empirical frequencies of consecutive token pairs. To avoid generating full datasets for clearly degenerate candidates, each candidate is first screened on a small draw of $256$ training and $128$ test sequences, and only if this screen passes is the full dataset generated and checked again against the same thresholds, which are never relaxed.

\begin{table}[h]
\centering
\footnotesize
\caption{Acceptance criteria applied to every synthetic candidate dataset. Statistics are computed on the union of the training and test splits unless noted. The token marginal is the empirical distribution of the $D$ tokens over all stored positions, ignoring position and context. NLL denotes the teacher-forced negative log-likelihood at the sampling temperature $\tau$ with $h_0=\mathbf 0$ per stored sequence.}
\label{tab:synthetic_acceptance}
\begin{tabular}{lp{0.62\linewidth}}
\toprule
Criterion & Threshold \\
\midrule
Global token balance & largest probability in the token marginal $\leq 1/D+0.10$ \\
Marginal entropy & Shannon entropy of the token marginal $\geq 0.90 \ln D$ (the maximum $\ln D$ is attained by the uniform distribution) \\
Per-sequence majority (test split) & for each test sequence, the fraction of its $S$ positions occupied by its most frequent token; the $95$th percentile of this fraction over the test sequences $\leq 1/D+0.20$ \\
Class coverage & each of the $D$ tokens occurs at least once in the training split and at least once in the test split \\
Constant sequences & fraction of sequences whose $S$ tokens are all identical $\le 0.01$ \\
Sequence diversity & number of distinct sequences divided by the number of sequences $\ge 0.95$ \\
Teacher informativeness & $\ln D-\mathrm{NLL}\geq 0.02$, i.e., the teacher beats the uniform predictor \\
No deterministic collapse & teacher perplexity $\exp(\mathrm{NLL})\ge 1.20$ \\
Predictive information & mutual information between consecutive tokens $x_t$ and $x_{t+1}$, estimated from their empirical joint frequencies, $\ge 0.005$ \\
Finiteness & all diagnostics, logits, steps $\Delta_t$, transitions $\bar A_t$, and hidden states are finite (no overflow) \\
Stability & the teacher is certified to lie in $\mathcal S_{\gamma,\delta}$ with realized stability limit $\hat\rho<1$ \\
\bottomrule
\end{tabular}
\end{table}

\textit{Gain and temperature choices.} As $ \rho \to1$, the hidden state of a random teacher accumulates input contributions over a horizon of order $1/(1-\rho)$, so its logits grow and sampling collapses into a few tokens, which violates the balance and entropy criteria. Instead of relaxing the criteria, we search a fixed grid of input-projection gains and sampling temperatures,
\begin{equation*}
\begin{aligned}
    (g,\tau) \in \big\{&(1,0.7),(1,1),(0.5,0.7),(0.5,1),(0.25,0.7),(0.25,1),(0.15,0.7),\\ 
    &(0.15,1),(0.1,0.7),(0.1,1),(0.05,0.7),(0.05,1),(0.03,1)\big\},
\end{aligned}
\end{equation*}
in this order for the teacher seeds $s\in\{0,1,2\}$ (outer loop), and accept the first candidate that passes every criterion in Table~\ref{tab:synthetic_acceptance}.  Reducing $g$ scales $W_B$ (and therefore $\bar B_t$) and the accumulated state without touching $W_\Delta$ or $A$, so the margins $(\gamma,\delta)$ and $\rho$ remain unchanged. Over the $72$ dataset configurations, $18$ were accepted at the first attempt and the longest search took $32$ attempts. Teacher seed $0$ was accepted in $63$ configurations, seed $1$ in $5$, and seed $2$ in $4$. The accepted gains were $g=1$ ($30$ configurations), $0.5$ ($9$), $0.25$ ($15$), $0.15$ ($8$), $0.1$ ($2$), $0.05$ ($6$), and $0.03$ ($2$). A gain below one was required in every configuration with $\rho\geq0.99$ and in only $7$ of the $27$ configurations with $\rho\leq0.7$. The accepted temperatures were $\tau=0.7$ ($39$ configurations) and $\tau=1$ ($33$ configurations). The resulting teachers improve on the uniform NLL by $0.02$ to $0.97$, have a consecutive-token mutual information between $0.008$ and $0.90$ and a teacher perplexity between $1.5$ and $15.7$. The most frequent token has frequency at most $1.76/D$, the normalized marginal entropy is at least $0.94$, and at least $98\%$ of the stored sequences are unique. 

\textbf{Federated language modeling (RedPajama).}
The federated experiments use the SlimPajama-627B~\citep{together2023redpajama} release of RedPajama, available as \texttt{cerebras/SlimPajama-627B}.. Of the seven SlimPajama domains: \texttt{arxiv}, \texttt{book}, \texttt{c4}, \texttt{common\_crawl} (\texttt{cc}), \texttt{github}, \texttt{stackexchange}, \texttt{wikipedia}, we use four-domain combinations to instantiate four federated clients:
\begin{itemize}
\item  Setting 1:\texttt{arxiv}, \texttt{github}, \texttt{stackexchange}, \texttt{wikipedia};
\item  Setting 2:\texttt{arxiv}, \texttt{c4}, \texttt{stackexchange}, \texttt{wikipedia};
\item  Setting 3:\texttt{c4}, \texttt{common\_crawl}, \texttt{github}, \texttt{wikipedia};
\item  Setting 4:\texttt{arxiv}, \texttt{common\_crawl}, \texttt{github}, \texttt{stackexchange}.
\end{itemize}
Due to copyright restrictions, the \texttt{book} domain is no longer accessible, and therefore not used in our experiments.

For every client, we set the maximum number of sequences to \texttt{sequences\_per\_client = 10{,}000}. Each client's dataset is split into 80\% training and 20\% validation data using \texttt{train\_val\_split = 0.8}. Due to the limited number of available sequences in certain domains, the effective per-client budgets are reduced for some settings. The resulting dataset statistics are summarized in Table~\ref{tab:data_splits}.

\begin{table}[h]
\centering
\caption{Effective client train/validation budgets and server validation sizes across heterogeneity settings.}
\begin{tabular}{llrrlll}
\toprule
Setting & $\beta$ & Train cap & Val cap & Train limiter data & Valid limiter data & Server val \\
\midrule
Setting 1 & 0.0 & 7{,}427 & 1{,}857 & \texttt{stackexchange} & \texttt{stackexchange} & 7{,}424 \\
Setting 1 & 0.7 & 7{,}427 & 1{,}857 & \texttt{stackexchange} & \texttt{stackexchange} & 7{,}431 \\
Setting 1 & 1.0 & 7{,}427 & 1{,}857 & \texttt{stackexchange} & \texttt{stackexchange} & 7{,}428 \\
\midrule
Setting 2 & 0.0 & 7{,}367 & 1{,}842 & \texttt{c4} & \texttt{c4} & 7{,}360 \\
Setting 2 & 0.7 & 7{,}367 & 1{,}842 & \texttt{c4} & \texttt{c4} & 7{,}364 \\
Setting 2 & 1.0 & 7{,}367 & 1{,}842 & \texttt{c4} & \texttt{c4} & 7{,}368 \\
\midrule
Setting 3 & 0.0 & 7{,}367 & 1{,}842 & \texttt{c4} & \texttt{c4} & 7{,}360 \\
Setting 3 & 0.7 & 7{,}367 & 1{,}842 & \texttt{c4} & \texttt{c4} & 7{,}364 \\
Setting 3 & 1.0 & 7{,}367 & 1{,}842 & \texttt{c4} & \texttt{c4} & 7{,}368 \\
\midrule
Setting 4 & 0.0 & 7{,}427 & 1{,}857 & \texttt{stackexchange} & \texttt{stackexchange} & 7{,}424 \\
Setting 4 & 0.7 & 7{,}427 & 1{,}857 & \texttt{stackexchange} & \texttt{stackexchange} & 7{,}431 \\
Setting 4 & 1.0 & 7{,}427 & 1{,}857 & \texttt{stackexchange} & \texttt{stackexchange} & 7{,}428 \\
\bottomrule
\end{tabular}
\label{tab:data_splits}
\end{table}

We use $\beta$ to label three domain-allocation regimes:
\begin{itemize}
    \item $\beta = 0.0$ (IID): each client receives an equal mixture over all $K=4$ domains.
    \item $\beta = 1.0$ (pure non-IID): each client receives data exclusively from its primary domain.
    \item $\beta = 0.7$ (mixed): each client receives $0.7$ of its data from its primary domain and $0.1$ from each of the other three.
\end{itemize}

Concretely, with \texttt{sequences\_per\_client = 10{,}000} and an 80/20 train/validation split, each client typically contains 8{,}000 training chunks and 2{,}000 validation chunks. The server held-out evaluation set is the union of all client validation pools, giving approximately 8{,}000 validation chunks for $K=4$ clients.

For each configuration, we use the Hugging Face gpt2 tokenizer (50,257 entries), with EOS reused for padding. The model configuration sets \texttt{vocab\_size} = 50,277 and \texttt{pad\_vocab\_size\_multiple} = 8, yielding 50,280 embedding/output rows.

\textbf{Model Merging.} For model merging, we evaluate on four text classification datasets: AG News~\citep{NIPS2015_250cf8b5}, DBpedia~\citep{lehmann2015dbpedia}, TREC~\citep{li2002learning}, and IMDB~\citep{maas2011learning}. The dataset statistics and evaluation limits are summarized in Table~\ref{tab:causal_lm_tasks}. 

Each task is formulated as causal-LM scoring without a classification head. At evaluation time, we score \texttt{prompt + ` ' + label} for every candidate label and select the label with the highest mean per-token log-probability over the label tokens. We use the mean, rather than the sum, so that labels of different lengths remain comparable. The evaluation pool is sub-sampled with a fixed seed of \texttt{99991} so that the evaluation split is identical across training seeds. The training pool is sub-sampled with the per-run training seed (\texttt{seed} $\in {0,1,2}$) so that the three runs differ both in optimization noise and in their training subsamples.

\begin{table}[h]
\centering
\caption{Task datasets and sequence limits for causal-LM evaluation.}
\begin{tabular}{llrrrr}
\toprule
Dataset & Task & \# Classes & Train cap & Eval cap & \texttt{max\_seq\_len} \\
\midrule
\texttt{ag\_news}    & News topic      & 4  & 40{,}000 & 5{,}000 & 256 \\
\texttt{dbpedia\_14} & Wikipedia topic & 14 & 40{,}000 & 5{,}000 & 256 \\
\texttt{trec}        & Question type   & 6  & $\approx 5{,}500$ (full) & 5{,}000 & 256 \\
\texttt{imdb}        & Sentiment       & 2  & 40{,}000 & 5{,}000 & 256 \\
\bottomrule
\end{tabular}
\label{tab:causal_lm_tasks}
\end{table}

The prompt templates are defined in Table~\ref{tab:prompt_templates}:

\begin{table}[h]
\centering
\caption{Prompt templates and candidate labels used for causal-LM scoring.}
\begin{tabular}{ll}
\toprule
Dataset & Prompt template and labels \\
\midrule
\texttt{ag\_news} & \texttt{Article: \{text\}\textbackslash nCategory:} \\
& Labels: \texttt{World}, \texttt{Sports}, \texttt{Business}, \texttt{Tech} \\
\midrule \texttt{dbpedia\_14} & \texttt{<title-aware template>...:Category:} \\
& Labels: 14-way Wikipedia topic labels \\
\midrule \texttt{trec} & \texttt{Question: \{text\}\textbackslash nType:} \\
& Labels: \texttt{Description}, \texttt{Entity}, \texttt{Abbreviation}, \texttt{Person}, \texttt{Location}, \texttt{Number} \\
\midrule \texttt{imdb} & \texttt{Review: \{text\}\textbackslash nSentiment:} \\
& Labels: \texttt{negative}, \texttt{positive} \\
\bottomrule
\end{tabular}
\label{tab:prompt_templates}
\end{table}

Tokenization uses each model's own tokenizer, namely GPT-2 BPE. Prompts are truncated or padded to \texttt{max\_seq\_len = 256}, with 16 tokens reserved for the label completion and template tail. During tokenization, we construct a boolean \texttt{label\_mask}. This mask is used during fine-tuning to compute the loss only over label-token positions, and during evaluation to average scores only over label tokens.

All datasets are  used under their original licenses and terms of use. 

\subsection{Models} 

\textbf{Synthetic Validation.} The learner SSM has the exact same architecture as the teacher: a single-layer selective SSM as defined in Section~\ref{sec:system_model}, with the four parameter blocks $W=\left(W_\Delta,W_B,W_C,A\right)$ of sizes $D\times D$, $N\times D$, $N\times D$, and $N\times1$, one-hot inputs, and no embedding, convolution, residual path, normalization, MLP, or Transformer components. The output $y_t\in\mathbb R^{D}$ is used directly as the logit vector for next-token  prediction. In every configuration, the learner uses the same input dimension $D$ and state dimension $N$ as the teacher that generated its dataset ($D$ must match the vocabulary), so its parameter count is $D^2+2ND+N$ (Table~\ref{tab:synthetic_param_counts}). The learner never receives the teacher parameters - the only teacher-related quantities it uses are the margins $(\gamma,\delta)$ of its configuration, which define the stability set $\mathcal S_{\gamma,\delta}$ in which it is trained.

\begin{table}[h]
\centering
\caption{Learner parameter counts $D^2+2ND+N$ for the synthetic configurations.}
\label{tab:synthetic_param_counts}
\begin{tabular}{lrrr}
\toprule
 & $N=4$ & $N=8$ & $N=16$ \\
\midrule
$D=4$  & $52$  & $88$  & $160$ \\
$D=8$  & $132$ & $200$ & $336$ \\
$D=16$ & $388$ & $520$ & $784$ \\
\bottomrule
\end{tabular}
\end{table}

\textit{Initialization.} For a learner seed $s\in\{0,1,2\}$, drawn independently of the teacher seed, the projections are initialized with i.i.d. Gaussian entries $\left(W_\Delta\right)_{d,d'},\left(W_B\right)_{n,d},\left(W_C\right)_{n,d}\sim\mathcal N(0,0.1^2)$, and the transition block as
\begin{equation*}
    A_n=-\left(\gamma+\epsilon\right)-0.1\,|z_n|, \qquad z_n\sim\mathcal N(0,1),\quad n=1,\dots,N,
\end{equation*}
with a safety margin $\epsilon=10^{-4}$. The initial parameters are then projected onto $\mathcal S_{\gamma+\epsilon,\delta+\epsilon}$ by clamping $A_n\le-(\gamma+\epsilon)$ and $\left(W_\Delta\right)_{d,d'}\ge\operatorname{Softplus}^{-1}(\delta+\epsilon)$, and membership in $\mathcal S_{\gamma,\delta}$ is verified over all $D$ one-hot inputs, exactly as for the teachers. The initialization scale $0.1$ is much smaller than the teacher's $0.75$, so every learner starts close to a uniform next-token predictor. The same projection is applied after every optimizer update during training, so the learner remains in $\mathcal S_{\gamma,\delta}$ with realized transitions bounded by $\rho$ throughout training. All learner parameters are stored and updated in single precision (\texttt{float32}).

\textbf{Federated Training.} All federated runs train Mamba-2 selective state-space models from scratch using \texttt{mamba\_ssm} 's \texttt{MambaLMHeadModel} with the configuration:
\begin{itemize}
    \item \texttt{d\_model} = 256
    \item \texttt{ssm\_cfg.layer} = \texttt{Mamba2}
    \begin{itemize}
        \item default \texttt{d\_state} = 128
        \item \texttt{d\_conv} = 4
        \item \texttt{expand} = 2
        \item \texttt{headdim} = 64
        \item \texttt{ngroups} = 1
    \end{itemize}
    \item \texttt{vocab\_size} = 50{,}277 \quad
    \item \texttt{pad\_vocab\_multiple} = 8, giving padded vocabulary size $50{,}280$
    \item \texttt{rms\_norm} = \texttt{True}
    \item \texttt{residual\_in\_fp32} = \texttt{True}
    \item \texttt{fused\_add\_norm} = \texttt{True}
    \item \texttt{dtype} = \texttt{float32} \quad (with \texttt{fp16} autocast inside forward pass)
    \item \texttt{depth} $\in \{4, 16\}$ \quad (observed depths: 4 and 16)
\end{itemize}
i.e. we consider custom Mamba2 models of depth $4$ or depth $16$. The parameter counts are summarized in Table~\ref{tab:param_breakdown}.

\begin{table}[h]
\centering
\caption{Parameter breakdown by model depth.}
\begin{tabular}{rcccc}
\toprule
Depth & Total params & SSM (mixer) params & Embedding params & Per-layer mixer \\
\midrule
4  & 14{,}733{,}664 & 1{,}860{,}704 & 12{,}871{,}680 & 465{,}176 \\
16 & 20{,}318{,}848 & 7{,}442{,}816 & 12{,}871{,}680 & 465{,}176 \\
\bottomrule
\end{tabular}
\label{tab:param_breakdown}
\end{table}

Notes: the input/output embedding are shared by the LM head ($50{,}280 \times 256 = 12,871,680$ parameters). The SSM (mixer) params denotes the parameter count restricted to tensors matching \texttt{*.mixer.*}, including \texttt{in\_proj}, \texttt{conv1d}, \texttt{A\_log}, \texttt{D}, \texttt{dt\_bias}, the block-gated \texttt{RMSNorm}, and \texttt{out\_proj}.

\begin{table}[h]
\centering
\caption{Backbones of the depth-$4$ ablation. Layer indices run from $0$ to $3$, every layer not listed as attention is a Mamba2 mixer.}
\label{tab:backbones}
\begin{tabular}{llcr}
\toprule
Backbone & Layout & SSM fraction & Total params \\
\midrule
\texttt{mamba2} & $4\times$ Mamba2 & $1.00$ & $14{,}733{,}664$ \\
\texttt{hybrid75} & attention at $\{2\}$ & $0.75$ & $14{,}531{,}656$ \\
\texttt{hybrid50} & attention at $\{1,3\}$ & $0.50$ & $14{,}329{,}648$ \\
\texttt{hybrid50alt} & attention at $\{0,2\}$ & $0.50$ & $14{,}329{,}648$ \\
\texttt{hybrid25} & attention at $\{0,2,3\}$ & $0.25$ & $14{,}127{,}640$ \\
\texttt{attn0} & $4\times$ attention & $0.00$ & $13{,}925{,}632$ \\
\bottomrule
\end{tabular}
\end{table}

\textbf{Hybrid and Attention Backbones.} The backbone ablation (see the ablations in the Practical Distributed Learning Framework subsection) replaces some or all of the four Mamba2 mixers of the depth-$4$ model by causal multi-head attention layers, keeping \texttt{d\_model} $=256$, the tied GPT-2 embedding and LM head, the RMSNorm layers, and the residual structure unchanged, so that the mixer of every layer is either a Mamba2 block with the configuration above or an attention block. Attention layers are \texttt{mamba\_ssm}'s \texttt{MHA} (Multi-Head Attention) with $8$ heads of dimension $32$, rotary position embeddings over the full head dimension, causal masking, biases on the input (\texttt{qkv}) and output projections, and no feed-forward sublayer. The attention layers are spread evenly through the stack. With $n_{\mathrm{attn}}$ attention layers among the $L=4$ layers, the $j$-th attention layer (counting from $j=0$) is placed at index $\lfloor(j+0.5)\,L/n_{\mathrm{attn}}\rfloor$, which puts a single attention layer in the middle of the stack (index $2$), two attention layers at indices $1$ and $3$, and three at indices $0$, $2$, and $3$, with the remaining layers being Mamba2 mixers. To check that the placement itself does not drive the results, \texttt{hybrid50alt} keeps two attention layers but moves them to indices $0$ and $2$. Table~\ref{tab:backbones} lists the resulting layouts and parameter counts: an attention layer has $263{,}168$ parameters against $465{,}176$ for a Mamba2 mixer, so the total parameter count decreases with the attention fraction. The attention-only model (\texttt{attn0}) serves as the zero-SSM control (since it has no SSM state to align).

\textbf{Model Merging.} The model merging experiments use public Mamba2 checkpoints loaded from Hugging Face with \texttt{MambaLMHeadModel.from\_pretrained}. The results evaluate the 130M, 370M, and 780M checkpoints.

\begin{table}[h]
\centering
\caption{Public Mamba2 checkpoints used for model merging.}
\begin{tabular}{lc}
\toprule
Checkpoint & Total params \\
\midrule
\texttt{state-spaces/mamba2-130m} & 130M \\
\texttt{state-spaces/mamba2-370m} & 370M \\
\texttt{state-spaces/mamba2-780m} & 780M \\
\bottomrule
\end{tabular}
\label{tab:mamba2_merge_checkpoints}
\end{table}

All checkpoints are loaded with \texttt{dtype = bfloat16}. The base state $\theta^0$ is the unmodified pretrained checkpoint and is fixed once at the start of each experiment. The per-task specialists $\theta_m$ are produced by task fine-tuning (details in~\ref{appendix:model_merging_framework}), and saved to the per-experiment directory. 

The pretrained models and the \texttt{mamba\_ssm} module are  used under their original licenses and terms of use.

\subsection{Synthetic Distributed Learning Framework}
For every configuration $(D,N,\rho)$ and learner seed $s\in\{0,1,2\}$, the synthetic pipeline runs four stages on the dataset corresponding to that configuration: (i) centralized training of the learner of the Models subsection, (ii) evaluation of the gradient bound of Theorem~\ref{thm:ssm_grad_bound} at the stored checkpoints, (iii) evaluation of the smoothness constants of Lemma~\ref{lem:ssm_local_smoothness} at the same checkpoints, and (iv) federated training with FedAvg and FedProx, which measures the quantities entering Theorem~\ref{thm:fedavg_ssm_convergence} and Theorem~\ref{cor:fedprox_ssm_convergence}. Stages (ii) and (iii) never train, they only read checkpoints. All stages use the margins $(\gamma,\delta)$ of the configuration. Table~\ref{tab:synthetic_training_schedule} lists the shared settings.

\begin{table}[h]
\centering
\caption{Shared settings of the synthetic centralized and federated runs.}
\label{tab:synthetic_training_schedule}
\begin{tabular}{ll}
\toprule
\textbf{Component} & \textbf{Value} \\
\midrule
Loss & mean softmax cross-entropy over the $S-1$ next-token positions \\
Optimizer & SGD, no momentum, weight decay, or gradient clipping \\
Learning rate $\eta$ & $0.05$ \\
Batch size & $64$ sequences \\
Centralized updates & $2{,}000$ (epochs of $64$ updates, reshuffled every epoch) \\
Checkpoints and evaluation & every $100$ updates; bound and smoothness \\
& checkpoints at updates $0$, $500$, $1{,}000$, $1{,}500$, $2{,}000$ \\
Stability & projection onto $\mathcal S_{\gamma+\epsilon,\delta+\epsilon}$ with $\epsilon=10^{-4}$ \\
& and certificate check after every update \\
Learner seeds & $\{0,1,2\}$ per configuration \\
Precision & training in \texttt{float32}; bound, smoothness, \\
& and certificate evaluation in \texttt{float64} \\
\bottomrule
\end{tabular}
\end{table}

\textbf{Centralized Training.} The learner minimizes the mean teacher-forced next-token cross-entropy over the $S-1$ shifted positions of every training sequence, i.e., the sequence loss of~\eqref{eq:global_loss_definition} with $\ell_t$ the softmax cross-entropy and $y_t^*=x_{t+1}$. SGD uses no momentum, weight decay, or clipping, while projection and epoch reshuffling differ from the analyzed independent-sampling protocol. Each epoch shuffles the $4{,}096$ training sequences with a seed derived from the learner seed and the epoch index and cuts them into batches of $64$, running for $2{,}000$ updates. After every update the parameters are projected onto $\mathcal S_{\gamma+\epsilon,\delta+\epsilon}$ as at initialization, and membership in $\mathcal S_{\gamma,\delta}$ is re-verified over all $D$ one-hot inputs. Checkpoints are written every $100$ updates, and the test NLL, perplexity, and next-token accuracy are evaluated on the $1{,}024$ test sequences at the same frequency, always restarting from $h_0=\mathbf 0_{N\times D}$ at the first stored token. 

\textbf{Gradient-Bound Evaluation.} At the five checkpoints of Table~\ref{tab:synthetic_training_schedule} and at the frozen teacher, every one of the $1{,}024$ test sequences is processed in \texttt{float64}, and the gradient norm of its sequence loss with respect to the four blocks is computed as follows: $\|\nabla_W\mathcal L^S\|_F=\big(\sum_{X\in\{\Delta,B,C,A\}}\|\nabla_{W_X}\mathcal L^S\|_F^2\big)^{1/2}$. It is compared with two constants evaluated at the same parameters. The \emph{configured} bound $\mathcal G_{\mathrm{SSM}}(W)$ is Theorem~\ref{thm:ssm_grad_bound} with its analytic constants, the configured $(\gamma,\delta,\rho)$, $L_x=1$, and $L_\ell=\sqrt2$ (the gradient of the softmax cross-entropy with respect to the logits has norm at most $\sqrt2$). The \emph{tightened} bound $\mathcal G^{\mathrm{tight}}_{\mathrm{SSM}}(W)$ uses realized stability margins, exact one-hot projection bounds, and finite-history sums: the realized margins $\hat\gamma=\min_n(-A_n)$, $\hat\delta$, and $\hat\rho=\exp(-\hat\gamma\hat\delta)$ replace $(\gamma,\delta,\rho)$, $\max_d\|\operatorname{Softplus}(W_\Delta e_d)\|_2$ replaces $K_\Delta$, the largest column norms $\max_d\|W_Be_d\|_2$ and $\max_d\|W_Ce_d\|_2$ replace $\|W_B\|_2$ and $\|W_C\|_2$, and every geometric factor $1/(1-\rho)$ is truncated to $(1-\hat\rho^{S-1})/(1-\hat\rho)$, since a stored sequence contributes at most $S-1$ terms to each history sum. These replacements upper-bound the same one-hot projection and history terms and do not exceed the configured constants. A violation is recorded when the measured norm exceeds the bound by more than a relative tolerance of $10^{-8}$. The utilization is the ratio of the measured norm to the bound. The gradient of the batch-mean loss over all test sequences is evaluated as well and checked against the mean of the per-sequence norms (triangle inequality) and against the bound.

\textbf{Smoothness Evaluation.} At the same checkpoints, $32$ test sequences drawn once with a fixed seed are processed in \texttt{float64}. For every sequence, (a) the operator norm of the Hessian of the sequence loss is estimated by power iteration with exact Hessian-vector products (at most $15$ iterations, stopping when the Rayleigh quotient changes by less than a relative $10^{-9}$) and compared with $L_{\mathrm{SSM},R}$ at $R=0$, (b) the finite-difference ratio $\|\nabla\mathcal L^S(W+hU)-\nabla\mathcal L^S(W)\|_F/h$ with $h=10^{-3}$ is evaluated along three unit directions $U$ (two random and the normalized gradient) and compared with $L_{\mathrm{SSM},h}$ whenever the ball of radius $h$ is guaranteed to lie in $\mathcal S_{\gamma,\delta}$. Between consecutive learner checkpoints $W_a\to W_b$, the per-sequence ratio $\max_i\|\nabla\ell_i(W_b)-\nabla\ell_i(W_a)\|_F/\|W_b-W_a\|_F$ is compared with $L_{\mathrm{SSM},R}$ centered at $W_a$ with $R=\|W_b-W_a\|_F$, and $|\mathcal G_{\mathrm{SSM}}(W_b)-\mathcal G_{\mathrm{SSM}}(W_a)|/\|W_b-W_a\|_F$ with $L_{\mathcal G_{\mathrm{SSM}},R}$ of Lemma~\ref{lem:local_lipschit_SSM_grad_bound}. The loss-regularity constants are $L_\ell=\sqrt2$ and $H_\ell=1/2$.

\textbf{Federated Training.} FedAvg and FedProx are run on the same learner with $K$ clients holding disjoint, equal-size shards of the $4{,}096$ training sequences and weights $p_k=1/K$. Under the IID partition, the sequences are randomly permuted (seeded by the learner seed) and cut into $K$ shards. Under the dominant-token partition, the sequences are first sorted by their most frequent token, then by the fraction of positions that token occupies, and then cut into $K$ shards, so that clients differ systematically in token composition. Every client starts each round from the broadcast model and performs $E$ local SGD updates with the settings of Table~\ref{tab:synthetic_training_schedule}, each on a fresh batch of $64$ sequences drawn with replacement from the client's shard, followed by the same projection and verification as in centralized training. FedProx adds $\mu(W-W^r)$ to every local gradient, taking projected SGD steps on $F_k^r$ from Section~\ref{sec:system_model}. The server forms $W^{r+1}=\sum_kp_kW_k^{r,E}$ in \texttt{float64} and applies the projection once more. All runs start from the centralized learner initialization of the same seed and use $T=60$ rounds. Table~\ref{tab:synthetic_federated_settings} lists the settings: the base setting is run on all $72$ configurations, and the one-factor variations of $K$, $E$, $\mu$, and the partition are run on the $16$ core configurations, namely the $\rho$ axis at $D=N=8$ and all $(D,N)$ pairs at $\rho=0.95$. With three seeds this gives $1{,}200$ federated runs.

\begin{table}[h]
\centering
\caption{Synthetic federated settings. The base setting is run on all $72$ configurations, each variation changes one factor of the base setting and is run on the $16$ core configurations.}
\label{tab:synthetic_federated_settings}
\begin{tabular}{ll}
\toprule
\textbf{Component} & \textbf{Value} \\
\midrule
Algorithms & FedAvg; FedProx\\
Base setting & $K=4$ clients, $E=16$ local updates, IID partition \\
Communication rounds $T$ & $60$ \\
Client-count variation & $K\in\{2,8,32\}$ \\
Local-step variation & $E\in\{1,4,64\}$ \\
Proximal variation (FedProx) & $\mu\in\{0.01,1\}$ \\
Partition variation & dominant-token (non-IID) \\
Evaluation & at round $0$ and every $5$ rounds; $\|\nabla f(W^r)\|_F^2$ at every broadcast \\
Seeds & $\{0,1,2\}$, shared across the settings of a configuration \\
\bottomrule
\end{tabular}
\end{table}

The following quantities are recorded. At every evaluated round, the subset gradient norm $\|\nabla f(W^r)\|_F$ on a fixed subset of $1{,}024$ training sequences is compared with the configured and tightened $\mathcal G_{\mathrm{SSM}}(W^r)$. The client drift $\|W_k^{r,E}-W^r\|_F$ (mean and maximum over clients) and the largest displacement of any local iterate $\|W_k^{r,e}-W^r\|_F$ are compared with the radius $R$. The heterogeneity $\sum_kp_k\|\nabla\mathcal L_k(W^r)-\nabla f(W^r)\|_F^2$, with client gradients over full shards, is compared with $4\,\mathcal G_{\mathrm{SSM}}(W^r)^2\sum_kp_k\,\mathrm{TV}_k^2$ of Lemma~\ref{lem:ssm_gamma_tv_bound}, where $\mathrm{TV}_k$ is the exact total variation between the empirical distribution of client $k$ over whole sequences and the pooled distribution (unigram and bigram variants are recorded for reference but do not enter the lemma). The test NLL of the global model is evaluated on the $1{,}024$ test sequences. Reference radii and envelopes are set from initialization: $R=\eta E\,\mathcal G_{\mathrm{SSM}}(W^0)$ and $M_X=\|W_X^0\|_2+(T+1)R$ for each block $X$, from which $G_{\mathrm{SSM},R}$, $L_{\mathrm{SSM},R}$, and $L_{\mathcal G_{\mathrm{SSM}},R}$ are computed. For FedProx the radius is $R_\mu=\min\left\{R,\,G_{\mathrm{SSM},R}[1-(1-\eta\mu)^E]/\mu\right\}$ with correspondingly reduced envelopes. Beyond the sufficient step-size range, we log iterate containment, projection activity, and the tests $\eta EL\le1/2$ and $12\eta^2(L+\mu)^2E(E-1)\le1$. We compare the right-hand side of Theorem~\ref{thm:fedavg_ssm_convergence} (coefficients $4,4,6,48$) or Theorem~\ref{cor:fedprox_ssm_convergence} (coefficients $3,3,3,24$) with the subset estimate of $\frac1T\sum_{r<T}\|\nabla f(W^r)\|_F^2$, using the full-training initial loss $f(W^0)$ and $f^*=0$.

\subsection{Practical Distributed Learning Framework}
The federated framework uses $K=4$ clients unless stated otherwise (see the client-count ablation) and initializes all clients from the same randomly initialized model $\theta^0$. At the start of every round $r$, the server broadcasts the global model $\theta^r$, each client $k$ runs $E$ local optimization steps on its private data starting from $\theta^r$, and returns its updated parameters $\theta_k^{r,E}$ to a server module that performs algorithm-specific aggregation. After aggregation, the new global model $\theta^{r+1}$ is evaluated on the union of all clients' local validation data (server validation pool). The per-round per-client training perplexities and the server validation perplexity are saved.

\textbf{Shared Training Schedule.} All algorithms are evaluated under the same overall optimization schedule. Across communication rounds, we keep the number of clients, participation rate, local optimizer, batch construction, sequence length, and random seeds fixed. The only differences across methods are the algorithm-specific server or client update rule and the corresponding method-specific hyperparameters. The shared hyperparameters are summarized in Table~\ref{tab:shared_training_schedule}. Local training runs in single precision with mixed-precision (\texttt{fp16}) autocast and dynamic loss scaling, and gradients are clipped after unscaling. FedSAM is the one exception and runs in full \texttt{float32}. A fresh AdamW optimizer is created at the start of every round, so no optimizer state carries across rounds.

\begin{table}[h]
\centering
\caption{Shared training hyperparameters used across all federated algorithms.}
\label{tab:shared_training_schedule}
\begin{tabular}{ll}
\toprule
\textbf{Component} & \textbf{Value} \\
\midrule
Number of clients, $K$ & $4$ \\
Participation rate & $1.0$; all clients participate each round \\
Communication rounds, $T$ & $50$ \\
Local steps per round, $E$ & $\{64, 128, 256\}$ \\
Per-step batch size & $8$ \\
Sequence length, $S$ & $512$ \\
Local optimizer & AdamW, $\beta_1=0.9$, $\beta_2=0.999$, $\epsilon=10^{-8}$ \\
Local weight decay & $0.01$ \\
Local learning rate & $10^{-3}$ \\
Gradient clipping & Global $\ell_2$ norm clipping with threshold $1.0$ \\
Sampling & Infinite \texttt{RandomSampler} with replacement \\
Random seeds & $\{0,1,2\}$ for PyTorch and NumPy \\
\bottomrule
\end{tabular}
\end{table}

The following metrics are saved for every round:
\begin{itemize}
    \item \textbf{Train Perplexity} (per client). During each local step, the client computes the token-level cross-entropy of the next-token objective: the logits at positions $1,\dots,S-1$ are scored against the tokens at positions $2,\dots,S$, with $S=512$. The pre-aggregation training perplexity of client $k$ in round $r$ is $\exp(\overline{\mathrm{CE}})$, where $\overline{\mathrm{CE}}$ is the mean cross-entropy over the $E$ local steps of the round.
    \item \textbf{Validation Perplexity}. After aggregation, the global model $\theta^{r+1}$ is evaluated on the server validation pool. We report $\exp(\overline{\mathrm{CE}})$ with $+\infty$ saturation on overflow.
\end{itemize}

\textbf{Federated Algorithms.} We evaluate nine federated learning algorithms, all sharing the schedule of Table~\ref{tab:shared_training_schedule}. Throughout, $p_k$ denotes the sample-count weight of client $k$, and since every client processes the same number of sequences per round, $p_k=1/K$ in all runs.
\begin{itemize}
    \item \textbf{FedAvg}~\citep{mcmahan2017communication}. After each client completes its $E$ local AdamW updates, the server sets $\theta^{r+1}=\sum_{k=1}^Kp_k\theta_k^{r,E}=: \bar\theta^{r+1}$. This method serves as the main baseline and introduces no algorithm-specific hyperparameters.

    \item \textbf{FedProx}~\citep{MLSYS2020_1f5fe839}. FedProx adds the uniform proximal penalty $\frac{\mu}{2}\|\theta-\theta^r\|_F^2$ to the local objective, which discourages local models from drifting far from the broadcast model. Server aggregation is identical to FedAvg.

    \item \textbf{SCAFFOLD}~\citep{pmlr-v119-karimireddy20a}. SCAFFOLD corrects each client's local update with the difference between a server control variate $c^r$ and the client control variate $c_k^r$: the local gradient is shifted by $\lambda_c\,(c^r-c_k^r)$ with correction scale $\lambda_c$. The control variates are refreshed every $f_c$ rounds with control learning rate $\eta_c$, and their norm is bounded by $\tau_c$. Aggregation of the model parameters is identical to FedAvg.

    \item \textbf{FedAdam}~\citep{reddi2021adaptive}. Clients run the unchanged AdamW local phase. The server applies Adam to the pseudo-gradient $\bar\theta^{r+1}-\theta^r$: it maintains first- and second-moment estimates $m^r$ and $v^r$ with coefficients $(\beta_1^s,\beta_2^s)$ and updates $\theta^{r+1}=\theta^r+\alpha_s\,m^r/(\sqrt{v^r}+\tau_s)$ coordinatewise, with server learning rate $\alpha_s$ and offset $\tau_s$.

    \item \textbf{FedAlign}~\citep{11099331}. FedAlign aligns the Mamba2 state coordinates of the client models before averaging. In the default mode (\texttt{align\_mode = mamba2\_ref}), the $B/C$ state features of every client are matched to the current global model as the reference within each group, applying the same permutation to all coupled parameters. Optionally (\texttt{align\_mamba2\_heads = 1}), heads are matched as well, permuting every head-indexed parameter. After alignment, the client parameters are averaged with the same weights as FedAvg.

    \item \textbf{FedDyn}~\citep{acar2021federated}. Each client keeps a persistent dual state $\nu_k^r$, initialized at zero, and trains the local objective $\mathcal L_k(\theta)-\langle\nu_k^r,\theta\rangle_F+\frac{\alpha}{2}\|\theta-\theta^r\|_F^2$, with the two regularization terms computed in \texttt{float32} outside autocast. After local training, the client sets $\nu_k^{r+1}=\nu_k^r-\alpha(\theta_k^{r,E}-\theta^r)$. The server keeps the average dual state $\nu^{r+1}=\nu^r-\alpha\sum_kp_k(\theta_k^{r,E}-\theta^r)$ and updates $\theta^{r+1}=\bar\theta^{r+1}-\nu^{r+1}/\alpha$. The local optimizer remains AdamW, and the gradient of the regularized objective is clipped like every other local gradient.

    \item \textbf{FedSAM}~\citep{pmlr-v162-qu22a}. Every local step is a first-order sharpness-aware step: the client computes the gradient $g$ at the current iterate $\theta$, perturbs the weights by $u=\mathfrak R_{\mathrm{SAM}}\,g/(\|g\|_F+10^{-12})$ with the norm taken jointly over all trainable parameters, computes the gradient at $\theta+u$, restores $\theta$ exactly, and applies the perturbed gradient (clipped) with AdamW. Aggregation is identical to FedAvg. Because the perturbation must be formed from exact gradients across two backward passes, FedSAM runs without mixed precision at roughly twice the wall-clock cost of the other methods.

    \item \textbf{FedExP}~\citep{jhunjhunwala2023fedexp}. Local training is identical to FedAvg. With the client displacements $\theta_k^{r,E}-\theta^r$ and their weighted average $\sum_kp_k\left(\theta_k^{r,E}-\theta^r\right)$, the server applies $\theta^{r+1}=\theta^r+\eta_g^r \sum_kp_k\left(\theta_k^{r,E}-\theta^r\right)$ with the scaling factor $\eta_g^r=\max\big\{1,\ \sum_kp_k\|\Delta_k^r\|_F^2\big/\big(2\|\Delta^r\|_F^2+\epsilon_g\big)\big\}$, a single scalar shared by all parameters, with norms taken over all floating-point tensors. The factor has no upper cap and is logged every round. We do not average the last two server iterates at evaluation time.

    \item \textbf{DiLoCo-style}~\citep{douillard2024diloco}. Clients run the same AdamW local phase as FedAvg. The server applies an outer Nesterov step to the pseudo-gradient $\bar\theta^{r+1}-\theta^r$: $m_o^r=\beta_om_o^{r-1}+\bar\theta^{r+1}-\theta^r$ and $\theta^{r+1}=\theta^r+\alpha_o(\beta_om_o^r+\bar\theta^{r+1}-\theta^r)$, with $m_o^{-1}=0$, outer learning rate $\alpha_o$, outer momentum $\beta_o$, and per-parameter momentum buffers kept in \texttt{float32}. Unlike DiLoCo, whose workers keep their inner AdamW state across outer steps, the inner optimizer is re-created every round, as for all other methods.
\end{itemize}

\textbf{Hyperparameter Tuning.} For each algorithm, we ran every value of the grid in Table~\ref{tab:fed_hyperparams} at every $(\text{depth},E,\beta,\text{setting})$ configuration and over three seeds.

\begin{table}[h]
\centering
\footnotesize
\caption{Hyperparameter grids of the federated algorithms, run at every configuration and seed. Values shared by all members of a grid are fixed.}
\label{tab:fed_hyperparams}
\begin{tabular}{p{0.16\linewidth} p{0.76\linewidth}}
\toprule
\textbf{Algorithm} & \textbf{Hyperparameters and values} \\
\midrule
FedAvg & none \\
FedProx & $\mu\in\{5{\times}10^{-5},\,10^{-4},\,5{\times}10^{-4},\,10^{-3},\,5{\times}10^{-2}\}$ \\
SCAFFOLD & $f_c\in\{4,8,16\}$; $\eta_c=0.01$; $\lambda_c=0.01$; $\tau_c=0.1$ \\
FedAdam & $\alpha_s\in\{10^{-3},\,5{\times}10^{-3},\,10^{-2}\}$; $\beta_1^s=0.9$; $\beta_2^s=0.99$; $\tau_s=10^{-3}$ \\
FedAlign & \texttt{align\_mode} $=$ \texttt{mamba2\_ref}; \texttt{align\_mamba2\_heads} $\in\{0,1\}$ \\
FedDyn & $\alpha\in\{10^{-3},\,3{\times}10^{-3},\,10^{-2},\,3{\times}10^{-2},\,10^{-1}\}$ \\
FedSAM & $\mathfrak R_{\mathrm{SAM}}\in\{10^{-3},\,3{\times}10^{-3},\,10^{-2},\,5{\times}10^{-2},\,10^{-1}\}$ \\
FedExP & $\epsilon_g\in\{10^{-4},\,3{\times}10^{-4},\,10^{-3},\,3{\times}10^{-3},\,10^{-2}\}$ \\
DiLoCo-style & $\alpha_o\in\{0.1,0.2,0.3,0.4,0.5,0.6,0.7,1.0\}$; $\beta_o=0.9$ \\
\bottomrule
\end{tabular}
\end{table}

The final reported metrics correspond, per configuration, to the hyperparameter value with the lowest validation perplexity averaged over the last $10$ rounds (rounds 41-50) and over the three seeds. Runs that end with a non-finite perplexity are counted as failures and are never replaced.

\textbf{Reporting Metrics.} For every $(\text{depth},E,\beta,\text{setting})$ configuration we report:
\begin{itemize}
    \item \texttt{train\_ppl}: the per-round, per-client training perplexity $\exp(\overline{\mathrm{CE}})$ averaged over rounds and clients, then averaged over seeds (with its standard deviation).
    \item \texttt{val\_ppl}: the server validation perplexity $\exp(\overline{\mathrm{CE}})$ averaged over all $50$ rounds, then over seeds (with its standard deviation).
    \item \texttt{val\_ppl\_last10}: the server validation perplexity averaged over rounds 41--50, then over seeds (with its standard deviation).
\end{itemize}

\textbf{Ablations.} We further add two ablation experiments to the main grid, both at the constant settings $\beta=0.7$, depth $4$, $E=64$, $T=50$ with three seeds and one fixed hyperparameter value per algorithm: FedProx $\mu=10^{-3}$ ($5{\times}10^{-4}$ in the backbone ablation), SCAFFOLD $f_c=8$, FedAdam $\alpha_s=10^{-2}$, FedAlign without head matching (with head matching in the backbone ablation), FedDyn $\alpha=10^{-2}$, FedSAM $\mathfrak R_{\mathrm{SAM}}=10^{-2}$, FedExP $\epsilon_g=10^{-3}$, and DiLoCo-style $(\alpha_o,\beta_o)=(0.7,0.9)$. 

The \emph{client-count ablation} varies $K\in\{4,8,16,32\}$ on all four domain settings. Clients are assigned primary domains round-robin over the four domains of the setting, so $K/4$ clients share each primary domain, and each client draws a fraction $\beta$ of its data from its primary domain and $(1-\beta)/3$ from each of the other three. Since the demand on every domain grows as $K/4$, the per-client budget is capped at the available domain pool divided by this demand: the $7{,}427/1{,}857$ training/validation sequences per client at $K=4$ become approximately $3{,}713/928$, $1{,}856/464$, and $928/232$ at $K=8$, $16$, and $32$ for Settings 1 and 4, and slightly fewer for Settings 2 and 3  (which are limited by by the size of \texttt{c4}).

The \emph{backbone ablation} keeps $K=4$ and Setting 1 and replaces some or all Mamba2 mixers of the depth-$4$ model by attention layers. The six backbones are described in the Models subsection.

\subsection{Model Merging Framework}\label{appendix:model_merging_framework}
The merging framework is a two-phase pipeline. In the first phase, we fine-tune the pretrained base $\theta^0$ on each of the four downstream tasks independently, producing four task specialists $\theta_m$. In the second phase, we merge the specialists with each configured method and evaluate the merged state on every task's evaluation split.

\textbf{Per-Task Local Fine-Tuning}
For each task $m \in \{\texttt{ag\_news},\,\texttt{dbpedia\_14},\,\texttt{trec},\,\texttt{imdb}\}$ the pretrained base $\theta^0$ is fine-tuned on the task's training pool with the causal-LM cross-entropy objective restricted to label-token positions. Specialists are checkpointed at every epoch in "Checkpoint epochs" so that phase 2 can be repeated at multiple training horizons without re-paying the fine-tuning cost. Table~\ref{tab:finetune_recipe} lists the optimization parameters.

\begin{table}[h]
\centering
\caption{Per-task fine-tuning recipe (identical across tasks and seeds).}
\label{tab:finetune_recipe}
\renewcommand{\arraystretch}{1.2}
\begin{tabular}{p{0.32\linewidth} p{0.60\linewidth}}
\hline
\textbf{Component} & \textbf{Value} \\
\hline
Optimizer & AdamW \\

Learning rate & $2\times10^{-4}$ \\

Weight decay & 0.0 \\

Gradient clipping & $\ell_2$ norm, threshold 1.0 \\

Train batch size & 4 for \texttt{mamba2-130m, mamba2-370m}, 8 for \texttt{mamba2-780m} \\

Eval batch size & 4 \\

LR schedule & linear warmup + cosine decay \\

Warmup ratio & 0.05 (5\% of total steps) \\

Total epochs &
4 \\

Checkpoint epochs & $\{1, 2, 4\}$ \\

Compute / storage dtype & \texttt{bfloat16} \\

Seeds per (model, task) & $\{0, 1, 2\}$ \\
Sequence length & 256 tokens \\

\hline
\end{tabular}
\end{table}
The cross-entropy loss is masked to label-token positions only, prompt and template tokens are ignored. The training-pool subsample varies with the seed (so each seed sees a different 40k subset of the underlying task), while the evaluation-pool subsample is held fixed across seeds (eval seed = 99,991) so that all reported accuracies are comparable across seeds and methods.

\subsection{Reproducibility}\label{appendix:reproducibility}
\textbf{Deterministic Data and Seed Handling.} All three frameworks share a single PyTorch + NumPy seed. The federated harness deterministically caches the tokenized RedPajama subsets to a single binary cache file shared across all algorithms, settings, and seeds. The merging harness deterministically samples train and eval pools from each downstream dataset, with the eval seed (99,991) held fixed across all training seeds so that all reported accuracies are directly comparable. 

\textbf{Hardware and Runtime.}
All reported experiments were run on single-GPU workers using NVIDIA A100-40GB and NVIDIA H100 GPUs. The A100-40GB workers were used for prototyping, debugging, and smaller validation runs, as well as the synthetic dataset generation and validation experiments. The remaining main sweeps of federated and model merging experiments were run on the H100 workers. Each run used one GPU worker, with host CPU/RAM used only for data loading, tokenization-cache access, logging, and checkpoint I/O. The tokenized RedPajama subsets were cached once and reused across algorithms, settings, and seeds.

For federated language-modeling experiments, the dominant cost was local client training. Runtime scaled approximately linearly with the number of local steps $E$. In our setup, single federated runs ranged from approximately $30$ minutes for $E=64$ to at most roughly 2 hours for the longest $E=256$ runs on an H100. All federated algorithms used the same base hardware and data pipeline. Their runtime differences were small
relative to the effect of $E$. Naturally, the $16$-layer model scaled the computation compared to the $4$-layer one. 

For model merging, each Phase-1 task-expert fine-tuning configuration required up to roughly 8 hours on a single H100, depending on checkpoint size. 

\section{Additional Experimental Results and Ablation Studies}\label{appendix:additional_experimental_results}

\subsection{Single Layer Synthetic Distributed Learning}\label{appendix:synthetic_results}
Note that we validate the single-layer statements only. The multi-layer bound of Theorem~\ref{thm:ssm_multi_layer_grad_bound} composes the same per-layer constants with the layerwise sensitivities $\kappa_q$ of~\eqref{eq:kappa_definition}: the contribution of layer $q$ is the single-layer constant of that layer (with its parameters adjusted to the multi-layer setting) multiplied by $\prod_{j>q}\kappa_j$ over the layers above it. When $\kappa_j\geq1$, these factors do not reduce the per-layer bounds, but this does not establish greater relative looseness. Useful layers can also contract perturbations. Numerical evaluation is restricted to the single-layer theory.

\textbf{Validity and tightness.} Table~\ref{tab:app_validity_by_rho} reports, per $\rho$, the number of evaluated cells, parameter points and sequences, the violation counts for the configured and tightened bounds (all zero), the median utilization with confidence intervals, the largest utilization and the median realized $\hat\rho$. Table~\ref{tab:app_utilization_by_checkpoint} gives the same statistics along training. Figure~\ref{fig:app_utilization} shows the utilization along training, one curve per $\rho$, and Figure~\ref{fig:app_scatter_ecdf} the tightened-constant companion of Figure~\ref{fig:synthetic_a} together with the full per-sequence utilization distributions: at every $\rho$ the distribution ends at least two orders of magnitude below the bound.

\begin{table}[h]
\centering
\scriptsize
\caption{Validity and tightness of the gradient bound of Theorem~\ref{thm:ssm_grad_bound} by stability limit (single layer). Utilization is the measured per-sequence gradient norm divided by $\mathcal G_{\mathrm{SSM}}$ (conf.) or $\mathcal G^{\mathrm{tight}}_{\mathrm{SSM}}$ (tight.). The row marked $^{*}$ is the one $\rho=0.3$ cell whose realized step $\hat\delta$ differs from the other eight and is grouped separately by the evaluation.}
\label{tab:app_validity_by_rho}
\begin{tabular}{lrrrrlrrr}
\toprule
$\rho$ & cells & points & sequences & viol.\ (conf./tight.) & median util.\ (conf.) & median util.\ (tight.) & max util.\ (tight.) & median $\hat\rho$ \\
\midrule
0.3$^{*}$ & 1 & 18 & 18,432 & 0 / 0 & $7.6\times10^{-5}$ & $7.2\times10^{-4}$ & $8.3\times10^{-3}$ & 0.0602 \\
0.3 & 8 & 144 & 147,456 & 0 / 0 & $2.0\times10^{-6}$ & $1.2\times10^{-4}$ & $4.1\times10^{-3}$ & 0.0345 \\
0.5 & 9 & 162 & 165,888 & 0 / 0 & $1.6\times10^{-6}$ & $3.7\times10^{-5}$ & $5.1\times10^{-3}$ & 0.146 \\
0.7 & 9 & 162 & 165,888 & 0 / 0 & $7.7\times10^{-7}$ & $1.6\times10^{-5}$ & $2.1\times10^{-3}$ & 0.369 \\
0.9 & 9 & 162 & 165,888 & 0 / 0 & $2.3\times10^{-7}$ & $4.1\times10^{-6}$ & $1.0\times10^{-3}$ & 0.746 \\
0.95 & 9 & 162 & 165,888 & 0 / 0 & $9.1\times10^{-8}$ & $2.0\times10^{-6}$ & $5.7\times10^{-4}$ & 0.868 \\
0.99 & 9 & 162 & 165,888 & 0 / 0 & $3.2\times10^{-8}$ & $7.3\times10^{-7}$ & $1.6\times10^{-4}$ & 0.97 \\
0.995 & 9 & 162 & 165,888 & 0 / 0 & $1.2\times10^{-8}$ & $7.6\times10^{-7}$ & $2.4\times10^{-4}$ & 0.982 \\
0.999 & 9 & 162 & 165,888 & 0 / 0 & $5.3\times10^{-10}$ & $6.2\times10^{-7}$ & $1.2\times10^{-4}$ & 0.996 \\
all & 72 & 1,296 & 1,327,104 & 0 / 0 & $1.2\times10^{-7}$ & $3.5\times10^{-6}$ & $8.3\times10^{-3}$ & 0.823 \\
\bottomrule
\end{tabular}
\end{table}

\begin{table}[h]
\centering
\scriptsize
\caption{Utilization along training (single layer).}
\label{tab:app_utilization_by_checkpoint}
\begin{tabular}{lrrrrrr}
\toprule
checkpoint & points & median util.\ (conf.) & median util.\ (tight.) & max util.\ (tight.) & median $\mathcal{G}_{\mathrm{SSM}}$ & median gradient norm \\
\midrule
frozen\_teacher & 216 & $2.7\times10^{-8}$ & $4.0\times10^{-7}$ & $1.8\times10^{-4}$ & $9.0\times10^{6}$ & 0.177 \\
learner\_final & 216 & $1.2\times10^{-7}$ & $4.4\times10^{-6}$ & $3.5\times10^{-3}$ & $6.5\times10^{5}$ & 0.0612 \\
learner\_initialization & 216 & $3.4\times10^{-7}$ & $1.9\times10^{-5}$ & $8.3\times10^{-3}$ & $1.1\times10^{5}$ & 0.0293 \\
learner\_intermediate & 648 & $1.4\times10^{-7}$ & $5.8\times10^{-6}$ & $7.4\times10^{-3}$ & $4.0\times10^{5}$ & 0.0438 \\
\bottomrule
\end{tabular}
\end{table}

\begin{figure}[h]
\centering
\begin{subfigure}[h]{0.48\linewidth}
  \centering\includegraphics[width=\linewidth]{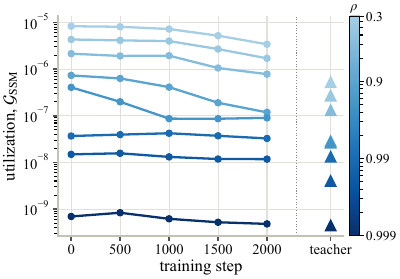}
  \caption{Configured constants}\label{fig:app_utilization_a}
\end{subfigure}\hfill
\begin{subfigure}[h]{0.48\linewidth}
  \centering\includegraphics[width=\linewidth]{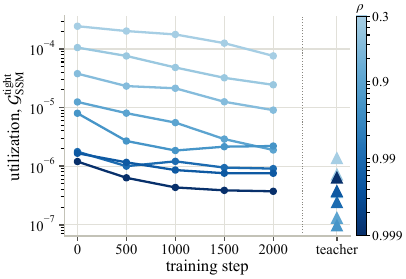}
  \caption{Tightened constants}\label{fig:app_utilization_b}
\end{subfigure}
\caption{Median bound utilization along training, one curve per $\rho$ (colour), with the frozen teacher at the last position. The bound is tightest at initialization and loosest at the frozen teacher at every $\rho$.}
\label{fig:app_utilization}
\end{figure}

\begin{figure}[h]
\centering
\begin{subfigure}[h]{0.48\linewidth}
  \centering\includegraphics[width=\linewidth]{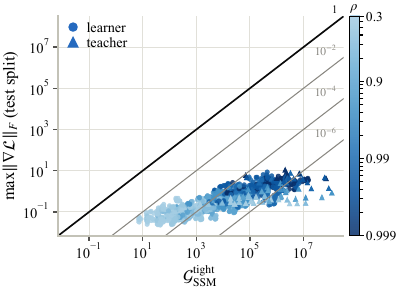}
  \caption{Gradient versus $\mathcal G^{\mathrm{tight}}_{\mathrm{SSM}}$}\label{fig:app_scatter_ecdf_a}
\end{subfigure}\hfill
\begin{subfigure}[h]{0.48\linewidth}
  \centering\includegraphics[width=\linewidth]{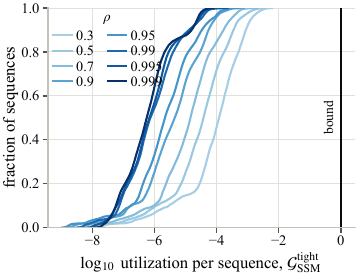}
  \caption{Per-sequence utilization}\label{fig:app_scatter_ecdf_b}
\end{subfigure}
\caption{(\subref{fig:app_scatter_ecdf_a})~Largest measured gradient norm over the test split against the tightened $\mathcal G^{\mathrm{tight}}_{\mathrm{SSM}}$ for all 1,296 parameter points (thin lines mark utilizations $10^{-2}$ to $10^{-6}$). (\subref{fig:app_scatter_ecdf_b})~Empirical distribution of the per-sequence tightened utilization per $\rho$ (one in eight sequences shown), the vertical line is the bound.}
\label{fig:app_scatter_ecdf}
\end{figure}

\textbf{Anatomy of the bound and of the gradient.} Table~\ref{tab:app_block_shares} lists the median share of the squared gradient norm carried by each parameter block at the final learner checkpoint and at the frozen teacher, next to the share of $\mathcal G_{\mathrm{SSM}}$ carried by its recurrent-Jacobian term. 

\begin{table}[h]
\centering
\scriptsize
\caption{Where the measured gradient lives (median share of squared Frobenius norm per parameter block) against the share of the certified bound carried by its recurrent-Jacobian term.}
\label{tab:app_block_shares}
\begin{tabular}{llrrrrr}
\toprule
checkpoint & $\rho$ & $W_\Delta$ & $W_B$ & $W_C$ & $A$ & recurrent share of $\mathcal{G}_{\mathrm{SSM}}$ \\
\midrule
final (step 2000) & 0.3 & 0.148 & 0.428 & 0.431 & 0.000 & 0.994 \\
 & 0.5 & 0.145 & 0.434 & 0.433 & 0.000 & 0.996 \\
 & 0.7 & 0.121 & 0.432 & 0.428 & 0.002 & 0.997 \\
 & 0.9 & 0.044 & 0.369 & 0.415 & 0.039 & 0.999 \\
 & 0.95 & 0.039 & 0.387 & 0.429 & 0.079 & 0.999 \\
 & 0.99 & 0.006 & 0.187 & 0.204 & 0.600 & 1.000 \\
 & 0.995 & 0.005 & 0.096 & 0.122 & 0.780 & 1.000 \\
 & 0.999 & 0.003 & 0.067 & 0.072 & 0.852 & 1.000 \\
frozen teacher & 0.3 & 0.050 & 0.594 & 0.331 & 0.000 & 0.999 \\
 & 0.5 & 0.077 & 0.568 & 0.404 & 0.000 & 0.999 \\
 & 0.7 & 0.065 & 0.548 & 0.350 & 0.003 & 0.999 \\
 & 0.9 & 0.043 & 0.482 & 0.301 & 0.111 & 1.000 \\
 & 0.95 & 0.028 & 0.354 & 0.124 & 0.325 & 1.000 \\
 & 0.99 & 0.002 & 0.393 & 0.015 & 0.581 & 1.000 \\
 & 0.995 & 0.002 & 0.420 & 0.008 & 0.570 & 1.000 \\
 & 0.999 & 0.001 & 0.362 & 0.004 & 0.631 & 1.000 \\
\bottomrule
\end{tabular}
\end{table}

\textbf{Dimension.} Table~\ref{tab:app_scaling_exponents} and Figure~\ref{fig:synthetic_d} give the log-log exponents of the configured bound, the tightened bound and the measured gradient in $D$ (at fixed $N$, seed and $\rho$) and in $N$ at the final learner checkpoint. The certified exponents follow the $N\,D^{2}$-type structure of $K_W$, the measured gradient is flat in $N$ and decreasing in $D$.

\begin{table}[h]
\centering
\scriptsize
\caption{Log-log exponents in the state and vocabulary sizes at the final learner checkpoint (median over one-factor fits with 95\% confidence intervals).}
\label{tab:app_scaling_exponents}
\begin{tabular}{lrrrrrr}
\toprule
$\rho$ & $\mathcal{G}$ in $D$ & $\mathcal{G}^{\mathrm{tight}}_{\mathrm{SSM}}$ in $D$ & $\|\nabla\mathcal{L}\|$ in $D$ & $\mathcal{G}$ in $N$ & $\mathcal{G}^{\mathrm{tight}}_{\mathrm{SSM}}$ in $N$ & $\|\nabla\mathcal{L}\|$ in $N$ \\
\midrule
0.3 & 1.78 [1.71, 2.84] & -0.29 [-1.28, 0.11] & -0.92 [-0.95, 0.41] & 1.96 [1.81, 2.10] & 1.81 [0.82, 2.30] & 0.17 [0.14, 0.22] \\
0.5 & 1.75 [1.48, 2.80] & 0.38 [-0.79, 0.77] & -0.97 [-1.09, 0.34] & 2.05 [1.76, 2.21] & 1.86 [1.21, 2.44] & 0.19 [0.12, 0.29] \\
0.7 & 1.39 [1.30, 2.47] & 0.04 [-0.40, 0.89] & -1.23 [-1.49, -0.07] & 1.70 [0.94, 1.80] & 1.41 [1.32, 1.96] & 0.03 [-0.37, 0.27] \\
0.9 & 1.86 [1.70, 2.19] & 0.48 [0.33, 1.64] & -1.28 [-1.66, -1.09] & 1.32 [0.74, 2.23] & 1.55 [0.30, 3.18] & 0.10 [-0.50, 0.87] \\
0.95 & 1.74 [1.25, 2.01] & 0.31 [0.10, 1.40] & -1.63 [-1.84, -1.58] & 1.09 [0.87, 1.14] & 1.37 [0.28, 1.55] & 0.05 [-0.07, 0.23] \\
0.99 & 1.85 [0.88, 2.23] & 1.05 [0.03, 1.96] & -1.38 [-1.58, -1.27] & 1.08 [0.95, 1.11] & 1.38 [0.75, 1.66] & 0.06 [-0.12, 0.12] \\
0.995 & 2.93 [0.71, 3.40] & 2.09 [0.12, 2.94] & -0.30 [-1.39, -0.21] & 1.70 [0.99, 2.15] & 1.55 [1.10, 2.48] & 0.50 [-0.09, 0.56] \\
0.999 & 2.90 [1.17, 3.85] & 2.33 [0.91, 3.71] & 0.16 [-0.06, 0.45] & 1.27 [0.63, 2.20] & 1.34 [0.81, 2.75] & 0.17 [-0.41, 0.49] \\
\bottomrule
\end{tabular}
\end{table}

\textbf{Realized stability and training.}  Table~\ref{tab:app_training_by_rho} and Figure~\ref{fig:app_training} demonstrate that the bounds use learners that generally outperform uniform prediction and approach the teachers' test Negative Log Likelihood (NLL). At short memory ($\rho$ far away from 1), the teachers' small predictive advantage over uniform makes normalized gap recovery sensitive to small absolute errors, while the learners remain well inside the stability limit. At long memory ($\rho$ approaches 1), learners nearly match the teachers and train close to that limit.
\begin{table}[h]
\centering
\scriptsize
\caption{Centralized learners by stability limit: fraction of runs whose test Negative Log Likelihood (NLL) falls below the uniform baseline, median excess NLL over the teacher, and the fraction of the uniform-to-teacher gap recovered.}
\label{tab:app_training_by_rho}
\begin{tabular}{lrrrrr}
\toprule
$\rho$ & runs & recovered fraction & median excess NLL & median gap recovered [95\% CI] & median final $\hat\rho$ \\
\midrule
0.3 & 27 & 0.93 & 0.0180 & 0.80 [0.52, 0.90] & 0.031 \\
0.5 & 27 & 0.93 & 0.0212 & 0.79 [0.49, 0.82] & 0.134 \\
0.7 & 27 & 0.93 & 0.0203 & 0.81 [0.40, 0.87] & 0.357 \\
0.9 & 27 & 0.93 & 0.0187 & 0.69 [0.64, 0.96] & 0.724 \\
0.95 & 27 & 1.00 & 0.0179 & 0.74 [0.52, 0.99] & 0.855 \\
0.99 & 27 & 1.00 & 0.0027 & 0.95 [0.42, 1.03] & 0.966 \\
0.995 & 27 & 1.00 & 0.0004 & 1.00 [0.75, 1.02] & 0.981 \\
0.999 & 27 & 1.00 & -0.0019 & 1.04 [0.95, 1.05] & 0.996 \\
\bottomrule
\end{tabular}
\end{table}

\begin{figure}[h]
\centering
\begin{subfigure}[h]{0.48\linewidth}
  \centering\includegraphics[width=\linewidth]{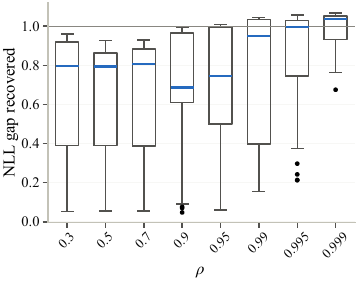}
  \caption{Gap recovered by the learners}\label{fig:app_training_a}
\end{subfigure}\hfill
\begin{subfigure}[h]{0.48\linewidth}
  \centering\includegraphics[width=\linewidth]{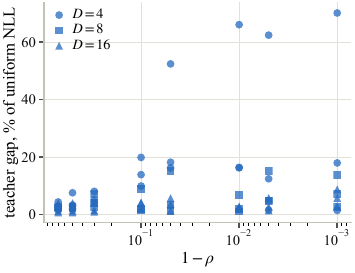}
  \caption{Predictive gap of the teachers}\label{fig:app_training_b}
\end{subfigure}
\caption{(\subref{fig:app_training_a})~Fraction of the uniform-to-teacher NLL gap recovered by the learners, per $\rho$. (\subref{fig:app_training_b})~The teachers' predictive gap over the uniform baseline as a percentage of the uniform NLL, per cell.}
\label{fig:app_training}
\end{figure}

\textbf{Federated ingredients.} Figure~\ref{fig:app_fed_scaling} shows how the quantities that enter Theorem~\ref{thm:fedavg_ssm_convergence} scale with the federated hyperparameters. Client drift, the end-of-round displacement $\|W_k^{r,E}-W^r\|_F$, grows as $E^{0.87}$ while the initialization-based reference radius $R=\eta E\mathcal G_{\mathrm{SSM}}(W^0)$ grows as $E^{1}$ (Figure~\ref{fig:app_fed_scaling_a}). The client disagreement $\sum_k p_k\|\nabla\mathcal L_k(W^r)-\nabla f(W^r)\|_F^2$, the quantity whose supremum $\Gamma^2_{\mathrm{SSM},R}$ Lemma~\ref{lem:ssm_gamma_tv_bound} bounds, grows as $K^{1.48}$ (Figure~\ref{fig:app_fed_scaling_b}) and as $(1-\rho)^{-1.15}$, whereas drift is flat in $\rho$ and $\|\nabla f\|_F^2$ grows only as $(1-\rho)^{-0.2}$ (Figure~\ref{fig:app_fed_scaling_c}), so disagreement grows faster than the squared gradient norm, consistent with the lemma's SSM-sensitivity dependence. Figure~\ref{fig:app_fed_rounds} shows the per-round curves behind these tail means. Figure~\ref{fig:app_heterogeneity} examines the heterogeneity lemma: measured disagreement is $10^{-11}$ to $10^{-14}$ of $4\mathcal G^2_{\mathrm{SSM}}\mathcal H_{\mathrm{TV}}$ (Figure~\ref{fig:app_heterogeneity_a}). The sequence-level squared-TV aggregate is $0.56$ at $K=4$ ($0.25$ to $0.94$ across $K$) and unchanged by the dominant-token partition, whereas a token-level squared-TV analogue rises from $1.6\times10^{-5}$ to $2.8\times10^{-3}$ and disagreement by $226\times$ (Figure~\ref{fig:app_heterogeneity_b}). Thus, the sequence-level term is insensitive to this partition; certifying a token-level or moment-based alternative requires a separate bound.

\begin{figure}[h]
\centering
\begin{subfigure}[h]{0.32\linewidth}
  \centering\includegraphics[width=\linewidth]{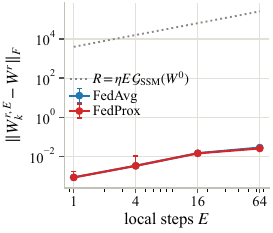}
  \caption{Client drift against $E$}\label{fig:app_fed_scaling_a}
\end{subfigure}\hfill
\begin{subfigure}[h]{0.32\linewidth}
  \centering\includegraphics[width=\linewidth]{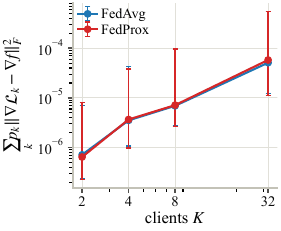}
  \caption{Dissimilarity against $K$}\label{fig:app_fed_scaling_b}
\end{subfigure}\hfill
\begin{subfigure}[h]{0.32\linewidth}
  \centering\includegraphics[width=\linewidth]{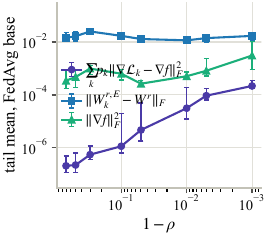}
  \caption{Base setting against $1-\rho$}\label{fig:app_fed_scaling_c}
\end{subfigure}
\caption{Federated scaling directions (medians with bootstrap 95\% intervals). (\subref{fig:app_fed_scaling_a})~End-of-round client drift against $E$ for FedAvg and FedProx on the core cells, with the initialization-based reference radius $R=\eta E\,\mathcal G_{\mathrm{SSM}}(W^0)$. (\subref{fig:app_fed_scaling_b})~Client dissimilarity $\sum_k p_k\|\nabla\mathcal L_k(W^r)-\nabla f(W^r)\|_F^2$ against $K$ on the core cells. (\subref{fig:app_fed_scaling_c})~$\sum_k p_k\|\nabla\mathcal L_k(W^r)-\nabla f(W^r)\|_F^2$, drift and $\|\nabla f\|_F^2$ against $1-\rho$ at the base setting (all 72 cells).}
\label{fig:app_fed_scaling}
\end{figure}

\begin{figure}[h]
\centering
\begin{subfigure}[h]{0.24\linewidth}
  \centering\includegraphics[width=\linewidth]{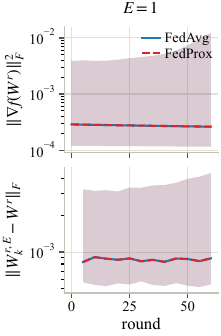}
  \caption{$E=1$}
\end{subfigure}\hfill
\begin{subfigure}[h]{0.24\linewidth}
  \centering\includegraphics[width=\linewidth]{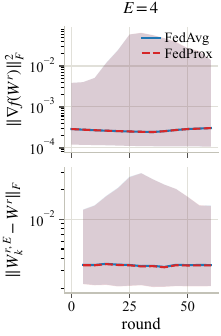}
  \caption{$E=4$}
\end{subfigure}\hfill
\begin{subfigure}[h]{0.24\linewidth}
  \centering\includegraphics[width=\linewidth]{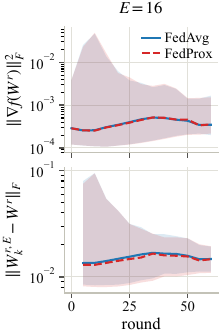}
  \caption{$E=16$}
\end{subfigure}\hfill
\begin{subfigure}[h]{0.24\linewidth}
  \centering\includegraphics[width=\linewidth]{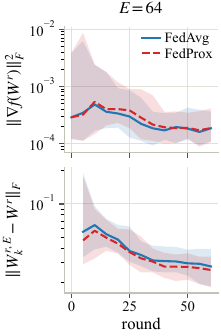}
  \caption{$E=64$}
\end{subfigure}
\caption{Per-round gradient (top) and client drift (bottom) on the core cells for each $E$}
\label{fig:app_fed_rounds}
\end{figure}

\begin{figure}[h]
\centering
\begin{subfigure}[h]{0.48\linewidth}
  \centering\includegraphics[width=\linewidth]{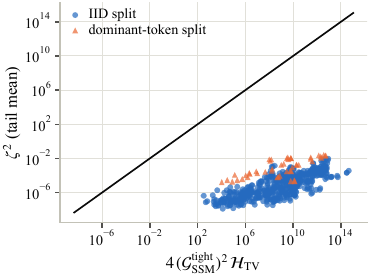}
  \caption{Heterogeneity lemma}\label{fig:app_heterogeneity_a}
\end{subfigure}\hfill
\begin{subfigure}[h]{0.48\linewidth}
  \centering\includegraphics[width=\linewidth]{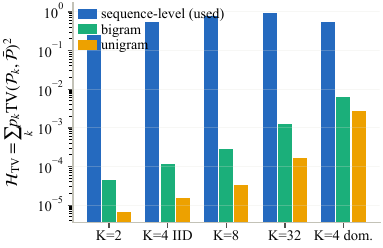}
  \caption{Total-variation terms}\label{fig:app_heterogeneity_b}
\end{subfigure}
\caption{(\subref{fig:app_heterogeneity_a})~Measured client dissimilarity against the certified heterogeneity bound with the tightened constant, for IID and dominant-token splits. (\subref{fig:app_heterogeneity_b})~Sequence-level, bigram and unigram weighted total variation for the client counts and partitions of the bank, only the sequence-level value enters the lemma.}
\label{fig:app_heterogeneity}
\end{figure}

\textbf{What drives the bound.} Table~\ref{tab:app_bound_drivers} and Figure~\ref{fig:app_drivers} give signed covariance attribution across the bank. The recurrent-history term carries at least $99\%$ of the bound, so $\log_{10}\mathcal G_{\mathrm{SSM}}$ is the sum of log-factors $\rho/(1-\rho)^2$, $ND$, $\|W_C\|_2$, $K_B$, $K_\Delta$, and $K_\Delta+\sqrt D\|A\|_2$, plus a residual below $0.1$ decades. Each share is the covariance of a log-factor with log-bound magnitude divided by the latter's variance. The memory factor spans $6.2$ decades and contributes $93\%$. Dimension, readout and step contribute $7\%$, $6\%$ and $4\%$, while the transition factor contributes $-13\%$, offsetting memory growth. This offset helps explain the exponent above $-2$ in Figure~\ref{fig:synthetic_b}. Teacher initialization gives $\|A\|_2=\gamma\sqrt{\sum_n(1+0.25|z_n|)^2}$ with $\gamma=-\ln\rho/\delta$. For log-utilization, memory contributes $102\%$ and measured gradient growth $-18\%$. Recurrence therefore dominates variation in tightness across this bank, not necessarily absolute looseness.

\begin{table}[h]
\centering
\scriptsize
\caption{What drives the bound across the bank. $\log_{10}\mathcal{G}_{\mathrm{SSM}}$ is the sum of the logarithms of the factors of its recurrent-history term (plus a residual for the remaining terms), so the covariance of each factor with $\log_{10}\mathcal{G}_{\mathrm{SSM}}$ divided by its variance is an additive attribution that sums to one over all 1,296 parameter points. The last column attributes the variation of the utilization $\|\nabla\mathcal L\|_F/\mathcal{G}_{\mathrm{SSM}}$ in the same way, with the measured gradient as an additional term. A negative share means the factor moves against the total: $\|A\|_2$ shrinks as $\rho\to1$.}
\label{tab:app_bound_drivers}
\begin{tabular}{lrrrr}
\toprule
factor & range (decades) & std (decades) & share of $\mathrm{Var}(\log_{10}\mathcal{G}_{\mathrm{SSM}})$ & share of $\mathrm{Var}(\log_{10}\,\mathrm{util.})$ \\
\midrule
memory $\rho/(1-\rho)^2$ & 6.2 & 2.06 & 93\% & 102\% \\
dimension $N\,D$ & 1.2 & 0.35 & 7\% & 12\% \\
readout $\|W_C\|_2$ & 1.5 & 0.31 & 6\% & 6\% \\
input $K_B$ & 1.8 & 0.28 & 3\% & 3\% \\
step $K_\Delta$ & 0.8 & 0.18 & 4\% & 6\% \\
transition $K_\Delta+\sqrt{D}\,\|A\|_2$ & 2.0 & 0.50 & -13\% & -11\% \\
remaining terms of $\mathcal{G}_{\mathrm{SSM}}$ & 0.1 & 0.01 & 0\% & 0\% \\
measured $\|\nabla\mathcal{L}\|_F$ & 2.7 & 0.62 & -- & -18\% \\
\bottomrule
\end{tabular}
\end{table}

\begin{figure}[h]
\centering
\begin{subfigure}[h]{0.48\linewidth}
  \centering\includegraphics[width=\linewidth]{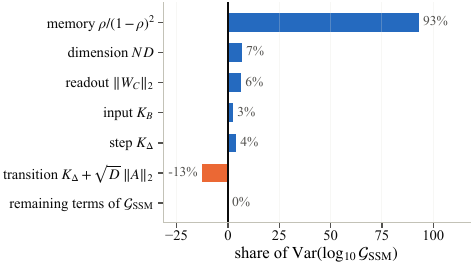}
  \caption{What drives $\mathcal G_{\mathrm{SSM}}$}\label{fig:app_drivers_a}
\end{subfigure}\hfill
\begin{subfigure}[h]{0.48\linewidth}
  \centering\includegraphics[width=\linewidth]{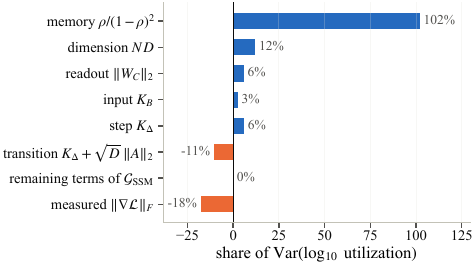}
  \caption{What drives the utilization $\|\nabla\mathcal L\|_F/\mathcal{G}_{\mathrm{SSM}}$}\label{fig:app_drivers_b}
\end{subfigure}
\caption{Additive attribution of the variance of (\subref{fig:app_drivers_a})~$\log_{10}\mathcal G_{\mathrm{SSM}}$ and (\subref{fig:app_drivers_b})~$\log_{10}$ utilization over all 1,296 parameter points to the factors of the bound (and, for the utilization, to the measured gradient). Bars sum to 100\%, a negative bar moves against the total.}
\label{fig:app_drivers}
\end{figure}

\subsection{Practical Distributed Learning}\label{appendix:practical_distributed_learning}
In this subsection we provide supplemental results to the Section~\ref{sec:experiments} Mamba2 federated learning experiments. Every table reports the three metrics defined in Appendix~\ref{appendix:detailed_exp_setup}: the client training perplexity, the server validation perplexity (both averaged over all $50$ rounds), and the server validation perplexity averaged over the last $10$ rounds. Appendix~\ref{appendix:practical_72_settings} evaluates the $9$ federated algorithms in all the $72$ main configurations in details, demonstrating how the different local rounds $E$, model depths, heterogeneity degrees $\beta$, and data compositions affect their performance. The reported standard deviation measures the per-seed average across configurations for 3 seeds (not the perplexity spread across configurations). Appendices~\ref{appendix:practical_client_count} and~\ref{appendix:practical_backbones} provide two ablation studies, on different client counts, and different model backbones respectively. Both use fixed hyperparameters listed in the Ablations paragraph of Appendix~\ref{appendix:detailed_exp_setup}.

\subsubsection{Detailed Results Across the 72 Main Configurations}\label{appendix:practical_72_settings}

Tables~\ref{tab:app_fl_by_beta}-\ref{tab:app_fl_by_depth} report the nine algorithms averaged over the $72$ configurations along each factor of the grid.

\providecommand{\meansd}[2]{#1\,{\scriptsize$\pm$\,#2}}
\providecommand{\bestmeansd}[2]{\textbf{#1}\,{\scriptsize$\pm$\,#2}}
\begin{table}[h]
\centering
\scriptsize
\setlength{\tabcolsep}{3pt}
\renewcommand{\arraystretch}{1.1}
\caption{\textbf{Effect of heterogeneity.} Training, validation, and last-10-round validation perplexity of the nine federated learning algorithms on Mamba2 at every $\beta$, averaged over the $24$ configurations.}
\label{tab:app_fl_by_beta}
\resizebox{\textwidth}{!}{%
\begin{tabular}{lccc ccc ccc}
\toprule
& \multicolumn{3}{c}{$\beta=0$} & \multicolumn{3}{c}{$\beta=0.7$} & \multicolumn{3}{c}{$\beta=1$} \\
\cmidrule(lr){2-4}\cmidrule(lr){5-7}\cmidrule(lr){8-10}
Algorithm & Train & Val & Last-10 & Train & Val & Last-10 & Train & Val & Last-10 \\
\midrule
FedAvg & \meansd{75.4}{0.5} & \meansd{63.7}{0.3} & \meansd{45.4}{0.2} & \meansd{99.3}{0.2} & \meansd{66.7}{0.6} & \meansd{45.8}{0.3} & \meansd{143.5}{0.7} & \meansd{136.5}{3.8} & \meansd{60.7}{0.5} \\
FedProx & \bestmeansd{69.0}{0.2} & \bestmeansd{57.6}{0.3} & \bestmeansd{36.9}{0.4} & \bestmeansd{88.0}{0.4} & \bestmeansd{60.4}{0.4} & \bestmeansd{36.5}{0.2} & \bestmeansd{119.3}{0.4} & \bestmeansd{121.3}{2.9} & \meansd{40.2}{0.3} \\
SCAFFOLD & \meansd{78.3}{0.5} & \meansd{66.3}{0.2} & \meansd{46.9}{0.1} & \meansd{102.4}{0.4} & \meansd{67.3}{0.3} & \meansd{45.2}{0.1} & \meansd{149.0}{0.5} & \meansd{131.6}{3.6} & \meansd{49.7}{0.2} \\
FedAdam & \meansd{88.4}{0.4} & \meansd{164.8}{1.6} & \meansd{53.0}{0.3} & \meansd{112.7}{0.4} & \meansd{160.0}{2.1} & \meansd{50.8}{0.1} & \meansd{157.2}{0.6} & \meansd{355.2}{13.8} & \meansd{61.7}{0.3} \\
FedAlign & \meansd{75.3}{0.4} & \meansd{63.6}{0.2} & \meansd{45.2}{0.0} & \meansd{99.2}{0.4} & \meansd{66.6}{0.3} & \meansd{45.7}{0.2} & \meansd{143.5}{0.5} & \meansd{135.3}{5.4} & \meansd{60.1}{0.4} \\
FedDyn & \meansd{78.9}{0.1} & \meansd{106.4}{2.1} & \meansd{44.1}{0.5} & \meansd{95.5}{0.2} & \meansd{120.1}{3.9} & \meansd{41.6}{0.1} & \meansd{127.1}{0.0} & \meansd{814.4}{50.2} & \bestmeansd{40.0}{0.2} \\
FedSAM & \meansd{75.4}{0.4} & \meansd{63.7}{0.3} & \meansd{45.2}{0.1} & \meansd{99.3}{0.4} & \meansd{66.7}{0.3} & \meansd{45.7}{0.1} & \meansd{144.5}{0.7} & \meansd{135.5}{3.4} & \meansd{59.8}{0.3} \\
FedExP & \meansd{75.3}{0.4} & \meansd{63.6}{0.2} & \meansd{45.2}{0.1} & \meansd{97.6}{0.4} & \meansd{66.5}{0.4} & \meansd{45.2}{0.2} & \meansd{137.8}{0.6} & \meansd{154.9}{7.1} & \meansd{68.9}{0.5} \\
DiLoCo-style & \meansd{72.4}{0.6} & \meansd{66.4}{0.4} & \meansd{46.1}{0.3} & \meansd{96.2}{0.5} & \meansd{71.7}{0.4} & \meansd{46.2}{0.2} & \meansd{140.5}{0.6} & \meansd{202.2}{9.8} & \meansd{61.3}{0.2} \\
\bottomrule
\end{tabular}}
\end{table}

\providecommand{\meansd}[2]{#1\,{\scriptsize$\pm$\,#2}}
\providecommand{\bestmeansd}[2]{\textbf{#1}\,{\scriptsize$\pm$\,#2}}
\begin{table}[h]
\centering
\scriptsize
\setlength{\tabcolsep}{3pt}
\renewcommand{\arraystretch}{1.1}
\caption{\textbf{Effect of local steps.} Training, validation, and last-10-round validation perplexity of the nine federated learning algorithms on Mamba2 at every $E$, averaged over the $24$ configurations.}
\label{tab:app_fl_by_steps}
\resizebox{\textwidth}{!}{%
\begin{tabular}{lccc ccc ccc}
\toprule
& \multicolumn{3}{c}{$E=64$} & \multicolumn{3}{c}{$E=128$} & \multicolumn{3}{c}{$E=256$} \\
\cmidrule(lr){2-4}\cmidrule(lr){5-7}\cmidrule(lr){8-10}
Algorithm & Train & Val & Last-10 & Train & Val & Last-10 & Train & Val & Last-10 \\
\midrule
FedAvg & \meansd{151.2}{0.5} & \meansd{108.5}{1.7} & \meansd{63.9}{0.2} & \meansd{100.9}{0.5} & \meansd{85.5}{1.1} & \meansd{48.7}{0.2} & \meansd{66.1}{0.3} & \meansd{72.9}{2.0} & \meansd{39.2}{0.4} \\
FedProx & \bestmeansd{137.7}{0.3} & \meansd{106.6}{1.7} & \bestmeansd{48.6}{0.2} & \bestmeansd{85.1}{0.4} & \bestmeansd{74.9}{0.9} & \bestmeansd{35.6}{0.4} & \bestmeansd{53.6}{0.3} & \bestmeansd{57.8}{1.0} & \bestmeansd{29.4}{0.2} \\
SCAFFOLD & \meansd{150.8}{0.4} & \bestmeansd{105.6}{1.7} & \meansd{61.0}{0.3} & \meansd{104.5}{0.3} & \meansd{85.3}{0.7} & \meansd{45.4}{0.3} & \meansd{74.5}{0.4} & \meansd{74.3}{2.0} & \meansd{35.4}{0.0} \\
FedAdam & \meansd{167.3}{0.5} & \meansd{261.5}{5.6} & \meansd{69.6}{0.5} & \meansd{113.0}{0.4} & \meansd{217.6}{6.6} & \meansd{53.1}{0.1} & \meansd{78.1}{0.3} & \meansd{200.9}{5.2} & \meansd{42.8}{0.1} \\
FedAlign & \meansd{151.2}{0.5} & \meansd{108.3}{2.1} & \meansd{63.6}{0.2} & \meansd{100.8}{0.4} & \meansd{85.1}{1.7} & \meansd{48.4}{0.3} & \meansd{66.0}{0.3} & \meansd{72.0}{2.3} & \meansd{39.0}{0.2} \\
FedDyn & \meansd{149.6}{0.1} & \meansd{257.9}{5.0} & \meansd{56.2}{0.3} & \meansd{90.6}{0.2} & \meansd{351.7}{43.0} & \meansd{38.8}{0.4} & \meansd{61.3}{0.1} & \meansd{431.3}{48.0} & \meansd{30.7}{0.0} \\
FedSAM & \meansd{151.3}{0.5} & \meansd{109.0}{1.5} & \meansd{63.8}{0.3} & \meansd{101.1}{0.4} & \meansd{85.6}{0.8} & \meansd{48.3}{0.1} & \meansd{66.8}{0.5} & \meansd{71.3}{1.8} & \meansd{38.6}{0.2} \\
FedExP & \meansd{147.1}{0.6} & \meansd{111.7}{1.2} & \meansd{65.3}{0.2} & \meansd{98.6}{0.5} & \meansd{91.5}{1.1} & \meansd{51.9}{0.3} & \meansd{65.0}{0.4} & \meansd{81.8}{5.6} & \meansd{42.1}{0.4} \\
DiLoCo-style & \meansd{144.6}{0.5} & \meansd{140.8}{5.4} & \meansd{63.5}{0.3} & \meansd{97.5}{0.6} & \meansd{113.4}{4.7} & \meansd{49.3}{0.3} & \meansd{67.0}{0.5} & \meansd{86.1}{3.2} & \meansd{40.9}{0.1} \\
\bottomrule
\end{tabular}}
\end{table}

\providecommand{\meansd}[2]{#1\,{\scriptsize$\pm$\,#2}}
\providecommand{\bestmeansd}[2]{\textbf{#1}\,{\scriptsize$\pm$\,#2}}
\begin{table}[h]
\centering
\scriptsize
\setlength{\tabcolsep}{3pt}
\renewcommand{\arraystretch}{1.1}
\caption{\textbf{Effect of the domain assignment.}  Training, validation, and last-10-round validation perplexity of the nine federated learning algorithms on Mamba2 for every domain setting (Setting 1: \texttt{arxiv}, \texttt{github}, \texttt{stackexchange}, \texttt{wikipedia}; Setting 2: \texttt{arxiv}, \texttt{c4}, \texttt{stackexchange}, \texttt{wikipedia}; Setting 3: \texttt{c4}, \texttt{common\_crawl}, \texttt{github}, \texttt{wikipedia}; Setting 4: \texttt{arxiv}, \texttt{common\_crawl}, \texttt{github}, \texttt{stackexchange}), averaged over the $18$ configurations.}
\label{tab:app_fl_by_setting}
\resizebox{\textwidth}{!}{%
\begin{tabular}{lccc ccc ccc ccc}
\toprule
& \multicolumn{3}{c}{Setting 1} & \multicolumn{3}{c}{Setting 2} & \multicolumn{3}{c}{Setting 3} & \multicolumn{3}{c}{Setting 4} \\
\cmidrule(lr){2-4}\cmidrule(lr){5-7}\cmidrule(lr){8-10}\cmidrule(lr){11-13}
Algorithm & Train & Val & Last-10 & Train & Val & Last-10 & Train & Val & Last-10 & Train & Val & Last-10 \\
\midrule
FedAvg & \meansd{45.1}{0.2} & \meansd{53.2}{1.5} & \meansd{25.2}{0.3} & \meansd{107.3}{0.0} & \meansd{96.5}{4.2} & \meansd{51.7}{0.6} & \meansd{140.6}{1.0} & \meansd{91.4}{3.3} & \meansd{56.9}{0.1} & \meansd{131.3}{0.5} & \meansd{114.8}{1.1} & \meansd{68.7}{0.8} \\
FedProx & \bestmeansd{42.3}{0.2} & \bestmeansd{46.4}{0.2} & \bestmeansd{19.0}{0.2} & \bestmeansd{93.8}{0.3} & \bestmeansd{86.4}{0.3} & \bestmeansd{38.9}{0.3} & \bestmeansd{117.7}{0.6} & \bestmeansd{87.3}{3.1} & \bestmeansd{43.0}{0.3} & \bestmeansd{114.6}{0.5} & \bestmeansd{98.9}{1.7} & \bestmeansd{50.5}{0.9} \\
SCAFFOLD & \meansd{45.9}{0.2} & \meansd{51.6}{1.5} & \meansd{22.6}{0.4} & \meansd{112.0}{0.0} & \meansd{95.6}{3.4} & \meansd{47.5}{0.2} & \meansd{146.5}{0.7} & \meansd{93.1}{3.2} & \meansd{55.6}{0.1} & \meansd{135.2}{0.5} & \meansd{113.3}{1.3} & \meansd{63.3}{0.7} \\
FedAdam & \meansd{51.4}{0.1} & \meansd{169.0}{4.3} & \meansd{27.2}{0.2} & \meansd{118.2}{0.4} & \meansd{260.3}{6.6} & \meansd{54.2}{0.2} & \meansd{158.0}{1.1} & \meansd{237.0}{6.4} & \meansd{62.2}{0.2} & \meansd{150.2}{0.2} & \meansd{240.3}{7.6} & \meansd{77.1}{1.3} \\
FedAlign & \meansd{45.0}{0.0} & \meansd{50.8}{1.6} & \meansd{24.8}{0.2} & \meansd{107.2}{0.3} & \meansd{96.7}{3.4} & \meansd{51.4}{0.3} & \meansd{140.6}{0.9} & \meansd{91.8}{2.7} & \meansd{56.7}{0.1} & \meansd{131.3}{0.6} & \meansd{114.6}{1.2} & \meansd{68.5}{0.9} \\
FedDyn & \meansd{46.1}{0.0} & \meansd{339.0}{57.0} & \meansd{20.8}{0.2} & \meansd{101.3}{0.2} & \meansd{502.0}{35.6} & \meansd{42.3}{0.3} & \meansd{127.1}{0.5} & \meansd{248.9}{70.4} & \meansd{47.7}{0.2} & \meansd{127.4}{0.2} & \meansd{297.9}{23.3} & \meansd{56.7}{0.7} \\
FedSAM & \meansd{45.4}{0.2} & \meansd{51.8}{1.6} & \meansd{25.0}{0.2} & \meansd{107.8}{0.3} & \meansd{96.2}{3.6} & \meansd{50.9}{0.2} & \meansd{141.0}{0.9} & \meansd{91.9}{2.2} & \meansd{56.6}{0.1} & \meansd{131.6}{0.5} & \meansd{114.8}{1.2} & \meansd{68.6}{0.9} \\
FedExP & \meansd{43.9}{0.3} & \meansd{60.7}{3.9} & \meansd{27.0}{0.1} & \meansd{104.8}{0.2} & \meansd{105.5}{3.8} & \meansd{55.0}{0.3} & \meansd{137.3}{1.0} & \meansd{97.5}{4.2} & \meansd{60.2}{0.1} & \meansd{128.3}{0.5} & \meansd{116.3}{1.1} & \meansd{70.2}{0.7} \\
DiLoCo-style & \meansd{43.2}{0.3} & \meansd{62.8}{2.5} & \meansd{25.6}{0.3} & \meansd{104.6}{0.2} & \meansd{130.7}{9.2} & \meansd{52.3}{0.3} & \meansd{137.6}{1.1} & \meansd{110.2}{3.2} & \meansd{58.4}{0.5} & \meansd{126.9}{0.5} & \meansd{149.9}{5.4} & \meansd{68.6}{0.8} \\
\bottomrule
\end{tabular}}
\end{table}

\providecommand{\meansd}[2]{#1\,{\scriptsize$\pm$\,#2}}
\providecommand{\bestmeansd}[2]{\textbf{#1}\,{\scriptsize$\pm$\,#2}}
\begin{table}[h]
\centering
\scriptsize
\setlength{\tabcolsep}{3pt}
\renewcommand{\arraystretch}{1.1}
\caption{\textbf{Effect of model depth.} Training, validation, and last-10-round validation perplexity of the nine federated learning algorithms for the depth-$4$ and depth-$16$ Mamba2 models, averaged over the $36$ configurations.}
\label{tab:app_fl_by_depth}
\begin{tabular}{lccc ccc}
\toprule
& \multicolumn{3}{c}{Depth $4$ ($14.7$M)} & \multicolumn{3}{c}{Depth $16$ ($20.3$M)} \\
\cmidrule(lr){2-4}\cmidrule(lr){5-7}
Algorithm & Train & Val & Last-10 & Train & Val & Last-10 \\
\midrule
FedAvg & \meansd{109.8}{0.5} & \meansd{88.8}{0.5} & \meansd{54.8}{0.3} & \meansd{102.3}{0.4} & \meansd{89.1}{2.6} & \meansd{46.4}{0.2} \\
FedProx & \bestmeansd{93.2}{0.4} & \bestmeansd{79.0}{1.3} & \bestmeansd{39.5}{0.4} & \bestmeansd{91.0}{0.3} & \bestmeansd{80.5}{1.2} & \bestmeansd{36.2}{0.1} \\
SCAFFOLD & \meansd{112.8}{0.4} & \meansd{88.0}{1.4} & \meansd{50.1}{0.0} & \meansd{107.0}{0.4} & \meansd{88.8}{2.0} & \meansd{44.5}{0.2} \\
FedAdam & \meansd{122.6}{0.3} & \meansd{235.1}{4.1} & \meansd{59.8}{0.3} & \meansd{116.2}{0.5} & \meansd{218.2}{7.4} & \meansd{50.5}{0.3} \\
FedAlign & \meansd{109.7}{0.5} & \meansd{89.1}{1.9} & \meansd{54.4}{0.2} & \meansd{102.4}{0.4} & \meansd{87.9}{2.2} & \meansd{46.3}{0.3} \\
FedDyn & \meansd{100.2}{0.2} & \meansd{398.3}{43.9} & \meansd{42.8}{0.3} & \meansd{100.8}{0.3} & \meansd{295.7}{15.5} & \meansd{41.0}{0.1} \\
FedSAM & \meansd{110.1}{0.6} & \meansd{89.8}{1.7} & \meansd{54.3}{0.1} & \meansd{102.7}{0.4} & \meansd{87.5}{1.1} & \meansd{46.2}{0.4} \\
FedExP & \meansd{107.5}{0.4} & \meansd{93.1}{2.5} & \meansd{57.6}{0.4} & \meansd{99.6}{0.6} & \meansd{96.9}{2.6} & \meansd{48.5}{0.3} \\
DiLoCo-style & \meansd{108.8}{0.6} & \meansd{112.2}{0.7} & \meansd{56.3}{0.1} & \meansd{97.3}{0.5} & \meansd{114.7}{6.8} & \meansd{46.2}{0.4} \\
\bottomrule
\end{tabular}
\end{table}

\textbf{Comments.} The results averages show a consistent overall ranking. FedProx achieves the lowest mean perplexity in $34$ of the $36$ metric columns. The two exceptions are last-10 validation perplexity at $\beta=1$, where FedDyn is marginally lower ($40.0$ versus $40.2$), and all-round validation perplexity at $E=64$, where SCAFFOLD is lower ($105.6$ versus $106.6$). On the last-10 metric, FedProx, FedDyn, and SCAFFOLD rank first, second, and third, respectively, at every local-step count, domain setting, and model depth in Tables~\ref{tab:app_fl_by_steps}-\ref{tab:app_fl_by_depth}. 

Heterogeneity produces the clearest changes in relative performance. FedProx reduces last-10 perplexity relative to FedAvg by approximately $19\%$, $20\%$, and $34\%$ at $\beta=0$, $0.7$, and $1$, respectively. SCAFFOLD moves from eighth place at $\beta=0$ ($46.9$ versus $45.4$ for FedAvg) to third at $\beta=1$, where it improves on FedAvg by $18\%$. FedDyn narrowly overtakes FedProx at $\beta=1$, whereas FedExP moves from near parity with FedAvg at $\beta\leq0.7$ to last place at $\beta=1$, with $13.5\%$ higher last-10 perplexity. Across the same increase in $\beta$, last-10 perplexity rises by approximately $32$-$34\%$ for FedAvg, FedAlign, FedSAM, and DiLoCo-style, by $52\%$ for FedExP, and by $16\%$ for FedAdam. The corresponding increases are only $9\%$ for FedProx and $6\%$ for SCAFFOLD. FedDyn instead improves by approximately $9\%$, although this comparison uses a separately selected $\alpha$ at each $\beta$.

Increasing the number of local steps strengthens the relative advantages of FedDyn and SCAFFOLD. Their reductions in last-10 perplexity relative to FedAvg grow from approximately $12\%$ and $4.5\%$ at $E=64$ to $22\%$ and $10\%$ at $E=256$, respectively. FedProx maintains a comparatively uniform reduction of $24$-$27\%$ across the three values of $E$. These trends are qualitatively consistent with the role of the $E(E-1)$ drift term in the bounds and the benefits of radius control or heterogeneity correction discussed in our analysis. Naturally, larger $E$ results in longer training, so all nine algorithms achieve approximately $35$-$45\%$ lower last-10 perplexity when $E$ increases from $64$ to $256$ at fixed $T=50$ communication rounds.

The remaining methods generally remain close to, or underperform, FedAvg on the last-10 metric. FedAlign and FedSAM stay within $2\%$ of FedAvg throughout the factor tables, while DiLoCo-style stays within $5\%$. FedAdam has approximately $9\%$ higher last-10 perplexity at every local-step count and model depth, with gaps of approximately $5$-$12\%$ across domain settings.

All-round validation perplexity provides an important contrast to these last-10 results. Despite its strong late-training performance, FedDyn has substantially higher all-round validation perplexity than FedAvg, as does FedAdam. Across the factor tables, their all-round averages are approximately $1.7$-$6.4$ times those of FedAvg, reaching $814.4$ for FedDyn and $355.2$ for FedAdam at $\beta=1$. The round-wise results place these averages in context: after round $1$, FedDyn and FedAdam have validation perplexities approximately $8$ and $9$ times that of FedAvg, respectively, using geometric means across configurations. FedDyn subsequently falls below FedAvg from round $4$ onward and finishes $18\%$ lower, whereas FedAdam remains above FedAvg throughout training and finishes $9\%$ higher.

Domain composition and model depth change the absolute perplexity levels without changing the ordering of the three leading methods. Setting 1 yields the lowest last-10 perplexity and Setting 4 the highest for every algorithm. For FedAvg, these values differ by approximately a factor of $2.7$. FedProx improves on FedAvg by approximately $24$-$26\%$ in each setting. Increasing depth from $4$ to $16$ reduces last-10 perplexity by approximately $15$-$18\%$ for FedAvg, FedAlign, FedSAM, FedExP, DiLoCo-style, and FedAdam, compared with $11\%$ for SCAFFOLD, $8\%$ for FedProx, and $4\%$ for FedDyn. Consequently, the relative improvements over FedAvg narrow from $28\%$ to $22\%$ for FedProx, from $22\%$ to $12\%$ for FedDyn, and from $9\%$ to $4\%$ for SCAFFOLD.

\subsubsection{Client-Count Ablation}\label{appendix:practical_client_count}

Table~\ref{tab:app_fl_clients} reports the client-count ablation described in Appendix~\ref{appendix:detailed_exp_setup}. We consider $K\in\{4,8,16,32\}$ client counts at fixed $\beta=0.7$, depth $4$, and $E=64$, averaged over the four domain settings, with $K/4$ clients sharing each primary domain and every client holding a $4/K$ share of the same domain pools.

\providecommand{\meansd}[2]{#1\,{\scriptsize$\pm$\,#2}}
\providecommand{\bestmeansd}[2]{\textbf{#1}\,{\scriptsize$\pm$\,#2}}
\begin{table}[h]
\centering
\scriptsize
\setlength{\tabcolsep}{3pt}
\renewcommand{\arraystretch}{1.1}
\caption{\textbf{Client-count ablation.} Training, validation, and last-10-round validation perplexity of the nine algorithms with $K\in\{4,8,16,32\}$ clients at $\beta=0.7$, depth $4$, $E=64$, averaged over the four domain settings.}
\label{tab:app_fl_clients}
\resizebox{\textwidth}{!}{%
\begin{tabular}{lccc ccc ccc ccc}
\toprule
& \multicolumn{3}{c}{$K=4$} & \multicolumn{3}{c}{$K=8$} & \multicolumn{3}{c}{$K=16$} & \multicolumn{3}{c}{$K=32$} \\
\cmidrule(lr){2-4}\cmidrule(lr){5-7}\cmidrule(lr){8-10}\cmidrule(lr){11-13}
Algorithm & Train & Val & Last-10 & Train & Val & Last-10 & Train & Val & Last-10 & Train & Val & Last-10 \\
\midrule
FedAvg & \meansd{144.4}{0.3} & \meansd{89.7}{0.9} & \meansd{64.2}{0.4} & \meansd{141.4}{0.4} & \meansd{87.3}{0.5} & \meansd{61.6}{0.2} & \meansd{136.9}{0.6} & \meansd{86.3}{0.7} & \meansd{60.9}{0.3} & \meansd{129.0}{0.3} & \meansd{86.2}{0.6} & \meansd{61.1}{0.2} \\
FedProx & \bestmeansd{131.6}{0.2} & \bestmeansd{82.4}{0.8} & \bestmeansd{49.0}{0.2} & \bestmeansd{129.1}{0.5} & \bestmeansd{80.6}{0.3} & \bestmeansd{47.3}{0.3} & \bestmeansd{125.2}{0.4} & \bestmeansd{79.9}{0.5} & \bestmeansd{46.8}{0.2} & \bestmeansd{118.5}{0.3} & \meansd{80.5}{0.4} & \bestmeansd{47.6}{0.1} \\
SCAFFOLD & \meansd{143.7}{0.2} & \meansd{87.6}{0.9} & \meansd{62.4}{0.4} & \meansd{139.5}{0.5} & \meansd{83.8}{0.5} & \meansd{58.7}{0.2} & \meansd{134.2}{0.7} & \meansd{81.2}{0.6} & \meansd{55.7}{0.4} & \meansd{126.8}{0.4} & \bestmeansd{79.7}{0.5} & \meansd{53.0}{0.2} \\
FedAdam & \meansd{159.9}{0.1} & \meansd{187.1}{2.0} & \meansd{69.5}{0.6} & \meansd{156.2}{0.6} & \meansd{183.0}{1.2} & \meansd{67.0}{0.4} & \meansd{151.3}{0.5} & \meansd{181.7}{2.0} & \meansd{66.1}{0.4} & \meansd{143.7}{0.2} & \meansd{182.8}{1.7} & \meansd{66.5}{0.5} \\
FedAlign & \meansd{144.4}{0.2} & \meansd{89.7}{0.9} & \meansd{64.2}{0.4} & \meansd{141.4}{0.4} & \meansd{87.3}{0.5} & \meansd{61.6}{0.2} & \meansd{136.9}{0.6} & \meansd{86.3}{0.7} & \meansd{60.9}{0.3} & \meansd{129.0}{0.3} & \meansd{86.2}{0.6} & \meansd{61.1}{0.2} \\
FedDyn & \meansd{164.5}{0.4} & \meansd{161.8}{15.9} & \meansd{59.4}{0.5} & \meansd{163.0}{0.7} & \meansd{156.9}{5.7} & \meansd{54.1}{0.4} & \meansd{162.2}{0.2} & \meansd{151.4}{10.3} & \meansd{51.6}{0.4} & \meansd{162.0}{0.2} & \meansd{151.9}{11.1} & \meansd{50.4}{0.3} \\
FedSAM & \meansd{145.1}{0.4} & \meansd{90.3}{0.8} & \meansd{64.3}{0.2} & \meansd{142.2}{0.5} & \meansd{88.1}{0.5} & \meansd{62.1}{0.2} & \meansd{137.5}{0.7} & \meansd{86.9}{0.7} & \meansd{61.2}{0.3} & \meansd{129.8}{0.4} & \meansd{86.8}{0.6} & \meansd{61.5}{0.2} \\
FedExP & \meansd{141.7}{0.1} & \meansd{89.9}{1.0} & \meansd{63.9}{0.7} & \meansd{137.1}{0.5} & \meansd{85.9}{0.4} & \meansd{59.5}{0.5} & \meansd{131.6}{0.8} & \meansd{84.0}{0.8} & \meansd{57.7}{0.4} & \meansd{123.6}{0.5} & \meansd{83.2}{0.5} & \meansd{57.3}{0.4} \\
DiLoCo-style & \meansd{138.7}{0.2} & \meansd{94.6}{1.4} & \meansd{65.8}{0.1} & \meansd{134.0}{0.8} & \meansd{90.2}{0.8} & \meansd{61.0}{0.3} & \meansd{128.3}{0.3} & \meansd{87.4}{0.9} & \meansd{58.0}{0.6} & \meansd{120.8}{0.6} & \meansd{86.5}{0.5} & \meansd{56.5}{0.1} \\
\bottomrule
\end{tabular}}
\end{table}

\textbf{Comments.}  With per-client work held fixed, FedAvg's last-10 validation perplexity decreases by approximately $4\%$ from $K=4$ to $K=8$ ($64.2$ to $61.6$), with little further change for the larger client counts ($60.9$ and $61.1$). The training perplexity continues to decrease, from $144$ at $K=4$ to $129$ at $K=32$, potentially reflecting more frequent reuse of smaller local datasets. The three leading methods in the main grid -FedProx, FedDyn, and SCAFFOLD - remain the strongest on last-10 validation perplexity at every $K$. FedAlign and FedSAM remain within $1\%$ of FedAvg, whereas FedAdam has $8$-$9\%$ higher perplexity.

FedProx maintains a reduction of $22$-$24\%$ compared to FedAvg across the evaluated values of $K$. In contrast, the reductions achieved by SCAFFOLD and FedDyn increase from $3\%$ to $13\%$ and from $7\%$ to $18\%$, respectively, between $K=4$ and $K=32$. DiLoCo-style moves from $3\%$ higher perplexity than FedAvg to $8\%$ lower, while FedExP moves from approximate parity to a $6\%$ reduction. As such, although the leading methods remain unchanged, their relative gains and the ordering of the remaining methods vary with client count.

One possible explanation for these method-dependent trends is that having multiple clients per primary domain at larger $K$ produces a more stable aggregate update than the one-client-per-domain partition at $K=4$. However, this interpretation is not directly tested by the ablation. Increasing $K$ also changes the total client work, local dataset sizes, and data partitioning. These results therefore describe performance under the evaluated client-count protocol at a single $(\beta,\text{depth},E)$ setting, rather than isolating the effect of client count or establishing a particular algorithmic mechanism.

\subsubsection{Model Backbone Ablation}\label{appendix:practical_backbones}

Table~\ref{tab:app_fl_backbones} reports the different model backbones ablation described in Appendix~\ref{appendix:detailed_exp_setup}: the depth-$4$ model with some or all Mamba2 mixers replaced by attention layers (Table~\ref{tab:backbones}), trained on Setting 1 at $\beta=0.7$, $E=64$, and $K=4$.

\providecommand{\meansd}[2]{#1\,{\scriptsize$\pm$\,#2}}
\providecommand{\bestmeansd}[2]{\textbf{#1}\,{\scriptsize$\pm$\,#2}}
\begin{table}[h]
\centering
\scriptsize
\setlength{\tabcolsep}{3pt}
\renewcommand{\arraystretch}{1.1}
\caption{\textbf{Backbone ablation.} Validation and last-10-round validation perplexity of the nine algorithms on the six depth-$4$ backbones of Table~\ref{tab:backbones} (Setting 1, $\beta=0.7$, $E=64$, $K=4$, the second header row gives the fraction of Mamba2 layers).}
\label{tab:app_fl_backbones}
\resizebox{\textwidth}{!}{%
\begin{tabular}{lcc cc cc cc cc cc}
\toprule
& \multicolumn{2}{c}{\texttt{mamba2}} & \multicolumn{2}{c}{\texttt{hybrid75}} & \multicolumn{2}{c}{\texttt{hybrid50}} & \multicolumn{2}{c}{\texttt{hybrid50alt}} & \multicolumn{2}{c}{\texttt{hybrid25}} & \multicolumn{2}{c}{\texttt{attn0}} \\
& \multicolumn{2}{c}{SSM 1.00} & \multicolumn{2}{c}{SSM 0.75} & \multicolumn{2}{c}{SSM 0.50} & \multicolumn{2}{c}{SSM 0.50} & \multicolumn{2}{c}{SSM 0.25} & \multicolumn{2}{c}{SSM 0.00} \\
\cmidrule(lr){2-3}\cmidrule(lr){4-5}\cmidrule(lr){6-7}\cmidrule(lr){8-9}\cmidrule(lr){10-11}\cmidrule(lr){12-13}
Algorithm & Val & Last-10 & Val & Last-10 & Val & Last-10 & Val & Last-10 & Val & Last-10 & Val & Last-10 \\
\midrule
FedAvg & \meansd{43.4}{0.4} & \meansd{30.7}{0.3} & \meansd{41.5}{0.8} & \meansd{26.5}{0.5} & \meansd{43.3}{0.6} & \meansd{29.3}{0.4} & \meansd{52.6}{1.2} & \meansd{31.3}{0.2} & \meansd{54.4}{1.2} & \meansd{32.7}{0.7} & \meansd{69.6}{1.1} & \meansd{42.3}{1.1} \\
FedProx & \bestmeansd{41.9}{0.4} & \bestmeansd{25.4}{0.2} & \meansd{39.8}{0.4} & \bestmeansd{21.0}{0.2} & \meansd{41.4}{0.3} & \bestmeansd{23.6}{0.2} & \meansd{50.4}{0.7} & \bestmeansd{24.7}{0.4} & \bestmeansd{52.2}{1.3} & \bestmeansd{26.8}{0.9} & \bestmeansd{67.7}{1.0} & \bestmeansd{34.8}{0.9} \\
SCAFFOLD & \meansd{42.2}{0.4} & \meansd{29.4}{0.3} & \bestmeansd{39.2}{1.0} & \meansd{22.5}{0.5} & \bestmeansd{41.3}{0.7} & \meansd{25.4}{0.8} & \bestmeansd{50.3}{1.2} & \meansd{27.1}{0.3} & \meansd{52.5}{1.2} & \meansd{29.6}{1.3} & \meansd{68.3}{1.0} & \meansd{41.8}{0.9} \\
FedAdam & \meansd{116.0}{1.1} & \meansd{33.2}{0.2} & \meansd{116.5}{3.4} & \meansd{30.7}{0.5} & \meansd{120.4}{2.8} & \meansd{32.5}{0.5} & \meansd{141.7}{3.9} & \meansd{38.8}{0.2} & \meansd{144.4}{2.3} & \meansd{40.3}{0.6} & \meansd{180.8}{4.6} & \meansd{59.4}{1.1} \\
FedAlign & \meansd{43.4}{0.4} & \meansd{30.7}{0.3} & \meansd{41.7}{1.0} & \meansd{26.7}{0.6} & \meansd{43.3}{0.6} & \meansd{29.3}{0.4} & \meansd{52.5}{1.2} & \meansd{31.2}{0.1} & \meansd{54.3}{1.1} & \meansd{32.7}{0.7} & \meansd{69.6}{1.2} & \meansd{42.0}{0.9} \\
FedDyn & \meansd{76.2}{6.4} & \meansd{29.6}{0.5} & \meansd{68.4}{2.1} & \meansd{28.2}{0.4} & \meansd{69.5}{5.5} & \meansd{27.5}{1.0} & \meansd{114.0}{18.4} & \meansd{34.8}{1.0} & \meansd{101.8}{16.0} & \meansd{35.5}{0.6} & \meansd{132.3}{22.7} & \meansd{41.6}{1.4} \\
FedSAM & \meansd{43.8}{0.4} & \meansd{30.8}{0.3} & \meansd{41.8}{0.8} & \meansd{26.7}{0.6} & \meansd{43.6}{0.7} & \meansd{29.6}{0.5} & \meansd{53.7}{0.9} & \meansd{32.1}{0.1} & \meansd{55.2}{1.4} & \meansd{33.4}{1.2} & \meansd{70.5}{0.9} & \meansd{42.7}{0.4} \\
FedExP & \meansd{43.6}{0.2} & \meansd{30.5}{0.0} & \meansd{41.0}{0.7} & \meansd{25.7}{0.4} & \meansd{43.1}{0.5} & \meansd{28.3}{0.3} & \meansd{51.8}{1.2} & \meansd{30.0}{0.2} & \meansd{53.5}{1.1} & \meansd{30.9}{0.7} & \meansd{68.9}{0.9} & \meansd{40.3}{1.3} \\
DiLoCo-style & \meansd{44.8}{0.2} & \meansd{31.1}{0.4} & \meansd{42.0}{0.9} & \meansd{26.3}{0.5} & \meansd{43.7}{0.4} & \meansd{29.0}{0.4} & \meansd{56.6}{2.6} & \meansd{33.5}{0.9} & \meansd{57.9}{0.9} & \meansd{35.2}{0.5} & \meansd{78.6}{2.5} & \meansd{48.1}{1.2} \\
\bottomrule
\end{tabular}}
\end{table}

\textbf{Comments.} Under FedAvg the backbones order as \texttt{hybrid75} ($26.5$) $<$ \texttt{hybrid50} ($29.3$) $<$ \texttt{mamba2} ($30.7$) $<$ \texttt{hybrid50alt} ($31.3$) $<$ \texttt{hybrid25} ($32.7$) $<$ \texttt{attn0} ($42.3$) on the last-10 validation perplexity. Notably, one attention layer in the middle of the stack improves on the pure Mamba2 model by $14\%$. Attention layers placed first (\texttt{hybrid50alt}, \texttt{hybrid50}) or a mostly-attention stack lose that gain, and the attention-only control is $38\%$ worse than pure Mamba2 at this model size and sequence length. The ranking of the algorithms transfers from the pure-SSM backbone to every hybrid, with two exceptions. FedProx is the best method on all six backbones ($-17$ to $-21\%$ relative to FedAvg), and FedAlign and FedSAM stay within $3\%$ of FedAvg everywhere. SCAFFOLD, in contrast, gains $13$--$15\%$ on the hybrids with two or three Mamba2 layers and $10\%$ on \texttt{hybrid25}, but only $4\%$ on \texttt{mamba2} and $1\%$ on \texttt{attn0}, and FedDyn at its fixed $\alpha$ is unstable across architectures ($-2$ to $-6\%$ on \texttt{mamba2}, \texttt{hybrid50}, and \texttt{attn0}, $+6$ to $+11\%$ on the other three hybrids). The server-step methods are the second exception: their penalty grows as Mamba2 layers are replaced, from $+8\%$ (FedAdam) and $+1\%$ (DiLoCo-style) on \texttt{mamba2} to $+40\%$ and $+14\%$ on \texttt{attn0}, whereas FedExP, neutral on \texttt{mamba2}, is $3$-$6\%$ better than FedAvg on every backbone that contains attention.

\subsection{Model Merging}\label{sec:appendix_mamba2_merging}

This subsection reports the one-shot merging experiments referenced in Appendix~\ref{appendix:model_mergin_discussion}. Following the framework of Appendix~\ref{appendix:model_merging_framework}, the pretrained Mamba2-130M, 370M, and 780M checkpoints (Table~\ref{tab:mamba2_merge_checkpoints}) are fine-tuned separately on AG News, DBpedia, TREC, and IMDB (Table~\ref{tab:causal_lm_tasks}) with the recipe of Table~\ref{tab:finetune_recipe}, and the four specialists $\theta_m=\theta^0+\mathbf T_m$ saved after $1$, $2$, and $4$ epochs are merged by weight averaging ($p_m=1/4$), Task Arithmetic ($\theta^0+\lambda\sum_m\mathbf T_m$), and TIES-Merging (trim each task vector to the fraction $d$ of its largest-magnitude entries, elect signs, average the agreeing entries, and scale the result by $\lambda$). Every merged model is evaluated on all four tasks by causal-LM label scoring, and we report the accuracy averaged over the four tasks (mean $\pm$ standard deviation over three fine-tuning seeds), as well as the accuracies of the pretrained model and the specialists evaluated on their own and on the other tasks. Task Arithmetic and TIES-Merging are shown at the best of the evaluated $\lambda$ and $(\lambda,d)$ for each checkpoint and epoch. Tables~\ref{tab:app_merge_ta} and~\ref{tab:app_merge_ties} give the full sweeps, Table~\ref{tab:app_merge_tasks} the per-task accuracies after four epochs, and Table~\ref{tab:app_merge_prox} a variant of the 130M experiment in which the specialists are fine-tuned with the FedProx proximal term. 

\providecommand{\meansd}[2]{#1\,{\scriptsize$\pm$\,#2}}
\providecommand{\bestmeansd}[2]{\textbf{#1}\,{\scriptsize$\pm$\,#2}}
\begin{table}[h]
\centering
\scriptsize
\setlength{\tabcolsep}{3.5pt}
\renewcommand{\arraystretch}{1.1}
\caption{\textbf{One-shot merging of Mamba2 task specialists: task-averaged accuracy (\%).} For each public checkpoint, the accuracy averaged over AG News, DBpedia, TREC, and IMDB of the pretrained model, of the four specialists on their own task and on the other three tasks, and of the merged model produced by weight averaging, Task Arithmetic, and TIES-Merging from specialists fine-tuned for $1$, $2$, or $4$ epochs.}
\label{tab:app_merge_avg}
\begin{tabular}{llcccl}
\toprule
Checkpoint & Model & $1$ epoch & $2$ epochs & $4$ epochs & Selected $\lambda$ / $(\lambda,d)$ \\
\midrule
\multirow{6}{*}{Mamba2-130M} & Zero-shot (pretrained) & \multicolumn{3}{c}{\meansd{25.7}{0.0}} & -- \\
 & Specialists (own task) & \meansd{81.8}{1.8} & \meansd{88.4}{0.5} & \meansd{89.2}{0.2} & -- \\
 & Specialists (other tasks) & \meansd{23.0}{0.2} & \meansd{23.6}{0.1} & \meansd{23.4}{0.3} & -- \\
 & Weight averaging & \meansd{34.4}{2.1} & \meansd{36.3}{1.7} & \meansd{37.4}{0.6} & -- \\
 & Task Arithmetic & \bestmeansd{53.5}{4.6} & \bestmeansd{61.9}{5.1} & \bestmeansd{61.3}{3.8} & $1$ / $1$ / $0.7$ \\
 & TIES-Merging & \meansd{47.0}{2.0} & \meansd{55.0}{3.9} & \meansd{56.6}{3.6} & $(1.3,0.1)$ / $(1.3,0.1)$ / $(1.3,0.1)$ \\
\midrule
\multirow{6}{*}{Mamba2-370M} & Zero-shot (pretrained) & \multicolumn{3}{c}{\meansd{24.2}{0.1}} & -- \\
 & Specialists (own task) & \meansd{86.6}{2.3} & \meansd{90.7}{0.1} & \meansd{90.8}{0.1} & -- \\
 & Specialists (other tasks) & \meansd{24.8}{0.5} & \meansd{24.8}{0.3} & \meansd{25.4}{0.1} & -- \\
 & Weight averaging & \meansd{45.6}{1.0} & \meansd{49.6}{0.5} & \meansd{50.4}{0.4} & -- \\
 & Task Arithmetic & \meansd{61.1}{2.5} & \bestmeansd{71.0}{1.7} & \bestmeansd{71.8}{1.2} & $0.7$ / $0.7$ / $0.7$ \\
 & TIES-Merging & \bestmeansd{62.6}{1.6} & \meansd{69.8}{2.0} & \meansd{70.9}{1.8} & $(1.3,0.5)$ / $(1.3,0.1)$ / $(1.3,0.1)$ \\
\midrule
\multirow{6}{*}{Mamba2-780M} & Zero-shot (pretrained) & \multicolumn{3}{c}{\meansd{24.7}{0.0}} & -- \\
 & Specialists (own task) & \meansd{82.4}{3.3} & \meansd{89.4}{0.5} & \meansd{90.8}{0.3} & -- \\
 & Specialists (other tasks) & \meansd{24.5}{1.0} & \meansd{24.6}{0.1} & \meansd{24.9}{0.2} & -- \\
 & Weight averaging & \meansd{39.4}{2.8} & \meansd{45.4}{1.4} & \meansd{51.7}{1.5} & -- \\
 & Task Arithmetic & \meansd{59.9}{2.8} & \meansd{61.9}{5.6} & \meansd{71.4}{2.5} & $0.7$ / $0.7$ / $0.7$ \\
 & TIES-Merging & \bestmeansd{60.3}{3.0} & \bestmeansd{64.5}{7.6} & \bestmeansd{72.1}{3.5} & $(1.3,0.1)$ / $(1.3,0.1)$ / $(1.3,0.1)$ \\
\bottomrule
\end{tabular}
\end{table}

\providecommand{\meansd}[2]{#1\,{\scriptsize$\pm$\,#2}}
\providecommand{\bestmeansd}[2]{\textbf{#1}\,{\scriptsize$\pm$\,#2}}
\begin{table}[h]
\centering
\scriptsize
\setlength{\tabcolsep}{3.5pt}
\renewcommand{\arraystretch}{1.1}
\caption{\textbf{Per-task accuracy (\%) of the merged models after 4 specialist epochs.} Columns are the evaluation tasks for each checkpoint, rows are the pretrained model, the specialist fine-tuned on that task, and the three merges.}
\label{tab:app_merge_tasks}
\resizebox{\textwidth}{!}{%
\begin{tabular}{lcccc cccc cccc}
\toprule
& \multicolumn{4}{c}{Mamba2-130M} & \multicolumn{4}{c}{Mamba2-370M} & \multicolumn{4}{c}{Mamba2-780M} \\
\cmidrule(lr){2-5}\cmidrule(lr){6-9}\cmidrule(lr){10-13}
Model & AG News & DBpedia & TREC & IMDB & AG News & DBpedia & TREC & IMDB & AG News & DBpedia & TREC & IMDB \\
\midrule
Zero-shot (pretrained) & \meansd{25.0}{0.0} & \meansd{9.9}{0.0} & \meansd{16.8}{0.0} & \meansd{51.1}{0.2} & \meansd{24.3}{0.0} & \meansd{8.1}{0.0} & \meansd{13.4}{0.2} & \meansd{51.1}{0.1} & \meansd{25.3}{0.0} & \meansd{8.7}{0.0} & \meansd{15.1}{0.1} & \meansd{49.9}{0.1} \\
Specialists (own task) & \meansd{89.5}{0.3} & \meansd{96.4}{0.2} & \meansd{86.7}{0.6} & \meansd{83.9}{0.4} & \meansd{90.8}{0.2} & \meansd{97.6}{0.2} & \meansd{88.7}{0.5} & \meansd{86.0}{0.1} & \meansd{90.8}{0.1} & \meansd{97.9}{0.2} & \meansd{88.4}{1.0} & \meansd{86.1}{0.5} \\
Weight averaging & \meansd{37.9}{2.6} & \meansd{21.6}{3.9} & \meansd{34.9}{5.3} & \meansd{55.3}{0.7} & \meansd{42.4}{2.3} & \meansd{36.8}{0.9} & \meansd{63.5}{4.4} & \meansd{59.1}{5.4} & \meansd{44.4}{0.6} & \meansd{41.1}{3.3} & \meansd{54.5}{1.6} & \meansd{66.8}{3.2} \\
Task Arithmetic & \meansd{64.7}{8.5} & \bestmeansd{53.3}{13.8} & \bestmeansd{61.7}{7.3} & \bestmeansd{65.5}{2.6} & \meansd{75.3}{7.0} & \bestmeansd{74.4}{2.8} & \bestmeansd{65.1}{2.8} & \bestmeansd{72.5}{2.2} & \meansd{80.6}{2.3} & \meansd{65.9}{7.9} & \bestmeansd{57.9}{7.7} & \meansd{81.1}{1.7} \\
TIES-Merging & \bestmeansd{72.2}{8.4} & \meansd{51.7}{17.9} & \meansd{46.2}{13.1} & \meansd{56.2}{0.1} & \bestmeansd{77.3}{6.8} & \meansd{73.4}{3.6} & \meansd{62.9}{3.0} & \meansd{69.9}{0.8} & \bestmeansd{83.4}{1.4} & \bestmeansd{69.5}{8.2} & \meansd{52.4}{11.6} & \bestmeansd{83.1}{1.1} \\
\bottomrule
\end{tabular}}
\end{table}

\providecommand{\meansd}[2]{#1\,{\scriptsize$\pm$\,#2}}
\providecommand{\bestmeansd}[2]{\textbf{#1}\,{\scriptsize$\pm$\,#2}}
\begin{table}[h]
\centering
\scriptsize
\setlength{\tabcolsep}{3.5pt}
\renewcommand{\arraystretch}{1.1}
\caption{\textbf{Task Arithmetic scale sweep: task-averaged accuracy (\%).} $\theta_{\mathrm{TA}}=\theta^0+\lambda\sum_m\mathbf T_m$ over the four specialists, so $\lambda=1/M=0.25$ would be weight averaging and $\lambda=1$ adds every task vector at full length. The 780M sweep evaluated only $\lambda\in\{0.7,1.0\}$.}
\label{tab:app_merge_ta}
\begin{tabular}{llcccccc}
\toprule
Checkpoint & Epochs & $\lambda=0.2$ & $\lambda=0.3$ & $\lambda=0.5$ & $\lambda=0.7$ & $\lambda=1$ & $\lambda=1.3$ \\
\midrule
\multirow{3}{*}{Mamba2-130M} & $1$ & \meansd{31.2}{1.2} & \meansd{38.1}{3.2} & \meansd{48.3}{3.3} & \meansd{52.4}{4.2} & \bestmeansd{53.5}{4.6} & \meansd{51.7}{3.8} \\
 & $2$ & \meansd{32.6}{1.1} & \meansd{41.3}{1.9} & \meansd{55.5}{2.0} & \meansd{61.9}{4.1} & \bestmeansd{61.9}{5.1} & \meansd{55.8}{5.4} \\
 & $4$ & \meansd{32.8}{0.8} & \meansd{42.7}{0.8} & \meansd{55.8}{2.1} & \bestmeansd{61.3}{3.8} & \meansd{57.2}{6.2} & \meansd{49.3}{7.0} \\
\midrule
\multirow{3}{*}{Mamba2-370M} & $1$ & \meansd{40.2}{1.9} & \meansd{49.5}{0.6} & \meansd{59.4}{2.9} & \bestmeansd{61.1}{2.5} & \meansd{58.6}{3.4} & \meansd{54.1}{5.6} \\
 & $2$ & \meansd{43.5}{1.5} & \meansd{54.7}{0.4} & \meansd{67.6}{1.0} & \bestmeansd{71.0}{1.7} & \meansd{68.2}{2.2} & \meansd{59.6}{2.7} \\
 & $4$ & \meansd{44.5}{1.5} & \meansd{55.4}{0.5} & \meansd{68.6}{0.5} & \bestmeansd{71.8}{1.2} & \meansd{66.3}{2.5} & \meansd{54.0}{3.3} \\
\midrule
\multirow{3}{*}{Mamba2-780M} & $1$ & -- & -- & -- & \bestmeansd{59.9}{2.8} & \meansd{58.1}{3.9} & -- \\
 & $2$ & -- & -- & -- & \bestmeansd{61.9}{5.6} & \meansd{57.9}{5.2} & -- \\
 & $4$ & -- & -- & -- & \bestmeansd{71.4}{2.5} & \meansd{69.5}{4.7} & -- \\
\bottomrule
\end{tabular}
\end{table}

\providecommand{\meansd}[2]{#1\,{\scriptsize$\pm$\,#2}}
\providecommand{\bestmeansd}[2]{\textbf{#1}\,{\scriptsize$\pm$\,#2}}
\begin{table}[h]
\centering
\scriptsize
\setlength{\tabcolsep}{3.5pt}
\renewcommand{\arraystretch}{1.1}
\caption{\textbf{TIES-Merging sweep after 4 specialist epochs: task-averaged accuracy (\%).} Rows are the output scale $\lambda$ applied to the task vector, columns the fraction $d$ of largest-magnitude entries kept per task vector. The 780M sweep evaluated only $\lambda=1.3$ with $d\in\{0.1,0.5\}$ in addition to the default $(\lambda,d)=(1.0,0.2)$.}
\label{tab:app_merge_ties}
\begin{tabular}{llcccc}
\toprule
Checkpoint & $\lambda$ & $d=0.1$ & $d=0.2$ & $d=0.3$ & $d=0.5$ \\
\midrule
\multirow{4}{*}{Mamba2-130M} & $0.5$ & \meansd{34.7}{1.5} & \meansd{33.7}{0.9} & \meansd{32.6}{0.8} & \meansd{32.1}{0.8} \\
 & $0.8$ & \meansd{46.1}{1.3} & \meansd{44.0}{1.1} & \meansd{42.2}{1.1} & \meansd{41.2}{1.4} \\
 & $1$ & \meansd{51.7}{1.6} & \meansd{49.6}{1.7} & \meansd{47.9}{1.9} & \meansd{46.9}{1.3} \\
 & $1.3$ & \bestmeansd{56.6}{3.6} & \meansd{54.7}{2.6} & \meansd{53.7}{2.3} & \meansd{52.9}{1.8} \\
\midrule
\multirow{4}{*}{Mamba2-370M} & $0.5$ & \meansd{45.4}{0.6} & \meansd{46.4}{0.3} & \meansd{45.0}{1.5} & \meansd{42.0}{1.3} \\
 & $0.8$ & \meansd{63.0}{1.7} & \meansd{62.6}{0.8} & \meansd{60.9}{1.1} & \meansd{57.5}{1.3} \\
 & $1$ & \meansd{68.0}{1.7} & \meansd{67.8}{1.2} & \meansd{67.1}{1.4} & \meansd{64.7}{1.8} \\
 & $1.3$ & \bestmeansd{70.9}{1.8} & \meansd{70.4}{1.7} & \meansd{70.6}{1.5} & \meansd{70.2}{1.6} \\
\midrule
\multirow{2}{*}{Mamba2-780M} & $1$ & -- & \meansd{67.8}{4.2} & -- & -- \\
 & $1.3$ & \bestmeansd{72.1}{3.5} & -- & -- & \meansd{67.2}{3.3} \\
\bottomrule
\end{tabular}
\end{table}

\providecommand{\meansd}[2]{#1\,{\scriptsize$\pm$\,#2}}
\providecommand{\bestmeansd}[2]{\textbf{#1}\,{\scriptsize$\pm$\,#2}}
\begin{table}[h]
\centering
\scriptsize
\setlength{\tabcolsep}{3.5pt}
\renewcommand{\arraystretch}{1.1}
\caption{\textbf{Weight averaging of proximally fine-tuned specialists (Mamba2-130M): task-averaged accuracy (\%).} Each specialist is fine-tuned with the FedProx local objective $\mathcal L_m(\theta)+\frac{\mu}{2}\|\theta-\theta^0\|_F^2$ for $1$, $2$, or $4$ epochs and the four specialists are merged by weight averaging}
\label{tab:app_merge_prox}
\begin{tabular}{l ccc ccc}
\toprule
& \multicolumn{3}{c}{Specialists (own task)} & \multicolumn{3}{c}{Weight averaging} \\
\cmidrule(lr){2-4}\cmidrule(lr){5-7}
$\mu$ & $1$ epoch & $2$ epochs & $4$ epochs & $1$ epoch & $2$ epochs & $4$ epochs \\
\midrule
$0$ (reference) & \meansd{81.8}{1.8} & \bestmeansd{88.4}{0.5} & \bestmeansd{89.2}{0.2} & \bestmeansd{34.4}{2.1} & \bestmeansd{36.3}{1.7} & \bestmeansd{37.4}{0.6} \\
$10^{-5}$ & \bestmeansd{82.6}{3.0} & \meansd{88.1}{0.9} & \meansd{88.9}{0.6} & \meansd{33.3}{0.9} & \meansd{35.7}{0.6} & \meansd{37.3}{0.6} \\
$10^{-4}$ & \meansd{82.5}{3.5} & \meansd{87.6}{0.3} & \meansd{88.1}{0.3} & \meansd{32.5}{2.0} & \meansd{33.1}{1.1} & \meansd{34.0}{0.8} \\
$10^{-3}$ & \meansd{80.6}{2.6} & \meansd{85.3}{0.6} & \meansd{86.1}{0.1} & \meansd{30.4}{1.1} & \meansd{30.3}{0.5} & \meansd{30.6}{0.5} \\
$10^{-2}$ & \meansd{75.3}{3.1} & \meansd{79.7}{0.9} & \meansd{81.9}{0.4} & \meansd{27.1}{0.7} & \meansd{27.6}{0.1} & \meansd{28.0}{0.1} \\
$10^{-1}$ & \meansd{65.0}{3.3} & \meansd{72.6}{0.9} & \meansd{74.3}{0.5} & \meansd{26.3}{0.3} & \meansd{26.2}{0.2} & \meansd{26.2}{0.3} \\
\bottomrule
\end{tabular}
\end{table}

\textbf{Comments.} After four epochs, specialists average $89$-$91\%$ accuracy on their own tasks but show little transfer to the other tasks. Task Arithmetic and TIES-Merging outperform weight averaging, with the best evaluated merges reaching $61.3$, $71.8$, and $72.1\%$ for the 130M, 370M, and 780M checkpoints, respectively, compared with $37.4$, $50.4$, and $51.7\%$ for averaging (Table~\ref{tab:app_merge_avg}). Larger checkpoints narrow the gap to specialists, although substantial deficits remain, especially on DBpedia and TREC (Table~\ref{tab:app_merge_tasks}). The sweeps favor Task Arithmetic scales of $\lambda=0.7$-$1$, rather than the averaging scale $\lambda=0.25$, and sparse, rescaled TIES updates, with $(\lambda,d)=(1.3,0.1)$ performing best among the four-epoch settings evaluated (Tables~\ref{tab:app_merge_ta} and~\ref{tab:app_merge_ties}). Additionally, fine-tuning with the proximal penalty $\frac{\mu}{2}\|\theta-\theta^0\|_F^2$ reduces four-epoch averaging accuracy from $37.4\%$ at $\mu=0$ to $26.2\%$ at $\mu=10^{-1}$, near the pretrained baseline (Table~\ref{tab:app_merge_prox}). 

\textbf{Relation to the analysis.} These results are consistent with the trade-off in the local loss control~\eqref{eq:ssm_merging_local_loss_control} between the linear alignment term $\langle\nabla f(\theta^0),\theta_{\mathrm{merge}}-\theta^0\rangle_F$ and the quadratic smoothness remainder. Weight averaging preserves neighborhood containment but rescales each task vector by $1/M$. For the 780M model parameters, $\|\sum_m\mathbf T_m\|_F^2/\sum_m\|\mathbf T_m\|_F^2=1.04$-$1.07$, indicating little net alignment among the task vectors. Together with the scale sweeps, this suggests that averaging suppresses useful task-specific adaptation: larger displacements can improve accuracy despite weaker containment guarantees, illustrating why containment alone does not ensure predictive performance (Appendix~\ref{appendix_model_merging_motivation_subseq}). The proximal results suggest a similar trade-off. Limiting specialist drift can also suppress useful adaptation, and, unlike iterative federated learning, one-shot merging has no subsequent rounds to recover it. Finally, the uniform smoothness bound does not explain the larger accuracy gaps on DBpedia and TREC, motivating further investigation of the sensitivity-aware per-block scales discussed in Appendix~\ref{appendix:ssm_merging_methods}.